\documentclass[11pt]{article}
\usepackage{fullpage}
\usepackage{amsmath,amssymb,amsthm,bbm,mathtools}
\usepackage{algorithm,algpseudocode}
\usepackage{graphicx,subfigure}
\usepackage{threeparttable}
\usepackage{enumerate}
\usepackage{multirow}
\usepackage{url}
\usepackage{bm}
\usepackage{color}
\usepackage{natbib}

\graphicspath{{figs/}}

\newcommand{\inprod}[2]{\left\langle #1,#2 \right\rangle}

\newcommand{\Romannumber}[1]{\uppercase\expandafter{\romannumeral #1}}

\newcommand{\real}{\mathbb{R}}

\newcommand{\integer}{\mathbb{Z}}

\newcommand{\ud}[1]{\underline{#1}}

\DeclareMathOperator{\nnz}{nz}

\theoremstyle{definition} 
\theoremstyle{remark}     
\theoremstyle{remark}     
\theoremstyle{plain}      \newtheorem{theorem}{Theorem}
\theoremstyle{plain}      
\theoremstyle{plain}      \newtheorem{proposition}[theorem]{Proposition}
\theoremstyle{plain}      
\theoremstyle{plain}      \newtheorem{lemma}[theorem]{Lemma}

\renewcommand{\algorithmicprocedure}{\textbf{subroutine}}
\renewcommand{\algorithmicrequire}{\textbf{Input:}}
\renewcommand{\algorithmicensure}{\textbf{Output:}}

        {\vskip 10pt \begin{algorithmic}[1]}%
        {\end{algorithmic} \vskip 10pt}

\begin{document}

\title{\textsf{Hierarchically Compositional Kernels for Scalable Nonparametric Learning}\thanks{Work supported by XDATA program of the Defense Advanced Research Projects Agency (DARPA), administered through Air Force Research Laboratory contract FA8750-12-C-0323.}}
\author{Jie Chen\thanks{IBM Thomas J. Watson Research Center. Email: \texttt{chenjie@us.ibm.com}}
  \and Haim Avron\thanks{Tel Aviv University. Email: \texttt{haimav@post.tau.ac.il}}
  \and Vikas Sindhwani\thanks{Google Brain, NYC. Email: \texttt{vikas.sindhwani@gmail.com}}}
\maketitle

\begin{abstract}
We propose a novel class of kernels to alleviate the high computational cost of large-scale nonparametric learning with kernel methods. The proposed kernel is defined based on a hierarchical partitioning of the underlying data domain, where the Nystr\"{o}m method (a globally low-rank approximation) is married with a locally lossless approximation in a hierarchical fashion. The kernel maintains (strict) positive-definiteness. The corresponding kernel matrix admits a recursively off-diagonal low-rank structure, which allows for fast linear algebra computations. Suppressing the factor of data dimension, the memory and arithmetic complexities for training a regression or a classifier are reduced from $O(n^2)$ and $O(n^3)$ to $O(nr)$ and $O(nr^2)$, respectively, where $n$ is the number of training examples and $r$ is the rank on each level of the hierarchy. Although other randomized approximate kernels entail a similar complexity, empirical results show that the proposed kernel achieves a matching performance with a smaller $r$. We demonstrate comprehensive experiments to show the effective use of the proposed kernel on data sizes up to the order of millions.
\end{abstract}

\section{Introduction}
Kernel methods~\citep{Schoelkopf2001a,Hastie2009} constitute a principled framework that extends linear statistical techniques to nonparametric modeling and inference. Applications of kernel methods span the entire spectrum of statistical learning, including classification, regression, clustering, time-series analysis, sequence modeling~\citep{RKHSHMM}, dynamical systems~\citep{RKHSPSR}, hypothesis testing~\citep{KernelsHypothesisTesting}, and causal modeling~\citep{KernelsCausality}. Under a Bayesian treatment, kernel methods also admit a parallel view in Gaussian processes (GP)~\citep{Rasmussen2006,Stein1999} that find broad applications in statistics and computational sciences, including geostatistics~\citep{Chiles2012}, design of experiments~\citep{Koehler1996}, and uncertainty quantification~\citep{Smith2013}. 

This power and generality of kernel methods, however, are limited to moderate sized problems because of the high computational costs. The root cause of the bottleneck is the fact that kernel matrices generated by kernel functions are typically dense and unstructured.  For $n$ training examples, storing the matrix costs $O(n^2)$ memory and performing matrix factorizations requires $O(n^3)$ arithmetic operations. One remedy is to resort to compactly supported kernels (e.g., splines~\citep{Monaghan1985} and Wendland functions~\citep{Wendl2004}) that potentially lead to a sparse matrix. In practice, however, the support of the kernel may not be sufficiently narrow for sparse linear algebra computations to be competitive. Moreover, prior work~\citep{Anitescu2012} revealed a subtle drawback of compactly supported kernels in the context of parameter estimation, where the likelihood surface is bumpy and the optimum is difficult to locate. For this reason, we focus on the dense setting in this paper and the goal is to exploit structures that can reduce the prohibitive cost of dense linear algebra computations.

\subsection{Preliminary}
Let $\mathcal{X}$ be a set and let $k(\cdot,\cdot):\mathcal{X}\times\mathcal{X}\to\real$ be a symmetric and strictly positive-definite function. Denote by $X=\{\bm{x}_i\in\mathcal{X}\}_{i=1,\ldots.n}$ a set of points. We write $K(X,X)$, or sometimes $K$ for short when the context is clear, to denote the kernel matrix of elements $k(\bm{x}_i,\bm{x}_j)$. \emph{Because of the confusing terminology on functions and their counterparts on matrices, here, we follow the convention that a strictly positive-definite function $k$ corresponds to a positive-definite matrix $K$, whereas a positive-definite function corresponds to a positive semi-definite matrix.} For notational convenience, we write $k(\bm{x},X)$ to denote the row vector of elements $k(\bm{x},\bm{x}_j)$ and similarly $k(X,\bm{x})$ to denote the column vector. In the context of regression/classification, the set $\mathcal{X}$ is often the $d$-dimensional Euclidean space $\real^d$ or a domain $S\subset\real^d$. Some of the methods discussed in this paper naturally generalize to a more abstract space. Associated to each point $\bm{x}_i$ is a target value $y_i\in\real$. We write $\bm{y}$ for the vector of all target values.

The Reproducing Kernel Hilbert Space $\mathcal{H}_k$ associated to a kernel $k$ is the completion of the function space
\[
\left\{\sum_{i=1}^m\alpha_ik(\bm{z}_i,\cdot)\mid\bm{z}_i\in\mathcal{X},\,\alpha_i\in\real,\,m\in\integer_+\right\}
\]
equipped with the inner product
\[
\left\langle \sum_i\alpha_ik(\bm{z}_i,\cdot),\,\,\sum_j\beta_jk(\bm{w}_j,\cdot)\right\rangle=\sum_{ij}\alpha_i\beta_jk(\bm{z}_i,\bm{w}_j).
\]
Given training data $\{(\bm{x}_i,y_i)\}_{i=1,\ldots,n}$, a typical kernel method finds a function $f\in\mathcal{H}_k$ that minimizes the following risk functional
\begin{equation}\label{eqn:KRR.obj}
\mathcal{L}(f)=\sum_{i=1}^n V(f(\bm{x_i}), y_i) + \lambda \|f\|_{{\cal H}_k}^2,
\end{equation}
where $V$ is a loss function and $\lambda>0$ is a regularization. When $V$ is the squared loss $V(t,y)=(t-y)^2$, the Representer Theorem~\citep{Schoelkopf2001} implies that the minimizer is
\begin{equation}\label{eqn:KRR}
f(\bm{x})=k(\bm{x},X)[K(X,X)+\lambda I]^{-1}\bm{y},
\end{equation}
which is nothing but the well-known kernel ridge regression. Similarly, when $V$ is the hinge loss, the minimizer leads to support vector machines.

In the GP view, the kernel $k$ serves as a covariance function. Assuming a zero-mean Gaussian prior with covariance $k$, for any separate set of points $X_*$ and the associated vector of target values $\bm{y}_*$, the joint distribution of $\bm{y}$ and $\bm{y}_*$ is thus
\[
\begin{bmatrix}\bm{y} \\ \bm{y}_*\end{bmatrix}\sim
\mathcal{N}\left(\bm{0},\begin{bmatrix} K(X,X) & K(X,X_*) \\ K(X_*,X) & K(X_*,X_*) \end{bmatrix}\right).
\]
Because of the Gaussian assumption, the posterior is the conditional $\bm{y}_*|X_*,X,\bm{y}\sim\mathcal{N}(\bm{\mu},\Sigma)$ where
\begin{equation}\label{eqn:GP.mean}
\bm{\mu}=K(X_*,X)K(X,X)^{-1}\bm{y},
\end{equation}
and
\begin{equation}\label{eqn:GP.var}
\Sigma=K(X_*,X_*)-K(X_*,X)K(X,X)^{-1}K(X,X_*).
\end{equation}
A white noise of variance $\lambda$ may be injected to the observations $\bm{y}$ so that the mean prediction $\bm{\mu}$ in~\eqref{eqn:GP.mean} is identical to~\eqref{eqn:KRR}. One may also impose a nonzero-mean model in the prior to capture the trend in the observations (see, e.g., the classic paper~\citet{OHagan1978} and also~\citet[Section 2.7]{Rasmussen2006}).

Equations~\eqref{eqn:KRR}--\eqref{eqn:GP.var} exemplify the demand for kernels that may simplify computations for a large $K(X,X)$. We discuss a few popular approaches in the following. To motivate the discussion, we consider stationary kernels whose function value depends on only the difference of the two input arguments; that is, we can write by abuse of notation $k(\bm{x},\bm{x}')=k(\bm{r})$, where $\bm{r}=\bm{x}-\bm{x}'$; for example, the Gaussian kernel
\begin{equation}\label{eqn:gauss}
k(\bm{x},\bm{x}')=\exp\left(-\frac{\|\bm{x}-\bm{x}'\|^2_2}{2\sigma^2}\right)
\end{equation}
parameterized by $\sigma$. The Fourier transform of $k(\bm{r})$, coined \emph{spectral density}, in a sense characterizes the decay of eigenvalues of the finitely dimensional covariance matrix $K$~\citep{Stein1999,Chen2013}. The decay is known to be the fastest among the Mat\'{e}rn class of kernels~\citep{Stein1999,Rasmussen2006,Chiles2012}, where Gaussian being a special case is the smoothest. Additionally, the range parameter $\sigma$ also affects the decay. When $\sigma\to\infty$, $K$ tends to a rank-1 matrix; whereas when $\sigma\to0$, $K$ tends to the identity. The numerical rank of the matrix varies when $\sigma$ moves between the two extremes. The decay of eigenvalues plays an important role on the effectiveness of the approximate kernels discussed below.

\subsection{Approximate Kernels}
The first approach is low-rank kernels. Examples are Nystr\"{o}m approximation~\citep{Williams2000,Schoelkopf2001a,Drineas2005}, random Fourier features~\citep{Rahimi2007}, and variants~\citep{Yang2014}. The Nystr\"{o}m approximation is based on a set of landmark points, $\ud{X}$, randomly sampled from the training data $X$. Then, the kernel can be written as
\begin{equation}\label{eqn:Nystrom}
k_{\text{Nystr\"{o}m}}(\bm{x},\bm{x}')=k(\bm{x},\ud{X})K(\ud{X},\ud{X})^{-1}k(\ud{X},\bm{x}').
\end{equation}
For the convenience of deriving approximation bounds, \citet{Drineas2005} consider sampling with repetition, which makes $\ud{X}$ possibly a multiset and the matrix inverse in~\eqref{eqn:Nystrom} necessarily replaced by a pseudo inverse. Various approaches for choosing the landmark points were compared in~\citet{Zhang2010}. For random Fourier features, let $\hat{k}(\bm{\omega})$ be the Fourier transform of $k(\bm{r})$ and let $\hat{k}$ be normalized such that it integrates to unity; that is, $\hat{k}$ is the normalized spectral density of $k$. Then, the kernel is
\begin{equation}\label{eqn:Fourier}
k_{\text{Fourier}}(\bm{x},\bm{x}')=\frac{2}{r}\sum_{i=1}^r\cos(\bm{\omega}_i^T\bm{x}+b_i)\cos(\bm{\omega}_i^T\bm{x}'+b_i),
\end{equation}
where $r$ is the rank, and $b_i$ and $\bm{\omega}_i$ are iid samples of Uniform$(0,2\pi)$ and of a distribution with density $\hat{k}$, respectively. Note that~\eqref{eqn:Nystrom} applies to any kernel whereas~\eqref{eqn:Fourier} applies to only stationary ones.

The Nystr\"{o}m approximation admits a conditional interpretation in the context of GP. The covariance kernel~\eqref{eqn:Nystrom} can be equivalently written as
\[
k_{\text{Nystr\"{o}m}}(\bm{x},\bm{x}')=k(\bm{x},\bm{x}')-k(\bm{x},\bm{x}'|\ud{X}),
\]
where
\[
k(\bm{x},\bm{x}'|\ud{X})=k(\bm{x},\bm{x}')-k(\bm{x},\ud{X})K(\ud{X},\ud{X})^{-1}k(\ud{X},\bm{x}')
\]
is nothing but the covariance of $\bm{x}$ and $\bm{x}'$ conditioned on $\ud{X}$. In other words, the covariance kernel of Nystr\"{o}m approximation comes from a deduction of the original covariance by a conditional covariance. The conditional covariance for any $\bm{x}$ (or symmetrically, $\bm{x}'$) within $\ud{X}$ vanishes and hence the approximation is lossless. Intuitively speaking, the closer the sites are to $\ud{X}$, the smaller the loss is. This explains frequent observations that when the size of the set $\ud{X}$ is small, using the centers of a k-means clustering as the landmark points often improves the approximation~\citep{Zhang2008,Yang2012}. A caveat is that the time cost of performing k-means clustering is often much higher than that of the Nystr\"{o}m calculation itself. Hence, the improvement gained from clustering may not be as significant as that from increasing the size of the conditioned set $\ud{X}$. When $\ud{X}$ is large, the improvement brought about by clustering is less significant (see, e.g., ~\citet[Section 8.3.7]{Rasmussen2006}).

A limitation of the low-rank kernels is that the size of the conditioned set, or equivalently the rank, needs to correlate with the decay of the spectrum of $K$ in order to yield a good approximation. For a slow decay, it is not rare to see in practice that the rank $r$ grows to thousands~\citep{Yen2014} or even several hundred thousands~\citep{Huang2014,Sindhwani2015} in order to yield comparable results with other methods, for a data set of size on the order of millions. See also the experimental results in Section~\ref{sec:exp}.

The second approach is a cross-domain independent kernel. Simply speaking, the kernel matrix is approximated by keeping only the diagonal blocks of the matrix. In a GP language, we partition the domain $S$ into $m$ sub-domains $S_j$, $j=1,\ldots,m$, and make an independence assumption across sub-domains. Then, the covariance between $\bm{x}$ and $\bm{x}'$ vanishes when the two sites come from different sub-domains. That is,
\begin{equation}\label{eqn:indep}
k_{\text{independent}}(\bm{x},\bm{x}')=
\begin{cases}
k(\bm{x},\bm{x}'), & \text{if } \bm{x},\bm{x}'\in S_j \text{ for some } j,\\
0, & \text{otherwise}.
\end{cases}
\end{equation}
Whereas such a kernel appears ad hoc and associated theory is possibly limited, the scenarios when it exhibits superiority over a low-rank kernel of comparable sizes are not rare (see the likelihood comparison in~\citet{Stein2014} and also the experimental results in Section~\ref{sec:exp}). An intuitive explanation exists in the context of classification. The cross-domain independent kernel works well when geographically nearby points possess a majority of the signal for classifying points within the domain. This happens more often when the kernel has a reasonably centralized bandwidth (outside of which the kernel value becomes marginal). In such a case, nearby points are the most influential.

The third approach is covariance tapering~\citep{Furrer2006,Kaufman2008}. It amounts to defining a new kernel by multiplying the original kernel $k$ with a compactly supported kernel $k_{\text{compact}}$:
\[
k_{\text{taper}}(\bm{x},\bm{x}')=k(\bm{x},\bm{x}')\cdot k_{\text{compact}}(\bm{x},\bm{x}').
\]
The tapered kernel matrix is an elementwise product of two positive-definite matrices; hence, it is positive-definite, too~\citep[Theorem 5.2.1]{Horn1994}. The primary motivation of this kernel is to introduce sparsity to the matrix. The supporting theory is drawn on the confidence interval (cf.~\eqref{eqn:GP.var}) rather than on the prediction~\eqref{eqn:GP.mean}. It is cast in the setting of fixed-domain asymptotics, which is similar to a usual practice in machine learning---a prescaling of each attribute to within a finite interval. The theory hints that if the spectral density of $k_{\text{compact}}$ has a lighter tail (i.e., the spectrum of the corresponding kernel matrix decays faster) than that of $k$, then the ratio between the prediction variance by using the tapered kernel $k_{\text{taper}}$ and that by using the original kernel $k$ tends to a finite limit, as the number of training data increases to infinity in the domain. The theory holds a guarantee on the prediction confidence if we choose $k_{\text{compact}}$ judiciously. Tapering is generally applicable to heavy-tailed kernels (e.g., Mat\'{e}rn kernels with low smoothness) rather than light-tailed kernels such as the Gaussian. Nevertheless, a drawback of this approach is similar to the one we stated earlier for using a compactly supported kernel alone: the range of the support must be sufficiently small for sparse linear algebra to be efficient.

\subsection{Proposed Kernel}
In this paper, we propose a novel approach for constructing approximate kernels motivated by low-rank kernels and cross-domain independent kernels. The construction aims at deriving a kernel that (a) maintains the (strict) positive-definiteness, (b) leverages the advantages of low-rank and independent approaches, (c) facilitates the evaluation of the kernel matrix $K(X,X)$ and the out-of-sample extension $k(X,\bm{x})$, and (d) admits fast algorithms for a variety of matrix operations. The premise of the idea is a hierarchical partitioning of the data domain and a recursive approximation across the hierarchy. Space partitioning is a frequently encountered idea for kernel matrix approximations~\citep{Si2014,March2014,Yu2017,Si2017}, but maintaining positive definiteness is quite challenging. Moreover, when the approximation is performed in a hierarchical fashion and is focused on only the matrix, it is not always easy to generalize to out of samples. A particularly intriguing property of the approach proposed in this article is that both positive definiteness and out-of-sample extensions are guaranteed, because the construction acts on the kernel function itself.

\section{Hierarchically Compositional Kernel}\label{sec:model}
The low-rank kernel (in particular, the Nystr\"{o}m approximation $k_{\text{Nystr\"{o}m}}$) and the cross-domain independent kernel $k_{\text{independent}}$ are complementary to each other in the following sense: the former acts on the global space, where the covariance at every pair of points $\bm{x}$ and $\bm{x}'$ are deducted by a conditional covariance based on the conditioned set $\ud{X}$ chosen globally; whereas the latter preserves all the local information but completely ignores the interrelationship outside the local domain. We argue that an organic composition of the two will carry both advantages and alleviate the shortcomings. Further, a hierarchical composition may reduce the information loss in nearby local domains.

\subsection{Composition of Low-Rank Kernel with Cross-Domain Independent Kernel}\label{sec:composit}
Let the domain $S$ be partitioned into disjoint sub-domains $\bigcup S_j=S$. Let $\ud{X}$ be a set of landmark points in $S$. For generality, $\ud{X}$ needs not be a subset of the training data $X$. Consider the function
\[
k_{\text{compositional}}(\bm{x},\bm{x}')=
\begin{cases}
k(\bm{x},\bm{x}'), & \text{if } \bm{x},\bm{x}'\in S_j \text{ for some } j,\\
k(\bm{x},\ud{X})K(\ud{X},\ud{X})^{-1}k(\ud{X},\bm{x}'), & \text{otherwise}.
\end{cases}
\]
Clearly, $k_{\text{compositional}}$ leverages both~\eqref{eqn:Nystrom} and~\eqref{eqn:indep}. When two points $\bm{x}$ and $\bm{x}'$ are located in the same domain, they maintain the full covariance~\eqref{eqn:indep}; whereas when they are located in separate domains, their covariance comes from the low-rank kernel $k_{\text{Nystr\"{o}m}}$~\eqref{eqn:Nystrom}. Such a composition complements missing information across domains in $k_{\text{independent}}$ and also complements the information loss in local domains caused by the Nystr\"{o}m approximation. The following result is straightforward in light of the fact that if $\bm{x}\in\ud{X}$, then $K(\ud{X},\ud{X})^{-1}k(\ud{X},\bm{x})$ is a column of the identity matrix where the only nonzero element (i.e., $1$) is located with respect to the location of $\bm{x}$ inside $\ud{X}$.

\begin{proposition}
We have
\[
k_{\text{compositional}}(\bm{x},\bm{x}')=k(\bm{x},\bm{x}'),
\]
if $\bm{x},\bm{x}'\in S_j \text{ for some } j$, or if either of $\bm{x},\bm{x}'$ belongs to $\ud{X}$.
\end{proposition}

An alternative view of the kernel $k_{\text{compositional}}$ is that it is an additive combination of a globally low-rank approximation and local Schur complements within each sub-domain. Hence, the kernel is (strictly) positive-definite. See Lemma~\ref{lem:pd} and Theorem~\ref{thm:pd} in the following.

\begin{lemma}\label{lem:pd}
The Schur-complement function
\[
k_{\text{Schur}}(\bm{x},\bm{x}')=k(\bm{x},\bm{x}')-k(\bm{x},\ud{X})K(\ud{X},\ud{X})^{-1}k(\ud{X},\bm{x}')
\]
is positive-definite, if $k$ is strictly positive-definite, or if $k$ is positive-definite and $K(\ud{X},\ud{X})$ is invertible.
\end{lemma}

\begin{proof}
For any set $X$, let $Y=X\cup\ud{X}$. It amounts to showing that the corresponding kernel matrix $K_{\text{Schur}}(Y,Y)$ is positive semi-definite; then, $K_{\text{Schur}}(X,X)$ as a principal submatrix is also positive semi-definite.

Denote by $\ud{X}^c=Y\backslash\ud{X}$, which could possibly be empty, and let the points in $\ud{X}$ be ordered before those in $\ud{X}^c$. Then,
\[
K_{\text{Schur}}(Y,Y)=\begin{bmatrix}
0 & 0 \\
0 & K(\ud{X}^c,\ud{X}^c)-K(\ud{X}^c,\ud{X})K(\ud{X},\ud{X})^{-1}K(\ud{X},\ud{X}^c)
\end{bmatrix}.
\]
By the law of inertia, the matrices
\[
\begin{bmatrix}
K(\ud{X},\ud{X}) & K(\ud{X},\ud{X}^c)\\
K(\ud{X}^c,\ud{X}) & K(\ud{X}^c,\ud{X}^c)
\end{bmatrix}
\quad\text{and}\quad
\begin{bmatrix}
K(\ud{X},\ud{X}) & 0 \\
0 & K(\ud{X}^c,\ud{X}^c)-K(\ud{X}^c,\ud{X})K(\ud{X},\ud{X})^{-1}K(\ud{X},\ud{X}^c)
\end{bmatrix}
\]
have the same number of positive, zero, and negative eigenvalues, respectively. If $k$ is strictly positive-definite, then the eigenvalues of both matrices are all positive. If $k$ is positive-definite, then the eigenvalues of both matrices are all nonnegative. In both cases, the Schur-complement matrix $K(\ud{X}^c,\ud{X}^c)-K(\ud{X}^c,\ud{X})K(\ud{X},\ud{X})^{-1}K(\ud{X},\ud{X}^c)$ is positive semi-definite and thus so is $K_{\text{Schur}}(Y,Y)$.
\end{proof}

\begin{theorem}\label{thm:pd}
The function $k_{\text{compositional}}$ is positive-definite if $k$ is positive-definite and $K(\ud{X},\ud{X})$ is invertible. Moreover, $k_{\text{compositional}}$ is strictly positive-definite if $k$ is so.
\end{theorem}

\begin{proof}
Write $k_{\text{compositional}}=k_1+k_2$, where $k_1(\bm{x},\bm{x}')=k(\bm{x},\ud{X})K(\ud{X},\ud{X})^{-1}k(\ud{X},\bm{x}')$ and
\[
k_2(\bm{x},\bm{x}')=
\begin{cases}
k(\bm{x},\bm{x}')-k(\bm{x},\ud{X})K(\ud{X},\ud{X})^{-1}k(\ud{X},\bm{x}'), & \bm{x},\bm{x}'\in S_j \text{ for some } j,\\
0, & \text{otherwise}.
\end{cases}
\]
Clearly, $k_1$ is positive-definite; by Lemma~\ref{lem:pd}, $k_2$ is so, too. Thus, $k_{\text{compositional}}$ is positive-definite.

We next show the strict definiteness when $k$ is strictly positive-definite. That is, for any set of points $\{\bm{x}_i\}$ and any set of coefficients $\{\alpha_i\}$ that are not all zero, the bilinear form $\sum_{il}\alpha_i\alpha_lk_{\text{compositional}}(\bm{x}_i,\bm{x}_l)$ cannot be zero. Note that $k_2(\bm{x},\bm{x}')=0$ whenever $\bm{x}$ or $\bm{x}'\in\ud{X}$. Moreover, we have seen in the proof of Lemma~\ref{lem:pd} that the Schur-complement matrix $K(\ud{X}^c,\ud{X}^c)-K(\ud{X}^c,\ud{X})K(\ud{X},\ud{X})^{-1}K(\ud{X},\ud{X}^c)$ is positive-definite when $k$ is strictly positive-definite. Therefore, we have $\sum_{il}\alpha_i\alpha_lk_2(\bm{x}_i,\bm{x}_l)=0$ only when $\alpha_i=0$ for all $i$ satisfying $\bm{x}_i\notin\ud{X}$. In such a case,
\[
\sum_{il}\alpha_i\alpha_lk_1(\bm{x}_i,\bm{x}_k)=\sum_{\bm{x}_i,\bm{x}_l\in\ud{X}}\alpha_i\alpha_lk(\bm{x}_i,\ud{X})K(\ud{X},\ud{X})^{-1}k(\ud{X},\bm{x}_l).
\]
Because of the strict positive-definiteness of $k$, the above summation cannot be zero if any of the involved $\alpha_i$ (that is, those satisfying $\bm{x}_i\in\ud{X}$) is nonzero. Then, $\sum_{il}\alpha_i\alpha_l[k_1(\bm{x}_i,\bm{x}_l)+k_2(\bm{x}_i,\bm{x}_l)]=0$ only when all $\alpha_i$ are zero.
\end{proof}

Since the composition replaces the Nystr\"{o}m approximation in local domains by the full covariance, it bares no surprise that $k_{\text{compositional}}$ improves over $k_{\text{Nystr\"{o}m}}$ in terms of matrix approximation.

\begin{theorem}\label{thm:mat}
Given a set $\ud{X}$ of landmark points and for any set $X\ne\ud{X}$,
\[
\|K(X,X)-K_{\text{compositional}}(X,X)\|<\|K(X,X)-K_{\text{Nystr\"{o}m}}(X,X)\|,
\]
where $\|\cdot\|$ is the 2-norm or the Frobenius norm.
\end{theorem}

\begin{proof}
Based on the split $k_{\text{compositional}}=k_1+k_2$ in the proof of Theorem~\ref{thm:pd}, one easily sees that
\[
K-K_{\text{compositional}}=(K-K_{\text{Nystr\"{o}m}})-\text{block-diag}(K-K_{\text{Nystr\"{o}m}}),
\]
where block-diag means keeping only the diagonal blocks of a matrix. Denote by $A=K-K_{\text{Nystr\"{o}m}}$ and $D=\text{block-diag}(A)$. In what follows we show that
\begin{equation}\label{eqn:AD}
\|A-D\|<\|A\|.
\end{equation}

Because $A=K(X,X)-K(X,\ud{X})K(\ud{X},\ud{X})^{-1}K(\ud{X},X)$ is positive semi-definite and nonzero, its diagonal cannot be zero. Then, eliminating the block-diagonal part $D$ reduces the Frobenius norm. Thus, \eqref{eqn:AD} holds for the Frobenius norm. To see that~\eqref{eqn:AD} also holds for the 2-norm, let $Y=X\backslash\ud{X}$ and let the points in $Y$ be ordered before those in $X\backslash Y$. Then,
\[
A=\begin{bmatrix}
K(Y,Y)-K(Y,\ud{X})K(\ud{X},\ud{X})^{-1}K(\ud{X},Y) & 0\\
0 & 0
\end{bmatrix}.
\]
Because the zero rows and columns do not contribute to the 2-norm, and because the top-left block of $A$ is positive-definite, it suffices to prove~\eqref{eqn:AD} for any positive-definite matrix $A$.

Note the following two straightforward inequalities
\begin{gather}
\lambda_{\min}(A-D)\ge\lambda_{\min}(A)-\lambda_{\max}(D),\label{eqn:eig1}\\
\lambda_{\max}(A-D)\le\lambda_{\max}(A)-\lambda_{\min}(D).\label{eqn:eig2}
\end{gather}
Because $D$ consists of the diagonal blocks of $A$, the interlacing theorem of eigenvalues states that for each diagonal block $D_i$, we have
\[
\lambda_{\min}(A)\le\lambda_{\min}(D_i)\le\lambda_{\max}(D_i)\le\lambda_{\max}(A).
\]
Then, taking the max/min eigenvalues of all blocks, we obtain
\begin{equation}\label{eqn:eig3}
\lambda_{\min}(A)\le\lambda_{\min}(D)\le\lambda_{\max}(D)\le\lambda_{\max}(A).
\end{equation}
Substituting~\eqref{eqn:eig3} into~\eqref{eqn:eig1} and~\eqref{eqn:eig2}, together with $\lambda_{\min}(A)>0$, we obtain
\[
\lambda_{\min}(A-D)>-\lambda_{\max}(A)
\quad\text{and}\quad
\lambda_{\max}(A-D)<\lambda_{\max}(A),
\]
which immediately implies that $\|A-D\|_2<\|A\|_2$.
\end{proof}

\subsection{Hierarchical Composition}\label{sec:hierarchical}
While $k_{\text{compositional}}$ maintains the full information inside each domain $S_j$, the information loss across domains caused by the low-rank approximation may still be dramatic. Consider the scenario of a large number of disjoint domains $S_j$; such a scenario is necessarily typical for the purpose of reducing the computational cost. If each domain is adjacent to only a few neighboring domains, it is possible to reduce the information loss in nearby domains.

The idea is to form a hierarchy. Let us first take a two-level hierarchy for example. A few of the neighboring domains $S_j$ form a super-domain $S_J=\bigcup_{j\in J}S_j$. These super-domains are formed such that they are disjoint and they collectively partition the whole domain; i.e., $\bigcup S_J=S$. Under such a hierarchical formation, instead of using landmark points $\ud{X}$ in $S$ to define the covariance across the bottom-level domains $S_j$, we may use landmark points $\ud{X}_J$ chosen from the super-domain $S_J$ to define the covariance. The intuition is that the conditional covariance $k(\bm{x},\bm{x}'|\ud{X}_J)$ for $\bm{x},\bm{x}'\in S_J$ tends to be smaller than $k(\bm{x},\bm{x}'|\ud{X})$, because $\bm{x}$ and $\bm{x}'$ are geographically closer to $\ud{X}_J$ than to $\ud{X}$. Then, the information loss is reduced for points inside the same super-domain $S_J$.

To formalize this idea, we consider an arbitrary hierarchy, which is represented by a rooted tree $T$. See Figure~\ref{fig:tree} for an example. The root node $1$ is associated with the whole domain $S=:S_1$. Each nonleaf node $i$ possesses a set of children $Ch(i)$; correspondingly, the associated domain $S_i$ is partitioned into disjoint sub-domains $S_j$ satisfying $\bigcup_{j\in Ch(i)}=S_i$. The partitioning tree $T$ is almost the most general rooted tree, except that no nodes in $T$ have exactly one child.

\begin{figure}[ht]
\centering
\includegraphics[width=.48\linewidth]{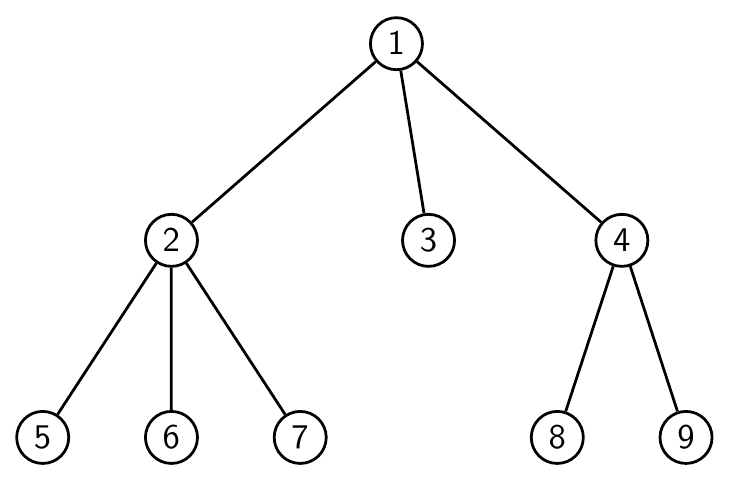}
\includegraphics[width=.48\linewidth]{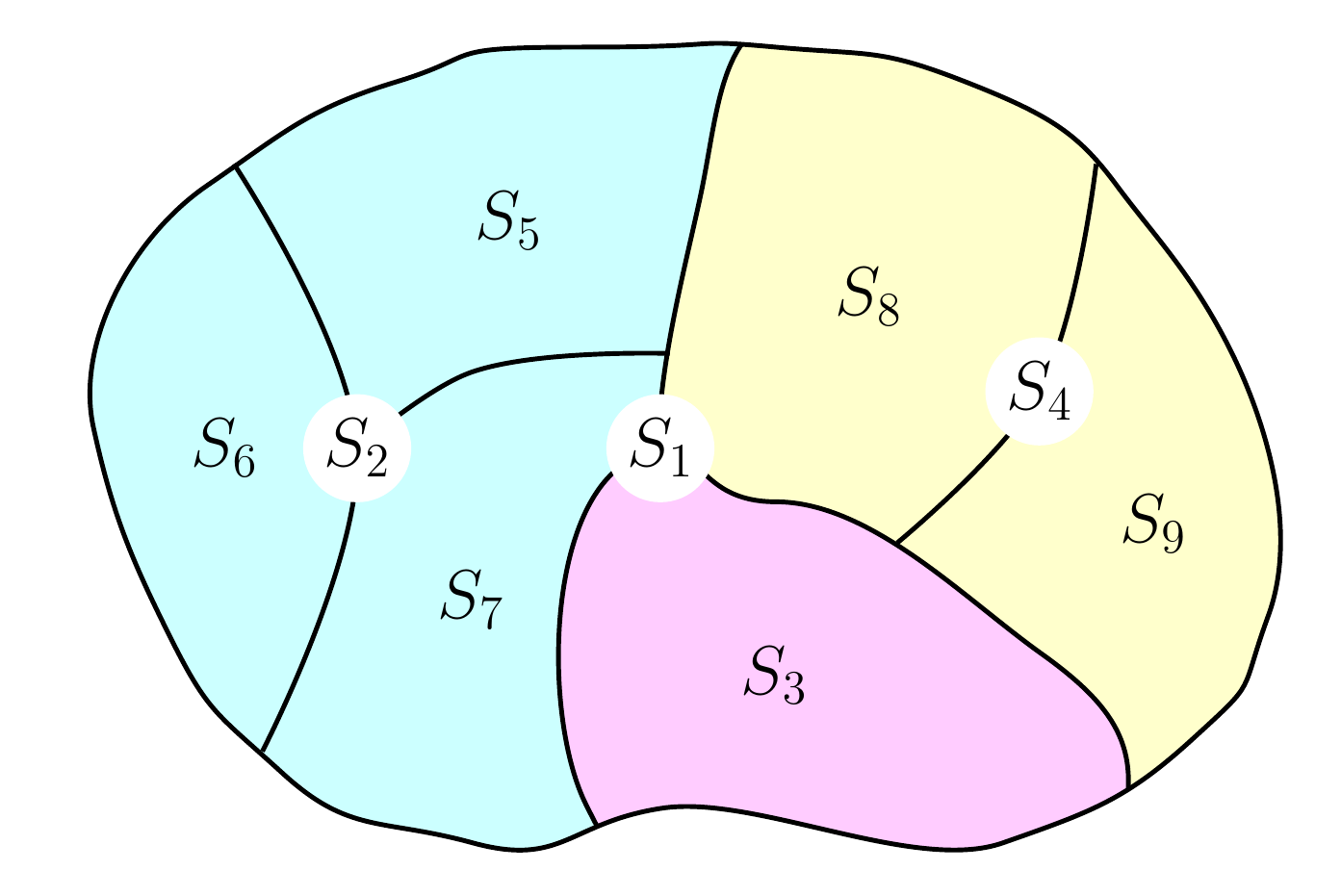}
\caption{Partitioning tree $T$ and the domain $S$.}
\label{fig:tree}
\end{figure}

Each nonleaf node $i$ is associated with a set $\ud{X}_i$ of landmark points, located within the domain $S_i$. We now recursively define the kernel on domains across levels. Continuing the example of Figure~\ref{fig:tree}, node 4 has two children 8 and 9. Since these two children are leaf nodes, the covariance within $S_8$ (or $S_9$) comes from the original kernel $k$, whereas the covariance across $S_8$ and $S_9$ comes from the Nystr\"{o}m approximation by using landmark points $\ud{X}_4$. That is, the kernel is equal to $k(\bm{x},\bm{x}')$ if $\bm{x}$ and $\bm{x}'$ are both in $S_8$ (or in $S_9$), and equal to $k(\bm{x},\ud{X}_4)K(\ud{X}_4,\ud{X}_4)^{-1}k(\ud{X}_4,\bm{x}')$ if they are in $S_8$ and $S_9$ separately. Such a covariance bares less information loss caused by the conditioned set, compared with the use of landmark points located within the whole domain $S$.

Next, consider the covariance between child domains of $S_2$ and those of $S_4$ (say, $S_6$ and $S_9$, respectively). At a first glance, we could have used $k(\bm{x},\ud{X}_1)K(\ud{X}_1,\ud{X}_1)^{-1}k(\ud{X}_1,\bm{x}')$ to define the kernel, because $\ud{X}_1$ consists of landmark points located in the domain that covers both $S_6$ and $S_9$. However, such a definition cannot guarantee the positive-definiteness of the overall kernel. Instead, we approximate $k(\bm{x},\ud{X}_1)$ by using the Nystr\"{o}m approximation $k(\bm{x},\ud{X}_2)K(\ud{X}_2,\ud{X}_2)^{-1}K(\ud{X}_2,\ud{X}_1)$ based on the landmark points $\ud{X}_2$. Then, the covariance for $\bm{x}\in S_6$ and $\bm{x}'\in S_9$ is defined as
\[
\Big[k(\bm{x},\ud{X}_2)K(\ud{X}_2,\ud{X}_2)^{-1}K(\ud{X}_2,\ud{X}_1)\Big]
K(\ud{X}_1,\ud{X}_1)^{-1}
\Big[K(\ud{X}_2,\ud{X}_1)K(\ud{X}_2,\ud{X}_2)^{-1}k(\ud{X}_2,\bm{x}')\Big].
\]

Formally, for a leaf node $j$ and $\bm{x},\bm{x}'\in S_j$, define $k^{(j)}(\bm{x},\bm{x}')\equiv k(\bm{x},\bm{x}')$. For a nonleaf node $i$ and $\bm{x},\bm{x}'\in S_i$, define
\begin{equation}\label{eqn:phi.i}
k^{(i)}(\bm{x},\bm{x}'):=
\begin{cases}
k^{(j)}(\bm{x},\bm{x}'), & \text{if } \bm{x},\bm{x}'\in S_j \text{ for some } j\in Ch(i),\\
\psi^{(i)}(\bm{x},\ud{X}_i)K(\ud{X}_i,\ud{X}_i)^{-1}\psi^{(i)}(\ud{X}_i,\bm{x}'), & \text{otherwise},
\end{cases}
\end{equation}
where if $j$ is a child of $i$ and if $\bm{x}\in S_j$, then
\begin{equation}\label{eqn:psi}
\psi^{(i)}(\bm{x},\ud{X}_i):=
\begin{cases}
k(\bm{x},\ud{X}_i), & \text{if } j \text{ is a leaf node},\\
\psi^{(j)}(\bm{x},\ud{X}_j)K(\ud{X}_j,\ud{X}_j)^{-1}K(\ud{X}_j,\ud{X}_i),
& \text{otherwise}.
\end{cases}
\end{equation}
The $k^{(i)}$ at the root level gives the \emph{hierarchically compositional kernel} of this paper:
\begin{equation}\label{eqn:hierarchical}
k_{\text{hierarchical}}:=k^{\text{(root)}}.
\end{equation}
Clearly, the kernel $k_{\text{compositional}}$ in Section~\ref{sec:composit} is a special case of $k_{\text{hierarchical}}$ when the partitioning tree consists of only the root and the leaf nodes (which are children of the root).

Expanding the recursive formulas~\eqref{eqn:phi.i} and~\eqref{eqn:psi}, for two distinct leaf nodes $j$ and $l$, we see that the covariance between $\bm{x}\in S_j$ and $\bm{x}'\in S_l$ is
\begin{multline}\label{eqn:k.expand}
k_{\text{hierarchical}}(\bm{x},\bm{x}')=k^{(r)}(\bm{x},\bm{x}')\\
=\underbrace{k(\bm{x},\ud{X}_{j_1})K(\ud{X}_{j_1},\ud{X}_{j_1})^{-1}K(\ud{X}_{j_1},\ud{X}_{j_2})\cdots K(\ud{X}_{j_s},\ud{X}_{j_s})^{-1}K(\ud{X}_{j_s},\ud{X}_{r})}_{\psi^{(r)}(\bm{x},\ud{X}_r)}K(\ud{X}_{r},\ud{X}_{r})^{-1}\\
\cdot\underbrace{K(\ud{X}_{r},\ud{X}_{l_t})K(\ud{X}_{l_t},\ud{X}_{l_t})^{-1}\cdots K(\ud{X}_{l_2},\ud{X}_{l_1})K(\ud{X}_{l_1},\ud{X}_{l_1})^{-1}k(\ud{X}_{l_1},\bm{x}')}_{\psi^{(r)}(\ud{X}_r,\bm{x}')},
\end{multline}
where $r$ is the least common ancestor of $j$ and $l$, and $(j,j_1,j_2,\ldots,j_s,r)$ and $(l,l_1,l_2,\ldots,l_t,r)$ are the paths connecting $r$ and the two leaf nodes, respectively. Therefore, we have the following result.

\begin{proposition}
Based on the notation in the preceding paragraph, for $\bm{x}\in S_j$ and $\bm{x}'\in S_l$, we have
\[
k_{\text{hierarchical}}(\bm{x},\bm{x}')=k(\bm{x},\bm{x}'),
\]
whenever $\bm{x}\in\ud{X}_j\cap\ud{X}_{j_1}\cdots\cap\ud{X}_{j_s}$, $\bm{x}'\in\ud{X}_l\cap\ud{X}_{l_1}\cdots\cap\ud{X}_{l_t}$, and either of $\bm{x},\bm{x}'$ belongs to $\ud{X}_r$.
\end{proposition}

\begin{proof}
On~\eqref{eqn:k.expand}, recursively apply the fact that if $\bm{x}\in\ud{X}$, then $K(\ud{X},\ud{X})^{-1}k(\ud{X},\bm{x})$ is a column of the identity matrix where the only nonzero element (i.e., $1$) is located with respect to the location of $\bm{x}$ inside $\ud{X}$.
\end{proof}

The following theorem guarantees the validity of the kernel. Its proof strategy is similar to that of Theorem~\ref{thm:pd}, but it is complex because of recursion. We defer the proof to Appendix~\ref{sec:proof.pd.2}.

\begin{theorem}\label{thm:pd.2}
The function $k_{\text{hierarchical}}$ is positive-definite if $k$ is positive-definite and $K(\ud{X}_i,\ud{X}_i)$ is invertible for all sets of landmark points $\ud{X}_i$ associated with the nonleaf nodes $i$. Moreover, $k_{\text{hierarchical}}$ is strictly positive-definite if $k$ is so.
\end{theorem}

\section{Matrix View}\label{sec:mat}
The kernel matrix $K_{\text{hierarchical}}(X,X)$ for a set of training points $X$ exhibits a hierarchical block structure. Figure~\ref{fig:matrix} pictorially shows such a structure for the example in Figure~\ref{fig:tree}. To avoid degenerate empty blocks, we assume that $X\cap S_j\ne\emptyset$ for all leaf nodes $j$ in the partitioning tree $T$. 

\begin{figure}[ht]
\centering
\includegraphics[width=.45\linewidth]{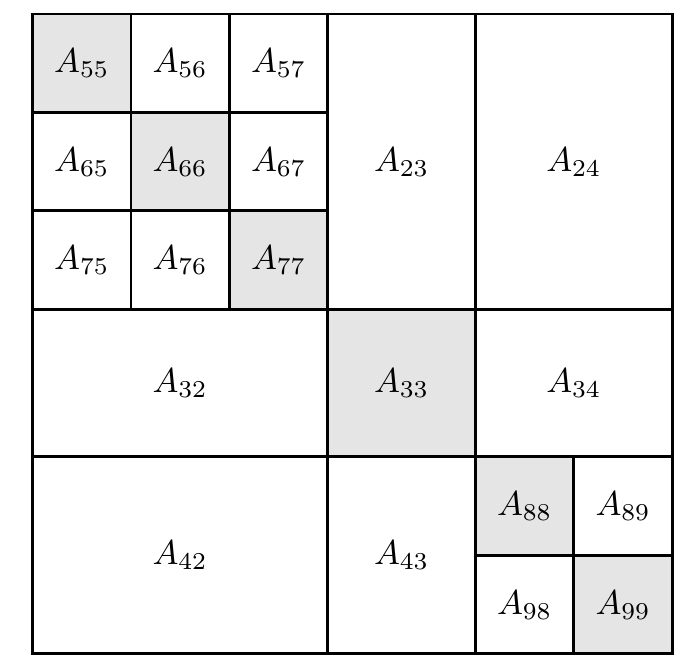}
\caption{The matrix $K_{\text{hierarchical}}$ corresponding to the partitioning tree $T$ in Figure~\ref{fig:tree}.}
\label{fig:matrix}
\end{figure}

Formally, for any node $i$, let $X_i=X\cap S_i$; that is, $X_i$ consists of the training points that fall within the domain $S_i$. Then, define a matrix $A\in\real^{n\times n}$ with the following structure:
\begin{enumerate}
\item For every node $i$, $A_{ii}$ is a diagonal block whose rows and columns correspond to $X_i$; for every pair of sibling nodes $i$ and $j$, $A_{ij}$ is an off-diagonal block whose rows correspond to $X_i$ and columns to $X_j$.
\item For every leaf node $i$, $A_{ii}=K(X_i,X_i)$.
\item For every pair of sibling nodes $i$ and $j$, $A_{ij}=U_i\Sigma_pU_j^T$, where $p$ is the parent of $i$ and $j$, with $\Sigma_p$ and $U_i$ defined next.
\item For every nonleaf node $p$, $\Sigma_p=K(\ud{X}_p,\ud{X}_p)$.
\item For every leaf node $i$, $U_i=K(X_i,\ud{X}_p)K(\ud{X}_p,\ud{X}_p)^{-1}$, where $p$ is the parent of $i$.
\item For every pair of child node $i$ and parent node $p$ not being the root, the block of $U_p$ corresponding to node $i$ is $U_iW_p$, where $W_p=K(\ud{X}_p,\ud{X}_r)K(\ud{X}_r,\ud{X}_r)^{-1}$ and $r$ is the parent of $p$. That is, in the matrix form,
\[
U_p=\begin{bmatrix}\vdots\\ U_i\\ \vdots\end{bmatrix}_{i\in Ch(p)}\cdot W_p.
\]
\end{enumerate}
One sees that $A$ is exactly equal to $K_{\text{hierarchical}}(X,X)$ by verifying against the definition of $k_{\text{hierarchical}}$ in~\eqref{eqn:phi.i}--\eqref{eqn:hierarchical}.

Such a hierarchical block structure is a special case of the \emph{recursively low-rank compressed matrix} studied in~\citet{Chen2014a}. In this matrix, off-diagonal blocks are recursively compressed into low rank through change of basis, while the main-diagonal blocks at the leaf level remain intact. Its connections and distinctions with related matrices (e.g., FMM matrices~\citep{Sun2001} and hierarchical matrices~\citep{Hackbusch1999}) were discussed in detail in~\citet{Chen2014a}. The signature of a recursively low-rank compressed matrix is that many matrix operations can be performed with a cost, loosely speaking, linear in $n$. The matrix structure in this paper, which results from the kernel $k_{\text{hierarchical}}$, specializes a general recursively low-rank compressed matrix in the following aspects: (a) the matrix is symmetric; (b) the middle factor $\Sigma_p$ in each $A_{ij}$ is the same for all child pairs $i,j$ of $p$; and (c) the change-of-basis factor $W_p$ is the same for all children $i$ of $p$. Hence, storage and time costs of matrix operations are reduced by a constant factor compared with those of the algorithms in~\citet{Chen2014a}. In what follows, we discuss the algorithmic aspects of the matrix operations needed to carry out kernel method computations, but defer the complexity analysis in Section~\ref{sec:computational}. Discussions of additional matrix operations useful in a general machine learning context are made toward the end of the paper.

\subsection{Matrix-Vector Multiplication}\label{sec:matvec}
First, we consider computing the matrix-vector product $y=Ab$. Let the blocks of a vector be labeled in the same manner as those of the matrix $A$. That is, for any node $i$, $b_i$ denotes the part of $b$ that corresponds to the point set $X_i$. Then, the vector $y$ is clearly an accumulation of the smaller matrix-vector products $A_{ij}b_j$ in the appropriate blocks, for all sibling pairs $i,j$ and for all leaf nodes $i=j$ (cf.\ Figure~\ref{fig:matrix}). When $i=j$ is a leaf node, the computation of $A_{ij}b_j$ is straightforward. On the other hand, when $i$ and $j$ constitute a pair of sibling nodes, we consider a descendant leaf node $l$ of $i$. The block of $A_{ij}b_j$ corresponding to node $l$ admits an expanded expression
\begin{equation}\label{eqn:complicated.brackets}
U_lW_{l_1}W_{l_2}\cdots W_{l_s}W_i\Sigma_pW_j^T\left(\sum_{q_t\in Ch(j)}W_{q_t}^T\cdots\left(\sum_{q_2\in Ch(q_3)}W_{q_2}^T\left(\sum_{q_1\in Ch(q_2)}W_{q_1}^T\left(\sum_{q\in Ch(q_1)}U_q^Tb_q\right)\right)\right)\right),
\end{equation}
where $(l,l_1,l_2,\ldots,l_s,i,p)$ is the path connecting $l$ and the parent $p$ of $i$. This expression assumes that all the leaf descendants of $j$ are on the same level so that the expression is not overly complex; but the subsequent reasoning applies to the general case. The terms inside the nested parentheses of~\eqref{eqn:complicated.brackets} motivate the following recursive definition:
\[
c_i=\begin{dcases}
U_i^Tb_i, & i \text{ is a leaf node},\\
W_i^T\sum_{j\in Ch(i)}c_j, & \text{otherwise}.\\
\end{dcases}
\]
Clearly, all $c_j$'s can be computed within one pass of a post-order tree traversal. Upon the completion of the traversal, we rewrite~\eqref{eqn:complicated.brackets} as $U_lW_{l_1}W_{l_2}\cdots W_{l_s}W_i\Sigma_pc_j$. Then we note that for any leaf node $l$, $y_l$ is a sum of these expressions with all $p$ along the path connecting $l$ and the root, and additionally, of $A_{ll}b_l$. In other words, we have
\begin{multline*}
y_l=A_{ll}b_l+\left(\sum_{l'\in Ch(l_1)\backslash\{l\}}U_l\Sigma_{l_1}c_{l'}\right)
+\left(\sum_{l'\in Ch(l_2)\backslash\{l_1\}}U_lW_{l_1}\Sigma_{l_2}c_{l'}\right)\\
+\cdots
+\left(\sum_{l'\in Ch(\text{root})\backslash\{l_t\}}U_lW_{l_1}W_{l_2}\cdots W_{l_t}\Sigma_{\text{root}}c_{l'}\right),
\end{multline*}
where $(l,l_1,l_2,\ldots,l_t,\text{root})$ is the path connecting $l$ and the root. Therefore, we recursively define another quantity
\[
d_j=W_id_i+\sum_{j'\in Ch(i)\backslash\{j\}}\Sigma_ic_{j'}
\]
for all nonroot nodes $j$ with parent $i$. Clearly, all $d_j$'s can also be computed within one pass of a pre-order tree traversal. Upon the completion of this second traversal, we have $y_l=A_{ll}b_l+U_ld_l$, which concludes the computation of the whole vector $y$.

As such, computing the matrix-vector product $y=Ab$ consists of one post-order tree traversal, followed by a pre-order one. We refer the reader to~\citet{Chen2014a} for a full account of the computational steps. For completeness, we summarize the pesudocode in Algorithm~\ref{algo:Ab} as a reference for computer implementation.

\begin{algorithm}[!ht]
\caption{Computing $y=Ab$}
\label{algo:Ab}
\begin{algorithmic}[1]
\State Initialize $c_i\gets0$, $d_i\gets0$ for each nonroot node $i$ of the tree
\State \Call{Upward}{\texttt{root}}
\State \Call{Downward}{\texttt{root}}
\Statex

\Function{Upward}{$i$}
\If{$i$ is leaf}
\State $c_i\gets U_i^Tb_i$; \,\, $y_i\gets A_{ii}b_i$
\Else
\ForAll{children $j$ of $i$}
\State \Call{Upward}{$j$}
\State $c_i\gets c_i+W_i^Tc_j$, \textbf{if} $i$ is not root
\EndFor
\EndIf
\If{$i$ is not root}
\State \textbf{for all} siblings $l$ of $i$ \textbf{do} $d_l\gets d_l+\Sigma_p c_i$ \textbf{end for}
\Comment{$p$ is parent of $i$}
\EndIf
\EndFunction
\Statex

\Function{Downward}{$i$}
\State \textbf{if} $i$ is leaf \textbf{then} $y_i\gets y_i+U_id_i$ and return \textbf{end if}
\ForAll{children $j$ of $i$}
\State $d_j\gets d_j+W_id_i$, \textbf{if} $i$ is not root
\State \Call{Downward}{$j$}
\EndFor
\EndFunction
\end{algorithmic}
\end{algorithm}

\subsection{Matrix Inversion}\label{sec:invert}
Next, we consider computing $A^{-1}$, for which we use the tilded notation $\tilde{A}:=A^{-1}$. One can show~\citep{Chen2014a} that $\tilde{A}$ has exactly the same hierarchical block structure as does $A$. In other words, the structure of $\tilde{A}$ can be described by reusing the verbatim at the beginning of Section~\ref{sec:mat}, with the factors $A_{ii}$, $A_{ij}$, $U_i$, $\Sigma_p$, $W_p$ replaced by the tilded version $\tilde{A}_{ii}$, $\tilde{A}_{ij}$, $\tilde{U}_i$, $\tilde{\Sigma}_p$, $\tilde{W}_p$, respectively. The proof is both inductive and constructive, providing explicit formulas to compute $\tilde{A}_{ii}$, $\tilde{A}_{ij}$, $\tilde{U}_i$, $\tilde{\Sigma}_p$, $\tilde{W}_p$ level by level. The essential idea is that if $(A_{ii}-U_i\Sigma_pU_i^T)^{-1}=\tilde{A}_{ii}-\tilde{U}_i\tilde{\Sigma}_p\tilde{U}_i^T$ have been computed for all children $i$ of a node $p$, and if $r$ is the parent of $p$, then
\[
A_{pp}-U_p\Sigma_rU_p^T=
\begin{bmatrix}
\ddots\phantom{A_{ii}-U_i\Sigma_pU_i^T}\\
A_{ii}-U_i\Sigma_pU_i^T\\
\phantom{A_{ii}-U_i\Sigma_pU_i^T}\ddots\\
\end{bmatrix}
+\begin{bmatrix}\vdots \\ U_i \\ \vdots\end{bmatrix}
\Sigma_p
\begin{bmatrix}\cdots & U_i^T & \cdots\end{bmatrix}
-U_p\Sigma_rU_p^T.
\]
Hence, the inversion of $A_{pp}-U_p\Sigma_rU_p^T$ can be done by applying the Sherman--Morrison--Woodbury formula, which results in a form $\tilde{A}_{pp}-\tilde{U}_p\tilde{\Sigma}_r\tilde{U}_p^T$, where $\tilde{W}_r$ is determined based on the change of basis from $\tilde{U}_i$ to $\tilde{U}_p$ and the middle factor $\tilde{\Sigma}_r$ is also determined. In addition, the factors $\tilde{\Sigma}_j$ (computed previously) for all descendants $j$ of $r$ must be corrected because the diagonal block of $\tilde{A}_{pp}$ corresponding to node $j$ is a low-rank correction of the corresponding diagonal block of $\tilde{A}_{ii}$ with the same basis. Such a down-cascading correction occurs whenever the induction proceeds from one child level to the parent level; but computationally, the correction on any node can be accumulated during the whole induction process. The net result of the accumulation is that each $\tilde{\Sigma}_i$ needs be corrected only once, which ensures an efficient computation.

We refer the reader to~\citet{Chen2014a} for a full account of the computational steps. Similar to the matrix-vector multiplication, the overall computation here consists of one post-order tree traversal (corresponding to the induction), followed by a pre-order one (corresponding to the down-cascading correction). For completeness, we summarize the pseudocode in Algorithm~\ref{algo:invA} as a reference for computer implementation.

{%
\renewcommand{\baselinestretch}{1.1}
\begin{algorithm}[!ht]
\caption{Computing $\tilde{A}=A^{-1}$}
\label{algo:invA}
\begin{algorithmic}[1]
\State \Call{Upward}{\texttt{root}}
\State \Call{Downward}{\texttt{root}}
\Statex

\Function{Upward}{$i$}
\If{$i$ is leaf}
\State\label{algo.ln:patch1} $\tilde{A}_{ii}\gets (A_{ii}-U_i\Sigma_pU_i^T)^{-1}$; \,\,
$\tilde{U}_i\gets \tilde{A}_{ii}U_i$; \,\,
$\tilde{\Theta}_i\gets U_i^T\tilde{U}_i$
\Comment{$p$ is parent of $i$}
\State\label{algo.ln:patch4} return
\EndIf
\ForAll{children $j$ of $i$}
\State \Call{Upward}{$j$}
\State $\tilde{W}_j\gets (I+\tilde{\Sigma}_j\tilde{\Xi}_j)W_j$ \textbf{if} $j$ is not leaf
\State $\tilde{\Theta}_j\gets W_j^T\tilde{\Xi}_j\tilde{W}_j$ \textbf{if} $j$ is not leaf
\EndFor
\State $\tilde{\Xi}_i\gets\sum_{j\in Ch(i)}\tilde{\Theta}_j$
\State \textbf{if} $i$ is not root \textbf{then}
$\tilde{\Lambda}_i\gets \Sigma_i-W_i\Sigma_pW_i^T$ \textbf{else} $\tilde{\Lambda}_i\gets \Sigma_i$
\textbf{end if}
\Comment{$p$ is parent of $i$}
\State $\tilde{\Sigma}_i\gets -(I+\tilde{\Lambda}_i\tilde{\Xi}_i)^{-1}\tilde{\Lambda}_i$
\ForAll{children $j$ of $i$}
\State $\tilde{E}_j\gets\tilde{W}_j\tilde{\Sigma}_i\tilde{W}_j^T$ \textbf{if} $j$ is not leaf
\EndFor
\State $\tilde{E}_i\gets0$ if $i$ is root
\EndFunction
\Statex
\Function{Downward}{$i$}
\If{$i$ is leaf}
\State $\tilde{A}_{ii}\gets\tilde{A}_{ii}+\tilde{U}_i\tilde{\Sigma}_p\tilde{U}_i^T$ if $i$ is not root
\Comment{$p$ is parent of $i$}
\Else
\State $\tilde{E}_i\gets \tilde{E}_i+\tilde{W}_i\tilde{E}_p\tilde{W}_i^T$ if $i$ is not root
\Comment{$p$ is parent of $i$}
\State $\tilde{\Sigma}_i\gets\tilde{\Sigma}_i+\tilde{E}_i$
\State \textbf{for all} children $j$ of $i$ \textbf{do} \Call{Downward}{$j$} \textbf{end for}
\EndIf
\EndFunction
\end{algorithmic}
\end{algorithm}
}

\subsection{(Implicit) Out-of-Sample Construction}\label{sec:out.of.sample}
Now, we consider the vector $k_{\text{hierarchical}}(X,\bm{x})$ for an existing training set $X$ and a new point $\bm{x}$ in the testing set. Generally, this vector is not used alone but it appears in the inner product with some other vector $w$ (see~\eqref{eqn:KRR}). Whereas the construction of $k_{\text{hierarchical}}(X,\bm{x})$ can be done at $O(n)$ cost (details omitted here), we shall consider instead the computation of the inner product $w^Tk_{\text{hierarchical}}(X,\bm{x})$, because the cost of computing this inner product is proportional to only the height of the tree per $\bm{x}$ after an $O(n)$ preprocessing. The preprocessing cost is amortized on each $\bm{x}$ and thus is generally negligible for a large number of $\bm{x}$'s.

The computation of $w^Tk_{\text{hierarchical}}(X,\bm{x})$ was not described in~\citet{Chen2014a}, but the idea is similar to that of the matrix-vector multiplication in Section~\ref{sec:matvec}. To simplify notation, let $v\equiv k_{\text{hierarchical}}(X,\bm{x})$. Assume that $\bm{x}$ falls in the domain $S_j$ for some leaf node $j$. Then, $v_j=k(X_j,\bm{x})$ and for any leaf node $l\ne j$,
\[
v_l=U_lW_{l_1}W_{l_2}\cdots W_{l_s}\Sigma_pW_{j_t}^T\cdots W_{j_2}^TW_{j_1}^TK(\ud{X}_{j_1},\ud{X}_{j_1})^{-1}k(\ud{X}_{j_1},\bm{x}),
\]
where $p$ is the least common ancestor of $l$ and $j$, and $(l,l_1,l_2,\ldots,l_s,p)$ and $(j,j_1,j_2,\ldots,j_t,p)$ are the paths connecting $l$ and $p$, and $j$ and $p$, respectively. Therefore, if we define
\begin{equation}\label{eqn:dd}
d_{j_u}=W_{j_u}^Td_{j_{u-1}},\quad
\cdots\quad
d_{j_2}=W_{j_2}^Td_{j_1},\quad
d_{j_1}=W_{j_1}^Td_j,\quad
d_j=K(\ud{X}_{j_1},\ud{X}_{j_1})^{-1}k(\ud{X}_{j_1},\bm{x}),
\end{equation}
where $(j,j_1,j_2,\ldots,j_u,\text{root})$ is the path connecting $j$ and the root, then we have
\begin{equation}\label{eqn:wv}
w^Tv=w_j^Tk(X_j,\bm{x})+\sum_{l\ne j,\,\, l \text{ is leaf}}w_l^TU_lW_{l_1}W_{l_2}\cdots W_{l_s}\Sigma_pd_{j_t}.
\end{equation}
If we further define
\begin{equation}\label{eqn:ec}
e_l=\begin{dcases}
U_l^Tw_l, & l \text{ is leaf},\\
\sum_{i\in Ch(l)}W_l^Te_i, & \text{otherwise},
\end{dcases}
\qquad\text{and}\qquad
c_q=\Sigma_p^Te_l,
\end{equation}
where $p$ is the parent of the sibling pair $q$ and $l$, then~\eqref{eqn:wv} is simplified to
\begin{equation}\label{eqn:wv2}
w^Tv=w_j^Tk(X_j,\bm{x})+\sum_{j_t\in\text{path}(j,\text{root})}c_{j_t}^Td_{j_t}.
\end{equation}

Based on~\eqref{eqn:dd}--\eqref{eqn:wv2}, we see that $w^Tv\equiv w^Tk_{\text{hierarchical}}(X,\bm{x})$ can be computed in the following two phases. In the first phase, we compute $e_l$ and $c_l$ for all nonleaf nodes $l$ according to~\eqref{eqn:ec}. Such a computation can be done by using a post-order tree traversal and is independent of $\bm{x}$. Thus, this computation is the preprocessing step. In the second phase, we compute $d_{j_t}$ for all $j_t$ along the path $(j,j_1,j_2,\ldots,j_u,\text{root})$ according to~\eqref{eqn:dd}. Once they are computed, the summation~\eqref{eqn:wv2} is straightforward and we thus conclude the computation. Note that the second phase is always conducted along a certain path connecting the root and the leaf $j$. As long as the determination of which leaf $j$ the point $\bm{x}$ falls in is restricted on this path, the cost of the second phase is always asymptotically smaller than that of a tree traversal. We summarize the pseudocode in Algorithm~\ref{algo:out.of.sample}.

\begin{algorithm}[!ht]
\caption{Computing $z=w^Tk_{\text{hierarchical}}(X,\bm{x})$ for $\bm{x}\notin X$}
\label{algo:out.of.sample}
\begin{algorithmic}[1]
\State Prefactorize $K(\ud{X}_p,\ud{X}_p)$ for all parents $p$ of leaf nodes
\State Initialize $c_i\gets0$, $d_i\gets0$ for each nonroot node $i$ of the tree
\State \Call{Common-Upward}{\texttt{root}}
\Statex $\triangleright$ The above three steps are independent of $\bm{x}$ and are treated as precomputation. In computer implementation, the intermediate results $c_i$ are carried over to the next step \Call{Second-Upward}{}, whereas the contents of $d_i$ are discarded and the allocated memory is reused.
\State \Call{Second-Upward}{\texttt{root}}
\Statex
\Function{Common-Upward}{$i$}
\If{$i$ is leaf}
\State $d_i\gets U_i^Tw_i$
\Else
\ForAll{children $j$ of $i$}
\State \Call{Common-Upward}{$j$}
\State $d_i\gets d_i+W_i^Td_j$, \textbf{if} $i$ is not root
\EndFor
\EndIf
\If{$i$ is not root}
\State \textbf{for all} siblings $l$ of $i$ \textbf{do} $c_l\gets\Sigma_p^T d_i$ \textbf{end for}
\Comment{$p$ is parent of $i$}
\EndIf
\EndFunction
\Statex
\Function{Second-Upward}{$i$}
\If{$i$ is leaf}
\State\label{algo.ln:pp} $d_i\gets K(\ud{X}_p,\ud{X}_p)^{-1}k(\ud{X}_p,\bm{x})$
\Comment{$p$ is parent of $i$}
\State\label{algo.ln:qq} $z\gets w_i^Tk(X_i,\bm{x})$
\Else
\State\label{algo.ln:fall} Find the child $j$ (among all children of $i$) where $\bm{x}$ lies on
\State \Call{Second-Upward}{$j$}
\State $d_i\gets W_i^Td_j$ if $i$ is not root
\EndIf
\State $z\gets z+c_i^Td_i$ if $i$ is not root
\EndFunction
\end{algorithmic}
\end{algorithm}

\section{Practical Considerations}\label{sec:computational}
Sections~\ref{sec:model} and~\ref{sec:mat} provide a general framework for the interpretation of and the computation with the proposed kernel $k_{\text{hierarchical}}$; they, however, have not covered a range of practical issues. Some issues are dimension dependent (this paper concerns particularly dimensions higher than three, which dominates the use of related matrices, e.g., FMM and hierarchical matrices, in scientific computing); whereas others are numerically related. In this section, we discuss the handling of several of these issues and ensure that the proposed kernel is computationally efficient.

\subsection{Partitioning of Domain}\label{sec:part}
Whereas the partitioning of a low-dimensional domain can be made regular, the partitioning of a high-dimensional domain is inherently difficult because of the curse of dimensionality. Several data-driven approaches appear to be natural choices but they have pros and cons.

The k-d tree~\citep{Bentley1975} approach is known to be efficient for nearest neighbor search. Generally, the approach iteratively selects an axis of the bounding box that contains the training points and partitions the axis such that the numbers of points on both sides are balanced. A drawback of this approach, particularly in the context of classification, lies in the case that some attributes of the data are binary and the counts are highly imbalanced. These attributes sometimes result from a conversion of categorical attributes to numeric ones. Partitioning on these axes is unlikely to be balanced. An alternative is to partition the axis into two segments of equal length. This approach often results in highly imbalanced partitioning because data are generally not evenly distributed.

The k-means~\citep{MacQueen1967} approach partitions a point set through a k-means clustering and hence the partitioning is a Voronoi diagram of the cluster centers. The resulting tree is not necessarily binary if there exhibits more than two natural clusters in the data on some level. An advantage of this approach is that the clustering often results in a tight grouping of the points, such that the subsequent Nystr\"{o}m approximations bare a good quality. A disadvantage is that the approach suffers from loss of clusters during iterations and it is less robust if only one set of initial guesses is used. Our experience indicates that a robust implementation of k-means sometimes costs much more than does the rest of the training computation.

The PCA~\citep{Pearson1901} approach recursively partitions the data according to the principal axis (the direction along which the data varies the most). This approach is based on a Gaussian assumption and often results in compact partitions if the data indeed conforms to this assumption. However, it ignores the skewness and other higher order moments, which are also crucial to the shape of data in high dimensions. In the standard case, the hyperplane that partitions the data passes through the mean, which may result in highly imbalanced partitions. An alternative is to move the hyperplane along the principal direction such that the two partitions are always balanced. Computationwise, this approach requires computing the dominant singular vectors of the shifted data matrix~\citep{divide.and.conquer.knn}, which can be achieved by using a power iteration~\citep{Golub1996} or the Lanczos algorithm~\citep{Saad2003}. However, even if we compute the singular vectors only approximately, we find that the cost is still too high (see e.g., Section~\ref{sec:pca}).

For computational efficiency, we recommend the random projection~\citep{Johnson1984} approach. This approach uses a random vector as the normal direction of the partitioning hyperplane and positions the hyperplane such that the numbers of points on the two sides are balanced. Computationwise, it amounts to projecting the points along the random direction and splitting them in two halves around the median. This approach is motivated by the study of dimension reduction, where it was observed that the projection quality on a random subspace is as good as that on the space spanned by the dominant singular vectors, in the sense that Euclidean distances are approximately preserved~\citep{Dasgupta2002}. Moreover, the heuristic use of a one-dimensional subspace for projection is robust and is computationally highly efficient. In such a case, the purpose is not to preserve the Euclidean distance, but rather, to serve as an efficient procedure for partitioning. In our experience, the eventual regression/classification performance of using this approach is almost identical to that of the PCA approach (see  Section~\ref{sec:pca}).

Under random projection, one can quickly determine which partition a new point $\bm{x}$ lies in, by comparing its projected coordinate with the median of those of the training points.

\subsection{Choice of Landmark Points}
The set of landmark points $\ud{X}_i$ associated to each domain $S_i$ simply consists of uniformly random samples of the point set $X_i$.

Some alternatives exist. Using the k-means centers of $X_i$ as the landmark points may improve the Nystr\"{o}m approximations, but computing them for every nonleaf node $i$ is often much more expensive than the rest of the training computation. On the other hand, using the uniformly random samples of the domain $S_i$ as the landmark points may sound theoretically preferable, but they are difficult to compute because the domains are polyhedra rather than regular shapes (e.g., boxes).

We mentioned in Section~\ref{sec:hierarchical} that the compositional kernel $k_{\text{compositional}}$ is a special case of $k_{\text{hierarchical}}$, if in the partitioning tree every child of the root is a leaf node. Interestingly, another interpretation exists. If we relax the requirement that the landmark points $\ud{X}_i$ must reside within the domain $S_i$, then for any partitioning tree, using the same set of landmark points $\ud{X}_i$ for every $i$ also results in the compositional kernel $k_{\text{compositional}}$. By following the proof of Theorem~\ref{thm:pd.2}, one sees that the relation $\ud{X}_i\subset S_i$ is in fact nonessential. Hence, the requirement $\ud{X}_i\subset S_i$ reflects only the modeling desire of a better kernel approximation in local domains; it is not a necessary condition for kernel validity.

\subsection{Numerical Stability}
A notorious numerical challenge for kernel matrices $K$ is that they are exceedingly ill-conditioned as the size increases. Hence, the regularization $\lambda I$ serves as an important rescue of matrix inversion. This challenge is less severe for the matrix of the proposed kernel, $K_{\text{hierarchical}}(X,X)$, but since the components of this matrix are defined with the term $K(\ud{X}_i,\ud{X}_i)^{-1}$ for nonleaf nodes $i$, one must ensure that each $K(\ud{X}_i,\ud{X}_i)$ is not too ill conditioned. Generally, when the set $\ud{X}_i$ is not large, the conditioning of $K(\ud{X}_i,\ud{X}_i)$ is acceptable; however, for a robust safeguard, we use a small $\lambda'<\lambda$ for help. Instead of treating $k(\bm{x},\bm{x}')$ as the base kernel and $\lambda$ as the amount of regularization, we treat $k'(\bm{x},\bm{x}')=k(\bm{x},\bm{x}')+\lambda'\delta_{\bm{x},\bm{x}'}$ as the base kernel and $\lambda-\lambda'$ the regularization, where $\delta$ is the Kronecker delta. Then, we use $k'$ to build the kernel $k'_{\text{hierarchical}}$ and add to the matrix $K'_{\text{hierarchical}}$ a regularization $(\lambda-\lambda')I$.

\subsection{Sizes}
Two parameters exist for the proposed kernel. First, the recursive partitioning needs a stopping criterion; we say that a set of points is not further partitioned when the cardinality is less than or equal to  $n_0$. Second, we need to specify the sizes of the conditioned sets $\ud{X}_i$. Whereas one may argue that the conditioned set needs to grow with the size of the point set for maintaining the approximation quality, a substantial increase in the size of $\ud{X}_i$ across levels will result in too high a computational cost, forfeiting the purpose of the kernel. Therefore, we mandate that each conditioned set have the same size $r$.

It will be convenient if we can consolidate the two parameters so that different approximate kernels are comparable. It will also be beneficial to have a guided choice of them. These requirements are indeed achievable. Under a balanced binary partitioning, $n_0$ can effectively take only these values: $\lceil n/2^j\rceil$, for $j=0,1,\ldots\lfloor\log_2n\rfloor$. On the other hand, it is not sensible to have the rank $r$ greater than $n_0$ because the smallest off-diagonal blocks have a size at most $n_0\times n_0$. Then, we take
\begin{equation}\label{eqn:r}
n_0=\lceil n/2^j\rceil\quad\text{and}\quad
r=\lfloor n/2^j\rfloor\quad\text{for some } j.
\end{equation}

\subsection{Cost Analysis}
We are now ready to perform a full analysis of the computational costs based on the practical choices made in the preceding subsections.

\paragraph{Memory Cost}\label{sec:cost.mem}
Let us first consider the memory cost of storing the matrix $A$. For simiplicity, assume that $n$ is a power of 2 (and thus $n_0=r$). Then, the partitioning tree has $n/n_0$ leaf nodes and $n/n_0-1$ nonleaf nodes. Hence, the memory costs for the factors $A_{ii}$, $U_i$, $\Sigma_p$, and $W_p$ are $n_0^2n/n_0$, $nr$, $r^2(n/n_0-1)$, and $r^2(n/n_0-1)$, respectively. Therefore, the total memory cost is approximately $4nr$.

\paragraph{Arithmetic Cost, Algorithm~\ref{algo:Ab}}
The arithmetic cost of matrix-vector multiplication (Algorithm~\ref{algo:Ab}) is $O(nr)$, because the subroutines \Call{Upward}{} and \Call{Downward}{} for each node performs $O(r^2)$ calculations, discounting the recursion calls. These recursions essentially constitute a tree traversal, which visits all the $2n/n_0-1$ nodes. Therefore, the total cost is $O(nr)$. In fact, a more detailed analysis may reveal the constant inside the big-O notation. Note that the matrix-vector multiplications $W_i^Tc_j$ and $W_id_i$ occur for each node $j$, $\Sigma_pc_i$ occurs for each node $i$, and $U_i^Tb_i$, $A_{ii}b_i$, and $U_id_i$ occur for each leaf node $i$. The matrix size of all these multiplications is $r\times r$. Therefore, if each multiplication requires $2r^2$ arithmetic operations, then the total cost is approximately $2r^2\times(2n/n_0+2n/n_0+2n/n_0+n/n_0+n/n_0+n/n_0)=18nr$, ignoring lower-order terms.

\paragraph{Arithmetic Cost, Algorithm~\ref{algo:invA}}
Based on a similar argument, the arithmetic cost of matrix inversion (Algorithm~\ref{algo:invA}) is $O(nr^2)$. A more precise estimate is $37nr^2$, if factorizing a general $r\times r$ matrix, factorizing a symmetric positive-definite matrix, performing a triangular solve, and multiplying two $r\times r$ matrices require $2r^3/3$, $r^3/3$, $r^2$, and $2r^3$ operations, respectively. As before, we ignore the lower-order terms.

\paragraph{Arithmetic Cost, Algorithm~\ref{algo:out.of.sample}, Precomputation Phase}
Now consider the precomputation phase of the inner product $w^Tk_{\text{hierarchical}}(X,\bm{x})$ (Algorithm~\ref{algo:out.of.sample}). In computer implementation, we move the first line (prefactorizing $K(\ud{X}_p,\ud{X}_p)$) to the matrix construction part (to be discussed in Section~\ref{sec:cost.construction}). Hence, the precomputation of Algorithm~\ref{algo:out.of.sample} is dominated by the third line \Call{Common-Upward}{\texttt{root}}. The arithmetic cost is $O(nr)$, or $10nr$ in a more precise account.

\paragraph{Arithmetic Cost, Algorithm~\ref{algo:out.of.sample}, $\bm{x}$-Dependent Phase}
Assume that evaluating the kernel function $k(\bm{x},\bm{x}')$ requires $O(\nnz(\bm{x})+\nnz(\bm{x}'))$ arithmetic operations, where $\nnz(\cdot)$ denotes the number of nonzeros of a data point. Further assume that for two point sets $Y$ and $Z$ of sizes $n_Y$ and $n_Z$ respectively, evaluating the vector $k(Y,\bm{x})$ requires $O(n_Y\nnz(\bm{x})+\nnz(Y))$ operations, and evaluating the matrix $K(Y,Z)$ requires $O(n_Y\nnz(Z)+n_Z\nnz(Y))$ operations, where the notation $\nnz(\cdot)$ is extended for point sets in a straightforward manner.

Now we consider the cost of the $\bm{x}$-dependent phase of Algorithm~\ref{algo:out.of.sample}. Line~\ref{algo.ln:pp} costs $O(r\nnz(\bm{x})+\nnz(\ud{X}_p))+O(r^2)$, where the first term comes from evaluating $k(\ud{X}_p,\bm{x})$ and the second term a linear-solve. Note that the matrix $K(\ud{X}_p,\ud{X}_p)$ has already been evaluated and prefactorized prior to this phase. In fact, such a computation is done when constructing $K_{\text{hierarchical}}$. Similarly, line~\ref{algo.ln:qq} costs $O(n_0\nnz(\bm{x})+\nnz(X_i))+O(n_0)$. Line~\ref{algo.ln:fall} costs $O(\nnz(\bm{x}))$, because it amounts to projecting $\bm{x}$ on a direction. Then, the overall cost of the $\bm{x}$-dependent phase of Algorithm~\ref{algo:out.of.sample} is
\begin{equation}\label{eqn:testing.cost}
O\Big(r\nnz(\bm{x})+\nnz(X_i)+\nnz(\ud{X}_p)+(r^2+\nnz(\bm{x}))\log_2(n/r)\Big),
\end{equation}
where $\bm{x}$ lies in the leaf node $i$ and $p$ is its parent.

\paragraph{Arithmetic Cost, Matrix Construction}\label{sec:cost.construction}
We shall also consider the cost of hierarchical partitioning and the instantiation of the matrix $A=K_{\text{hierarchical}}$. In the partitioning of a set of $n$ points $X$, generating a random projection direction costs $O(d)$, computing the projected coordinates costs $O(\nnz(X))$, finding the median costs $O(n)$, and permuting the points costs $O(\nnz(X))$. Then, counting recursion, the total cost is $O((d+n+\nnz(X))\log_2(n/r))$. On the other hand, in the instantiation of the matrix $A$, computing $A_{ii}$ for all leaf nodes $i$ costs $O(r\nnz(X))$, computing $U_i$ for all leaf nodes $i$ costs $O(r[n+\nnz(X)+\sum_{i \text{ is leaf}}\nnz(\ud{X}_p)])$, computing $\Sigma_p$ for all nonleaf nodes $p$ costs $O(r[\sum_{p \text{ is nonleaf}}\nnz(\ud{X}_p)])$, factorizing these $\Sigma_p$'s costs $O(nr^2)$, computing $W_i$ for all nonleaf and nonroot nodes $i$ costs $O(r[n+\sum_{i \text{ is not leaf and not root}}\nnz(\ud{X}_i)+\nnz(\ud{X}_p)])$, and finding all landmark point sets costs $O(n)$. Therefore, the total cost of instantiating $A$ is $O(r[nr+\nnz(X)+\sum_{i \text{ is not leaf}}\nnz(\ud{X}_i)])$.

\paragraph{Arithmetic Cost Summary, Training and Testing}
Summarizing the above analysis, we see that the overall cost of training (including partitioning, instantiating $A$, matrix inversion, matrix-vector multiplication, and preprocessing of Algorithm~\ref{algo:out.of.sample}) is
\begin{equation}\label{eqn:training.cost}
O\left((d+n+\nnz(X))\log_2(n/r)+nr^2+
r\left[\nnz(X)+\sum_{i \text{ is not leaf}}\nnz(\ud{X}_i)\right]\right).
\end{equation}
The overall cost of testing per $\bm{x}$ has been given in~\eqref{eqn:testing.cost}.

We remark that the term $\log_2(n/r)$ in both~\eqref{eqn:testing.cost} and~\eqref{eqn:training.cost} is dominated by $r$ when $n\le r2^r$. Moreover, when the data $X$ is dense, the training cost~\eqref{eqn:training.cost} simplifies to $O(nr^2+ndr)$ and the testing cost per point~\eqref{eqn:testing.cost} simplifies to $O(r^3+dr)$. Suppressing the data dimension $d$, these costs are further simplified to $O(nr^2)$ and $O(r^3)$, respectively.

\section{Experimental Results}\label{sec:exp}
In this section, we present comprehensive experiments to demonstrate the effectiveness of the proposed kernel. The experiments were performed on one node of a computing cluster with 512 GB memory and an IBM POWER8 12-core processor, unless otherwise stated. The programs were implemented in C++. The linear algebra routines (BLAS and LAPACK) were linked against the IBM ESSL library and some obviously parallelizable for-loops were decorated by OpenMP pragmas, so that the experiments with large data sets can be completed in a reasonable time frame. However, the programs were neither fully optimized nor fully parallelized, leaving room for improvement on both memory consumption and execution time.

\begin{table}[ht]
\centering
\caption{Data sets.}
\label{tab:dataset}
\begin{tabular}{ccrrrcccrrr}
\\[-5pt]
\hline
Name & Type & $d$ & $n$ Train & $n'$ Test\\
\hline
cadata            & regression            &   8 &    16,512 &     4,128\\
YearPredictionMSD & regression            &  90 &   463,518 &    51,630\\
\hline
ijcnn1            & binary classification &  22 &    35,000 &    91,701\\
covtype.binary    & binary classification &  54 &   464,809 &   116,203\\
SUSY              & binary classification &  18 & 4,000,000 & 1,000,000\\
\hline
mnist             & 10 classes            & 780 &    60,000 &    10,000\\
acoustic          & 3 classes             &  50 &    78,823 &    19,705\\
covtype           & 7 classes             &  54 &   464,809 &   116,203\\
\hline
\end{tabular}
\end{table}

Table~\ref{tab:dataset} summarizes the data sets used for experiments. They were all downloaded from \url{http://www.csie.ntu.edu.tw/~cjlin/libsvm/} and are widely used as benchmarks of kernel methods. We selected these data sets primarily because of their varying sizes. Some of the data sets come with a training and a testing part; for those not, we performed a 4:1 split to create the two parts, respectively. In some of the data sets, the attributes had already been normalized to within $[0,1]$ or $[-1,1]$; for those not, we performed such a normalization. We also preprocessed the data sets by removing duplicate and conflicting records (whose labels are inconsistent) in the training sets. Such records are infrequent.

\subsection{Effect of Randomness}
Throughout Section~\ref{sec:exp}, we will compare various approximate kernels: Nystr\"{o}m approximation $k_{\text{Nystr\"{o}m}}$, random Fourier features $k_{\text{Fourier}}$, cross-domain independent kernel $k_{\text{independent}}$, and the proposed kernel $k_{\text{hierarchical}}$. The partitioning in the cross-domain independent kernel is the same as that in the proposed kernel, except that the hierarchy is flattened. Because of the random nature of all these kernels (e.g., landmark points, sampling, and partitioning), we first study how the performance is affected by randomization. Note that the quantity $r$ is comparable across kernels, even though its specific meaning is different.

The data set for demonstration is cadata. We use the Gaussian kernel~\eqref{eqn:gauss} as an example. As hinted earlier, the choice of the range parameter $\sigma$ affects the quality of various kernels. Therefore, the experiment setup is to use a reasonable regularization $\lambda=0.01$ and to vary the choice of $\sigma$ in a large interval (between $0.01$ and $100$) such that the optimal $\sigma$ falls within the interval. We set the rank $r$ (and the leaf size $n_0$) according to~\eqref{eqn:r} with three particular choices: $r=32$, $129$, and $516$.

\begin{figure}[!ht]
\centering
\subfigure[Nystr\"{o}m approximation]{
\includegraphics[width=.31\linewidth]{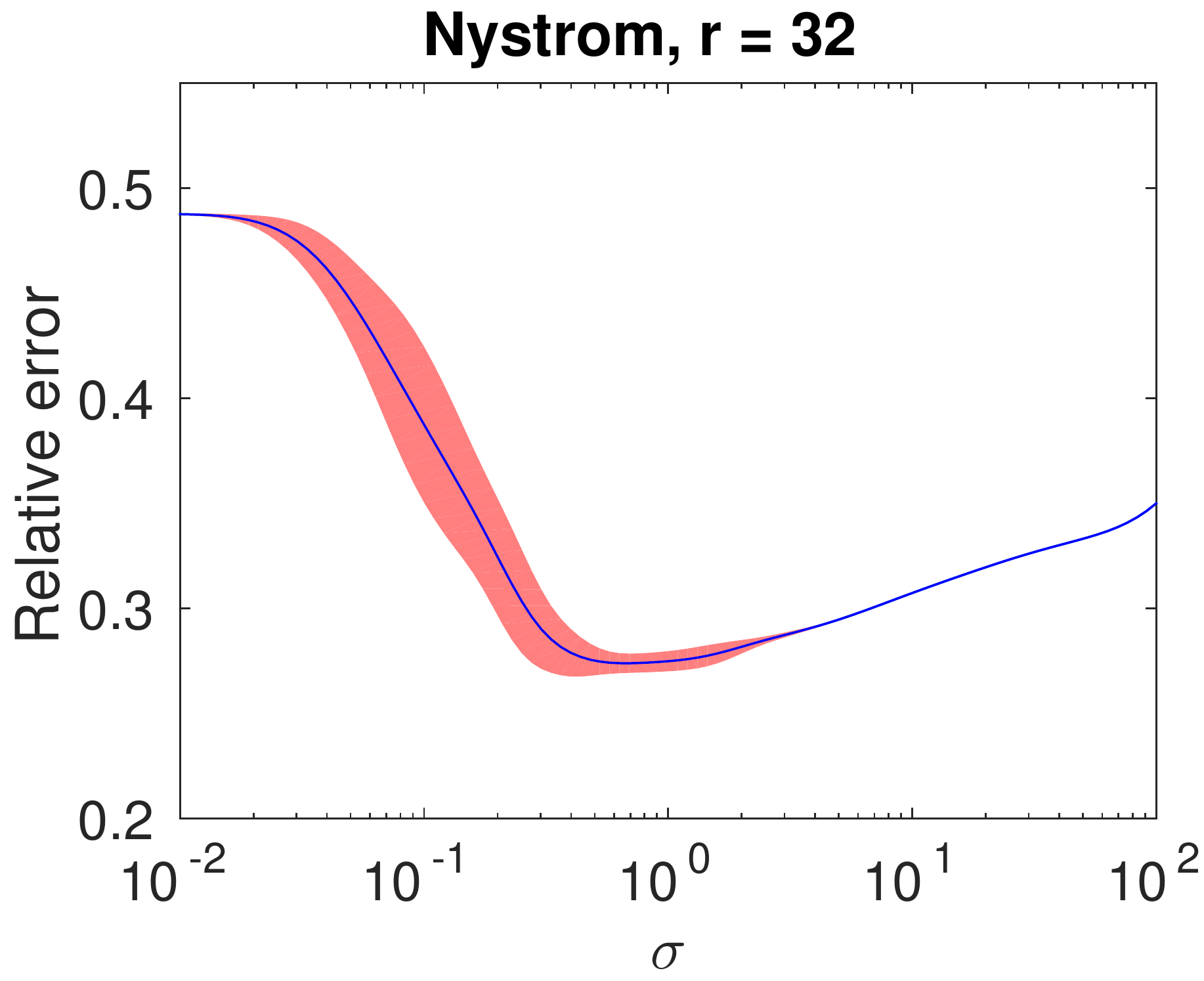}
\includegraphics[width=.31\linewidth]{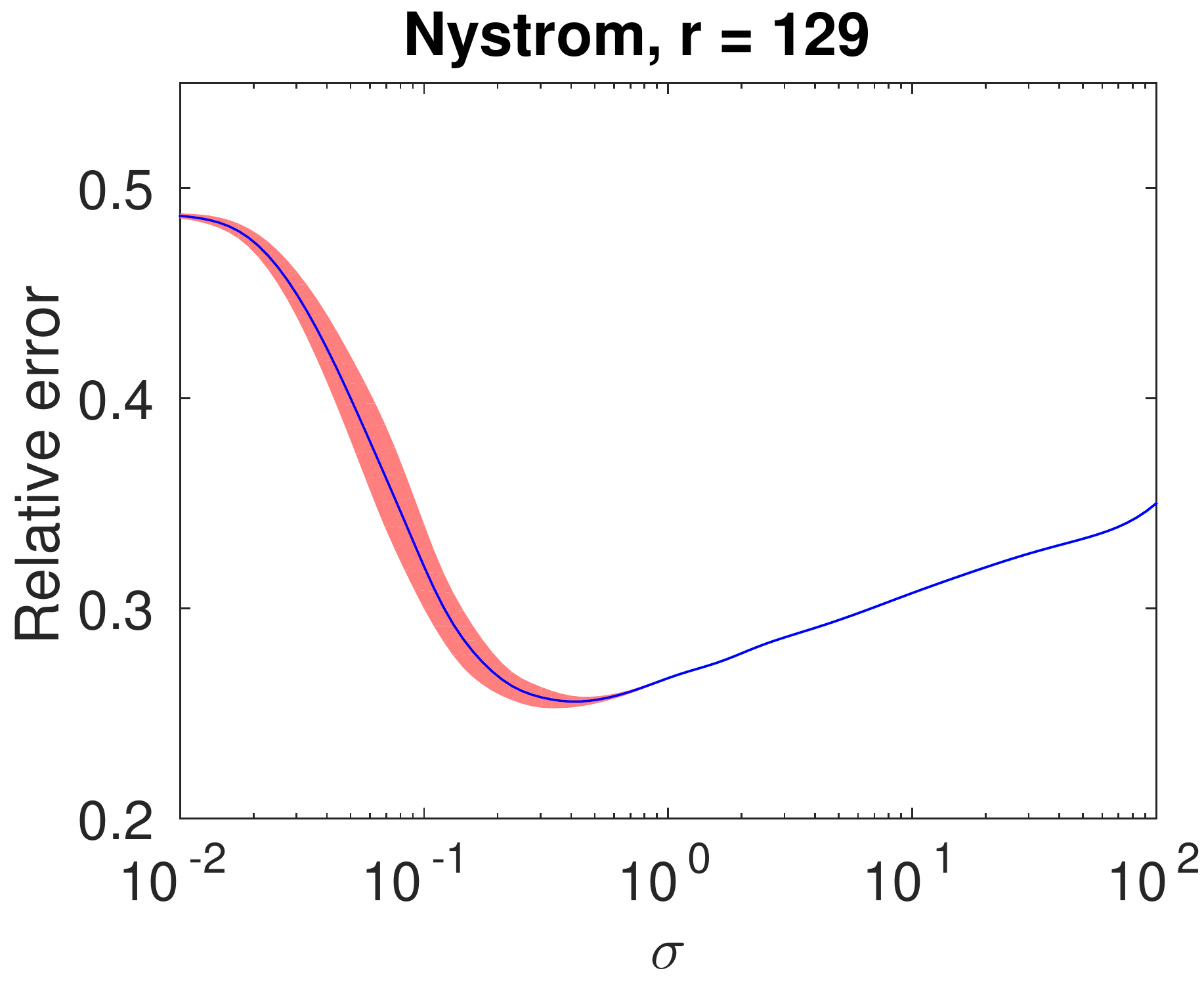}
\includegraphics[width=.31\linewidth]{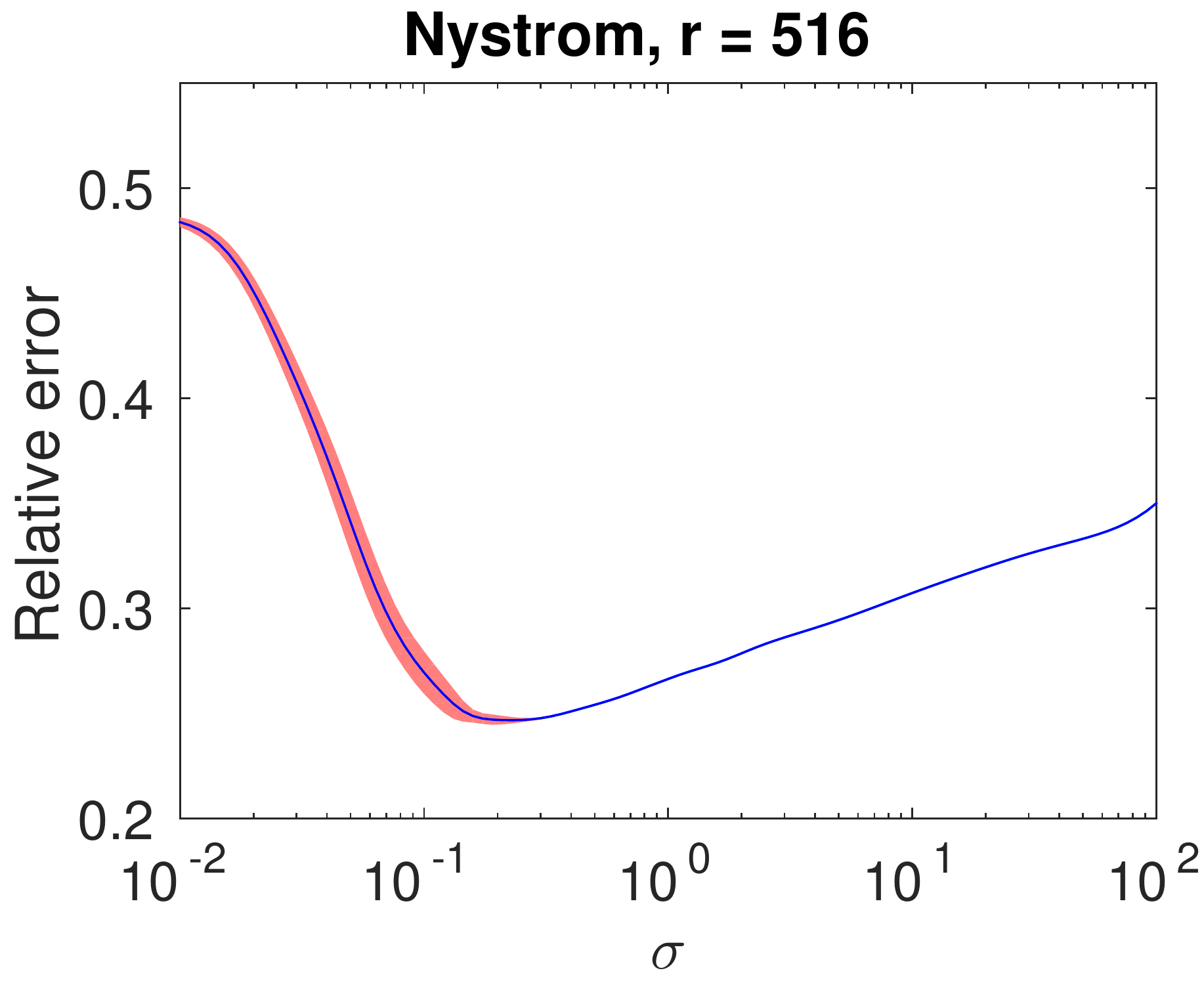}}
\subfigure[random Fourier features]{
\includegraphics[width=.31\linewidth]{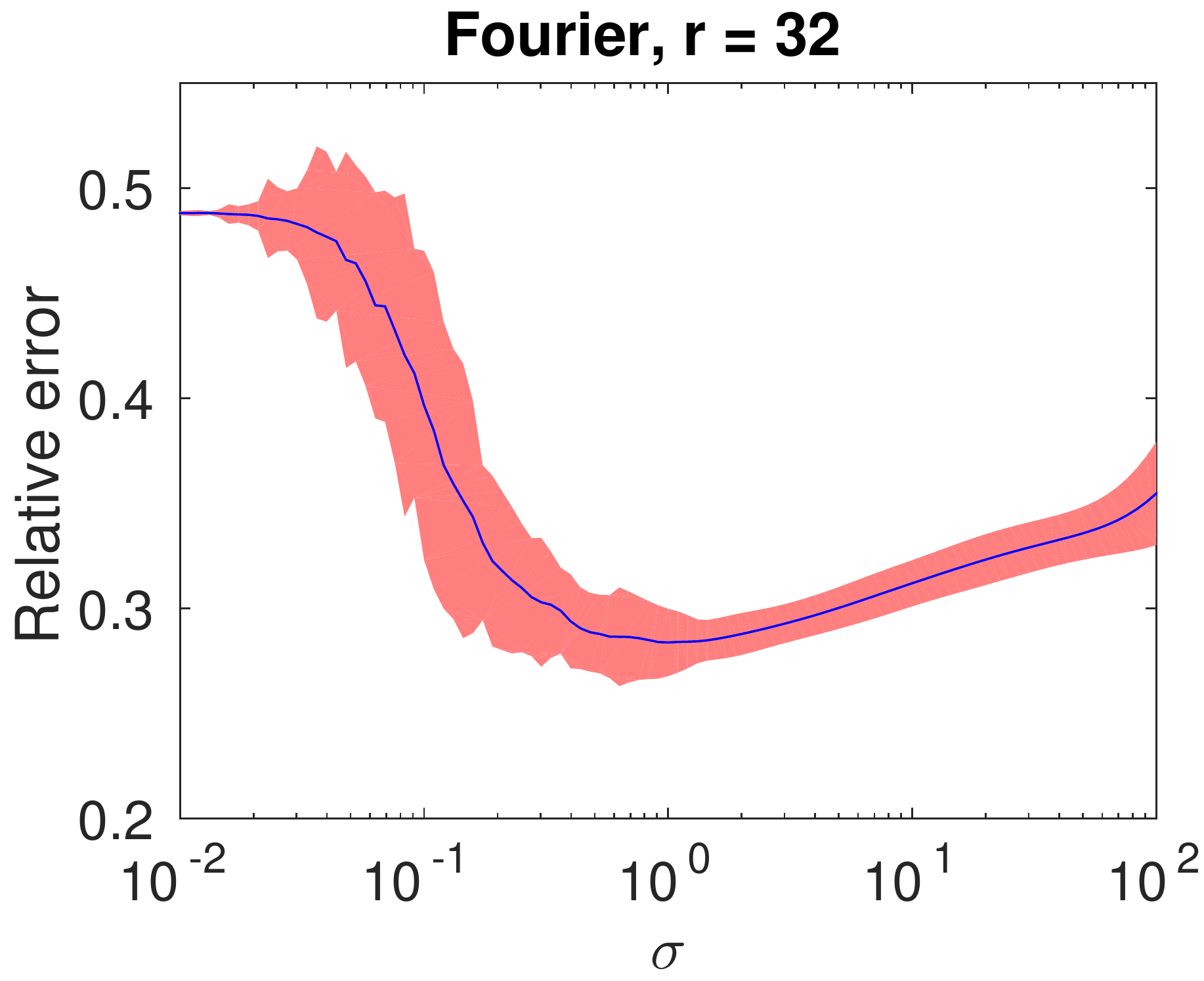}
\includegraphics[width=.31\linewidth]{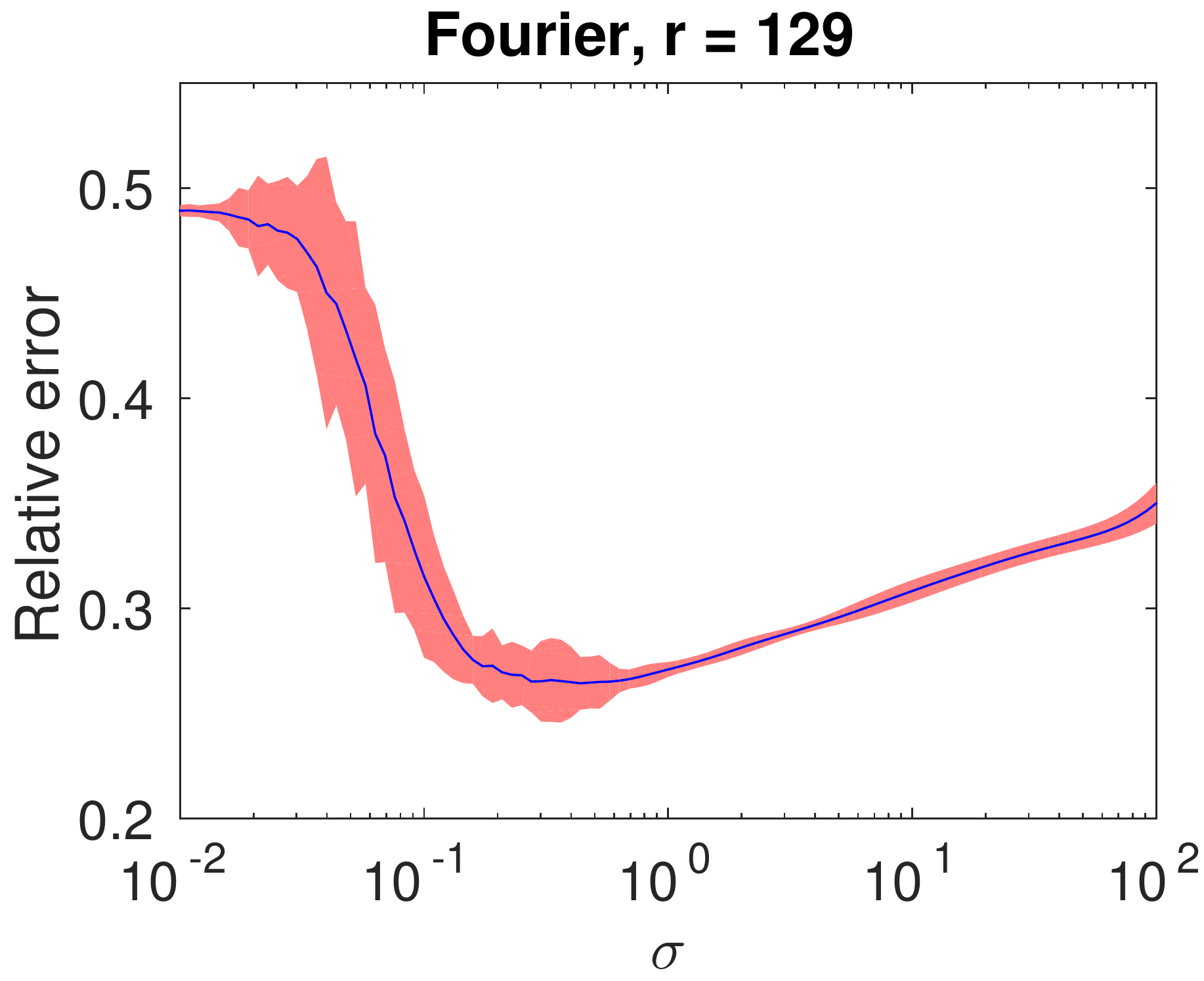}
\includegraphics[width=.31\linewidth]{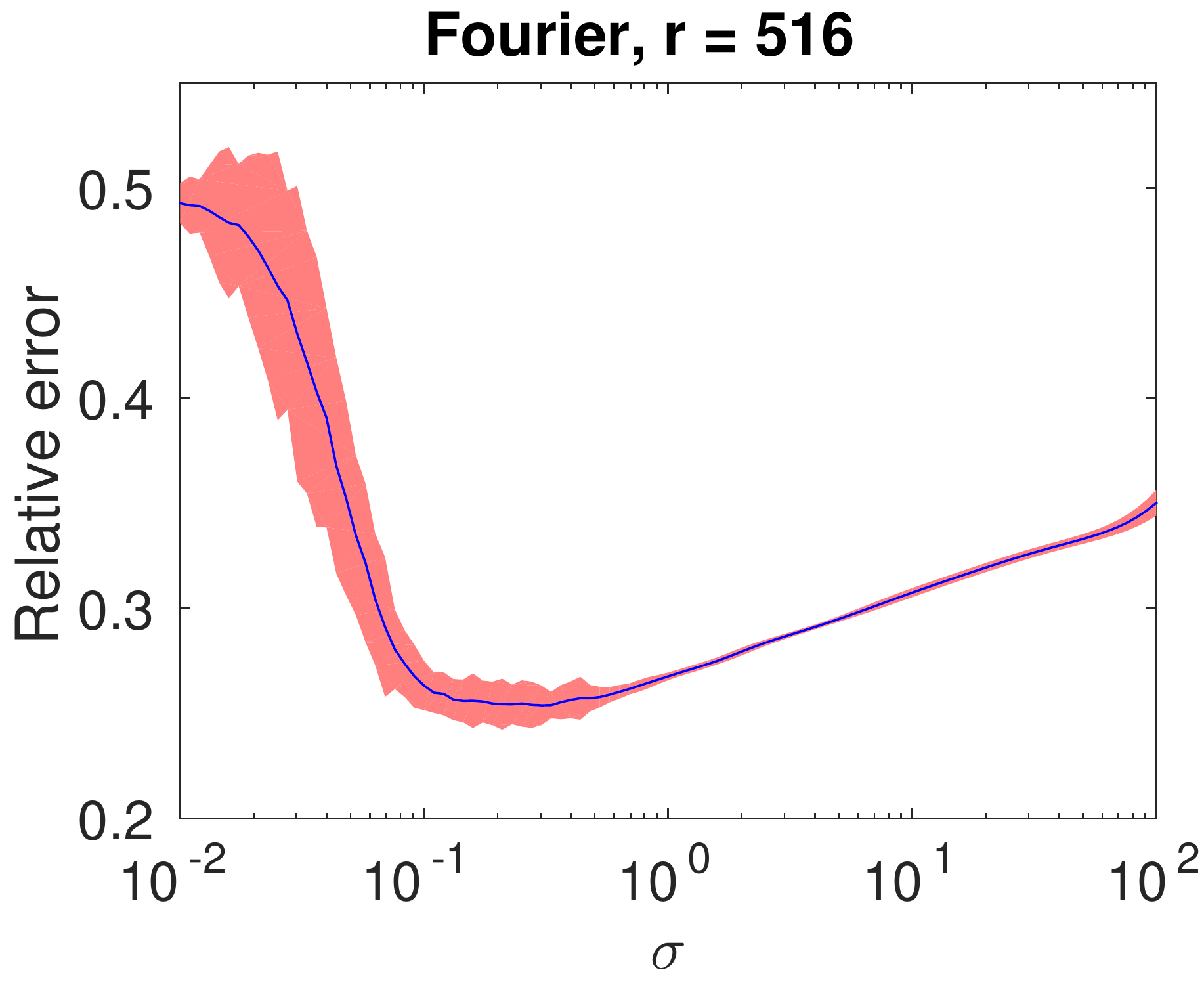}}
\subfigure[Cross-domain independent kernel]{
\includegraphics[width=.31\linewidth]{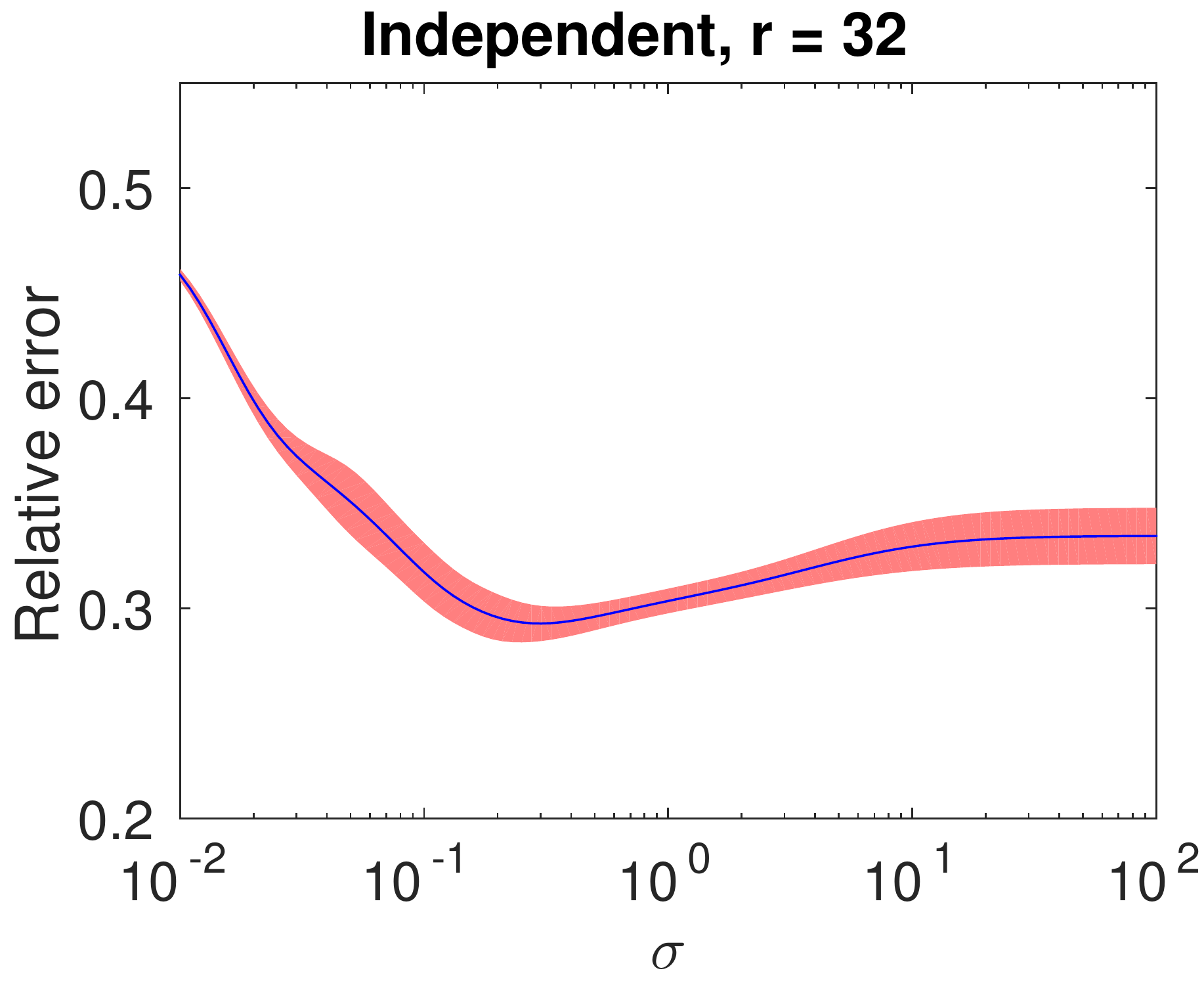}
\includegraphics[width=.31\linewidth]{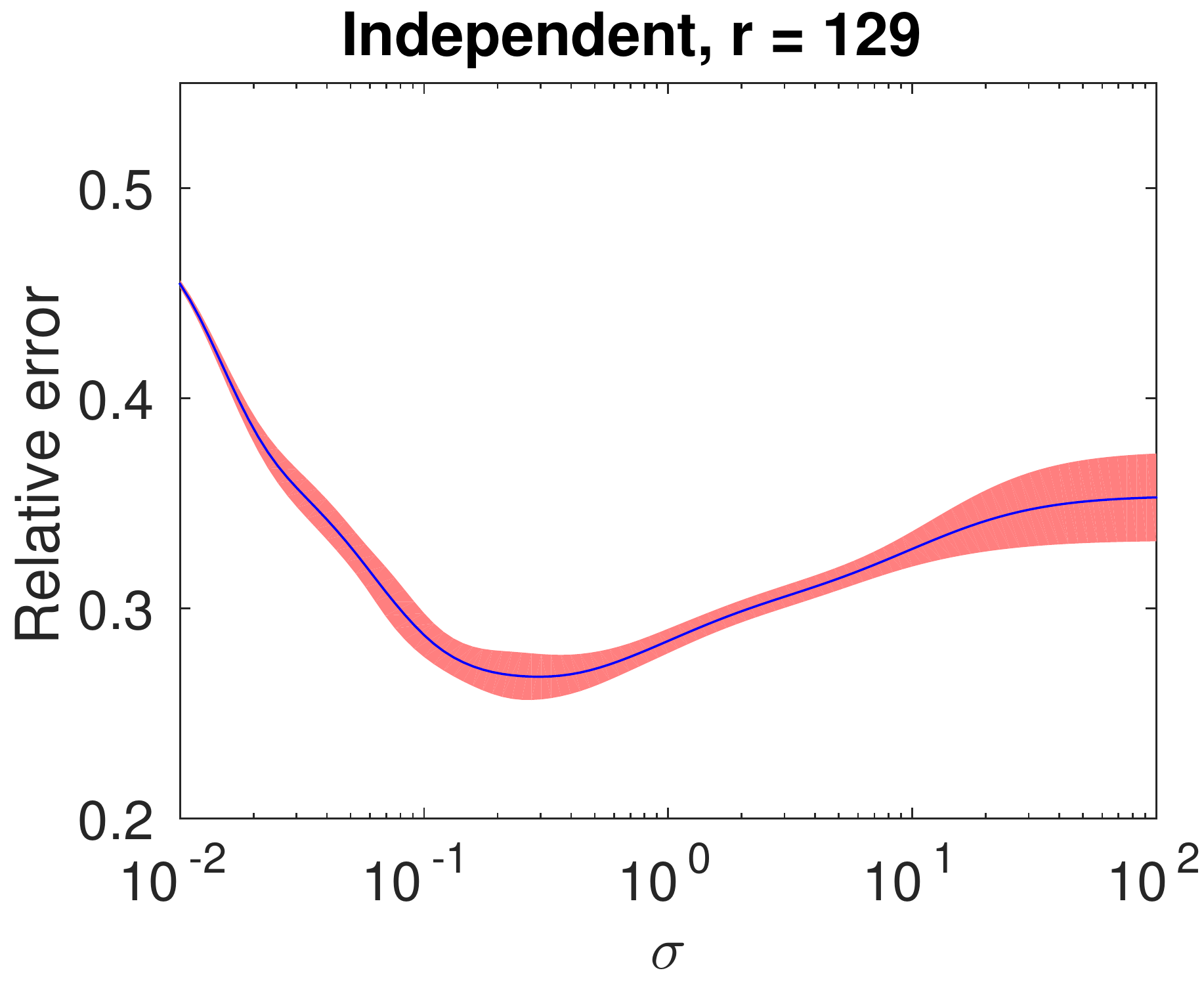}
\includegraphics[width=.31\linewidth]{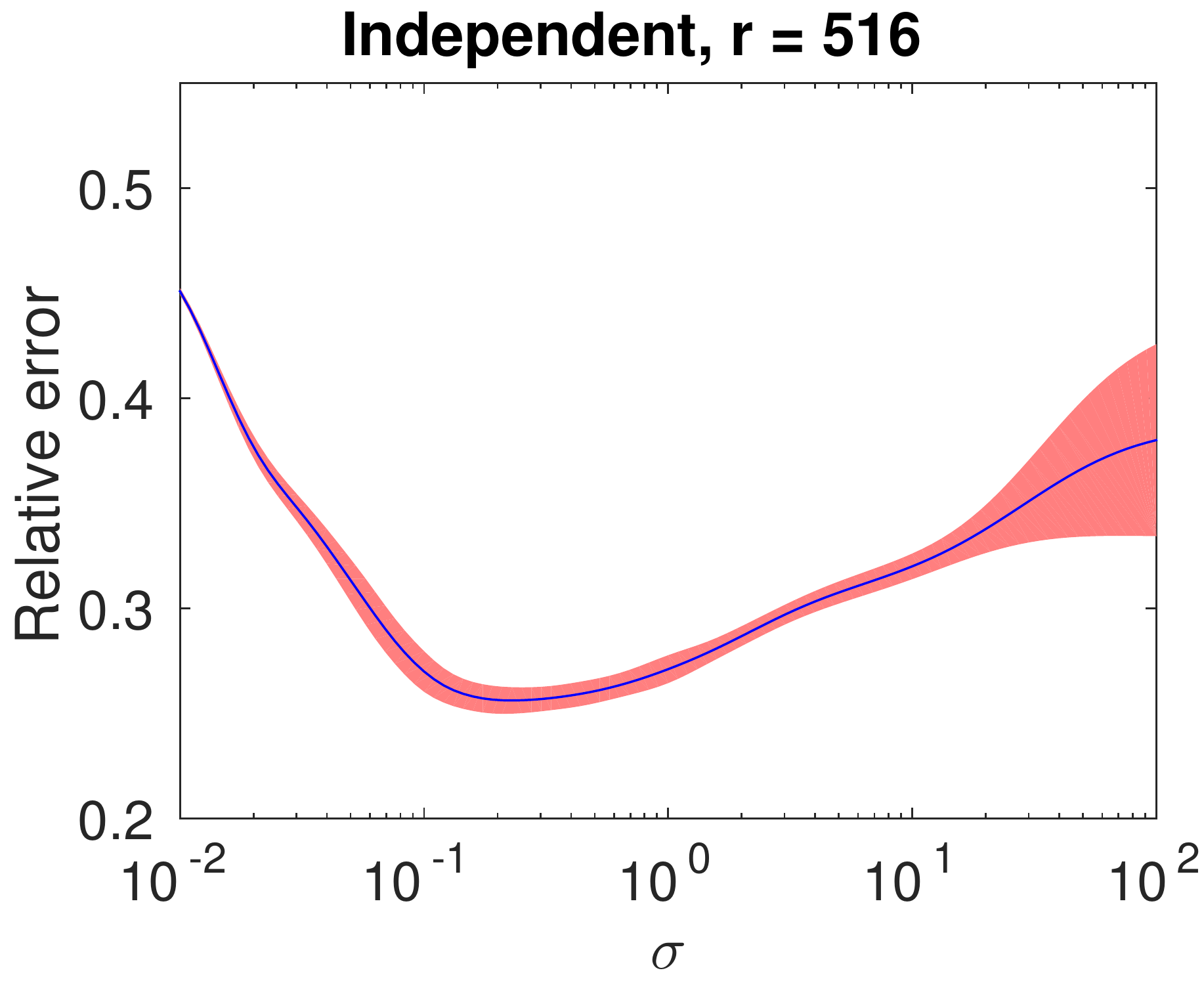}}
\subfigure[Hierarchically compositional kernel]{
\includegraphics[width=.31\linewidth]{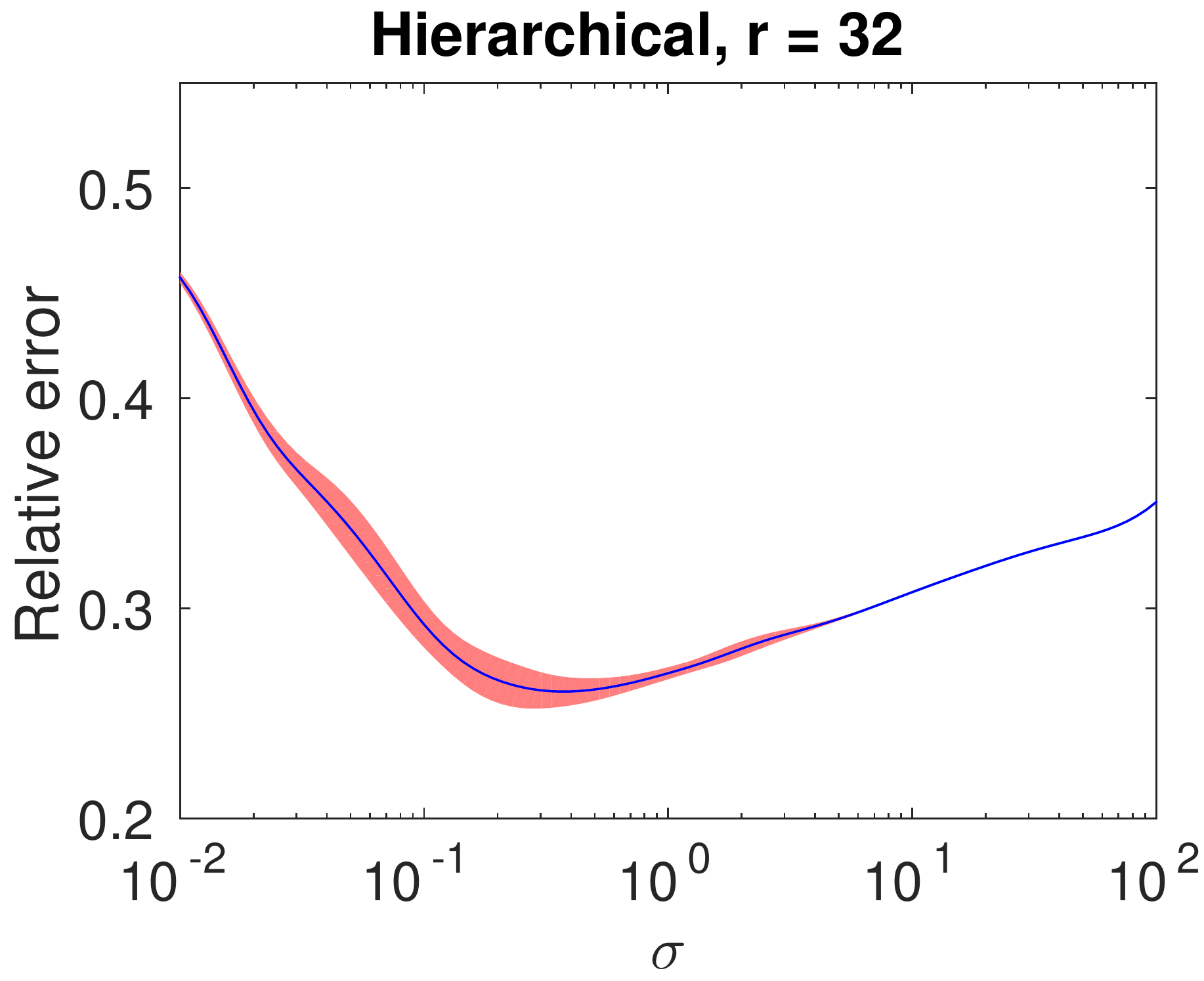}
\includegraphics[width=.31\linewidth]{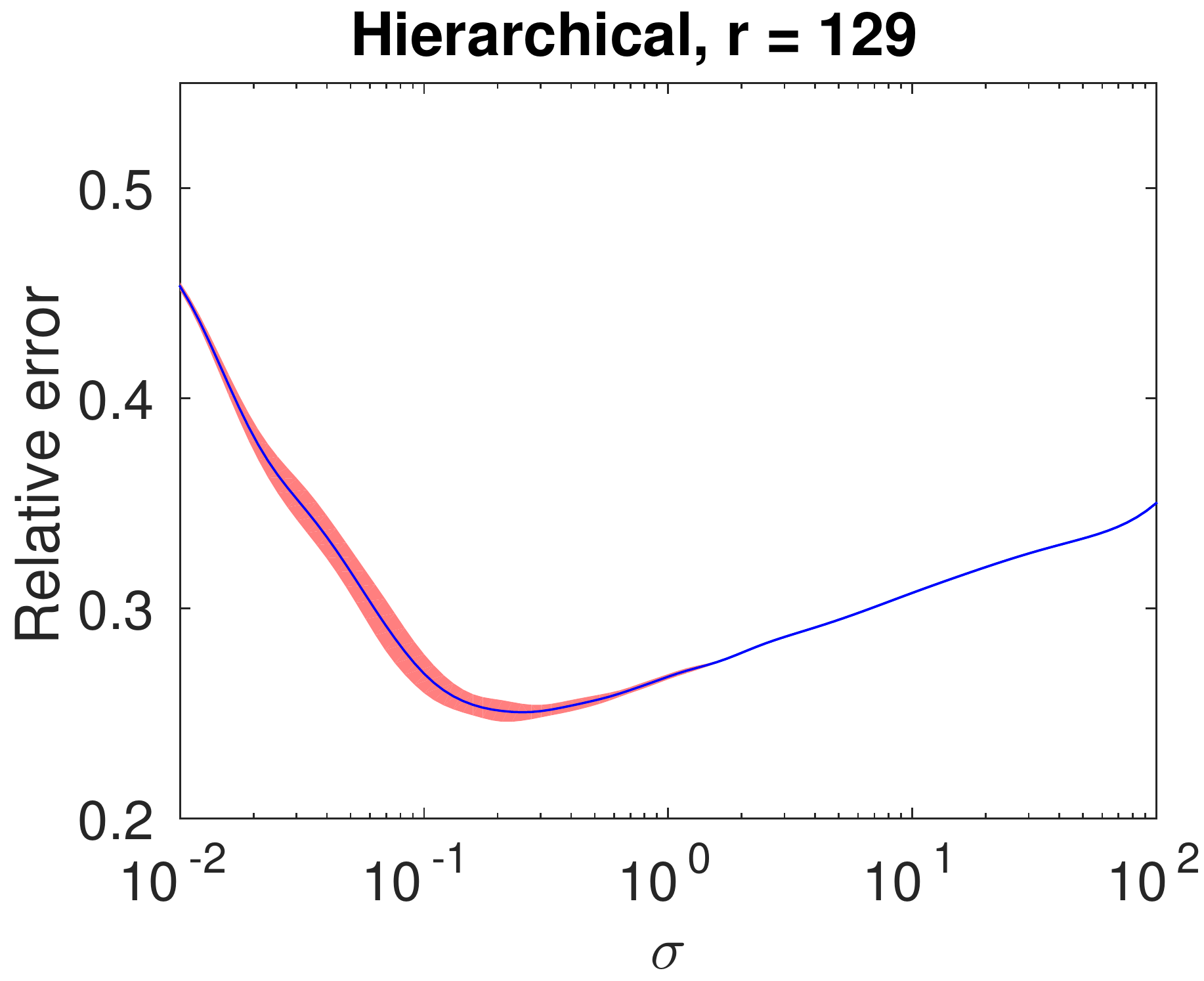}
\includegraphics[width=.31\linewidth]{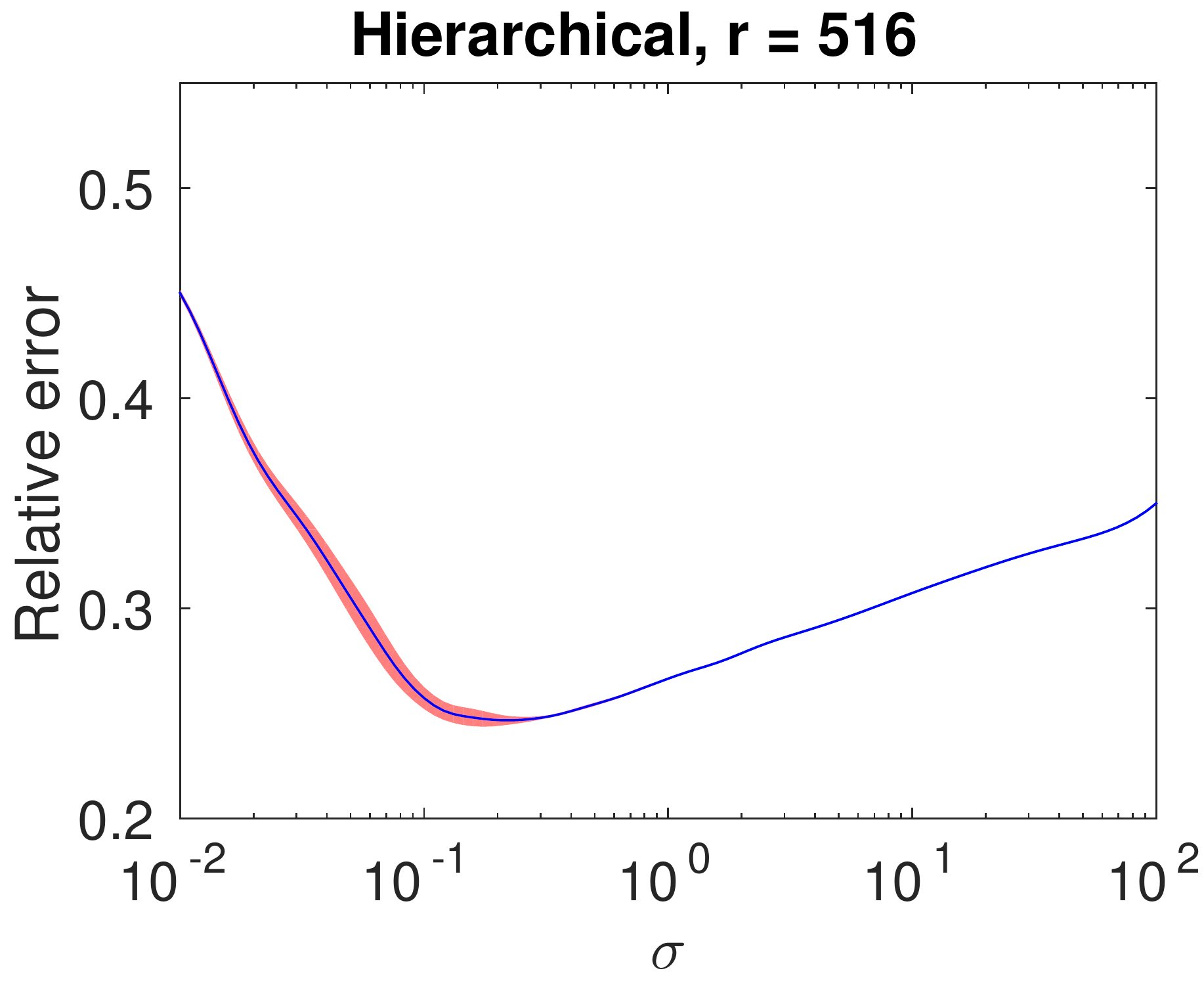}}
\caption{Regression error curves versus $\sigma$. Results are summarized from 30 repeated runs and the red bands show three times standard deviation. Data set: cadata.}
\label{fig:ZZ_plot_exp_1_randomness}
\end{figure}

For each $r$, we repeat 30 times with different random seeds; but the seed always stays the same every time when the range of $\sigma$ is swept. The results (relative testing error versus $\sigma$) are summarized in Figure~\ref{fig:ZZ_plot_exp_1_randomness}, with mean (blue curve) and standard deviation (red band) plotted. One sees that the red bands of random Fourier features are not smooth; this is because each single error curve from a fixed seed is nonsmooth. Moreover, the error curves for Nystr\"{o}m approximation have a nonnegligible variation when $\sigma$ is small, whereas those for the independent kernel vary significantly when $\sigma$ is large. The error curves of the proposed kernel are the most stable, with the narrowest standard deviation band. Generally speaking, as $r$ increases, all approximate kernels yield a more and more stable error curve, except the peculiar case of the independent kernel at large $\sigma$.

The unstable performance caused by randomness is unfavorable for parameter estimation, because the valley, where the optimal parameter stands, may move substantially. The nonsmoothness exhibited in the Fourier approach also renders difficulty. Although the unfavorable behaviors are substantially alleviated when $r$ increases to $516$ in this particular data set, to our experience, the relieving size for $r$ correlates with the data size; that is, the larger $n$, the higher $r$ needs to increase to. In this regard, the proposed kernel is the most favorable because its performance is relatively stable even for small $r$.

\subsection{Partitioning Approaches}\label{sec:pca}
We further compare the methods for partitioning needed by the proposed kernel. The comparison focuses on three aspects: effect of randomness, testing error, and computational efficiency. Among the several possible choices discussed in Section~\ref{sec:part}, only the PCA approach and the random projection approach (which is recommended) yield a balanced partitioning; hence, we compare only these two.

\begin{figure}[!ht]
\centering
\subfigure[Random partitioning]{
\includegraphics[width=.32\linewidth]{ZZ_plot_exp_1_randomness_RLCM_32}
\includegraphics[width=.32\linewidth]{ZZ_plot_exp_1_randomness_RLCM_129}
\includegraphics[width=.32\linewidth]{ZZ_plot_exp_1_randomness_RLCM_516}}
\subfigure[PCA partitioning]{
\includegraphics[width=.32\linewidth]{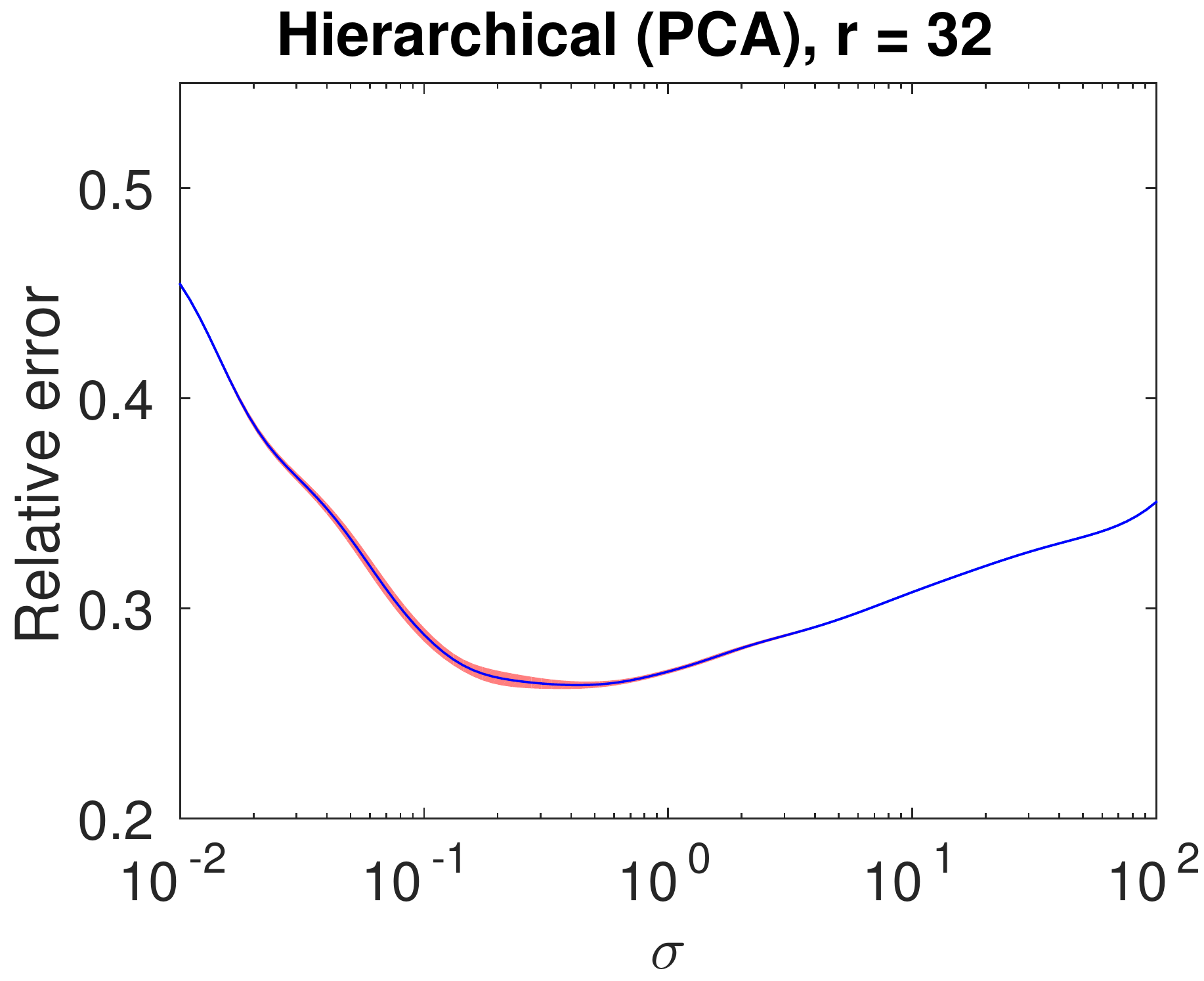}
\includegraphics[width=.32\linewidth]{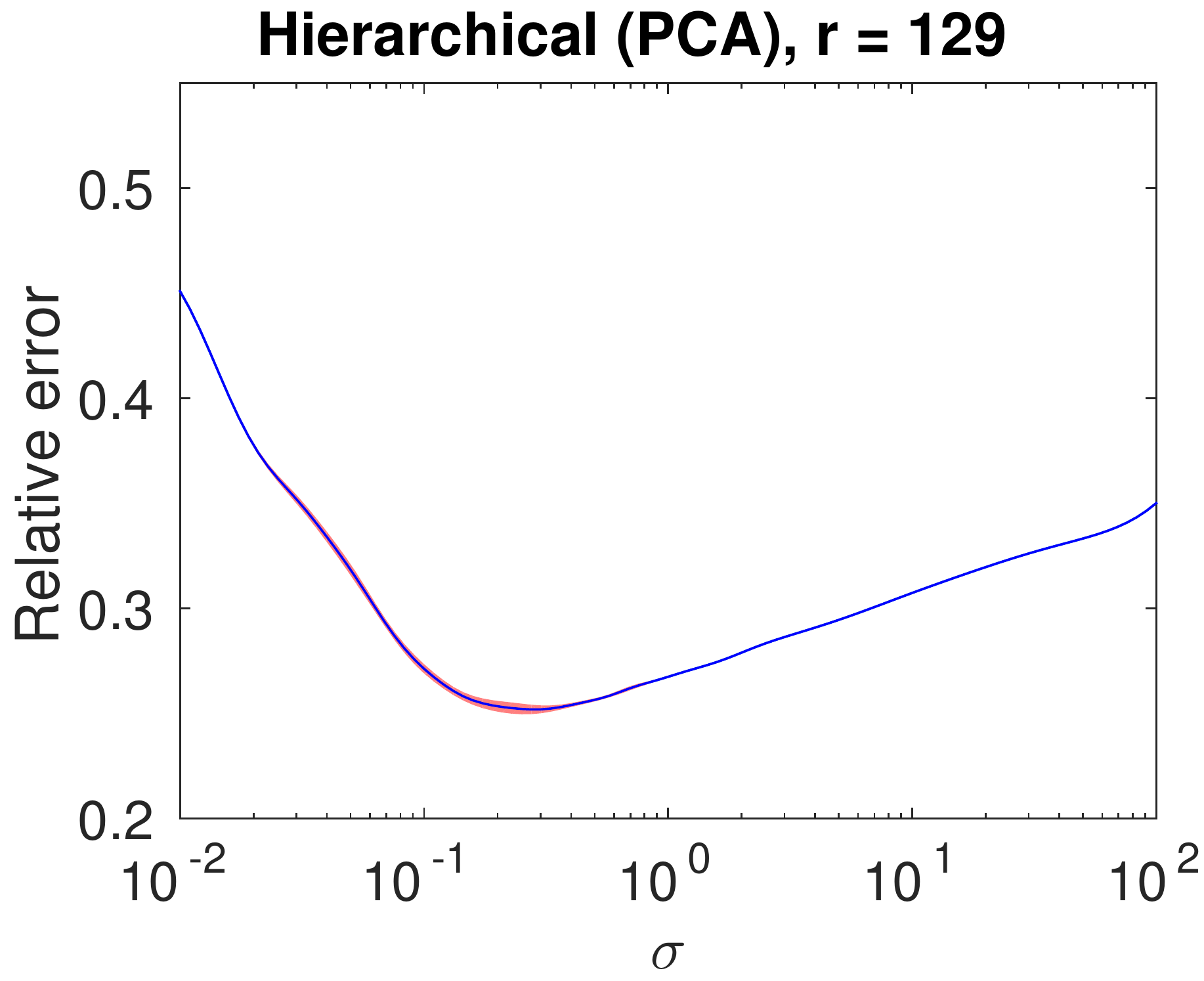}
\includegraphics[width=.32\linewidth]{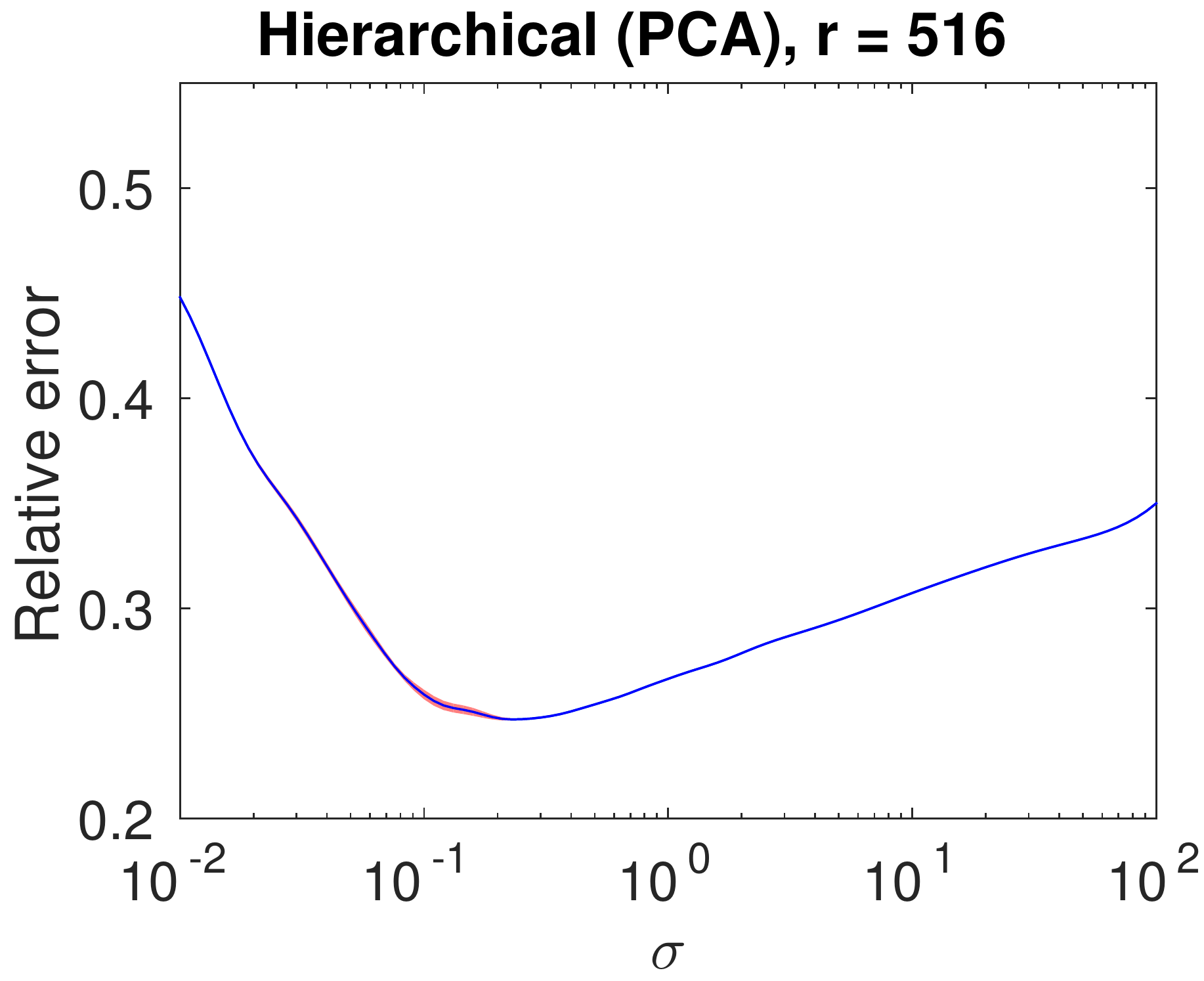}}
\caption{Regression error curves of the proposed method. Data are recursively partitioned by using random projections (top row) and by PCA (bottom row), respectively.}
\label{fig:ZZ_plot_exp_2_partitioning}
\end{figure}

Figure~\ref{fig:ZZ_plot_exp_2_partitioning} summarizes the error curves (mean and standard deviation) resulting from the two approaches, based on the same setting as in the preceding subsection. The top row corresponds to the recommended approach whereas the bottom row PCA. One sees that the mean error curves are almost identical in both approaches. PCA is less influenced by randomization. Such a result should not be surprising, because PCA bares no randomness at all on partitioning; the only randomness comes from the landmark points. Although the variation of the error curves of the random projection approach is higher, such a variation is acceptable, particularly when it is compared with that of other approximate kernels in the preceding subsection.

The essential reason why random projection is favored over PCA comes from the consideration of computational efficiency. To generate the normal direction of a partitioning hyperplane, random projection amounts to only generating $d$ random numbers, whereas PCA requires computing the dominant singular vector of the shifted data matrix. Table~\ref{tab:ZZ_plot_exp_6_partitioning_time} presents the time costs of the additional singular vector computation, against the partitioning cost and the overall training cost by using random projection. We call the singular vector computation ``overhead.'' The overhead is much higher with respect to the partitioning step, because such a step has a negligible cost compared to the overall training. The overhead is also generally higher when $r$ is smaller because there requires more partitioning. The overhead with respect to partitioning easily exceeds $100\%$ for quite a few of the data sets, and sometimes it is even a few thousand percents. For the data set mnist, which has the largest dimension $d$, the overhead with respect to overall training ranges from approximately $50\%$ to $800\%$.

\begin{table}[!ht]
\centering
\caption{PCA overhead with respect to partitioning and to overall training, under different $r$.}
\label{tab:ZZ_plot_exp_6_partitioning_time}
\begin{tabular}{rrrrrrrrrrr}
\\[-5pt]
\multicolumn{3}{c}{cadata} &&
\multicolumn{3}{c}{YearPredictionMSD} &&
\multicolumn{3}{c}{ijcnn1}\\
\cline{1-3}\cline{5-7}\cline{9-11}
\multicolumn{1}{c}{$r$} &
\multicolumn{1}{c}{partition.} &
\multicolumn{1}{c}{train.} &&
\multicolumn{1}{c}{$r$} &
\multicolumn{1}{c}{partition.} &
\multicolumn{1}{c}{train.} &&
\multicolumn{1}{c}{$r$} &
\multicolumn{1}{c}{partition.} &
\multicolumn{1}{c}{train.}\\
\cline{1-3}\cline{5-7}\cline{9-11}
32 & 91.16\% & 9.48\% && 56 & 687.69\% & 71.74\% && 34 & 199.52\% & 15.95\%\\
64 & 139.23\% & 7.59\% && 113 & 616.68\% & 50.97\% && 68 & 8.80\% & 0.51\%\\
129 & 51.33\% & 2.29\% && 226 & 630.07\% & 21.07\% && 136 & 153.16\% & 4.23\%\\
258 & 37.93\% & 0.94\% && 452 & 226.84\% & 5.64\% && 273 & 58.42\% & 1.37\%\\
516 & 52.86\% & 0.78\% && 905 & 216.37\% & 2.66\% && 546 & 64.74\% & 0.58\%\\
\cline{1-3}\cline{5-7}\cline{9-11}
\\[-5pt]
\multicolumn{3}{c}{covtype.binary} &&
\multicolumn{3}{c}{SUSY} &&
\multicolumn{3}{c}{mnist}\\
\cline{1-3}\cline{5-7}\cline{9-11}
\multicolumn{1}{c}{$r$} &
\multicolumn{1}{c}{partition.} &
\multicolumn{1}{c}{train.} &&
\multicolumn{1}{c}{$r$} &
\multicolumn{1}{c}{partition.} &
\multicolumn{1}{c}{train.} &&
\multicolumn{1}{c}{$r$} &
\multicolumn{1}{c}{partition.} &
\multicolumn{1}{c}{train.}\\
\cline{1-3}\cline{5-7}\cline{9-11}
56 & 86.40\% & 7.03\% && 61 & 85.40\% & 4.08\% && 58 & 3973.02\% & 805.67\%\\
113 & 74.75\% & 4.18\% && 122 & 99.44\% & 3.21\% && 117 & 3775.47\% & 508.89\%\\
226 & 56.26\% & 1.88\% && 244 & 52.10\% & 1.61\% && 234 & 4341.62\% & 383.34\%\\
453 & 28.75\% & 0.59\% && 488 & 27.41\% & 0.40\% && 468 & 3126.89\% & 151.64\%\\
907 & 93.58\% & 0.66\% && 976 & 79.71\% & 0.49\% && 937 & 2175.98\% & 51.17\%\\
\cline{1-3}\cline{5-7}\cline{9-11}
\\[-5pt]
\multicolumn{3}{c}{acoustic} &&
\multicolumn{3}{c}{covtype}\\
\cline{1-3}\cline{5-7}
\multicolumn{1}{c}{$r$} &
\multicolumn{1}{c}{partition.} &
\multicolumn{1}{c}{train.} &&
\multicolumn{1}{c}{$r$} &
\multicolumn{1}{c}{partition.} &
\multicolumn{1}{c}{train.}\\
\cline{1-3}\cline{5-7}
38 & 664.45\% & 53.07\% && 56 & 96.86\% & 7.02\%\\
76 & 411.56\% & 29.15\% && 113 & 155.34\% & 6.29\%\\
153 & 435.01\% & 21.66\% && 226 & 75.37\% & 2.77\%\\
307 & 264.07\% & 7.62\% && 453 & 105.68\% & 1.60\%\\
615 & 164.62\% & 2.28\% && 907 & 32.17\% & 0.30\%\\
\cline{1-3}\cline{5-7}
\end{tabular}
\end{table}

\subsection{Performance Results for Various Data Sets}
We now compare various approximate kernels on all the data sets listed in Table~\ref{tab:dataset}. For each kernel and each $r$, we obtain the performance result through a grid search of the optimal parameters $\sigma$ and $\lambda$. Because of the high cost of parameter tuning, no repetitions are run and hence the performances may be susceptible to randomization. Hence, conclusions are drawn across data sets and we refrain from over-interpreting the results of an individual data set. We run experiments with a few $r$'s and we are particularly interested in the performance trend as $r$ increases.

\begin{figure}[!ht]
\centering
\subfigure[cadata, regression]{
\includegraphics[width=.33\linewidth]{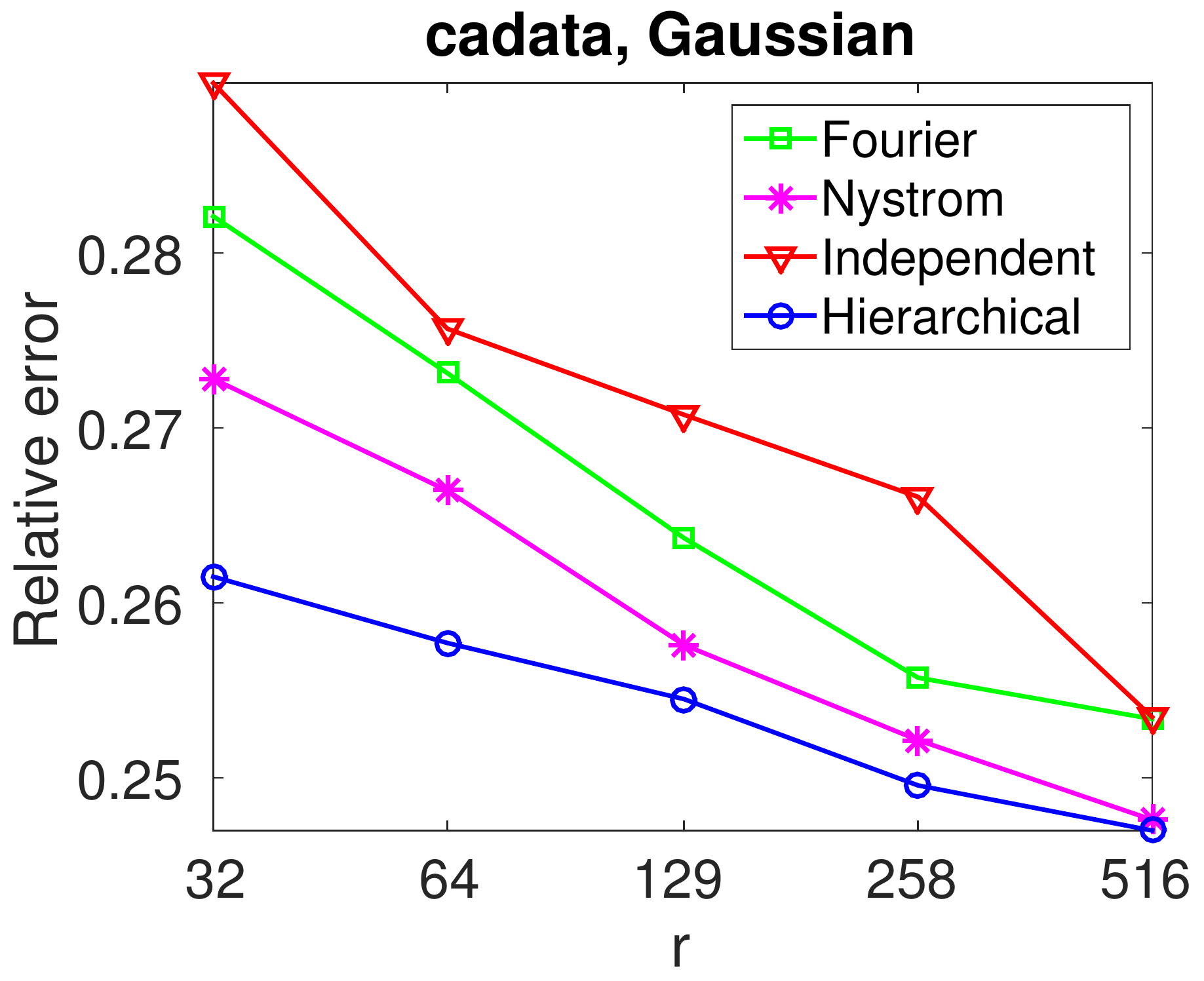}
\includegraphics[width=.32\linewidth]{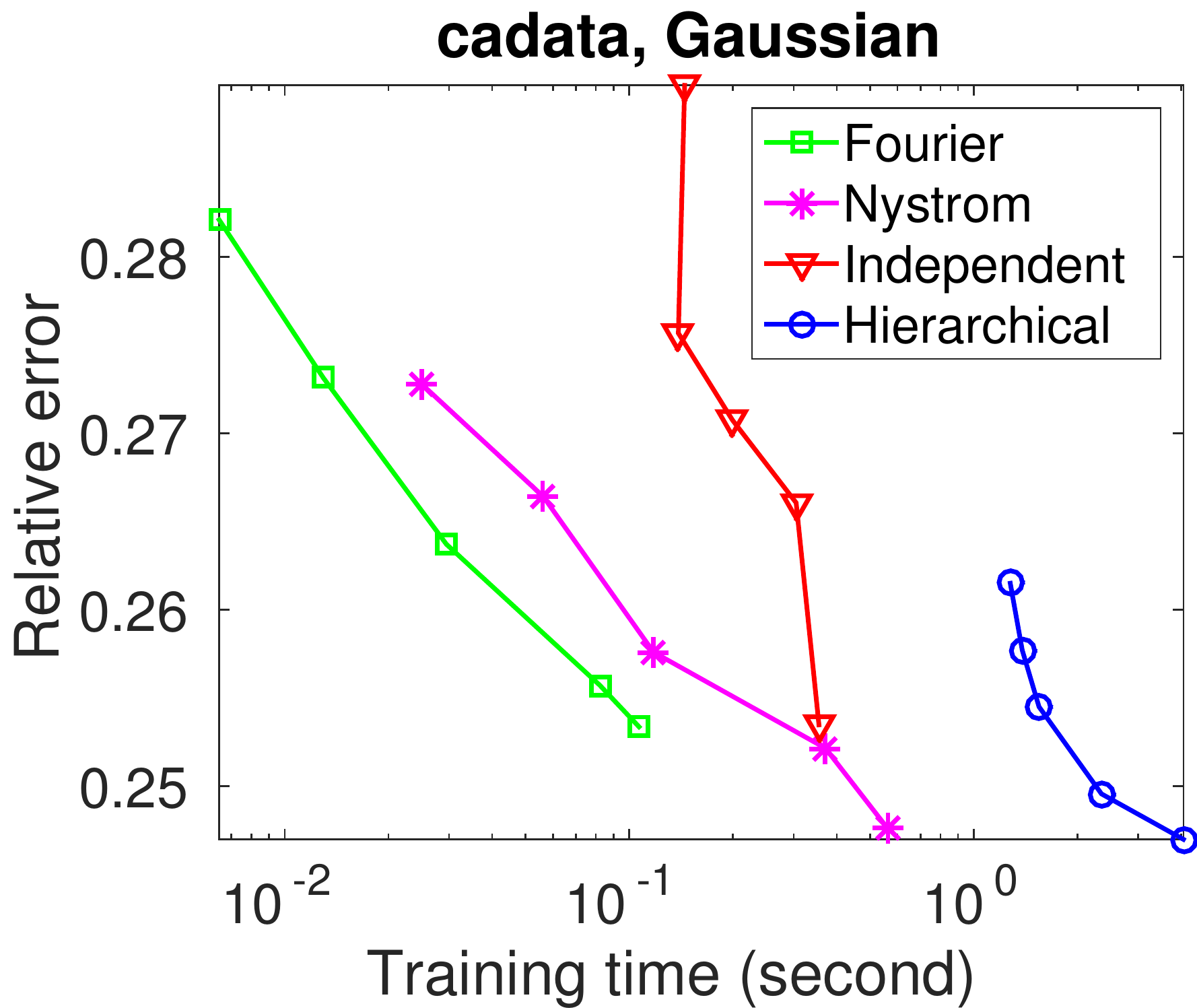}
\includegraphics[width=.33\linewidth]{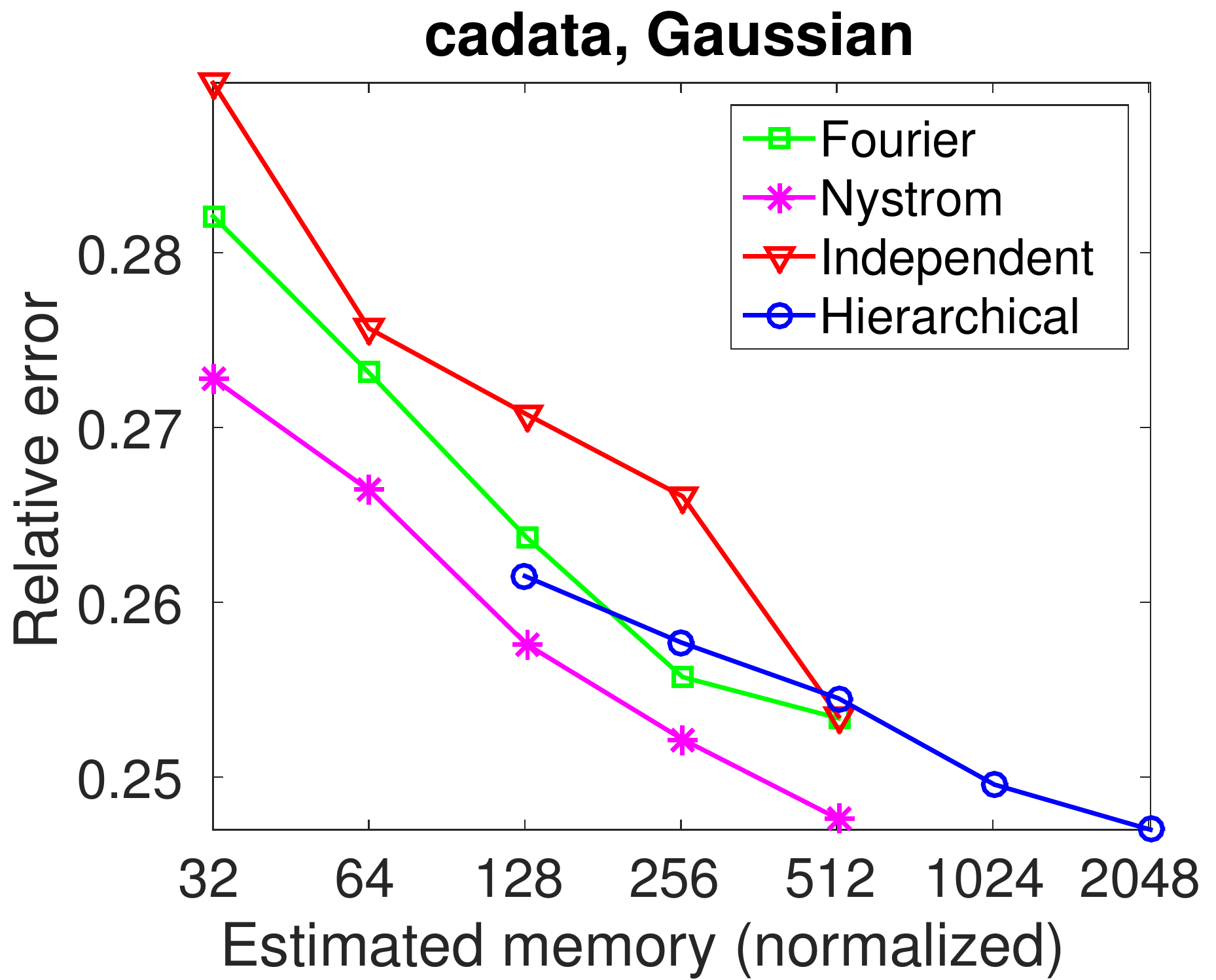}}
\subfigure[YearPredictionMSD, regression]{
\includegraphics[width=.33\linewidth]{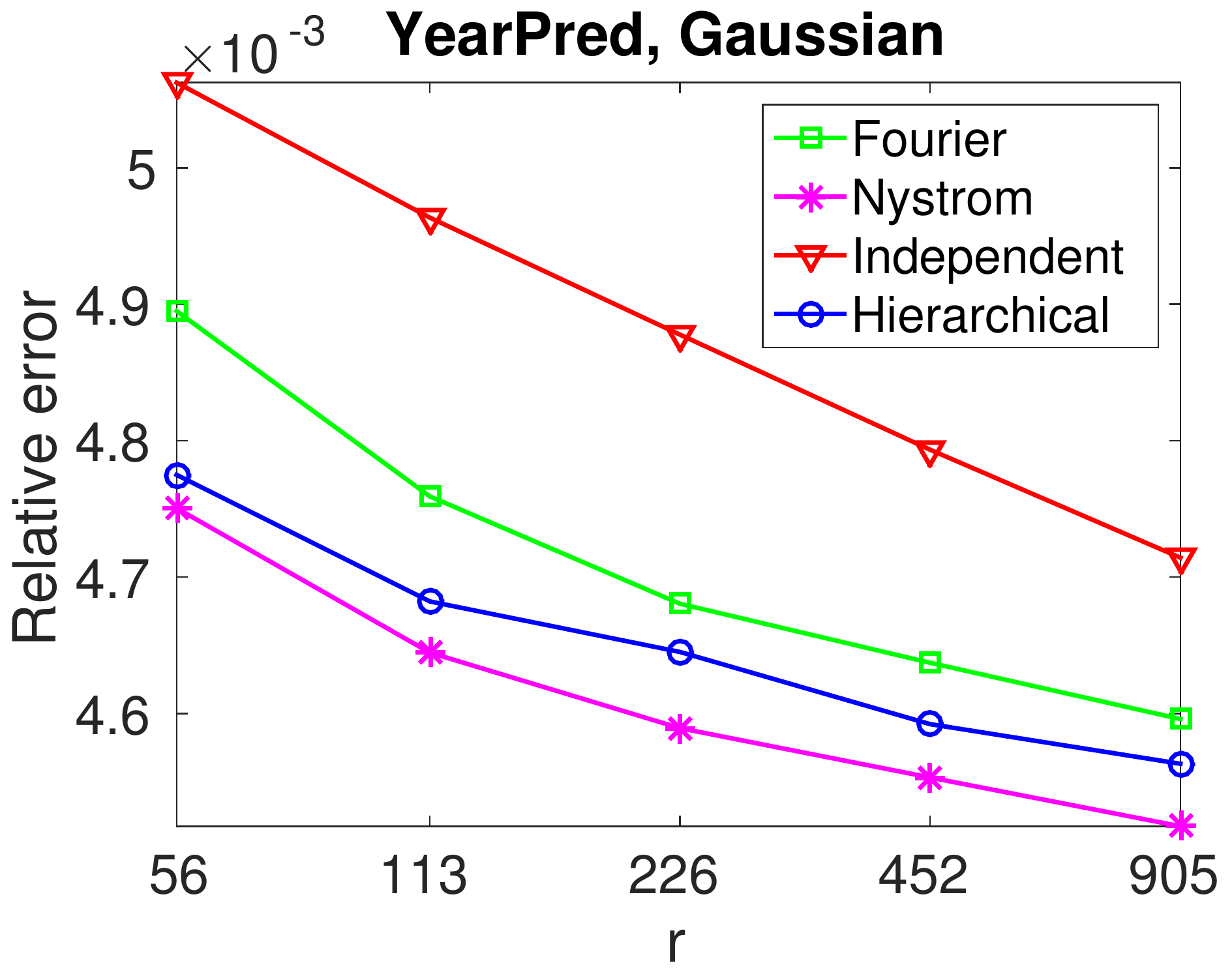}
\includegraphics[width=.32\linewidth]{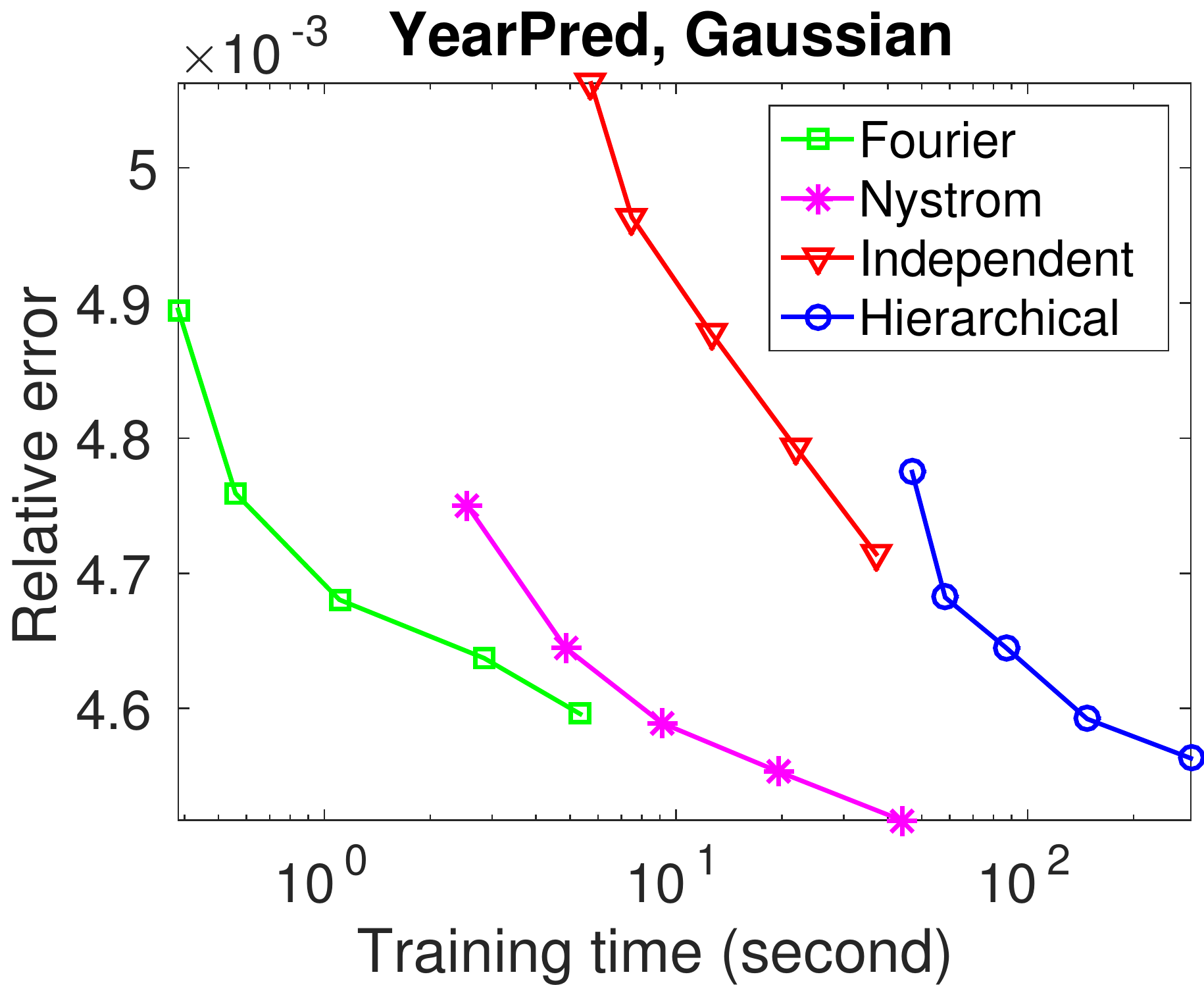}
\includegraphics[width=.33\linewidth]{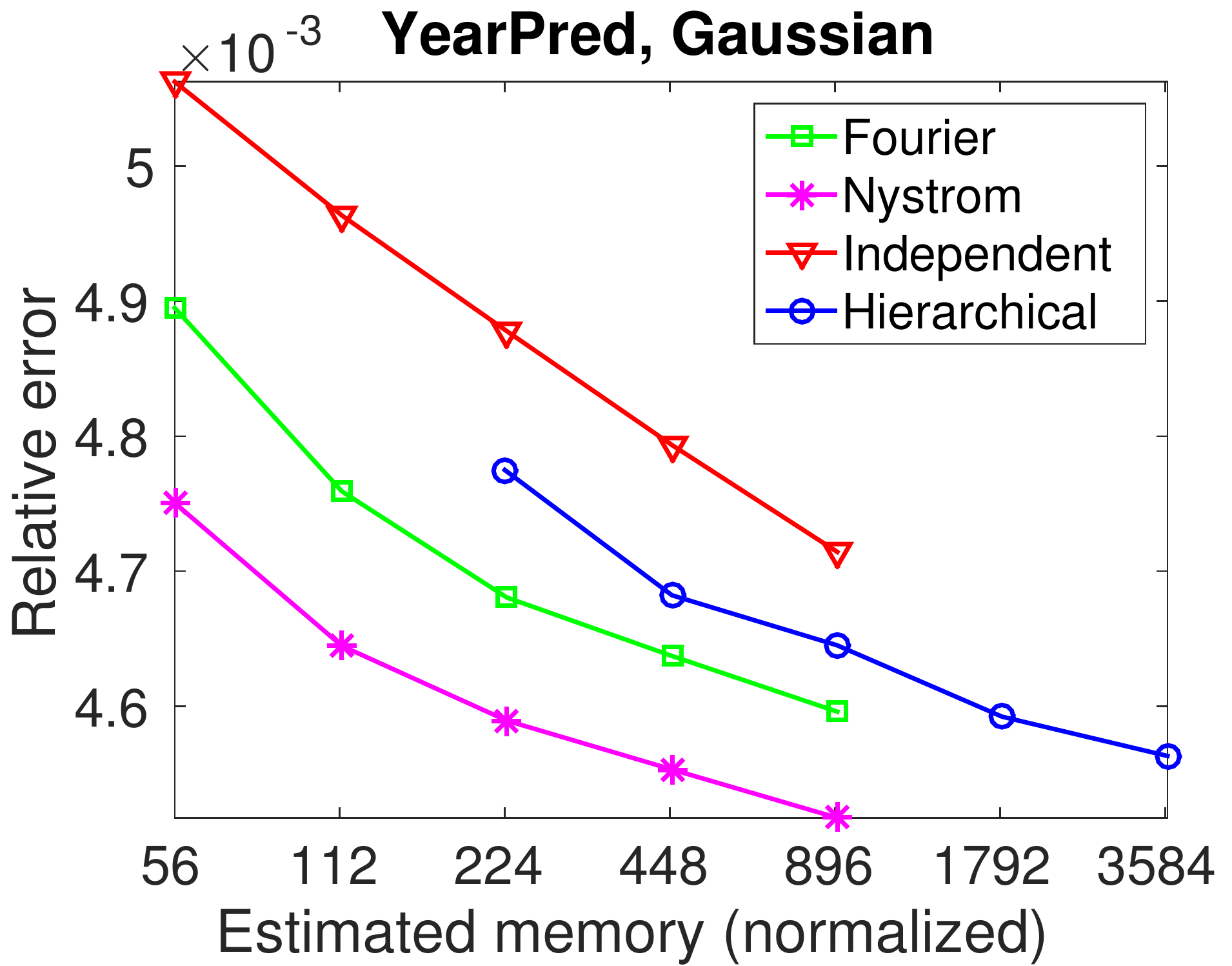}}
\subfigure[ijcnn1, binary classification]{
\includegraphics[width=.33\linewidth]{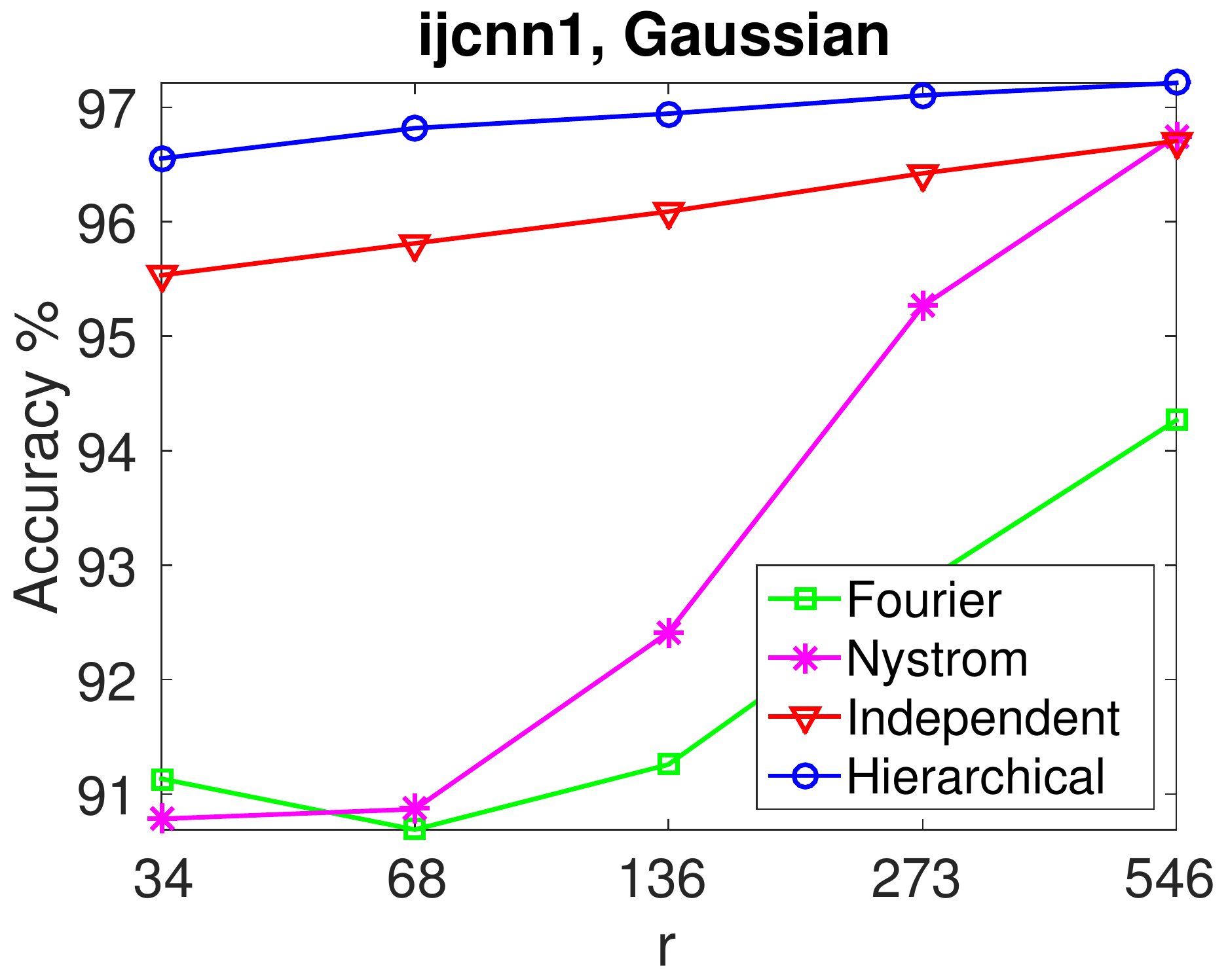}
\includegraphics[width=.32\linewidth]{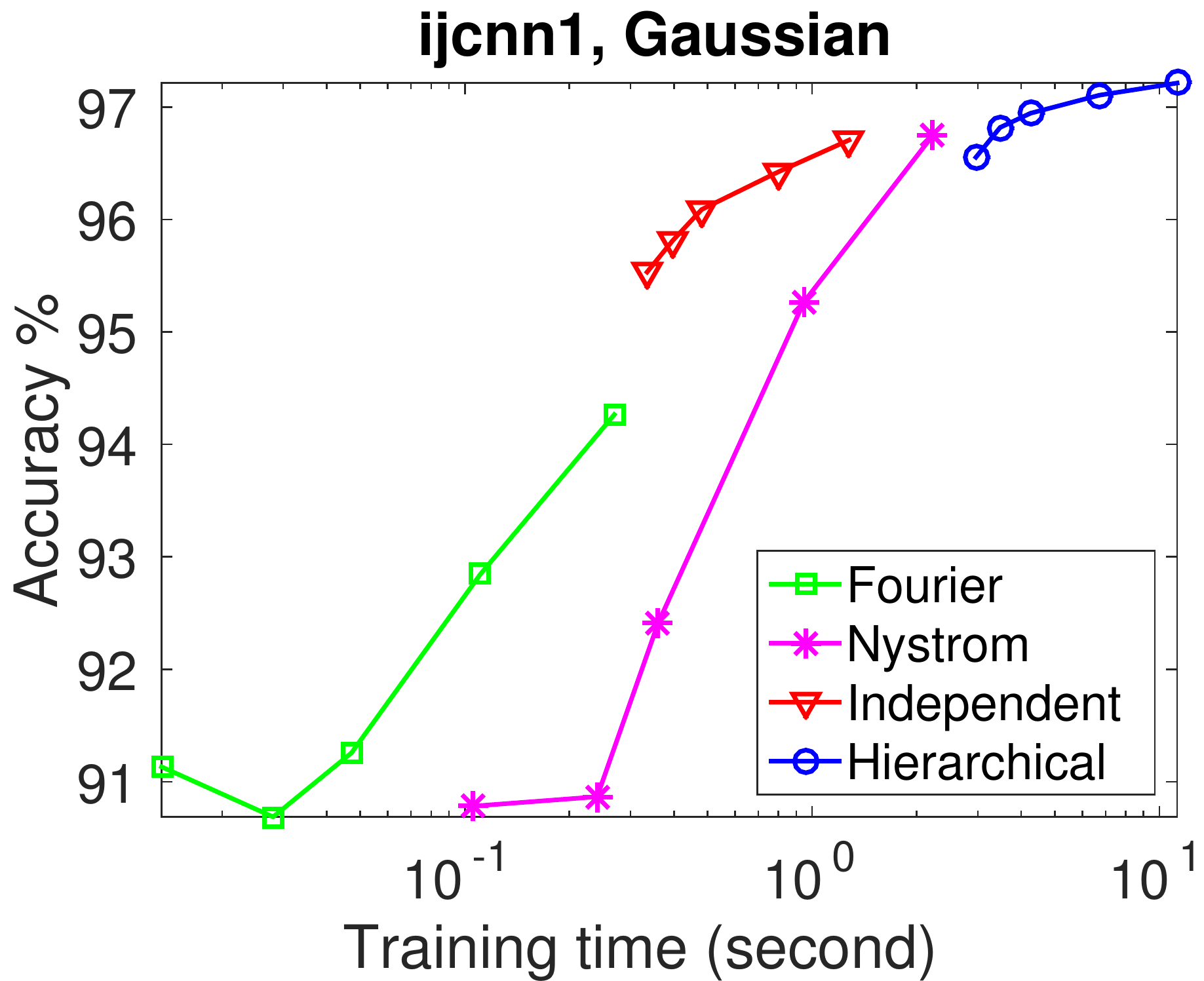}
\includegraphics[width=.33\linewidth]{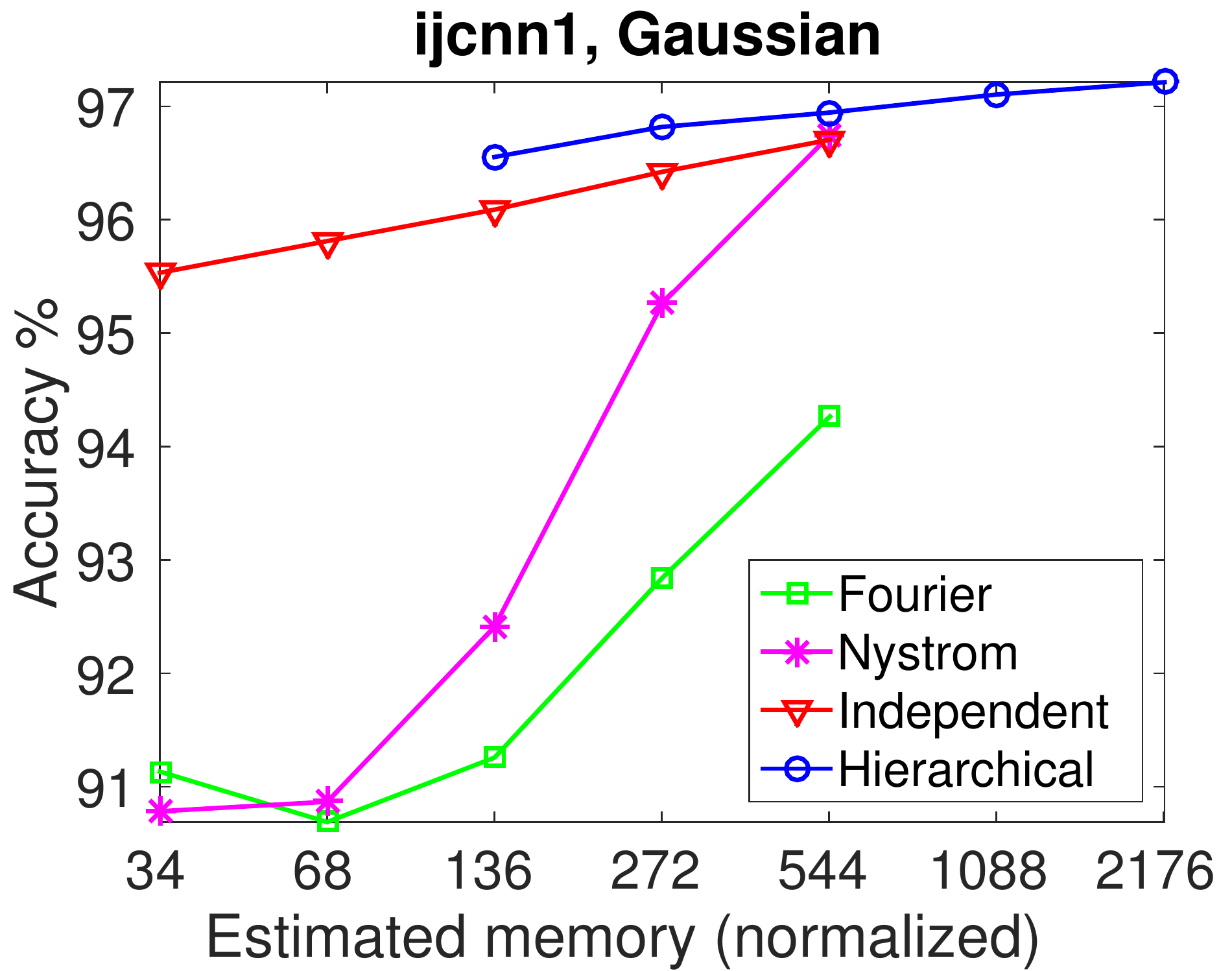}}
\subfigure[covtype.binary, binary classification]{
\includegraphics[width=.33\linewidth]{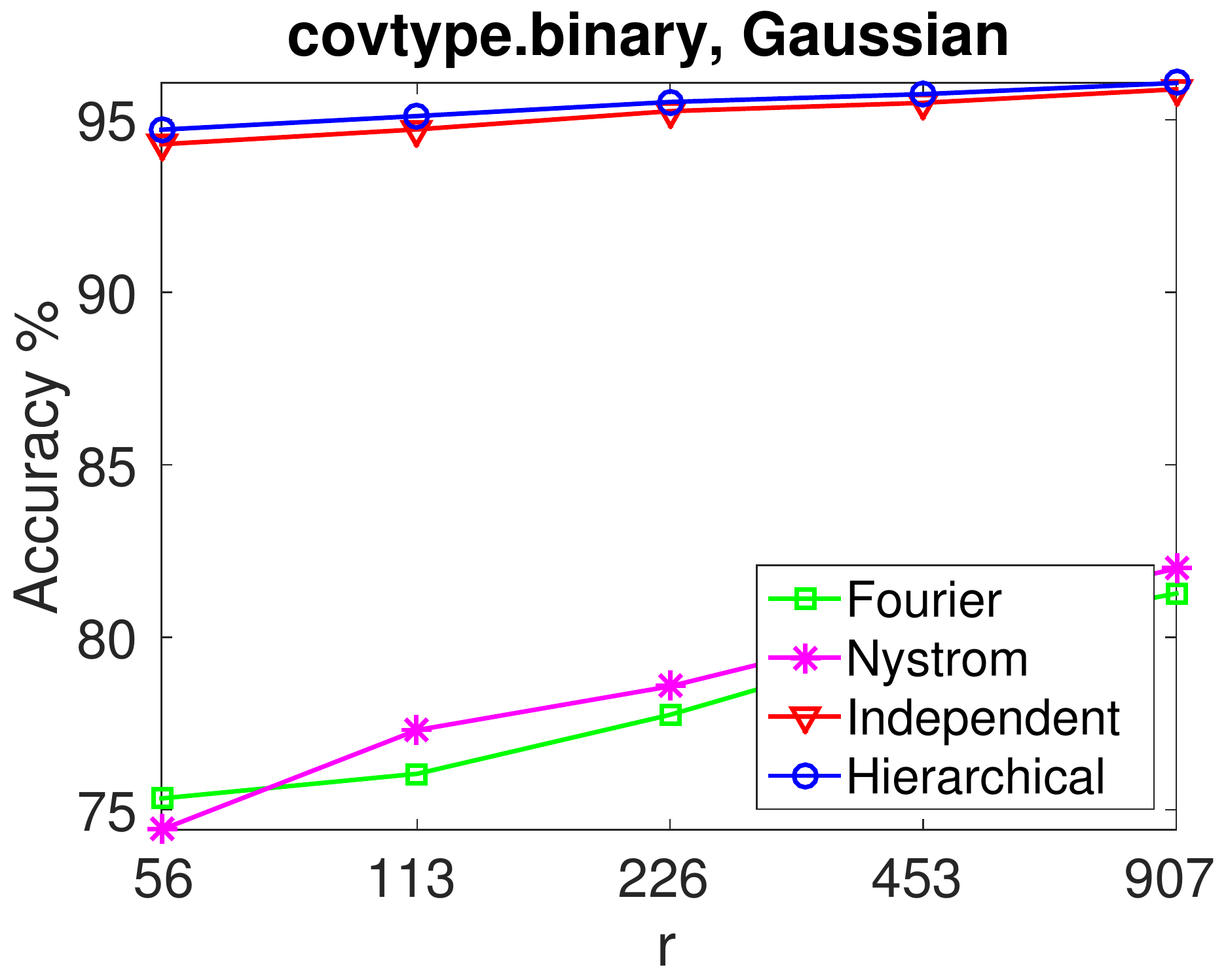}
\includegraphics[width=.32\linewidth]{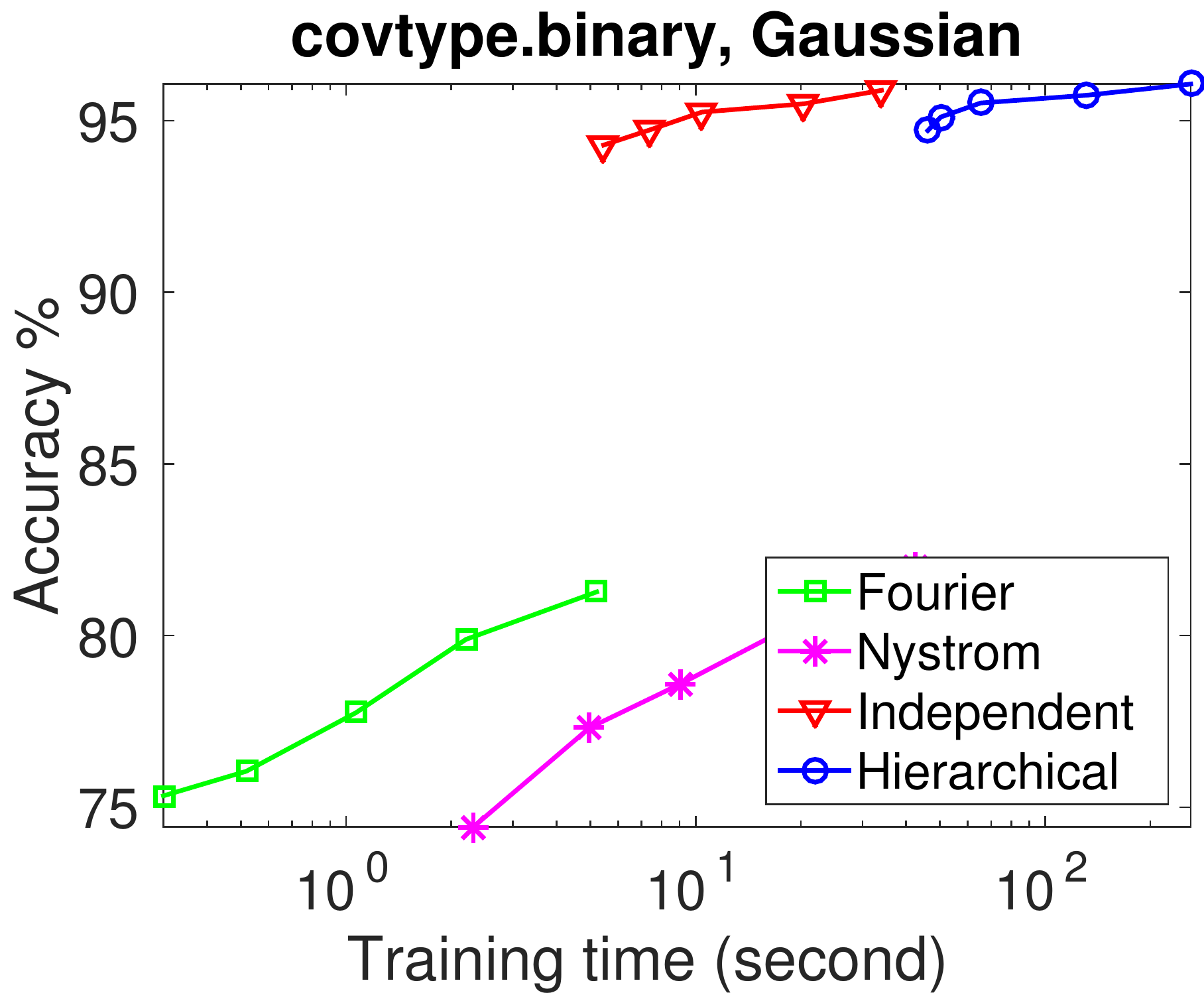}
\includegraphics[width=.33\linewidth]{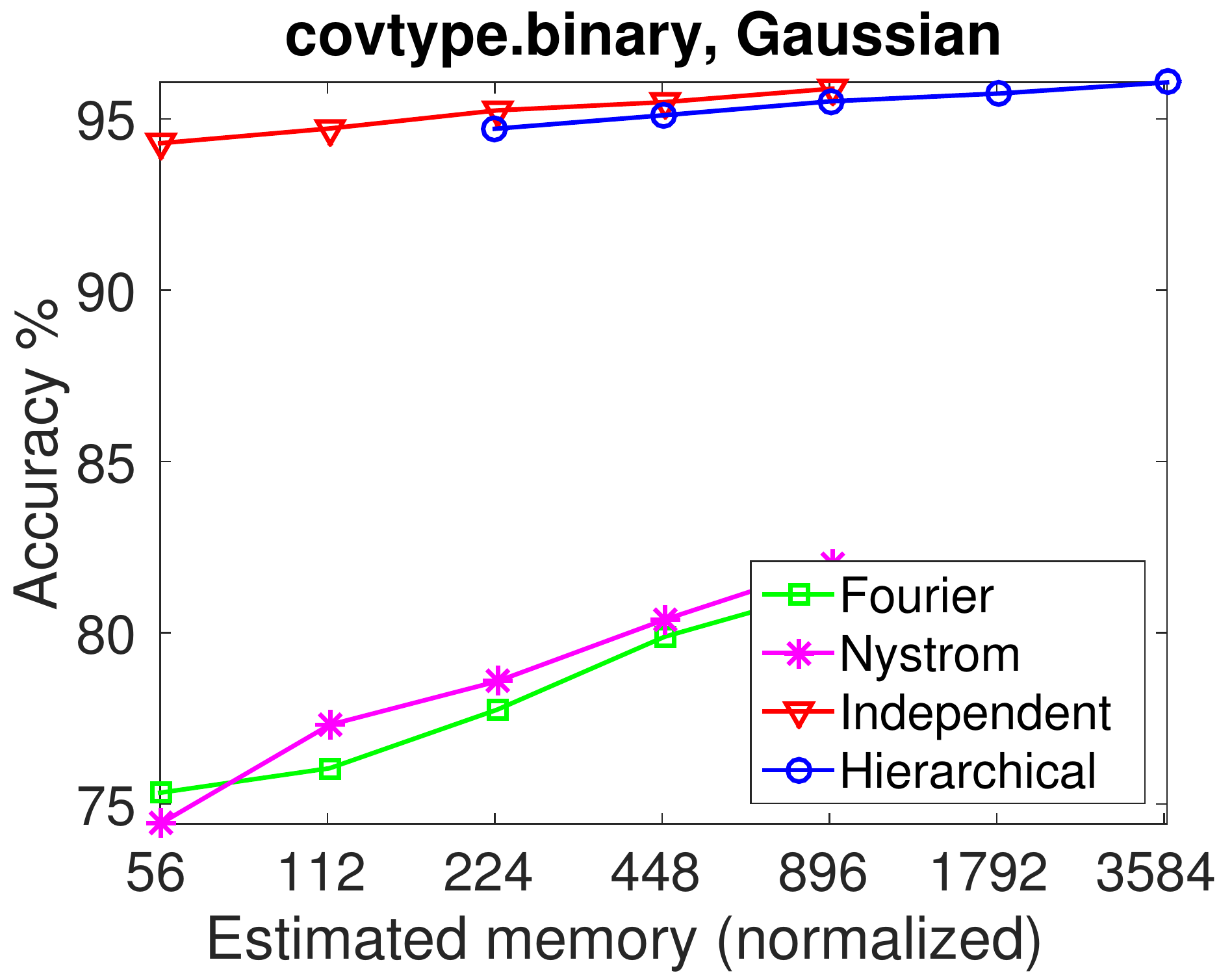}}
\caption{Performance versus $r$, time, and memory. Gaussian kernel.}
\label{fig:ZZ_plot_exp_3_comprehensive_1}
\end{figure}

\begin{figure}[!ht]
\centering
\subfigure[SUSY, binary classification]{
\includegraphics[width=.33\linewidth]{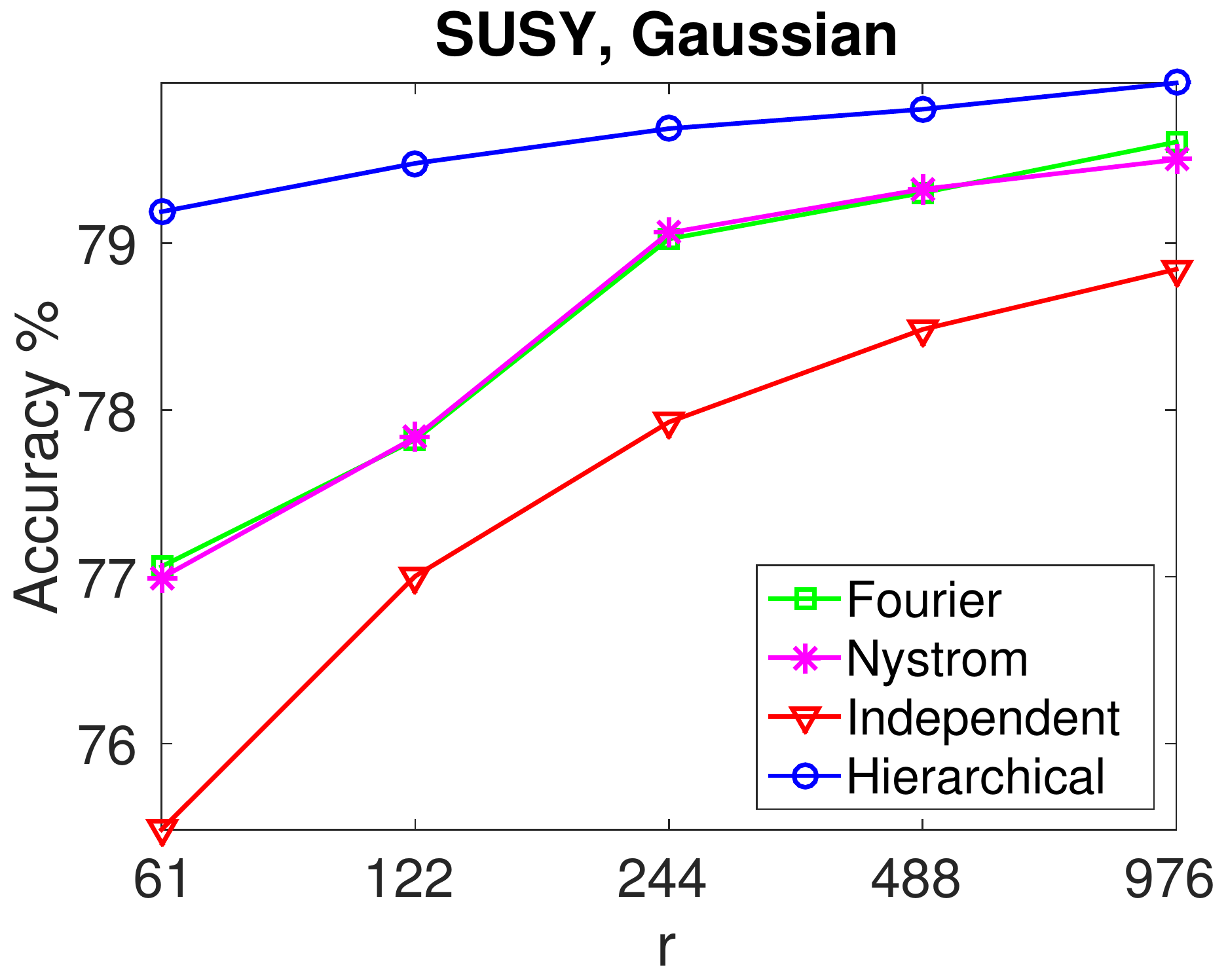}
\includegraphics[width=.32\linewidth]{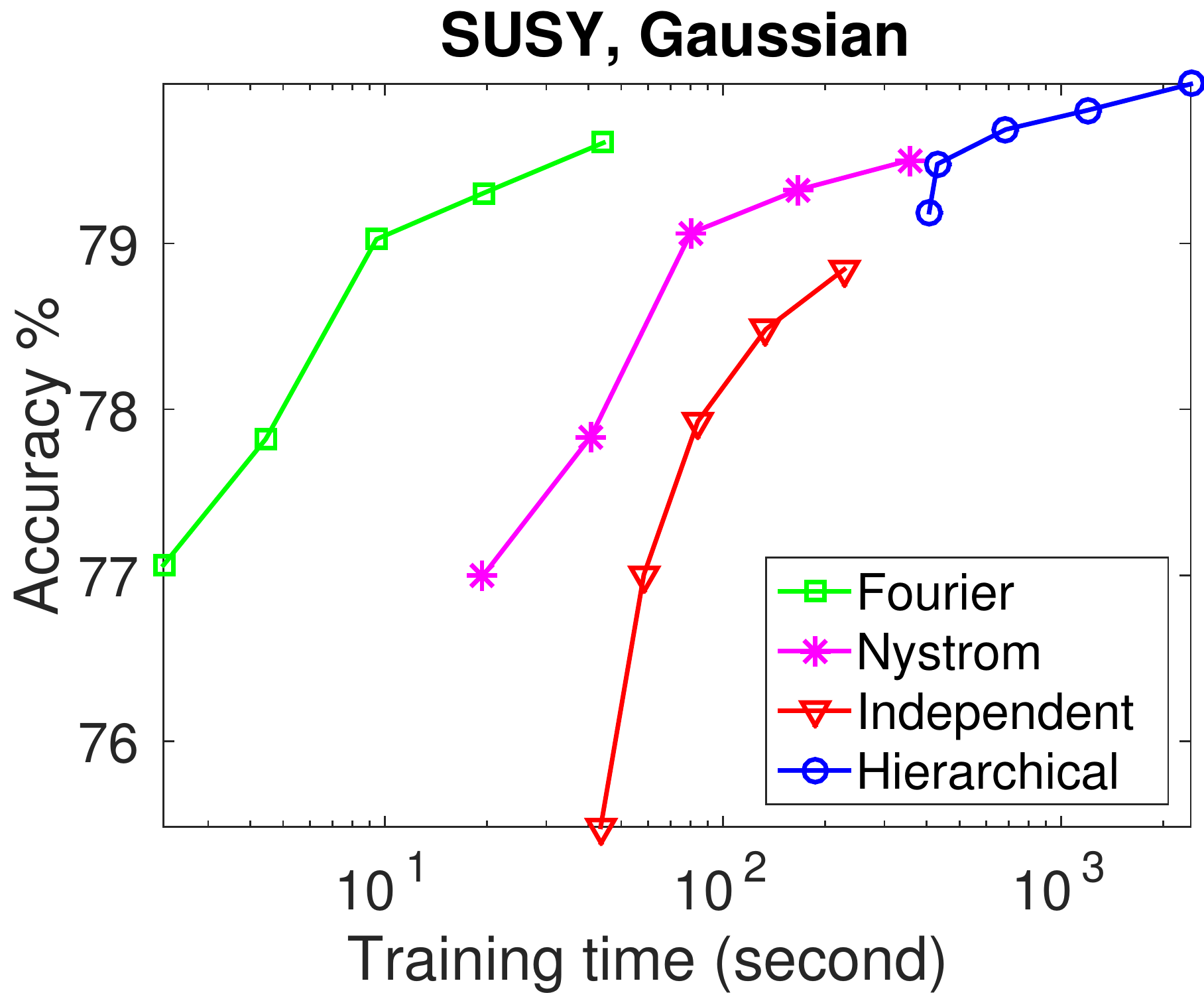}
\includegraphics[width=.33\linewidth]{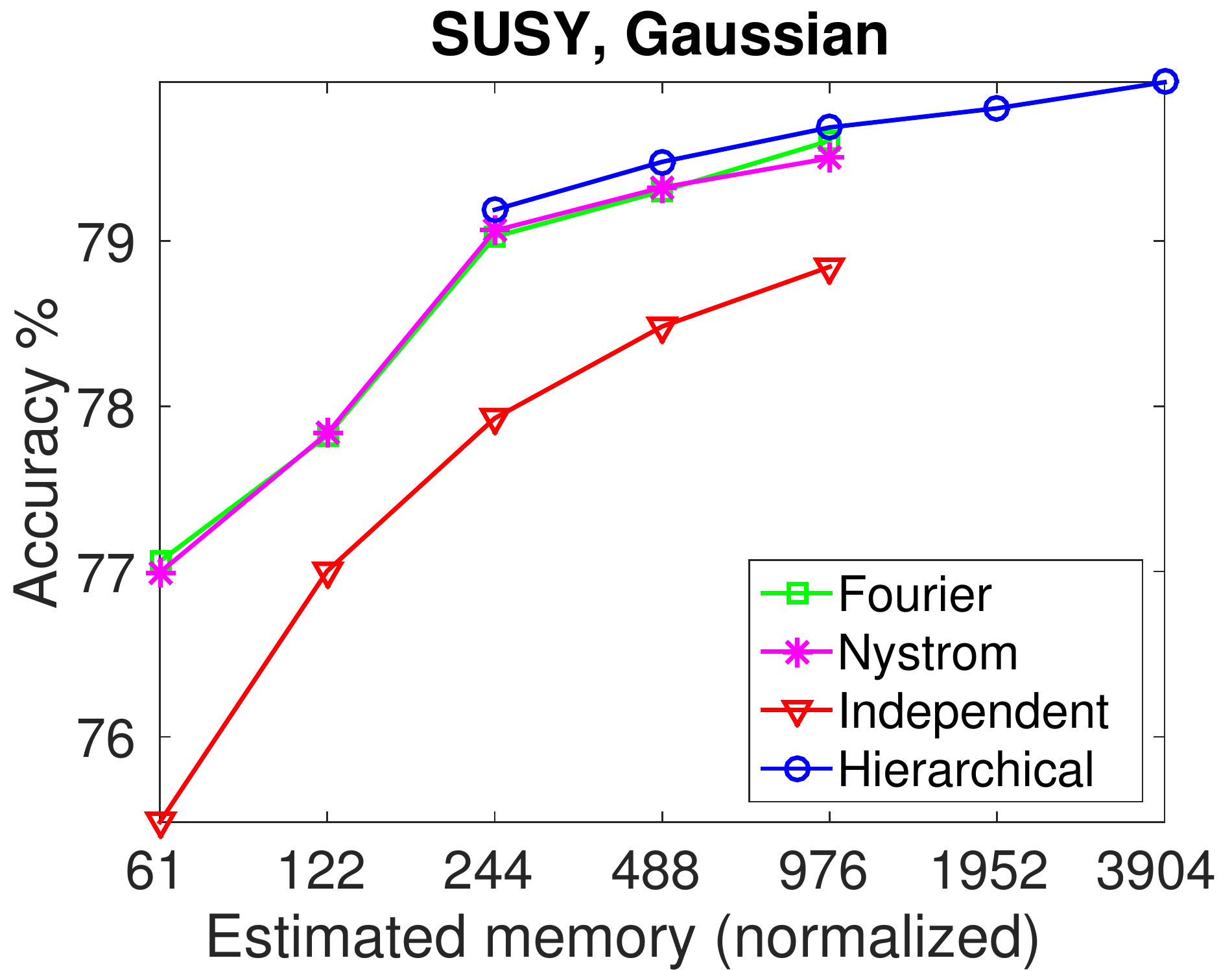}}
\subfigure[mnist, multiclass classification]{
\includegraphics[width=.33\linewidth]{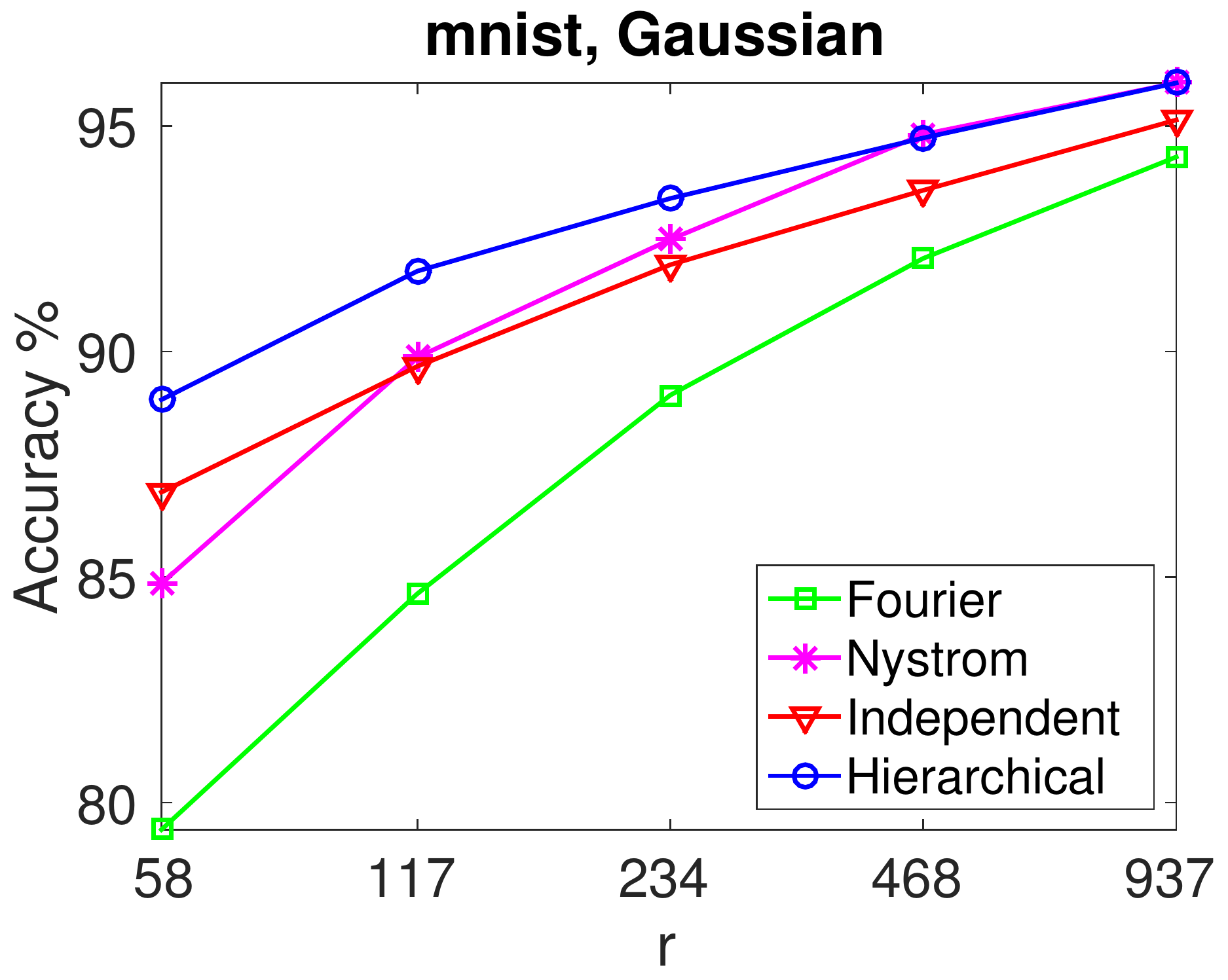}
\includegraphics[width=.32\linewidth]{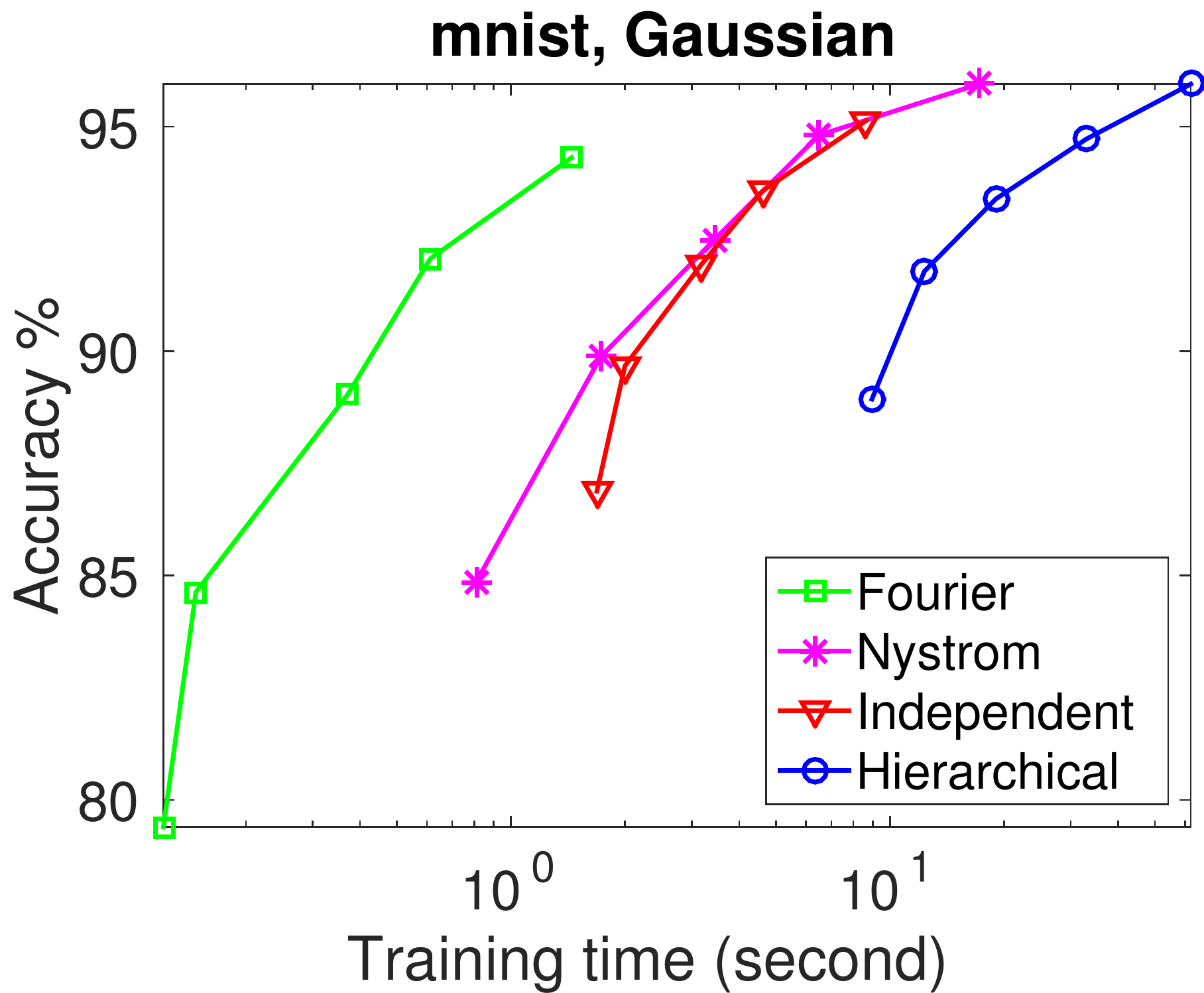}
\includegraphics[width=.33\linewidth]{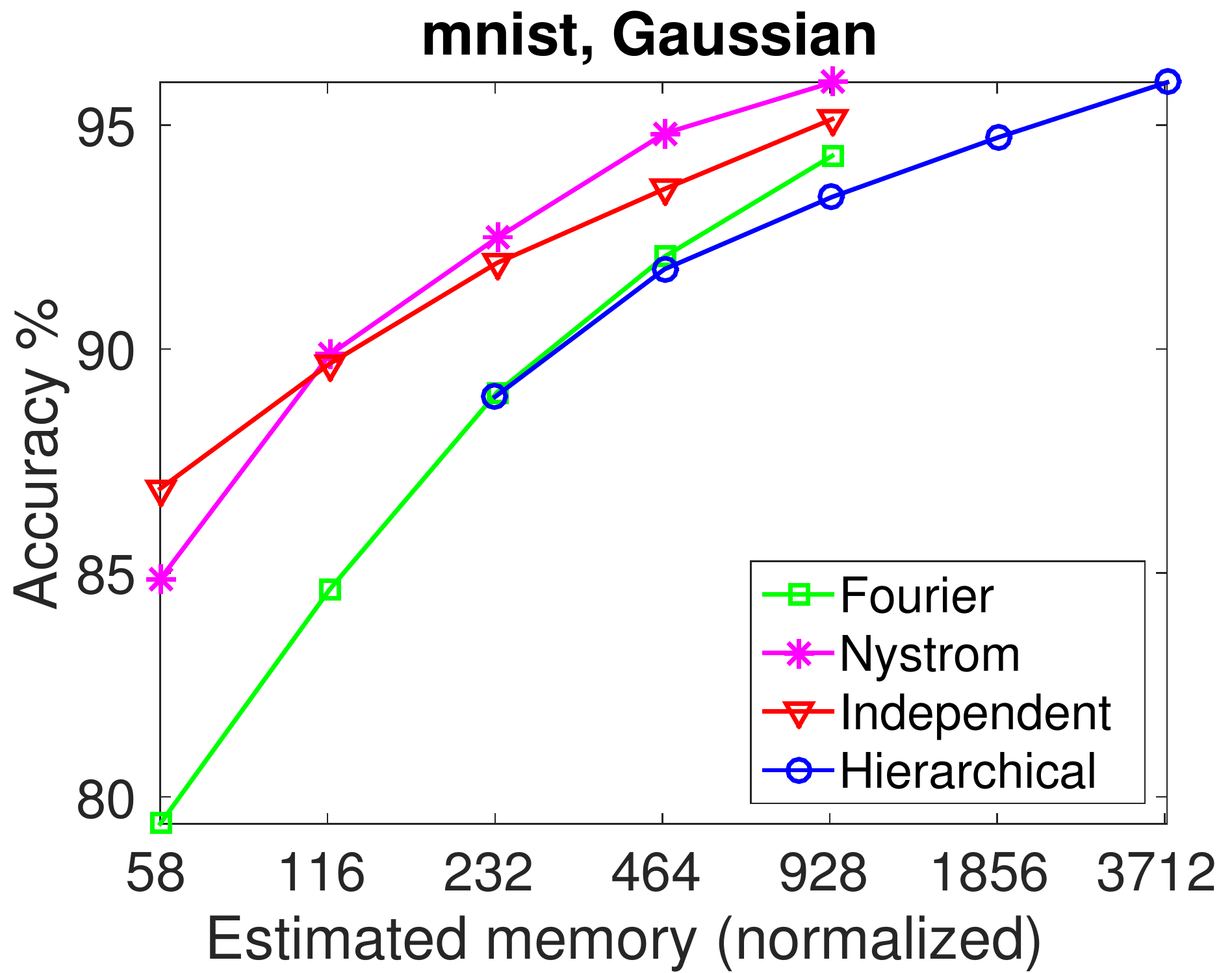}}
\subfigure[acoustic, multiclass classification]{
\includegraphics[width=.33\linewidth]{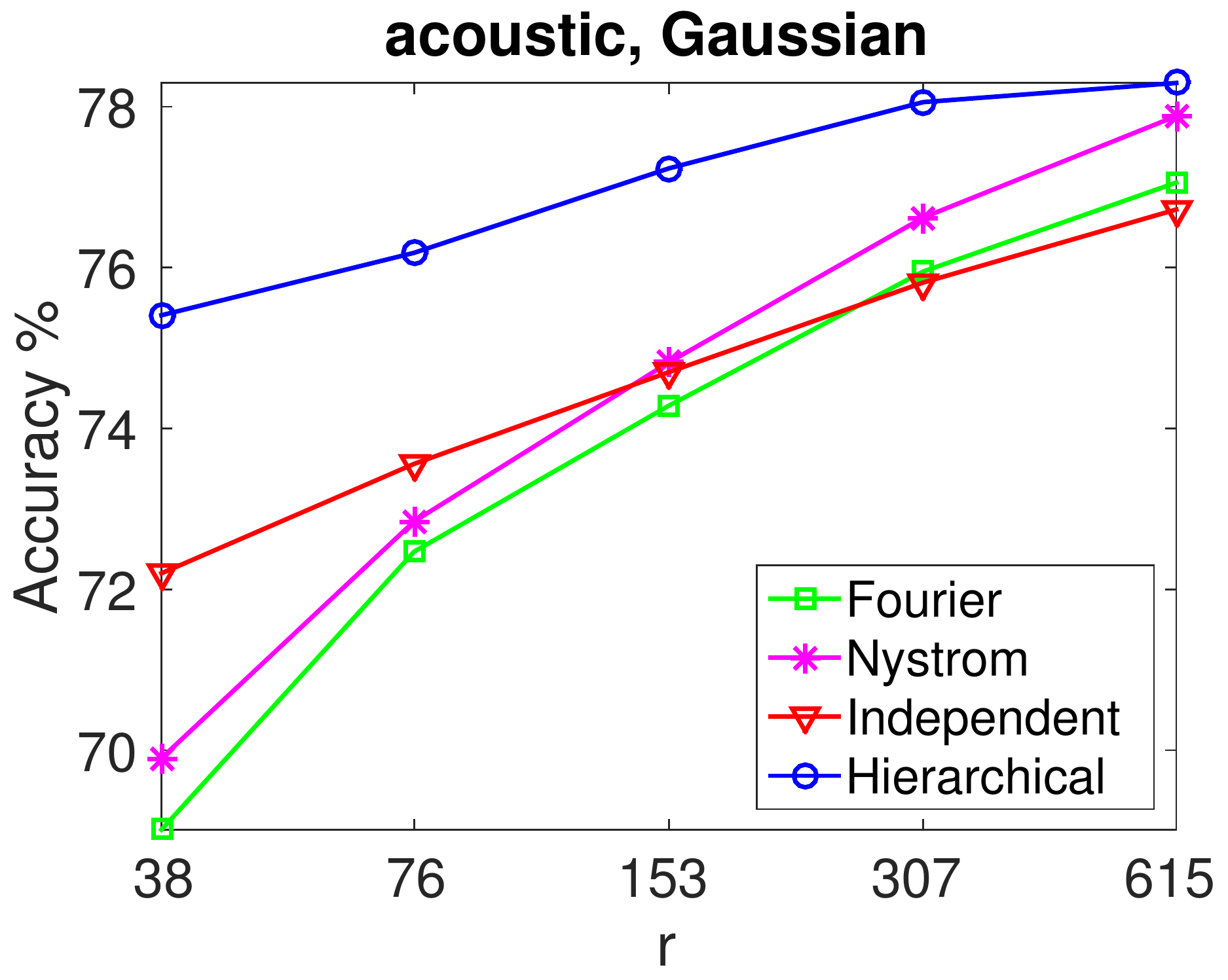}
\includegraphics[width=.32\linewidth]{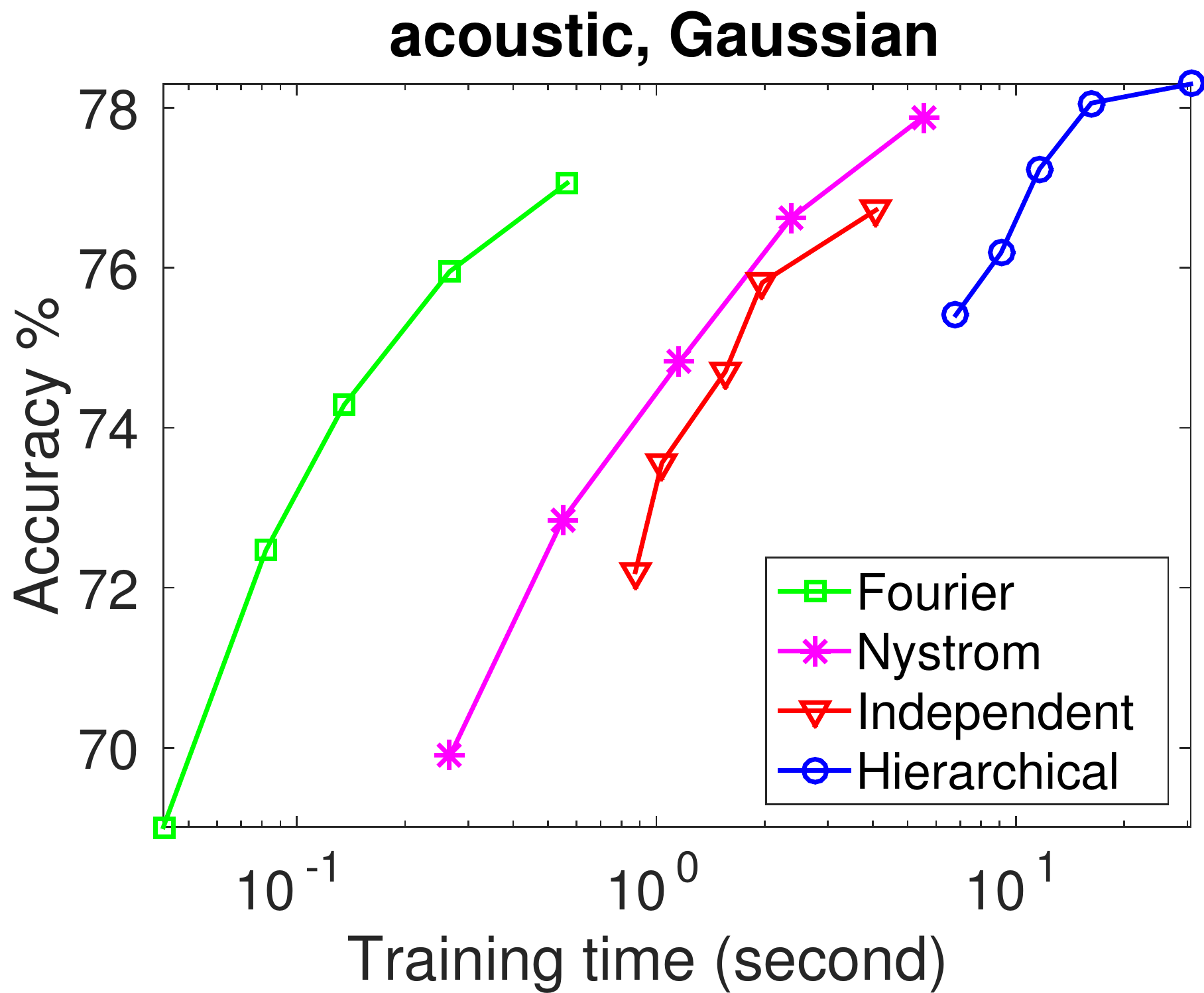}
\includegraphics[width=.33\linewidth]{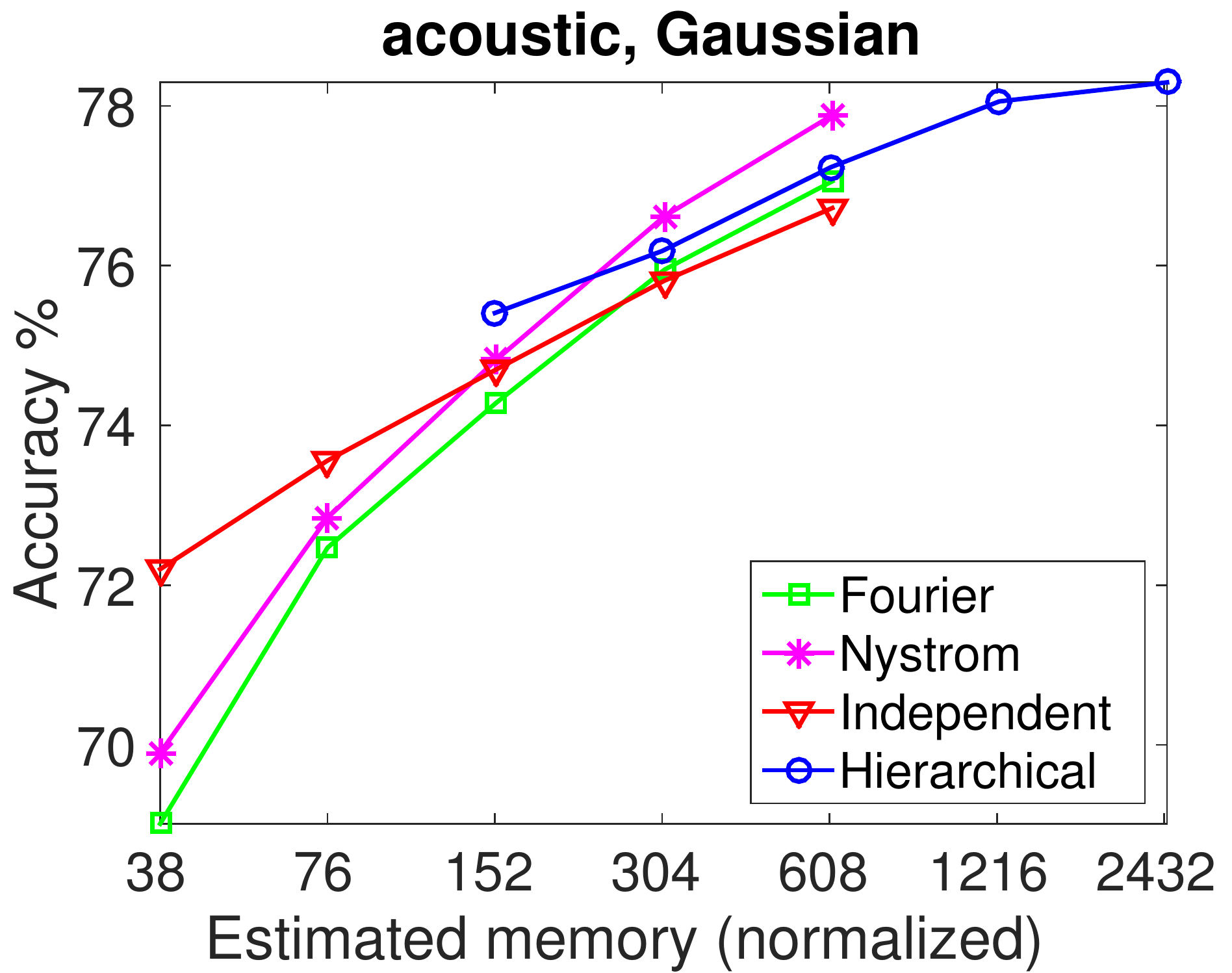}}
\subfigure[covtype, multiclass classification]{
\includegraphics[width=.33\linewidth]{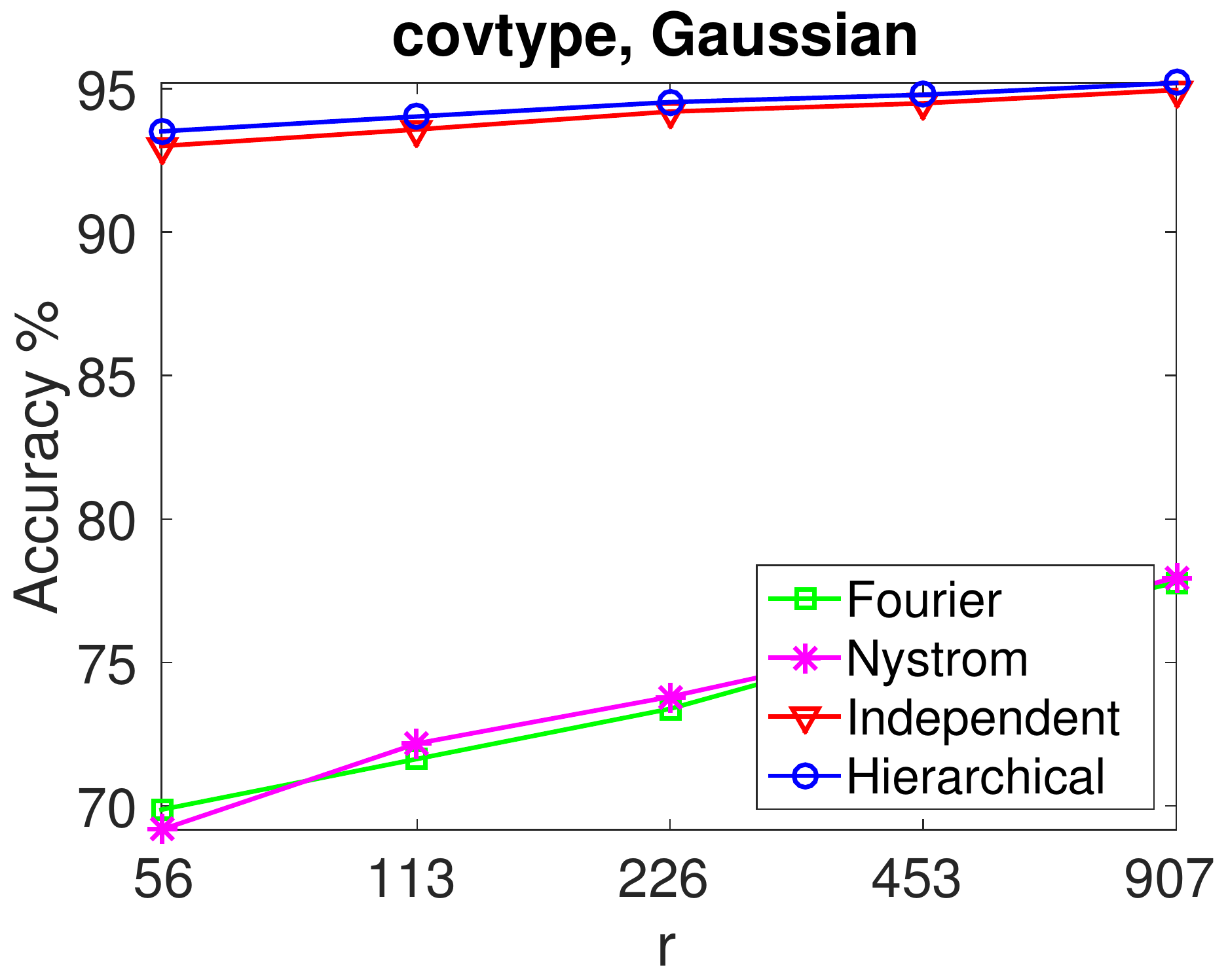}
\includegraphics[width=.32\linewidth]{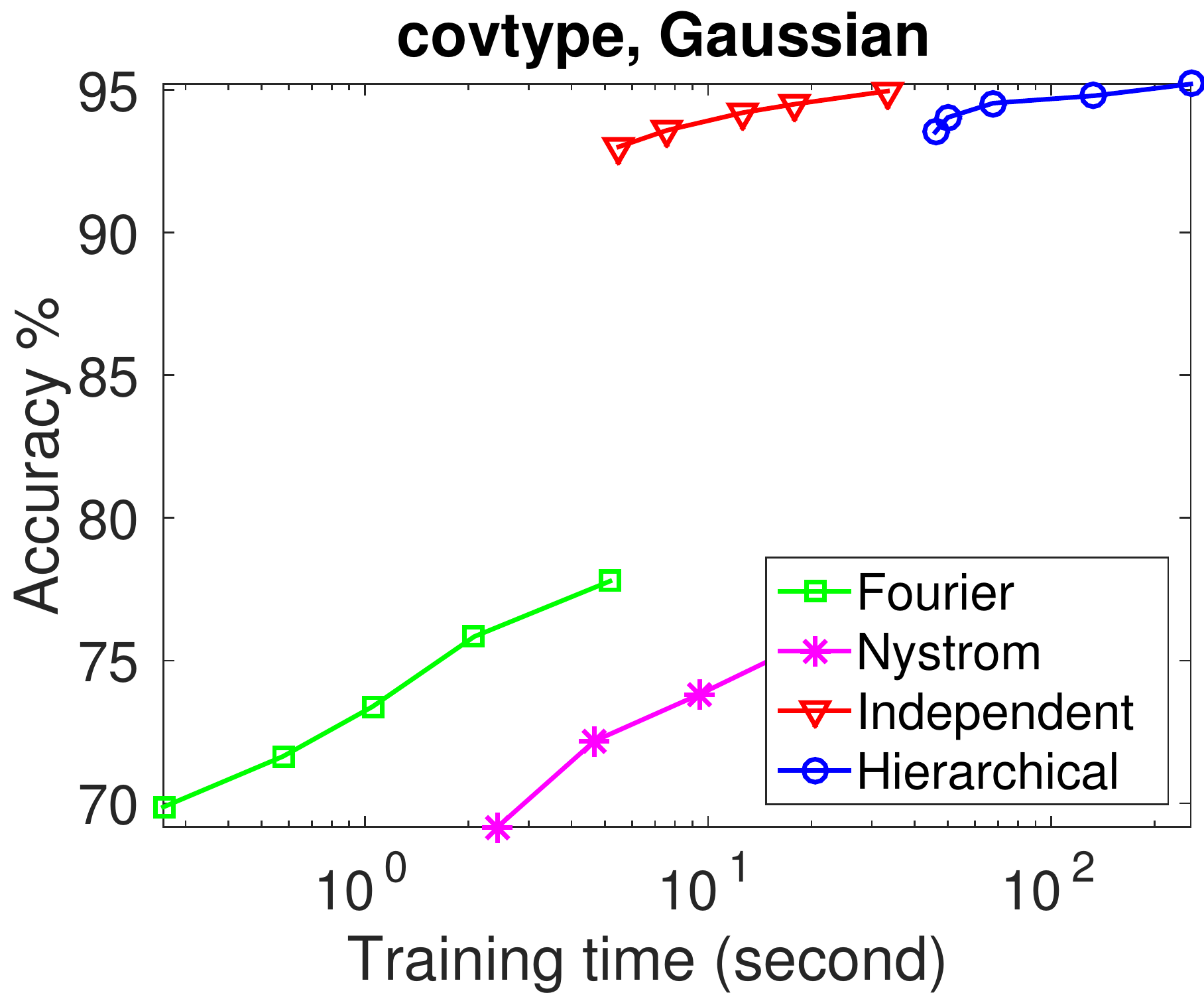}
\includegraphics[width=.33\linewidth]{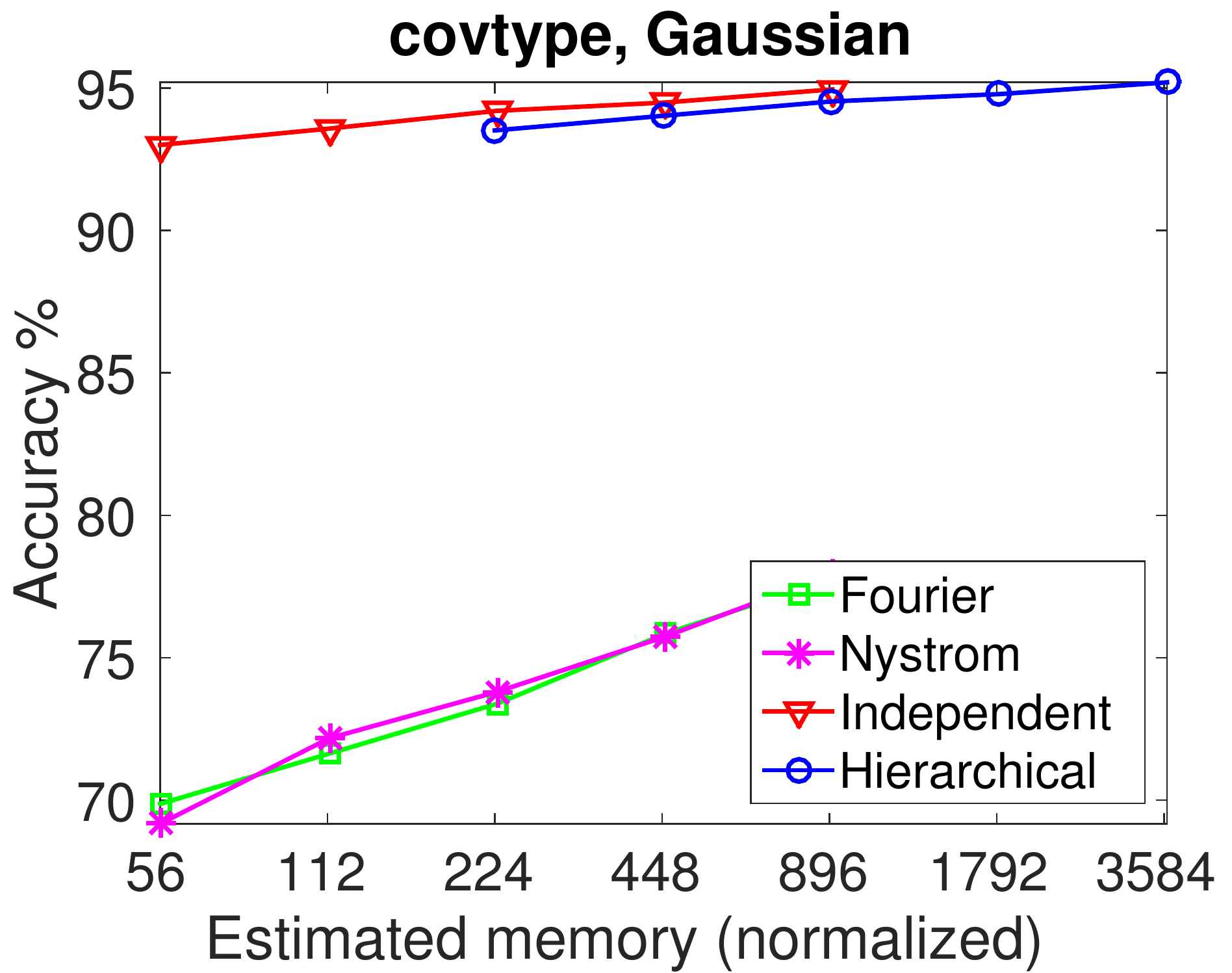}}
\caption{(Continued) Performance versus $r$, time, and memory. Gaussian kernel.}
\label{fig:ZZ_plot_exp_3_comprehensive_2}
\end{figure}

The results are summarized in Figures~\ref{fig:ZZ_plot_exp_3_comprehensive_1} and~\ref{fig:ZZ_plot_exp_3_comprehensive_2}. In the panel of the plots, each row corresponds to one data set. The three columns are performance versus $r$, training time, and memory consumption, respectively. The performance is measured as the relative error in the regression case and the accuracy in the classification case. The memory cost is estimated and normalized. According to the analysis in Section~\ref{sec:cost.mem}, the memory cost of the proposed kernel is approximately $4r$ per training point, whereas for the other approximate kernels we use an estimation of $r$ per point.

A few observations follow. First, the proposed kernel almost always yields the best performance versus $r$, except for the data set YearPredictionMSD. Such an observation confirms the effectiveness of the proposed method, which combines the advantages of the Nystr\"{o}m method globally and the cross-domain independent kernel locally.

Second, the Fourier method generally runs the fastest, followed by Nystr\"{o}m and independent, whose speeds are comparably similar; and the proposed method falls behind. Although all methods have the same asymptotic cost $O(nr^2)$, in a finer level of analysis, Fourier wins over Nystr\"{o}m because the generation of random features in the former is often less expensive than the kernel evaluation in the latter. Of course, such a fact could also be interpreted on the other hand as a weakness of the Fourier method: It is applicable to only a numeric data representation, whereas Nystr\"{o}m assumes no numeric form of data but a kernel. Nystr\"{o}m has a similar time cost as does the independent kernel; the former requires more kernel evaluations but the latter needs computing the sequence of subdomains each point belongs to. The independent kernel and the proposed kernel need similar computations (small matrix multiplications and factorizations), but the number of such operations is a constant times more for the proposed kernel; hence, not surprisingly it requires more computational effort.

Note that in the Fourier and the Nystr\"{o}m methods, large matrices are operated in a small number of times. The implementation straightforwardly makes use of sophisticated linear algebra libraries, where parallelism and cache reuse have been aggressively optimized~\citep{Anderson1999,Goto2008}. On the other hand, in the independent and the proposed kernels, small matrices are operated in a large number of times. Moreover, there exist fragmented computations such as tree traversals and data shuffles. Optimizing these computations in terms of parallelism and cache reuse is possible and the executation times have the potential to be substantially improved.

The third observation of Figures~\ref{fig:ZZ_plot_exp_3_comprehensive_1} and~\ref{fig:ZZ_plot_exp_3_comprehensive_2} focuses on the performance versus memory consumption. Since all the other methods are estimated to use the same amount of memory, the plots are essentially moving only the curve of the proposed kernel to the right compared with the performance-versus-$r$ plots. Then, with the same amount of memory, some methods behave better than others in some data sets. There does not exist a consistent winner. One sees that the proposed method is the best performing for the data sets ijcnn1 and SUSY.

The last observation is that for the data set covtype (both the binary and the multiclass case), a significant performance gap occurs between the low-rank (Fourier and Nystr\"{o}m) and the full-rank (independent and hierarchical) kernels. This is the exemplified case where the eigenvalues in the kernel matrix decay too slowly, such that low-rank kernels with an insufficiently large rank $r$ are clearly underperforming.

\subsection{A Different Base Kernel}
We consider using a different base kernel and inspect whether the observations made in the preceding subsection still hold. Here, we use the Laplace kernel
\[
k(\bm{x},\bm{x}')=\exp\left(-\frac{\|\bm{x}-\bm{x}'\|_1}{\sigma}\right)
\]
popularized by~\citet{Rahimi2007}. Clearly, this kernel is the tensor product of one-dimensional exponential kernels $\exp(-|x_i-x_i'|/\sigma)$ for all attributes $i$. In the GP context, the exponential kernel gives rise to the well-known Ornstein-Uhlenbeck process. Both the exponential kernel and the Gaussian kernel are special cases of the Mat\'{e}rn family of kernels, but their characteristics are opposite: the latter yields an extremely smooth process whereas the former highly rough~\citep{Stein1999}. Hence, one might expect that their results substantially differ.

However, the answer is to the contrary. To avoid cluttering, we leave the plots to the appendix (see Figures~\ref{fig:ZZ_plot_exp_4_laplace_1} and~\ref{fig:ZZ_plot_exp_4_laplace_2}); they are similar to those shown in Figures~\ref{fig:ZZ_plot_exp_3_comprehensive_1} and~\ref{fig:ZZ_plot_exp_3_comprehensive_2} of the Gaussian kernel. Specifically, the general trends and the specific performance values are quite close across the two base kernels. A possible reason is that the (optimal) regularization $\lambda$ generally lies on the order of $0.01$ to $0.1$, which is relatively large compared with the kernel values, whose peak occurs at $k(0)=1$. In the GP context, the noise level $\lambda$ is so high that it eclipses the effect of smoothness---the rough variation of data may as well be interpreted as noise instead. Then, the smoothness of the kernel matters little and thus the results of different base kernels look similar.

Based on this observation, we additionally perform experiments with the inverse multiquadric kernel
\[
k(\bm{x},\bm{x}')=\frac{\sigma^2}{\sqrt{\|\bm{x}-\bm{x}'\|_2^2+\sigma^2}}.
\]
The strict positive definiteness of this kernel is proved in~\citet{Micchelli1986}; but its Fourier transform is little known and hence we do not compare with the random Fourier method. The results are shown in the appendix (see Figures~\ref{fig:ZZ_plot_exp_7_invmultiquadric_1} and~\ref{fig:ZZ_plot_exp_7_invmultiquadric_2}). One sees again that they are quite similar to those of the Gaussian and Laplace kernels.

\subsection{Trade-Off between $n$ and $r$}\label{sec:budget}
A folk wisdom in machine learning is that more data beats a more complex model~\citep{Domingos2012}. The applicability of this knowledge lies in the regime where the hypothesis space has not been saturated by data. For kernel methods, a natural question to ask is whether subsampling is viable, considering that the computational effort is nontrivial with respect to the data size $n$. A relevant question in the context of approximate kernels is whether a trade-off exists between $n$ and $r$, given a fixed budget $nr$. This budget comes from the memory constraint, because the memory cost is $O(nr)$ in various approximate kernels. Note that if $nr$ is fixed, the arithmetic cost $O(nr^2)$ will increase when $n$ becomes smaller, but the kernels tend to admit a better approximation quality.

Another scenario that resolves the interplay between $n$ and $r$ is the following: Suppose the current computational resources have been fully leveraged for computation and one is offered an upgrade such that the memory capacity is increased $t$ times. To expect the best performance improvement, should one seek $t$ times more training samples (provided feasible), or a smaller increase in training samples but meanwhile also an increase in $r$?

To answer this question, we use two data sets, YearPredictionMSD and covtype.binary, and progressively downsize them by a factor of two. Recall that the former is a regression problem whereas the latter classification. We plot in Figure~\ref{fig:ZZ_plot_exp_5_n_vs_r} the performance curves versus the training size, for a few $r$'s in the progression of approximately a factor of two. For comparison, we also perform the full-fledged computation with the nonapproximate kernel through solving~\eqref{eqn:KRR} directly by using a preconditioned Krylov method on a cluster of AWS EC2 machines (see~\citet{avron2016} for details). The performance curve is black with circle markers.

\begin{figure}[ht]
\centering
\subfigure[YearPredictionMSD, regression]{
\includegraphics[width=.47\linewidth]{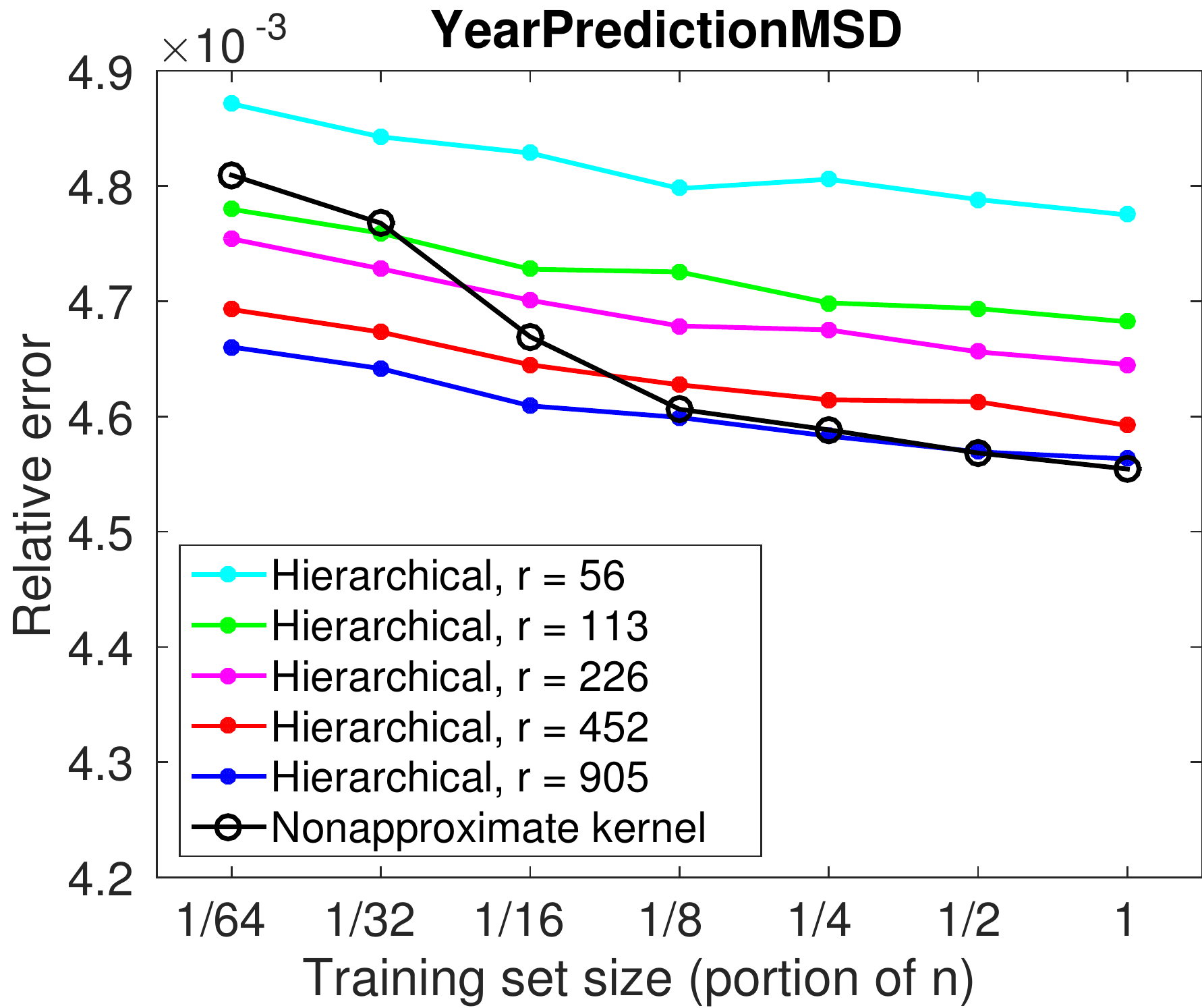}}
\subfigure[covtype.binary, binary classification]{
\includegraphics[width=.47\linewidth]{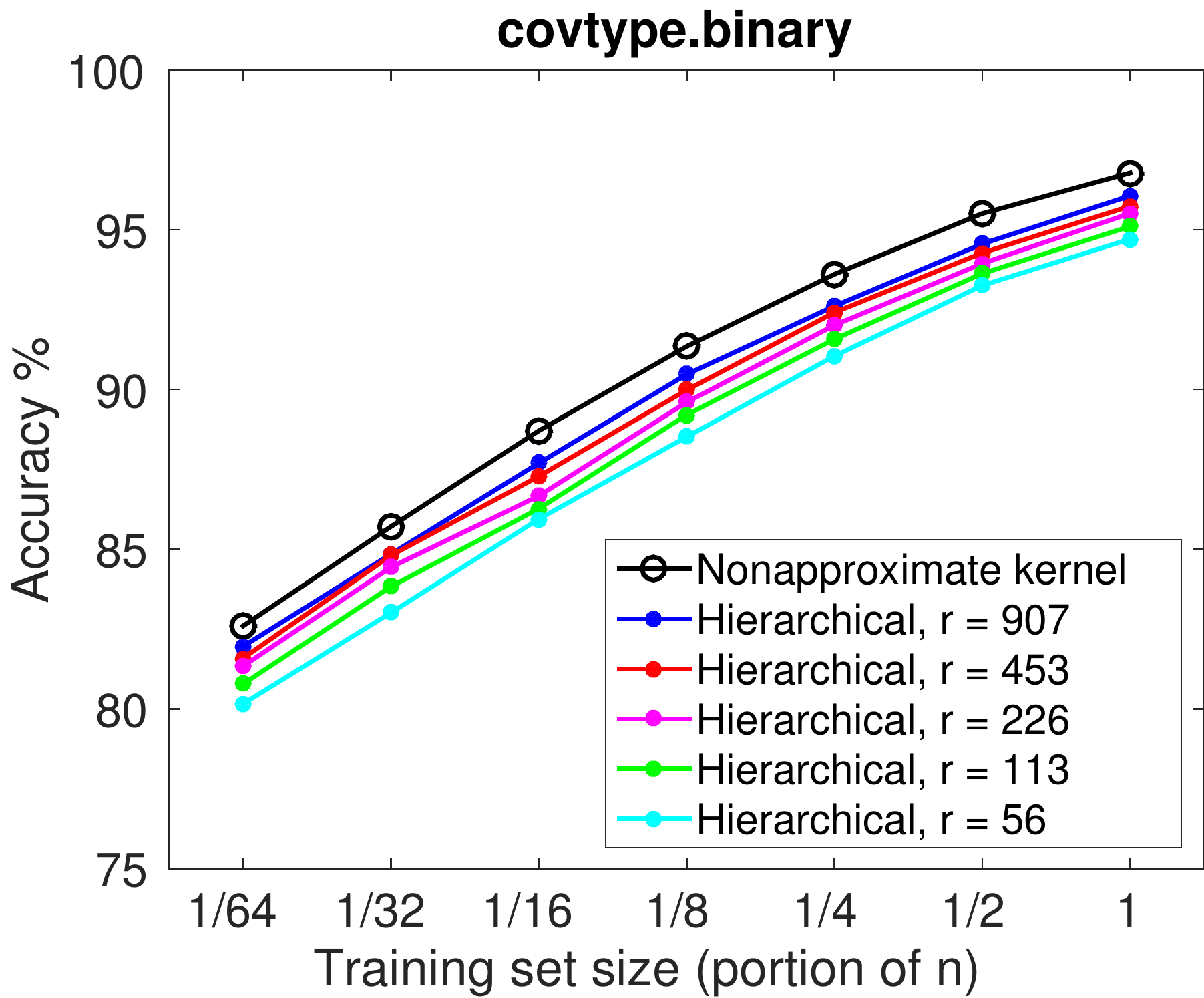}}
\caption{Performance under different training sizes and $r$. Kernel: Gaussian.}
\label{fig:ZZ_plot_exp_5_n_vs_r}
\end{figure}

The question does not seem to bare a clear and simple answer after one investigates the two plots in Figure~\ref{fig:ZZ_plot_exp_5_n_vs_r}. Commonly, the performance of the proposed kernel improves in a consistent pace as $r$ increases. For covtype.binary, the curves approach that of the nonapproximate kernel; however, such is not the case for YearPredictionMSD. In this data set, when the training size is small, increasing $r$ may surpass the performance of the nonapproximate kernel. Furthermore, for covtype.binary, increasing the training size clearly yields a much higher performance than does increasing $r$ by the same factor; however, for YearPredictionMSD, increasing $r$ is to the contrary more beneficial. Hence, what the trade-off between $n$ and $r$ favors appears to be data set dependent.

\subsection{Kernel PCA}\label{sec:kpca}

Kernel principal component analysis (kernel PCA; see~\citet{Schoelkopf1998}) is another popular application of kernel methods. The standard PCA defines the embedding of a point $\bm{x}$ as the projected coordinates along the principal components of the training set $X$; whereas kernel PCA defines the embedding as the projected coordinates of $\phi(\bm{x})$ along the principal components of $\phi(X)$, where $\phi$ denotes the mapping from the input space to the feature space. For low-rank kernels (e.g., Nystr\"{o}m and Fourier) that give the explicit feature map $\phi$, kernel PCA may be straightforwardly computed through singular value decomposition of the feature points $\phi(X)$ after centering. For other kernels (e.g., cross-domain independent kernel and the proposed kernel), one may leverage the relation $k(\bm{x},\bm{x}')=\inprod{\phi(\bm{x})}{\phi(\bm{x}')}$ and compute the embedding through eigenvalue decomposition of the kernel matrix $K(X,X)$ after centering.

\begin{figure}[ht]
\centering
\subfigure[cadata, $n=16,512$]{
\includegraphics[width=.47\linewidth]{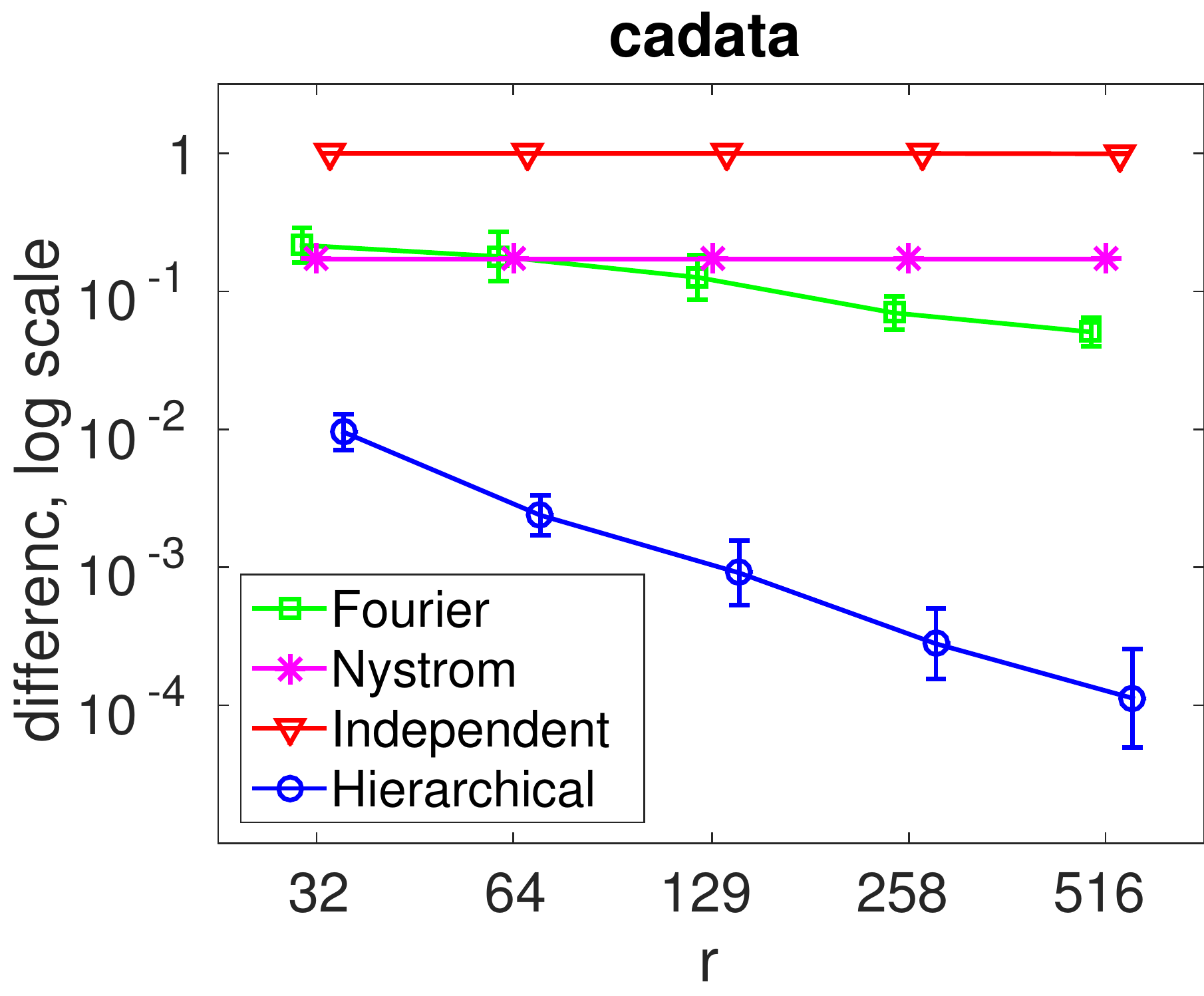}}
\subfigure[ijcnn1, $n=35,000$]{
\includegraphics[width=.47\linewidth]{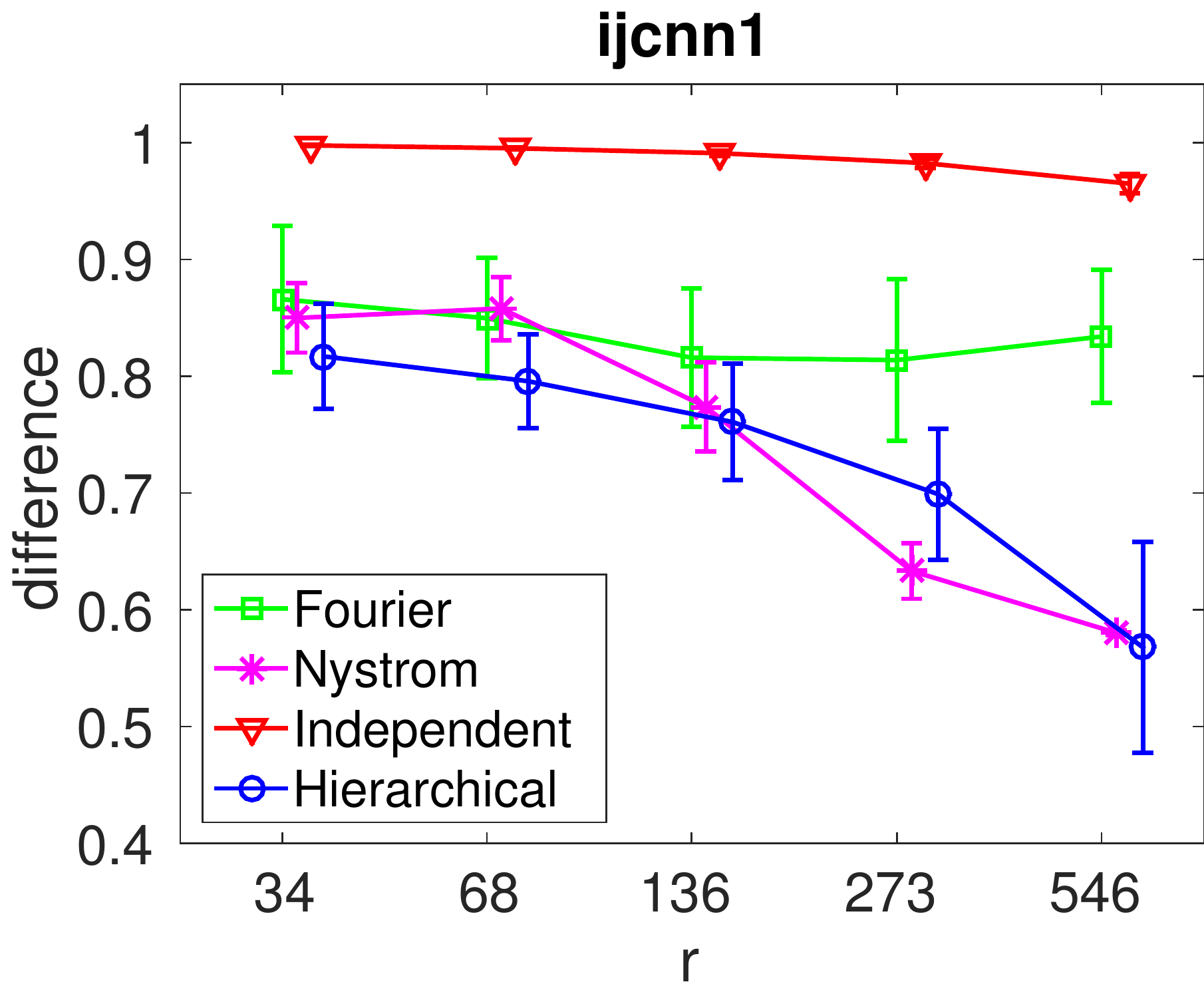}}
\caption{Kernel PCA: Embedding alignment difference (embedding dimension $=3$). The horizontal placement of each curve is slightly jittered to avoid cluttering.}
\label{fig:ZZ_plot_exp_8_kpca}
\end{figure}

We follow~\citet{Zhang2008} and evaluate the embedding quality of different kernels against that of the base kernel. Specifically, let $U$ be the embedding matrix, where each row corresponds to the embedding of one point $\bm{x}$; similarly, let $\widetilde{U}$ be the embedding matrix resulting from an approximate kernel. We use a matrix $M$ to align $\widetilde{U}$ with $U$; that is, $M$ minimizes $\|U-\widetilde{U}M\|_F$. Then, we show in Figure~\ref{fig:ZZ_plot_exp_8_kpca} the alignment difference $\|U-\widetilde{U}M\|_F/\|U\|_F$ for different approximate kernels with various (rank) $r$. The embedding dimension is fixed to $3$ and the base kernel is Gaussian, with the choice of bandwidth $\sigma$ approximately equal to the optimum found in the preceding experiments. One sees that generally the proposed kernel results in the smallest alignment difference.

\section{Concluding Remarks}\label{sec:conclude}
Kernel methods enjoy a strong theoretical support for data interpretation and predictive modeling; however, they come with a prohibitive computational price. Substantial efforts exist in the search of approximate kernels. Low-rank kernels are one of the most popular approaches because their costs have a linear dependency on the data size $n$ and because of their relatively simple linear algebra implementation. From the matrix angle, the effectiveness of a low-rank approximation, however, depends heavily on the decay of eigenvalues, which is often found to be insufficiently fast for large data. We may interpret that low-rank approximations tackle the problem globally, because its purpose is to approximate the whole matrix. On the other hand, block diagonal approximations (e.g., the cross-domain independent kernel) focus on the preservation of local information only. They may approximate very poorly the kernel matrix, but they often produce surprisingly good results for regression and classification (see the covtype example in Section~\ref{sec:exp}). Hence, the hierarchically compositional kernel proposed in this paper is motivated by the separate characteristics of these approximations: it preserves the full information in local domains and gradually applies low-rank approximations when the domains expand in a hierarchical fashion. The net effect is that the kernel admits an $O(nr)$ storage and an $O(nr^2)$ arithmetic cost similar to that of others, but the factor $r$ may be reduced for a matching performance when $n$ is large. This is favorable when computational resources are constrained; that is, as the data size $n$ increases, the additional factor $r$ must remain small in order for computation to be affordable.

The computational efficiency of the proposed kernel is supported by a delicate matrix structure. The resulting kernel matrix $K_{\text{hierarchical}}$ is a special case of the \emph{recursively low-rank compressed matrices} studied in~\citet{Chen2014a}. Such a hierarchical structure allows many otherwise $O(n^2)$--$O(n^3)$ dense matrix operations done with an $O(nr)$--$O(nr^2)$ cost. Examples of these matrix operations are matrix-vector multiplication, matrix inversion, log-determinant calculation~\citep{Chen2014a}, and square-root decomposition~\citep{Chen2014}. In addition to the focus of this paper---regression and classification, such matrix operations are crucial in many machine learning and statistical settings, including Gaussian process modeling, simulation of random processes, likelihood calculation, Markov Chain Monte Carlo, and Bayesian inference. The proposed kernel hints a direction for performing scalable computations facing an ever increasing $n$.

The recursively low-rank compressed structure is not a unique discovery in the prosperous literature of structured matrix computations. Many practically useful structures in scientific computing (e.g., FMM matrices~\citep{Sun2001} and hierarchical matrices~\citep{Hackbusch1999}) exploit the same low-rank property exhibited in the off-diagonal blocks of the matrix, but they differ from ours in fine details, including the admissibility of low-rank approximation, the method of approximation, the nesting of basis, the shape of the hierarchy tree, and the supported matrix operations. Apart from these differences, the most distinguishing feature of our work is that the matrix structure is amenable for high dimensional data, an arena where other structured matrix methods are rarely applicable. This peculiarity roots in the fact that the decay of singular values of the off-diagonal blocks becomes much slower when the data dimension increases. These blocks generally have a full, or nearly full, rank when the data dimension is larger than three or four. Hence, it is challenging to develop computational methods from the angle of matrix approximations, if evaluation of quality is based on the approximation error. Instead, we develop theory from the perspective of kernels and ensure the preservation of positive-definiteness. Although the new kernel arises from some approximation of the old one, we may bypass the approximation interpretation and consider them separate kernels. The choice of kernels will then be justified via model selection. An additional benefit is that out-of-sample extension is straightforward.

Although the setting of this paper casts the data domain $\mathcal{X}$ as an Euclidean space for simplicity, it is possible to generalize the proposed method to a more abstract space (e.g., a metric space). In fact, the construction of the proposed kernel and the instantiation of the kernel matrix requires as minimal as the fact that the base kernel $k$ is strictly positive-definite. The only exception is the partitioning of the data domain, where the partition membership must be inferred for every data point (including out of samples). In a metric space, one solution is the k-means clustering, because the clustering straightforwardly defines the partitions for the training data; moreover, the partition membership for a testing point is determined by the minimal distance (metric) from the point to the cluster centers.

Through experimentation, we have demonstrated the complex performance curves with respect to kernel parameters. The curves may vary significantly due to randomization, may be multimodal, and may even be nonsmooth. These phenomena pose a challenge for selecting the optimal parameters through cross validation and grid search. To make things worse, grid search is applicable only when the number of parameters is very small. An example when one may want more parameters is to introduce anisotropy to the kernel through specifying one range parameter for each or a few attributes. Hence, a more principled approach, which in theory can incorporate an arbitrary number of parameters, is to take the GP view and to maximize the Gaussian log-likelihood
\begin{equation}\label{eqn:loglik}
L(\bm{\theta})=-\frac{1}{2}\bm{y}^TK(\bm{\theta})^{-1}\bm{y}-\frac{1}{2}\log\det K(\bm{\theta})-\frac{n}{2}\log2\pi
\end{equation}
through numerical optimization, where $\bm{\theta}$ is the vector of unknown parameters and $K(\bm{\theta})$ is the kernel matrix depending on $\bm{\theta}$. This approach, coined \emph{maximum likelihood estimation} (MLE), is also a central subject in estimation theory. Optimizing~\eqref{eqn:loglik} nevertheless is nontrivial because of the $O(n^3)$ cost for evaluating the objective function and its derivatives when $K$ is dense. Existing research~\citep{Anitescu2012,Stein2013} solves an approximate problem and bounds the variance of the result against that of the original problem~\eqref{eqn:loglik}. Fortunately, most of the approximate kernels discussed in this paper allow an $O(nr^2)$ cost for evaluating $L(\bm{\theta})$. For example, the algorithm for computing the log-determinant term for a recursively low-rank compressed matrix is described in~\citet{Chen2014a}. An avenue of future work is to adapt the log-determinant calculation discussed in~\citet{Chen2014a} to the proposed kernel and to develop robust optimization for parameter estimation.

\bibliographystyle{abbrvnat}
\bibliography{reference}

\appendix
\section{Proof of Theorem~\ref{thm:pd.2}}\label{sec:proof.pd.2}
The proof technique is similar to that of Theorem~\ref{thm:pd}: we decompose the kernel $k_{\text{hierarchical}}$ as a sum of positive-definite terms and show that if $k$ is strictly positive-definite, then the bilinear form cannot be zero if the coefficients are not all zero. To prevent cluttering, we use the short-hand notation $K_{\ud{i},\ud{j}}$ to denote $K(\ud{X}_i,\ud{X}_j)$.

For any node $i$ with parent $p$ (if any), define a function $\xi^{(i)}$ for $\bm{x},\bm{x}'\in S_i$:
\begin{equation}\label{eqn:xi}
\xi^{(i)}(\bm{x},\bm{x}')=
\begin{cases}
k(\bm{x},\bm{x}')-k(\bm{x},\ud{X}_p)K_{\ud{p},\ud{p}}^{-1}k(\ud{X}_p,\bm{x}'), & i \text{ is leaf},\\
\psi^{(i)}(\bm{x},\ud{X}_i)K_{\ud{i},\ud{i}}^{-1}[K_{\ud{i},\ud{i}}-K_{\ud{i},\ud{p}}K_{\ud{p},\ud{p}}^{-1}K_{\ud{p},\ud{i}}]K_{\ud{i},\ud{i}}^{-1}\psi^{(i)}(\ud{X}_i,\bm{x}'), & i \text{ is neither leaf nor root},\\
\psi^{(i)}(\bm{x},\ud{X}_i)K_{\ud{i},\ud{i}}^{-1}\psi^{(i)}(\ud{X}_i,\bm{x}'), & i \text{ is root}.
\end{cases}
\end{equation}
Through telescoping, one sees that $k_{\text{hierarchical}}(\bm{x},\bm{x}')=k^{(\text{root})}(\bm{x},\bm{x}')$ is a sum of $\xi^{(i)}(\bm{x},\bm{x}')$ for all nodes $i$ where $\bm{x},\bm{x}'\in S_i$. Because each $\xi^{(i)}$ is clearly positive-definite, so is $k_{\text{hierarchical}}$.

We now show the strict definiteness when $k$ is strictly positive-definite. For any set of points $\{\bm{x}_i\}$ and any set of coefficients $\{\alpha_i\}$, the bilinear form $\sum_{jl}\alpha_j\alpha_lk_{\text{hierarchical}}(\bm{x}_j,\bm{x}_l)$ is zero if and only if
\begin{equation}\label{eqn:short}
\sum_{\bm{x}_j,\bm{x}_l\in S_i}\alpha_j\alpha_l\xi^{(i)}(\bm{x}_j,\bm{x}_l)=0,\quad\forall \text{ nodes } i.
\end{equation}
If~\eqref{eqn:short} holds, we shall prove the following statement:
\begin{quote}
(Q) If $i$ is not the root, it holds that $\alpha_j=0$ for all $j$ satisfying
\[
\bm{x}_j\in S_i\left\backslash\left[\bigcap_{q\in Des(i)}\ud{X}_q\cap\ud{X}_p\right]\right.,
\]
where $p$ is the parent of $i$ and $Des(i)$ means all the descendant nonleaf nodes of $i$, including $i$ itself.
\end{quote}
Once we prove this statement, we see that the only $\alpha_j$ that can possibly be nonzero are those satisfying $\bm{x}_j\in S\cap_{\text{all nodes } q}\ud{X}_q$. However, if any of these $\alpha_j$ is nonzero, then applying the third case of~\eqref{eqn:xi} on the root yields
\begin{align*}
\sum_{\bm{x}_j,\bm{x}_l\in S_{\text{root}}}\alpha_j\alpha_l\xi^{(\text{root})}(\bm{x}_j,\bm{x}_l)
&=\sum_{\bm{x}_j,\bm{x}_l\in S_{\text{root}}}\alpha_j\alpha_l\psi^{(\text{root})}(\bm{x}_j,\ud{X}_{\text{root}})K_{\ud{\text{root}},\ud{\text{root}}}^{-1}\psi^{(\text{root})}(\ud{X}_{\text{root}},\bm{x}_l)\\
&=\sum_{\bm{x}_j,\bm{x}_l\in S\cap_{\text{all nodes } q}\ud{X}_q}\alpha_j\alpha_le_j^TK_{\ud{\text{root}},\ud{\text{root}}}e_l
\ne0,
\end{align*}
contradicting~\eqref{eqn:short}. Here, $e_j$ means the $j$th column of the identity matrix. Therefore, the bilinear form $\sum_{jl}\alpha_j\alpha_lk_{\text{hierarchical}}(\bm{x}_j,\bm{x}_l)$ cannot be zero if the coefficients are not all zero.

We now prove statement (Q) by induction. The base case is when $i$ is a leaf node. Applying the same argument as that in the proof of Theorem~\ref{thm:pd} on the first case of~\eqref{eqn:xi}, we have that $\alpha_j=0$ for all $j$ satisfying $\bm{x}_j\in S_i\backslash\ud{X}_p$. In the induction step, if (Q) holds for all children nodes $i$ of $p$ and if $r$ is the parent of $p$, we apply the second case of~\eqref{eqn:xi} and obtain
\begin{align}
0=\sum_{\bm{x}_j,\bm{x}_l\in S_p}\alpha_j\alpha_l\xi^{(p)}(\bm{x}_j,\bm{x}_l)
&=\sum_{\bm{x}_j,\bm{x}_l\in S_p}\alpha_j\alpha_l\psi^{(p)}(\bm{x}_j,\ud{X}_p)K_{\ud{p},\ud{p}}^{-1}[K_{\ud{p},\ud{p}}-K_{\ud{p},\ud{r}}K_{\ud{r},\ud{r}}^{-1}K_{\ud{r},\ud{p}}]K_{\ud{p},\ud{p}}^{-1}\psi^{(p)}(\ud{X}_p,\bm{x}_l)\nonumber\\
&=\sum_{\bm{x}_j,\bm{x}_l\in S_p\backslash\cap_{q\in Des(p)}\ud{X}_q}\alpha_j\alpha_le_j^T[K_{\ud{p},\ud{p}}-K_{\ud{p},\ud{r}}K_{\ud{r},\ud{r}}^{-1}K_{\ud{r},\ud{p}}]e_l.\label{eqn:long}
\end{align}
For the matrix $K_{\ud{p},\ud{p}}-K_{\ud{p},\ud{r}}K_{\ud{r},\ud{r}}^{-1}K_{\ud{r},\ud{p}}$ inside the square bracket, the block corresponding to $\ud{X}_p\backslash\ud{X}_r$ has full rank, whereas the rest is zero. Then, in order for~\eqref{eqn:long} to be zero, one must have that $\alpha_j=0$ for all $j$ satisfying $\bm{x}_j\in\ud{X}_p\backslash\ud{X}_r$. Making a disjunction of this relation with that in the induction hypothesis: $\bm{x}_j\in S_p\backslash\cap_{q\in Des(p)}\ud{X}_q$, we obtain
\[
\bm{x}_j\in S_p\left\backslash\left[\bigcap_{q\in Des(p)}\ud{X}_q\cap\ud{X}_r\right]\right.,
\]
which concludes the induction and subsequently concludes the proof of Theorem~\ref{thm:pd.2}.

\section{Performance Results with the Laplace Kernel}\label{sec:append.perf.laplace}
See Figures~\ref{fig:ZZ_plot_exp_4_laplace_1} and~\ref{fig:ZZ_plot_exp_4_laplace_2}. These plots are qualitatively similar to those of the Gaussian kernel (Figures~\ref{fig:ZZ_plot_exp_3_comprehensive_1} and~\ref{fig:ZZ_plot_exp_3_comprehensive_2}).

\begin{figure}[!ht]
\centering
\subfigure[cadata, regression]{
\includegraphics[width=.33\linewidth]{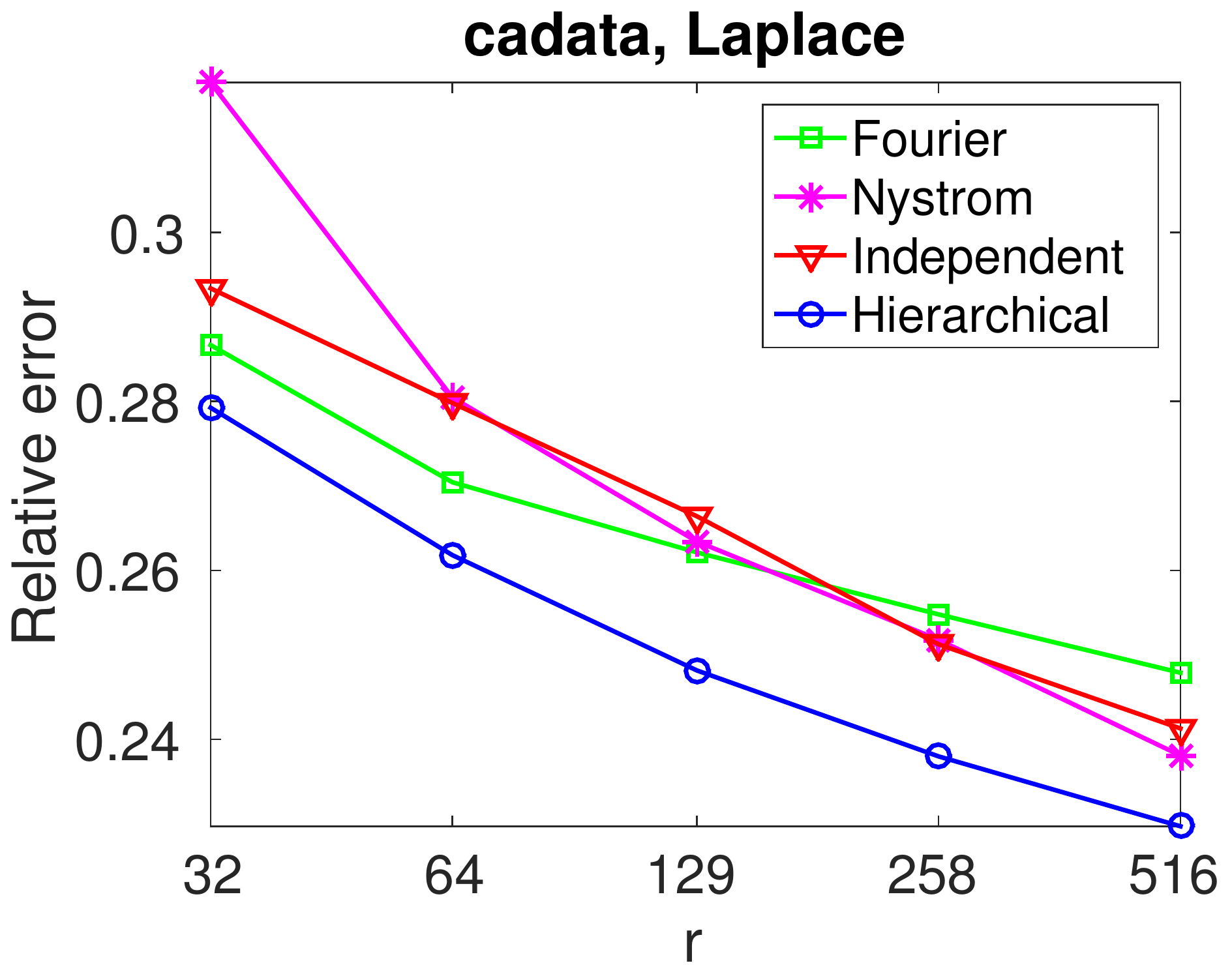}
\includegraphics[width=.32\linewidth]{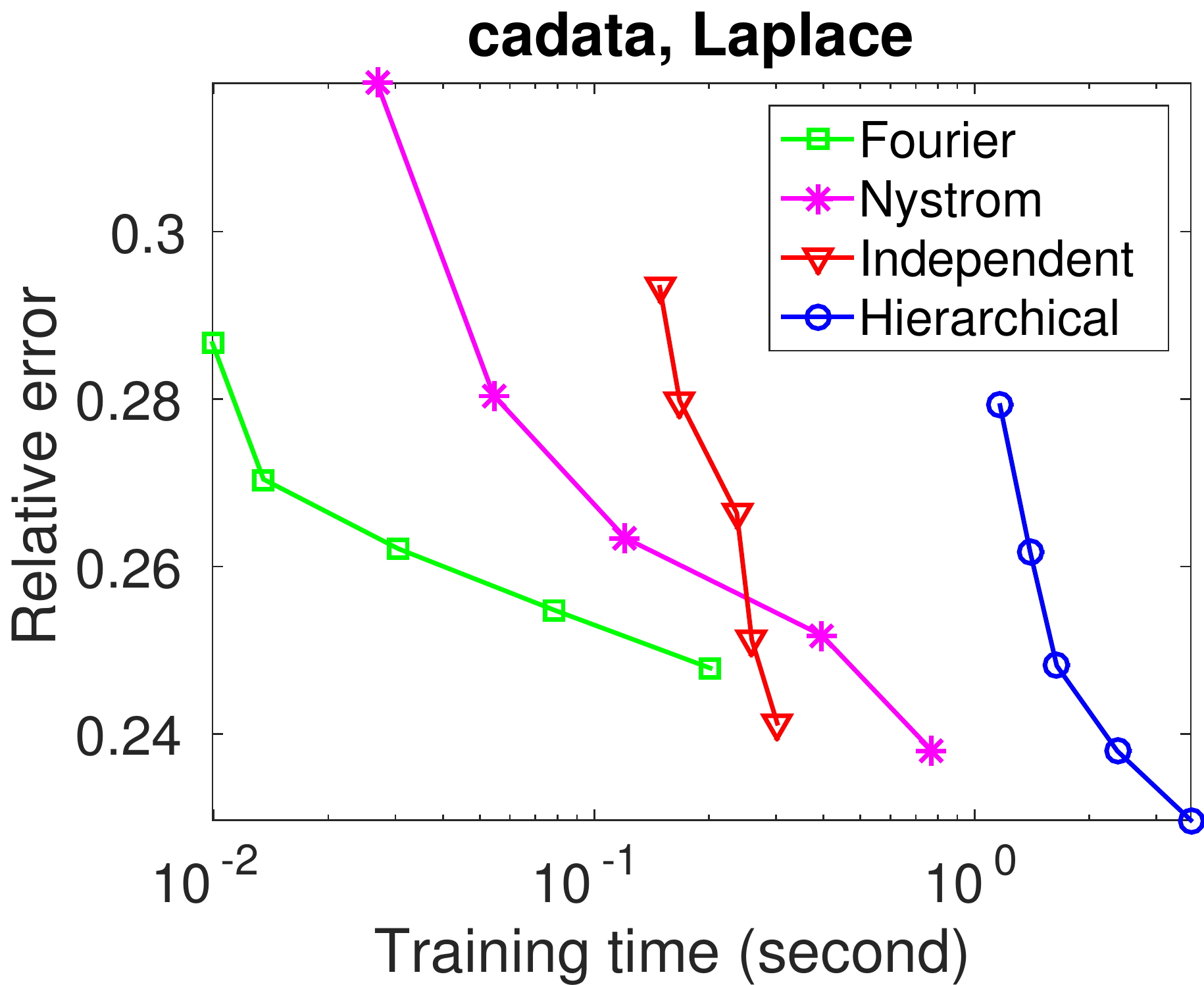}
\includegraphics[width=.33\linewidth]{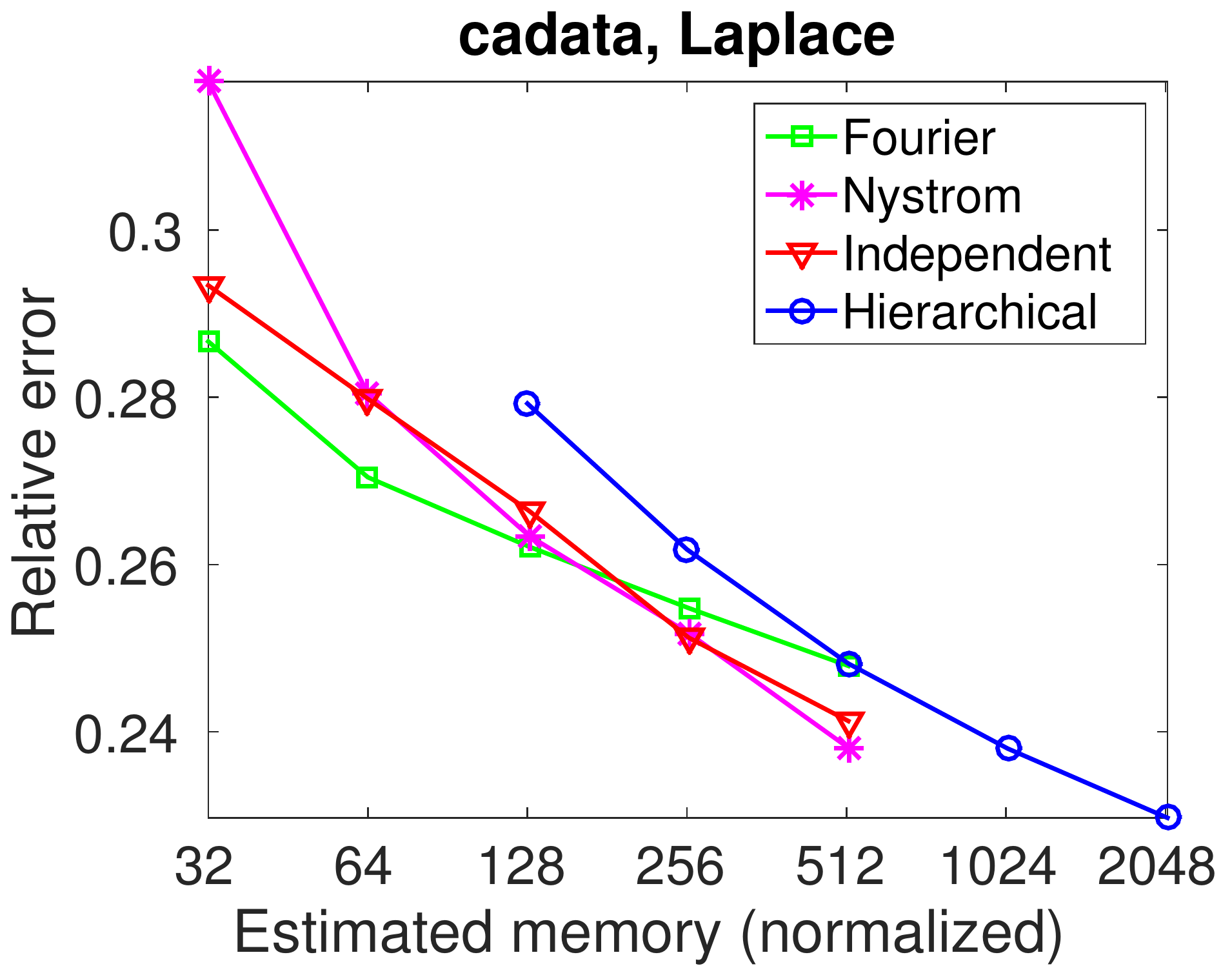}}
\subfigure[YearPredictionMSD, regression]{
\includegraphics[width=.33\linewidth]{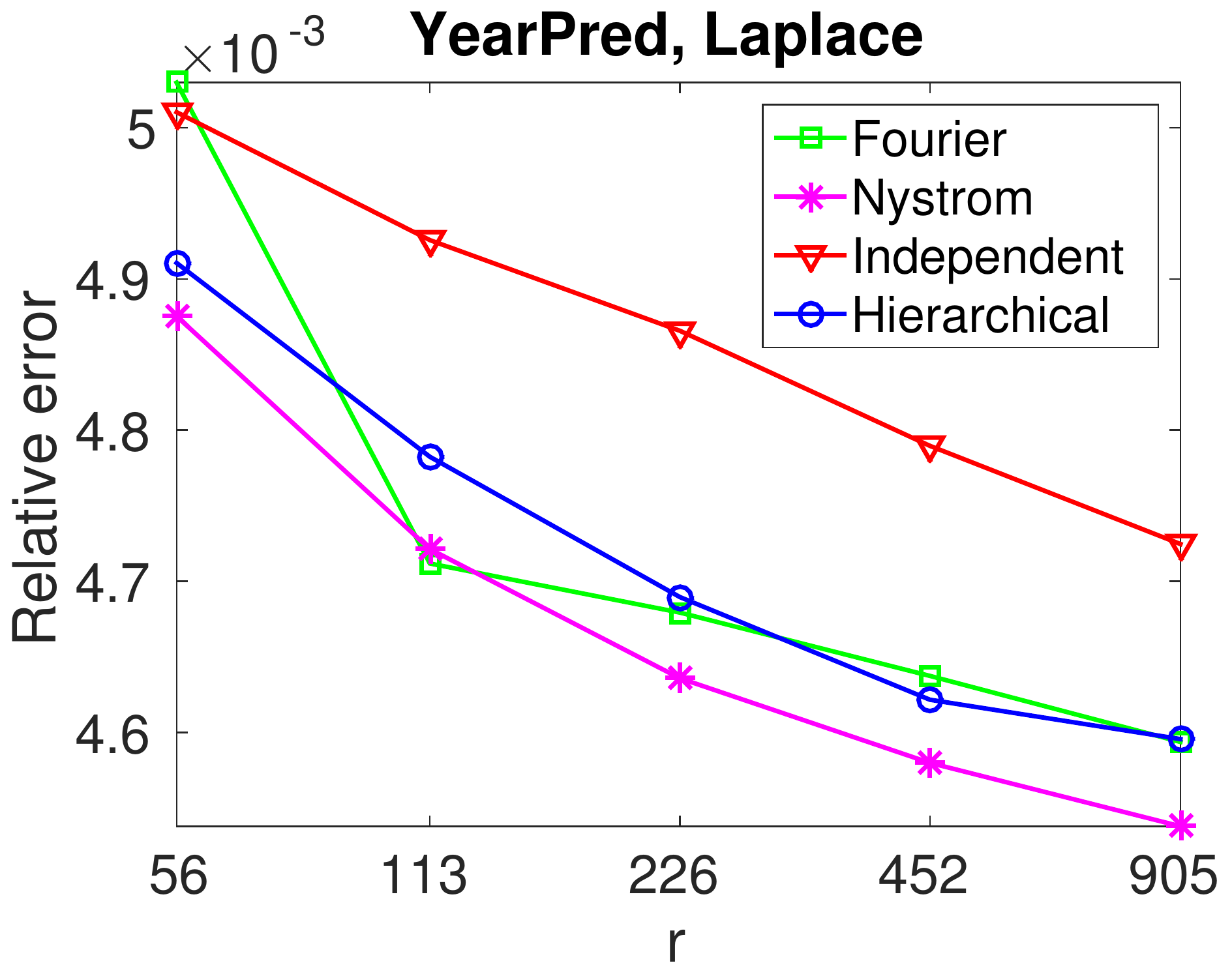}
\includegraphics[width=.32\linewidth]{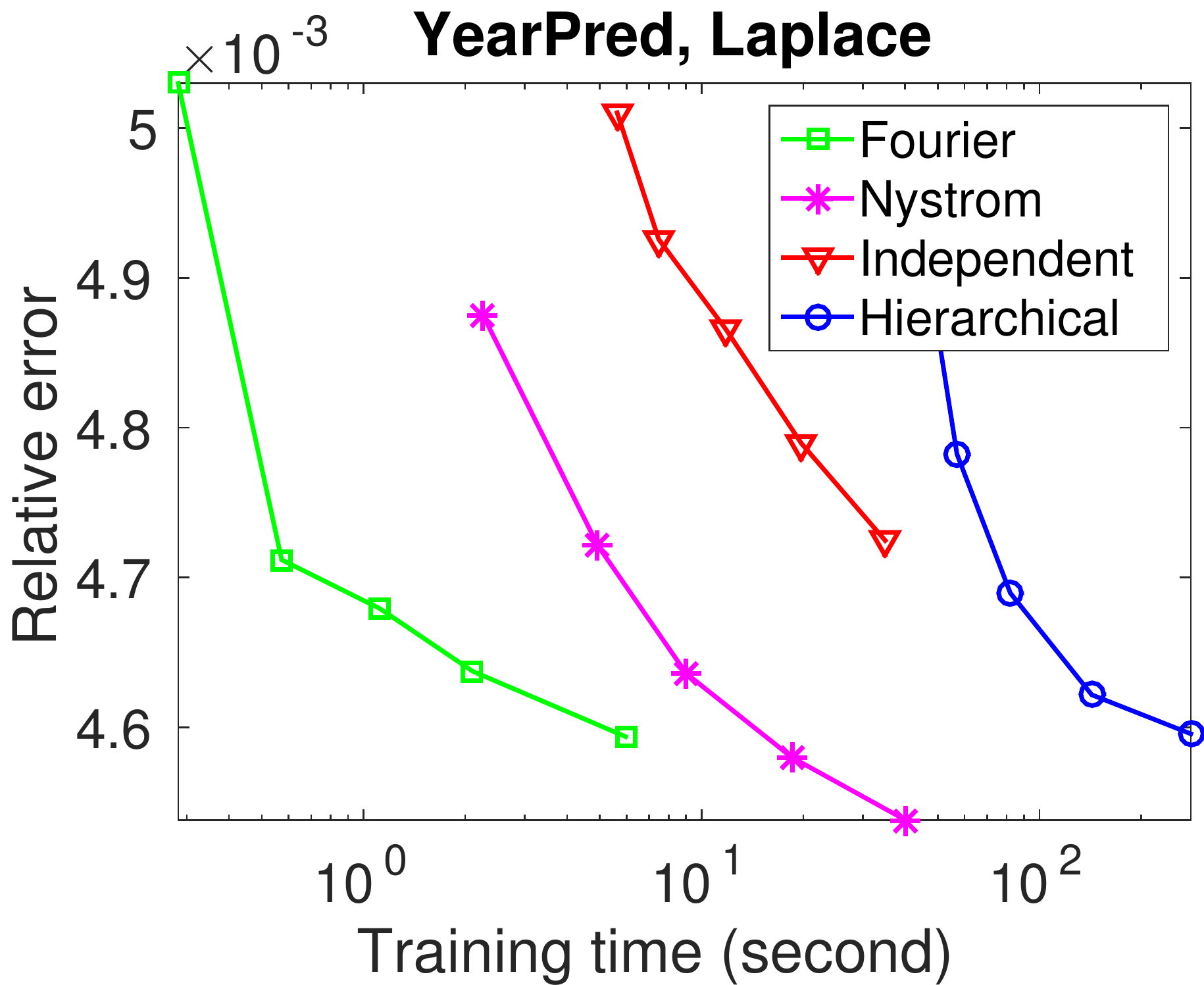}
\includegraphics[width=.33\linewidth]{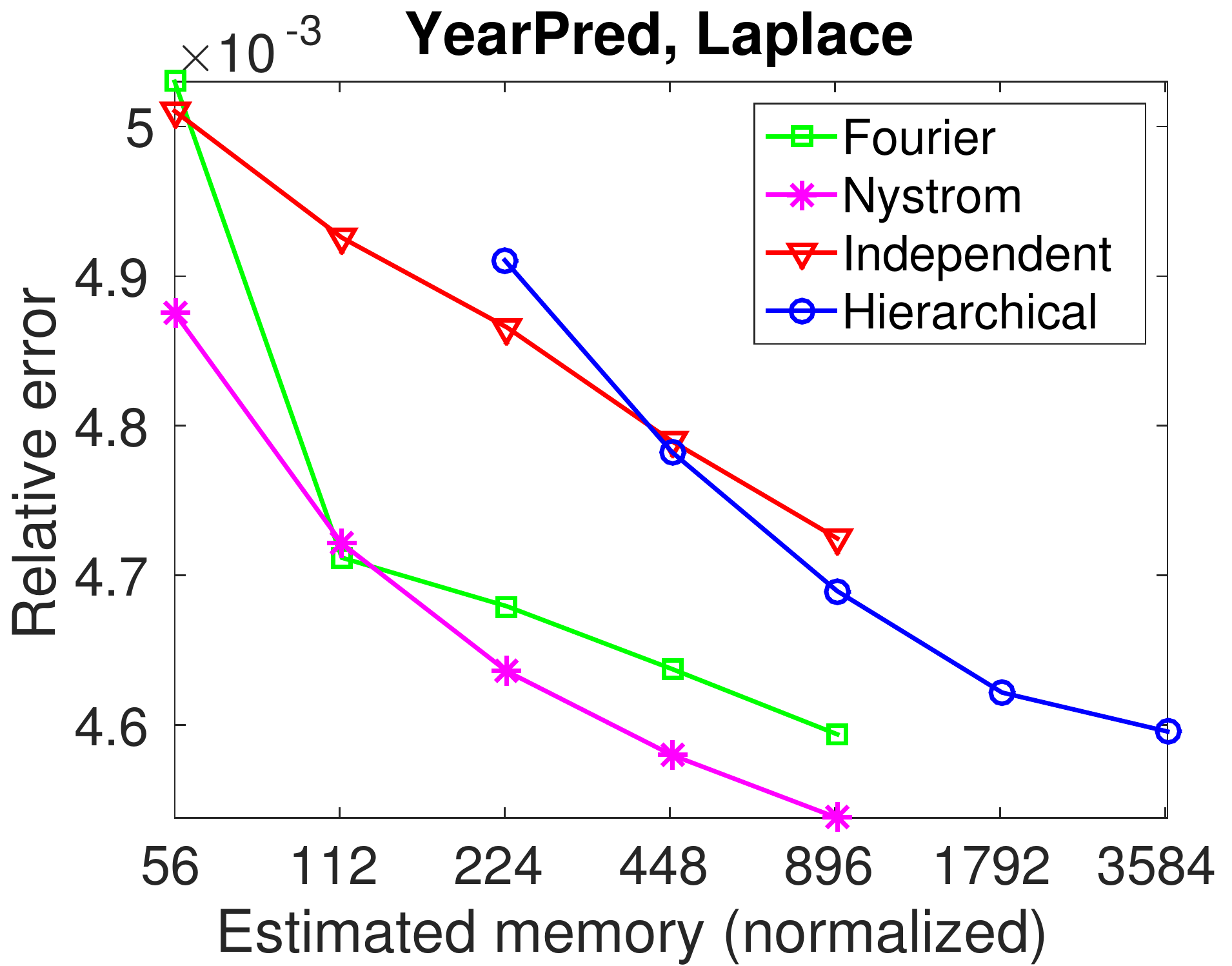}}
\subfigure[ijcnn1, binary classification]{
\includegraphics[width=.33\linewidth]{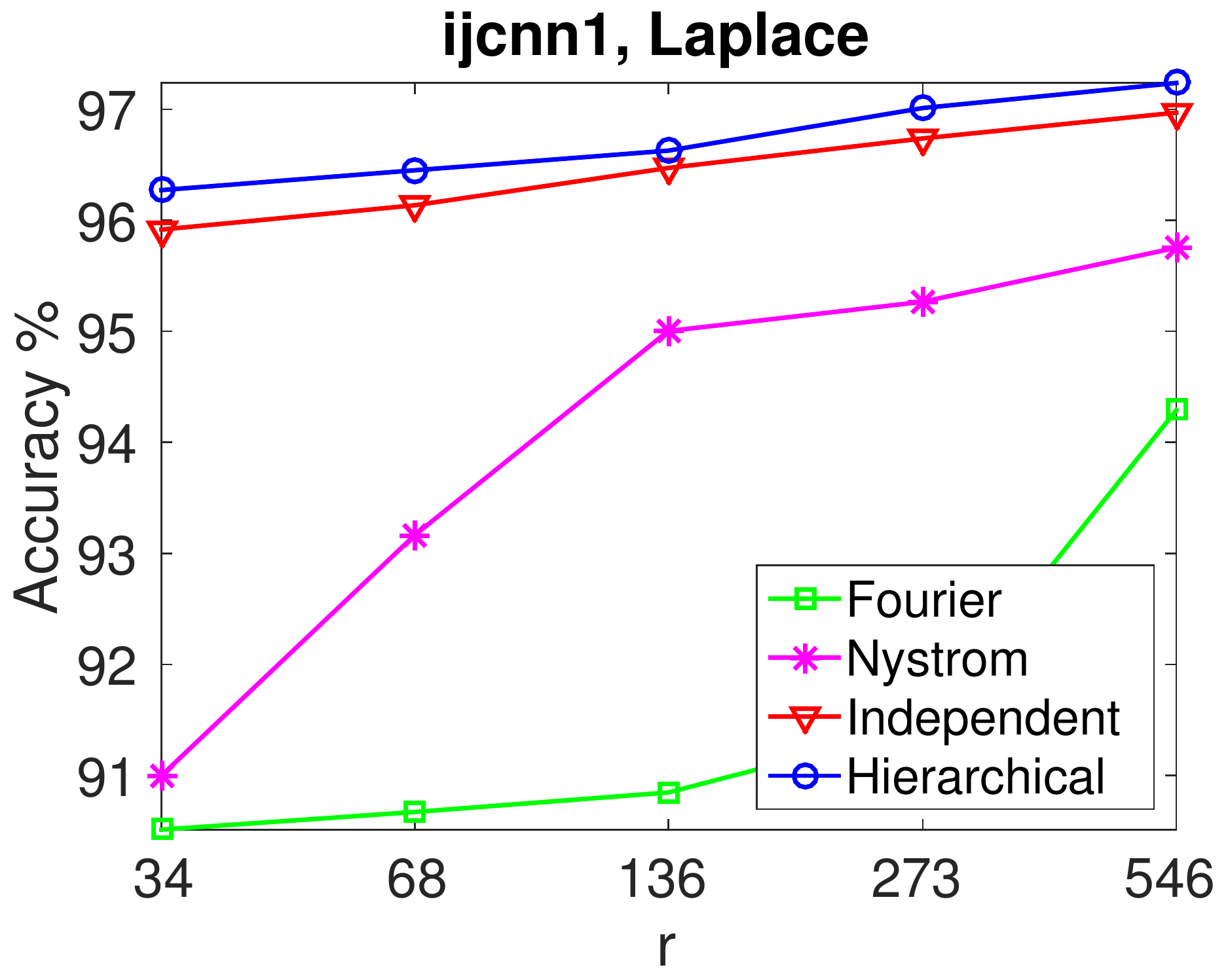}
\includegraphics[width=.32\linewidth]{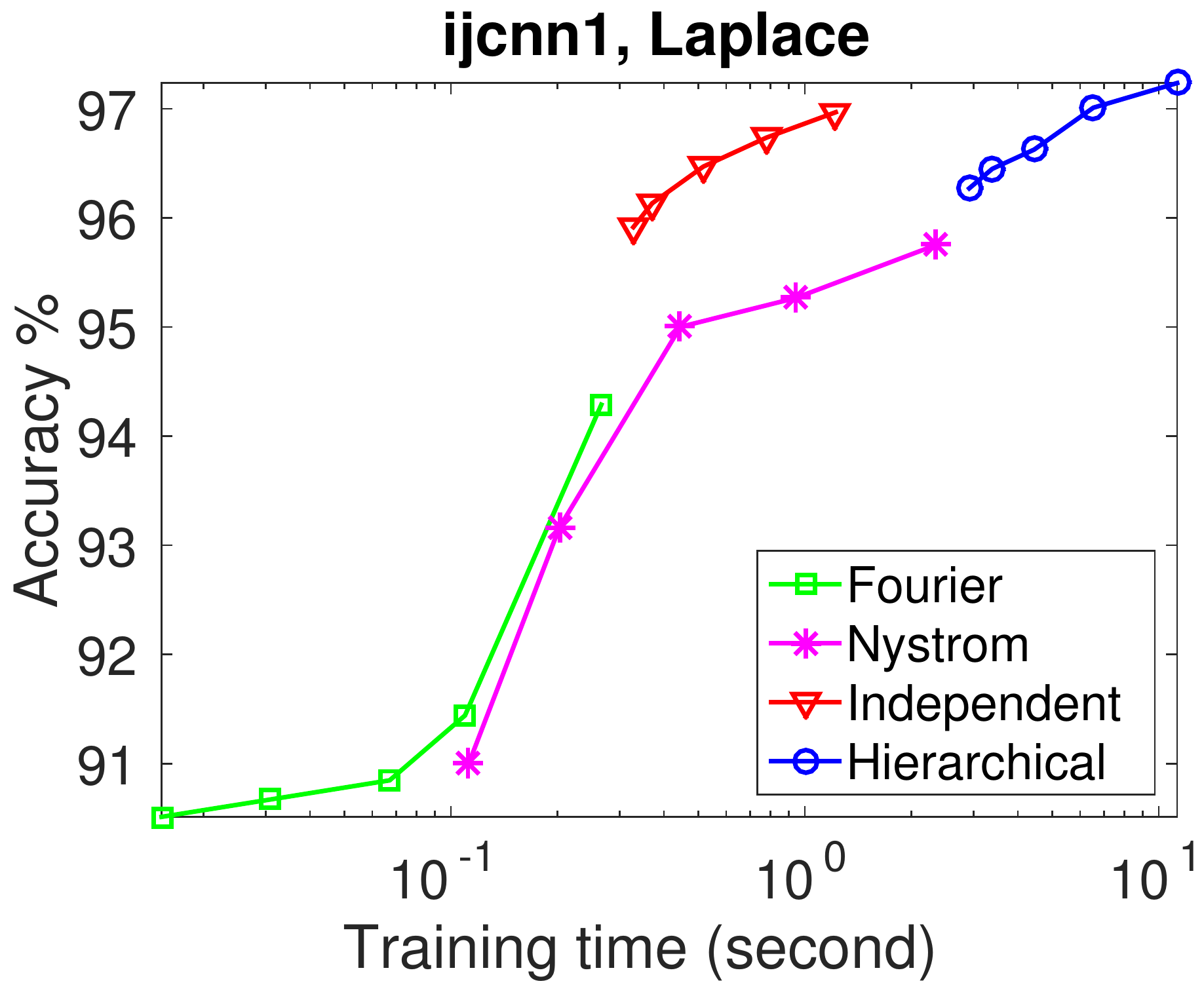}
\includegraphics[width=.33\linewidth]{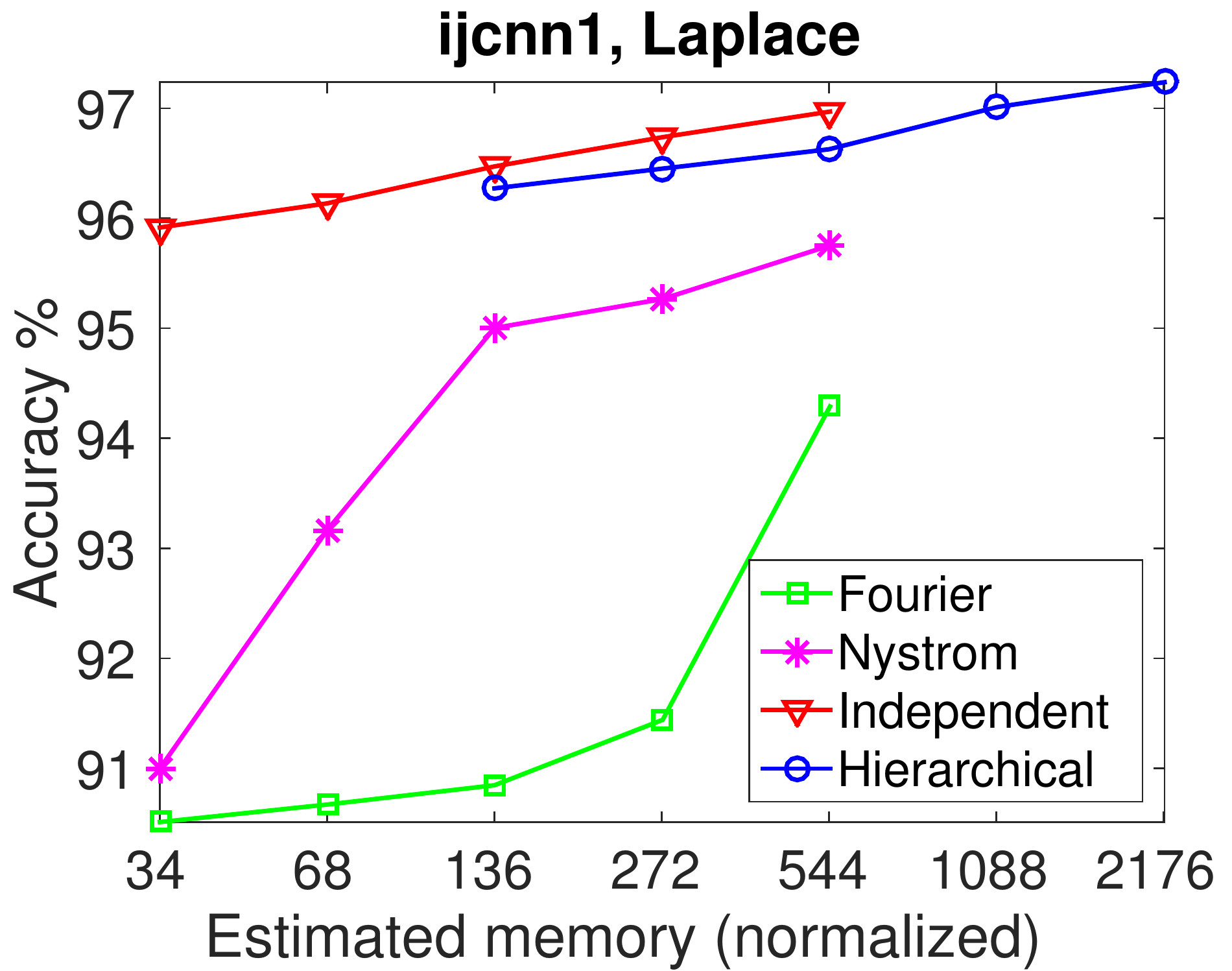}}
\subfigure[covtype.binary, binary classification]{
\includegraphics[width=.33\linewidth]{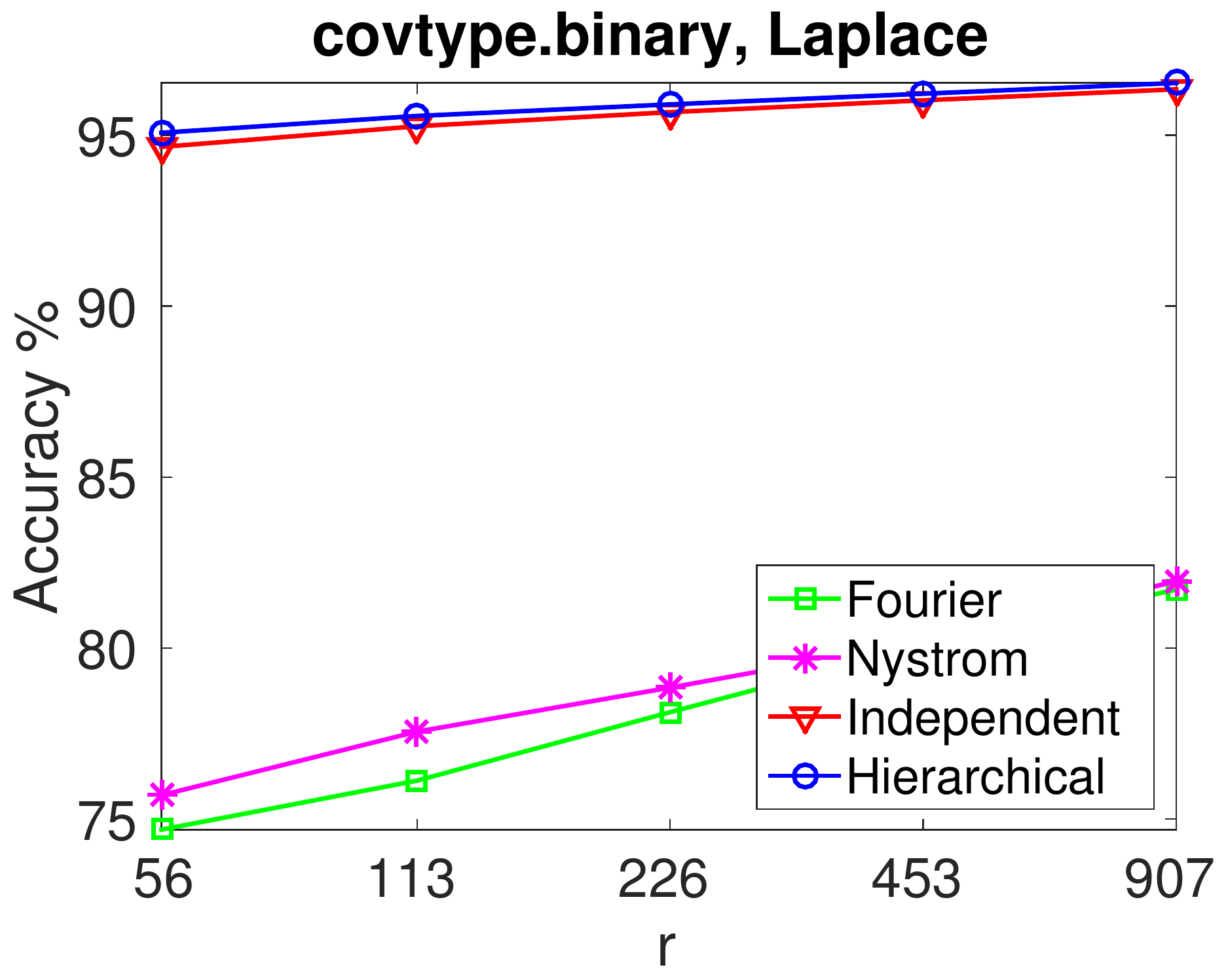}
\includegraphics[width=.32\linewidth]{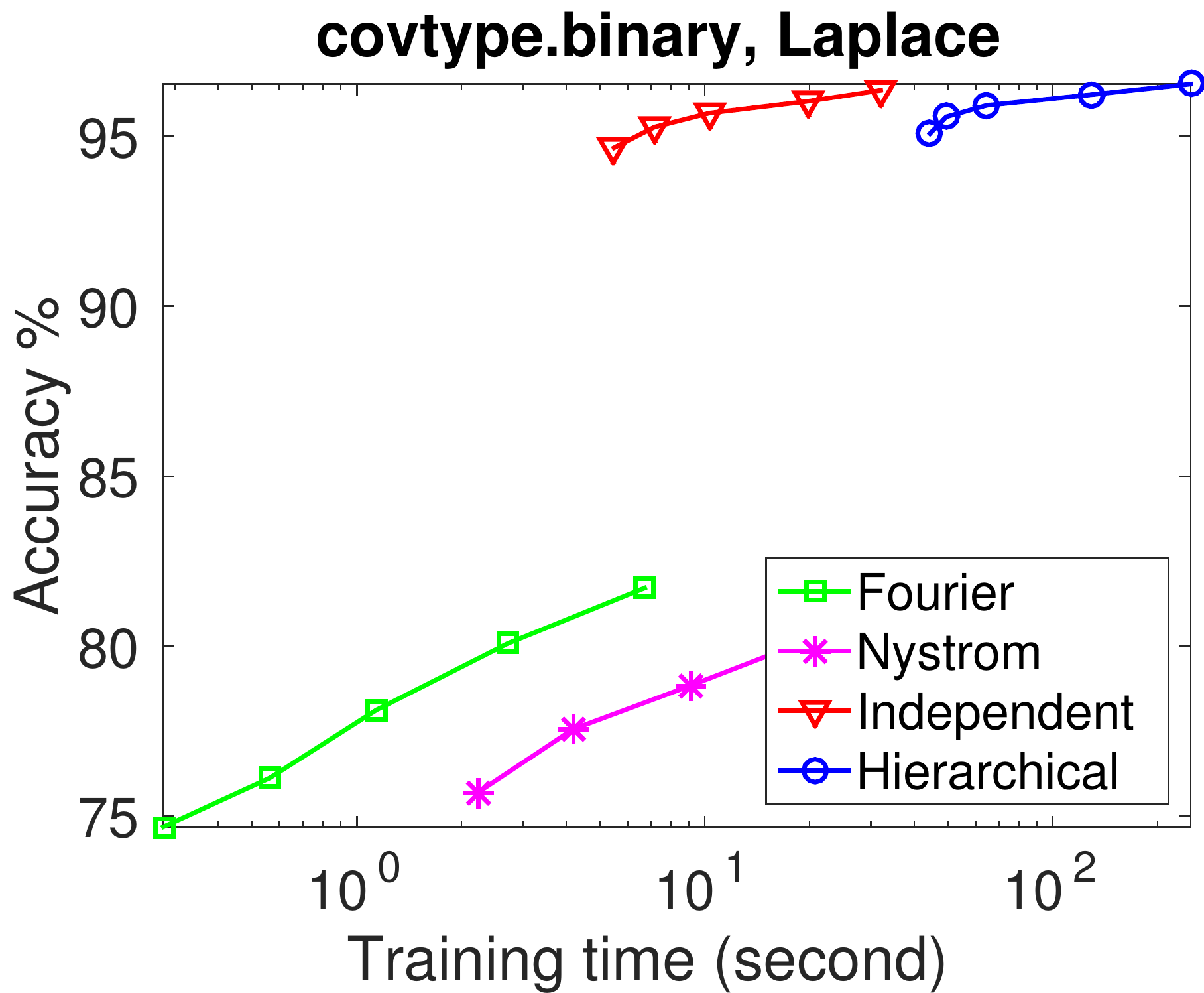}
\includegraphics[width=.33\linewidth]{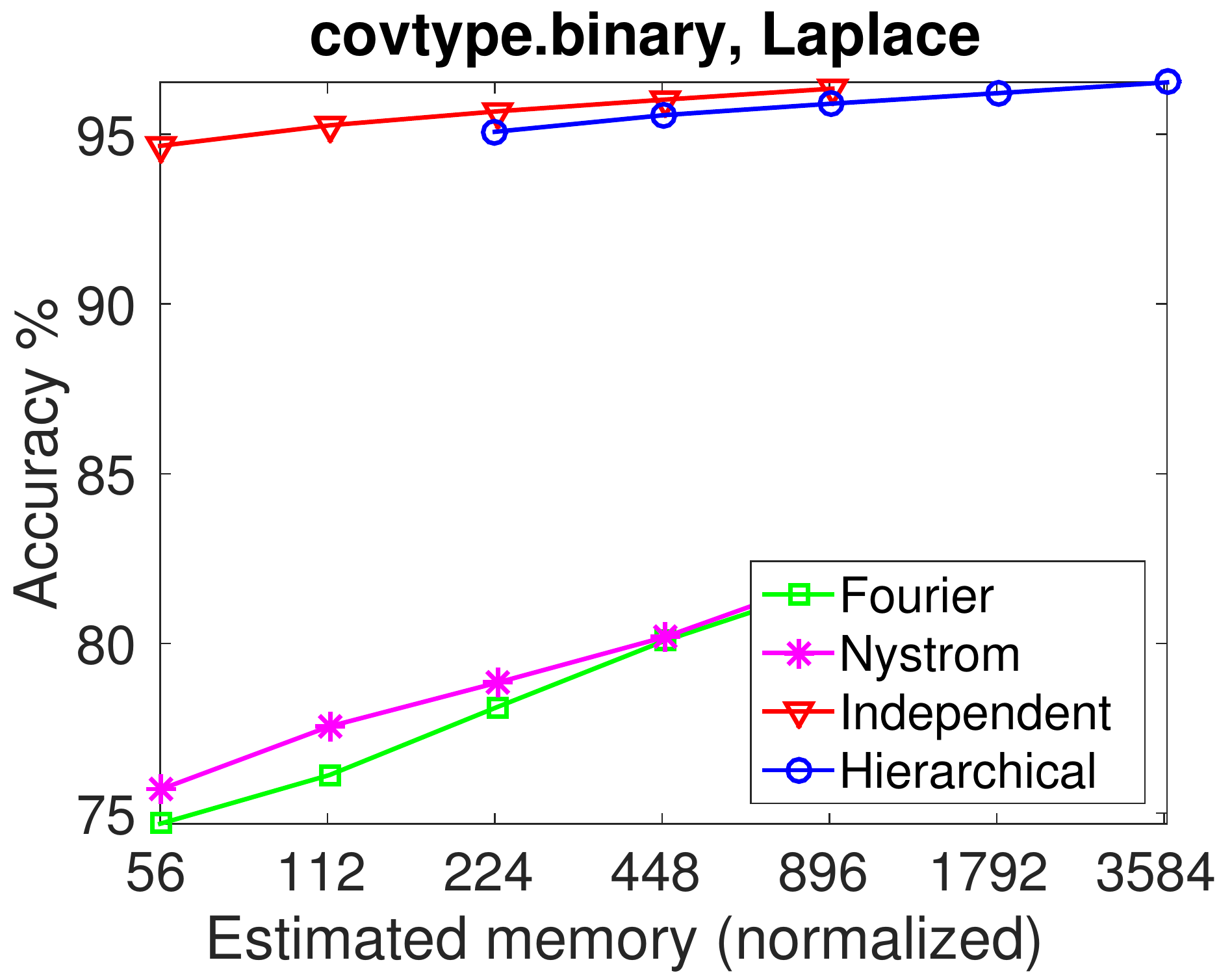}}
\caption{Performance versus $r$, time, and memory. Laplace kernel.}
\label{fig:ZZ_plot_exp_4_laplace_1}
\end{figure}

\begin{figure}[!ht]
\centering
\subfigure[SUSY, binary classification]{
\includegraphics[width=.33\linewidth]{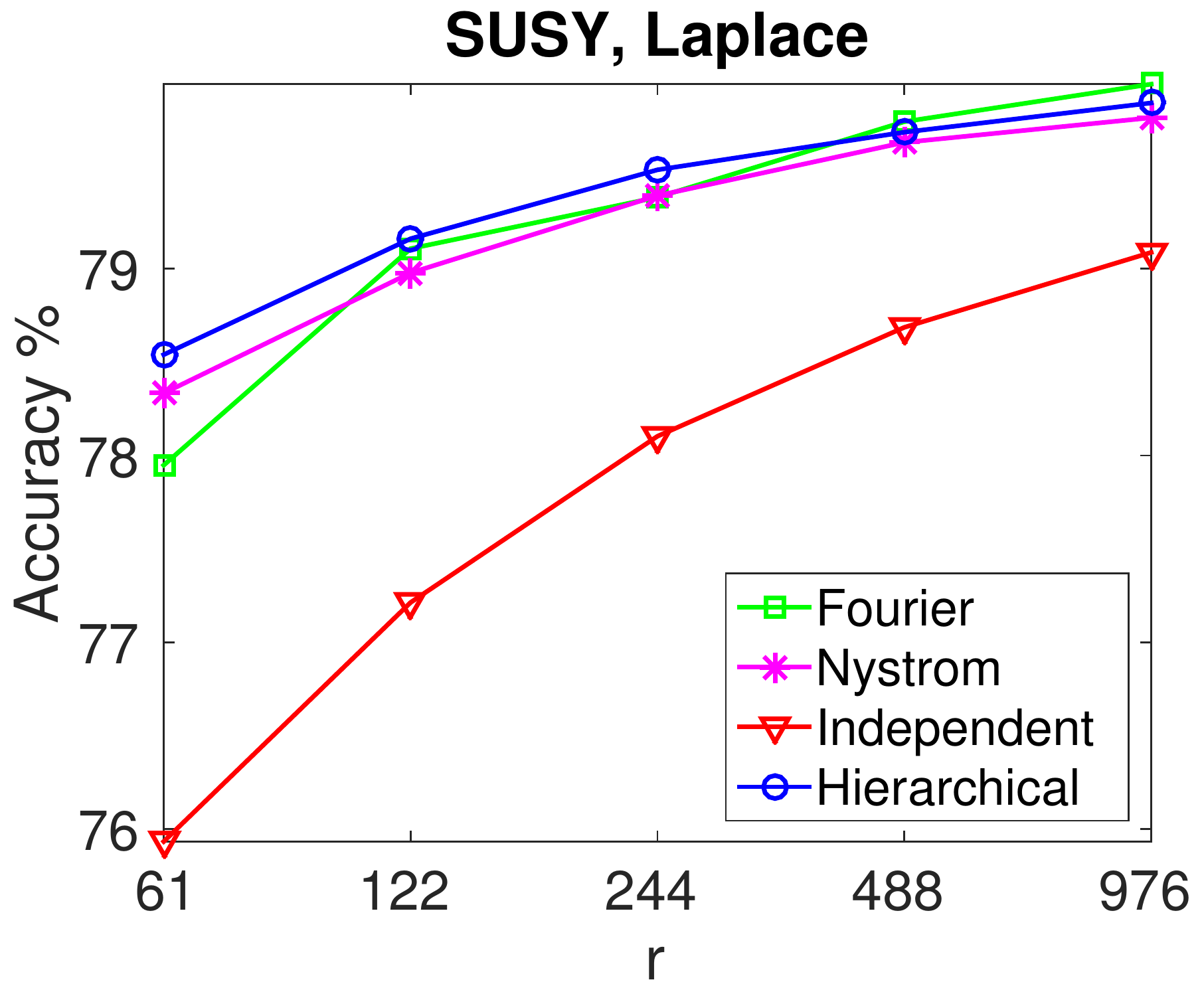}
\includegraphics[width=.32\linewidth]{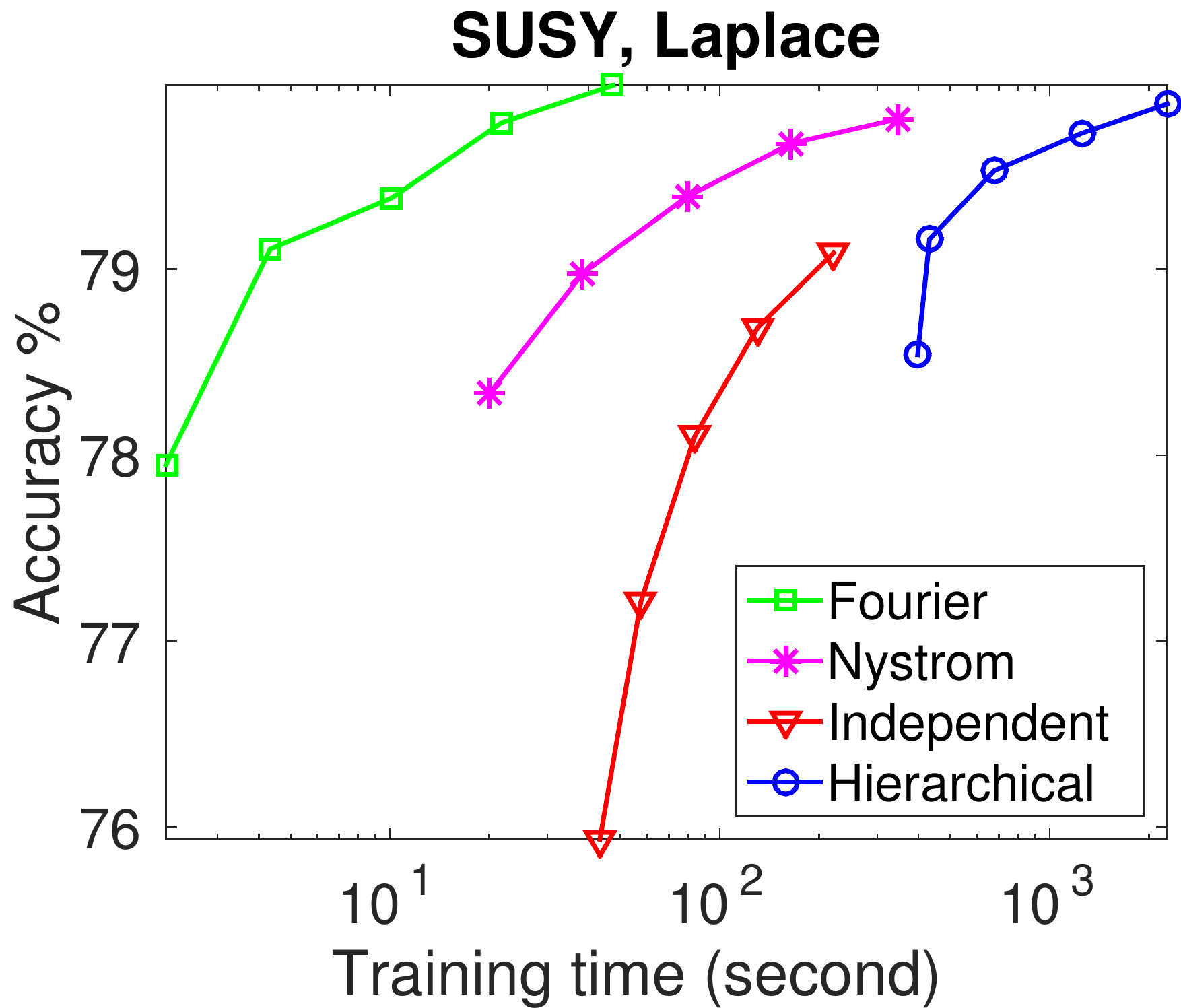}
\includegraphics[width=.33\linewidth]{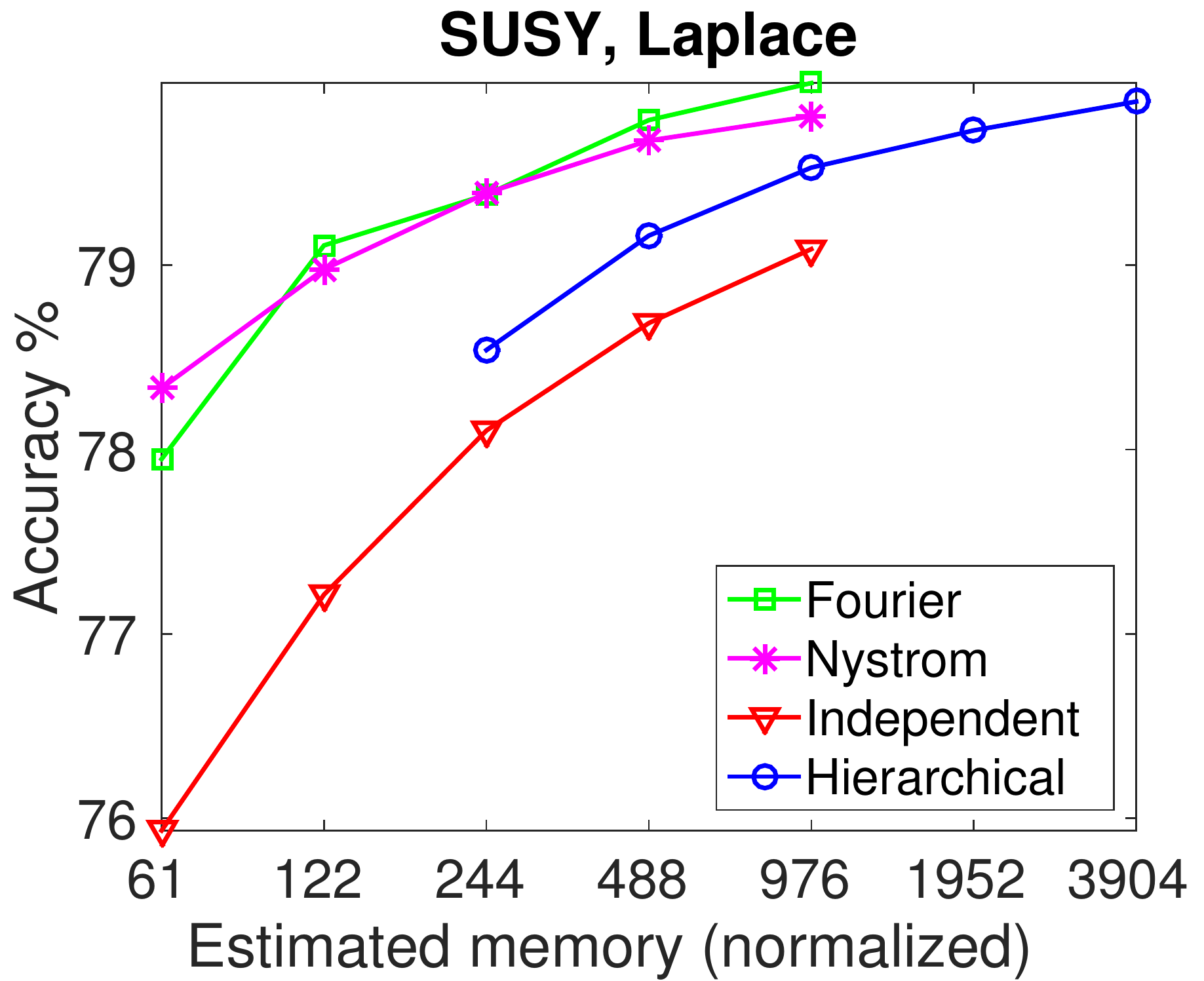}}
\subfigure[mnist, multiclass classification]{
\includegraphics[width=.33\linewidth]{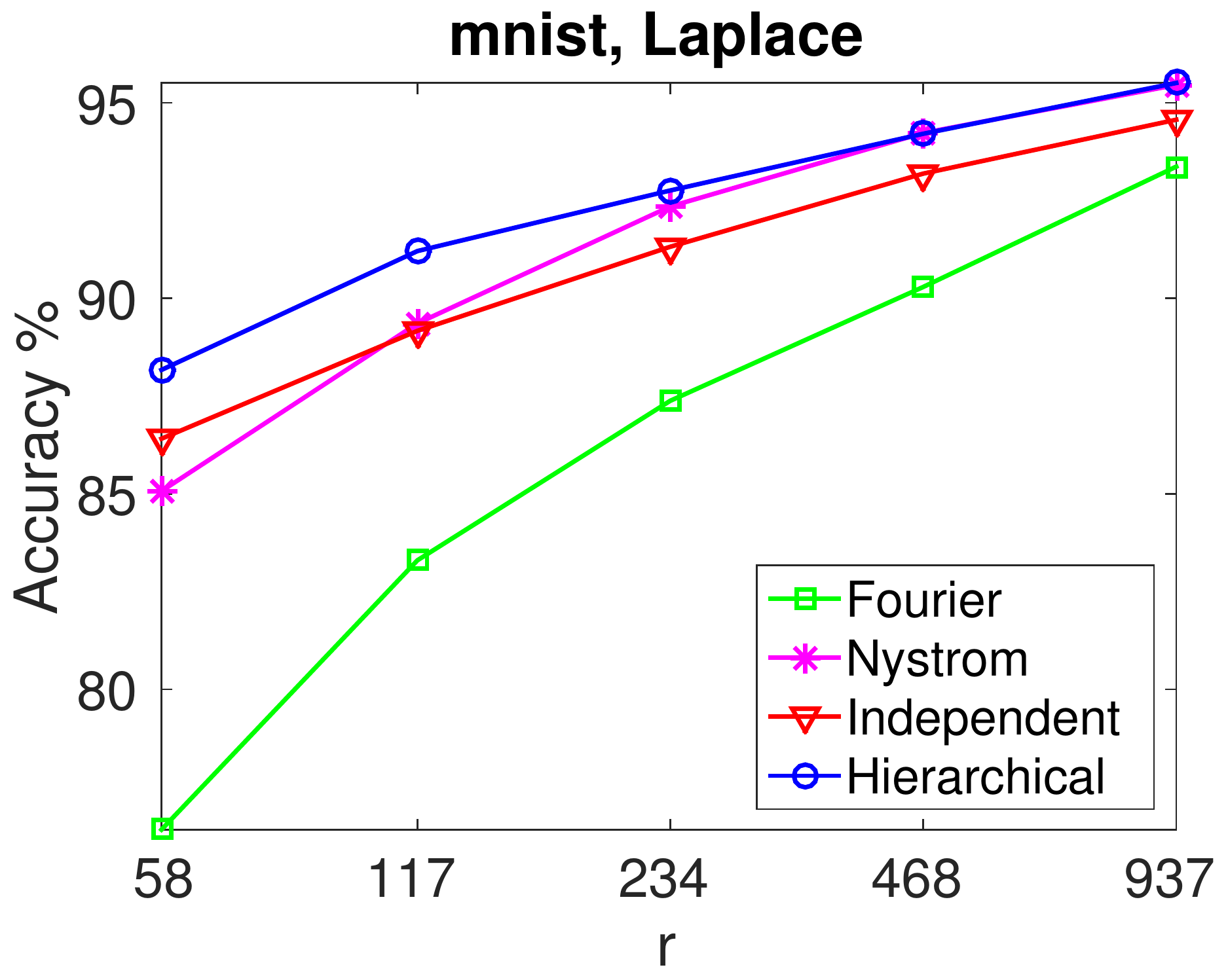}
\includegraphics[width=.32\linewidth]{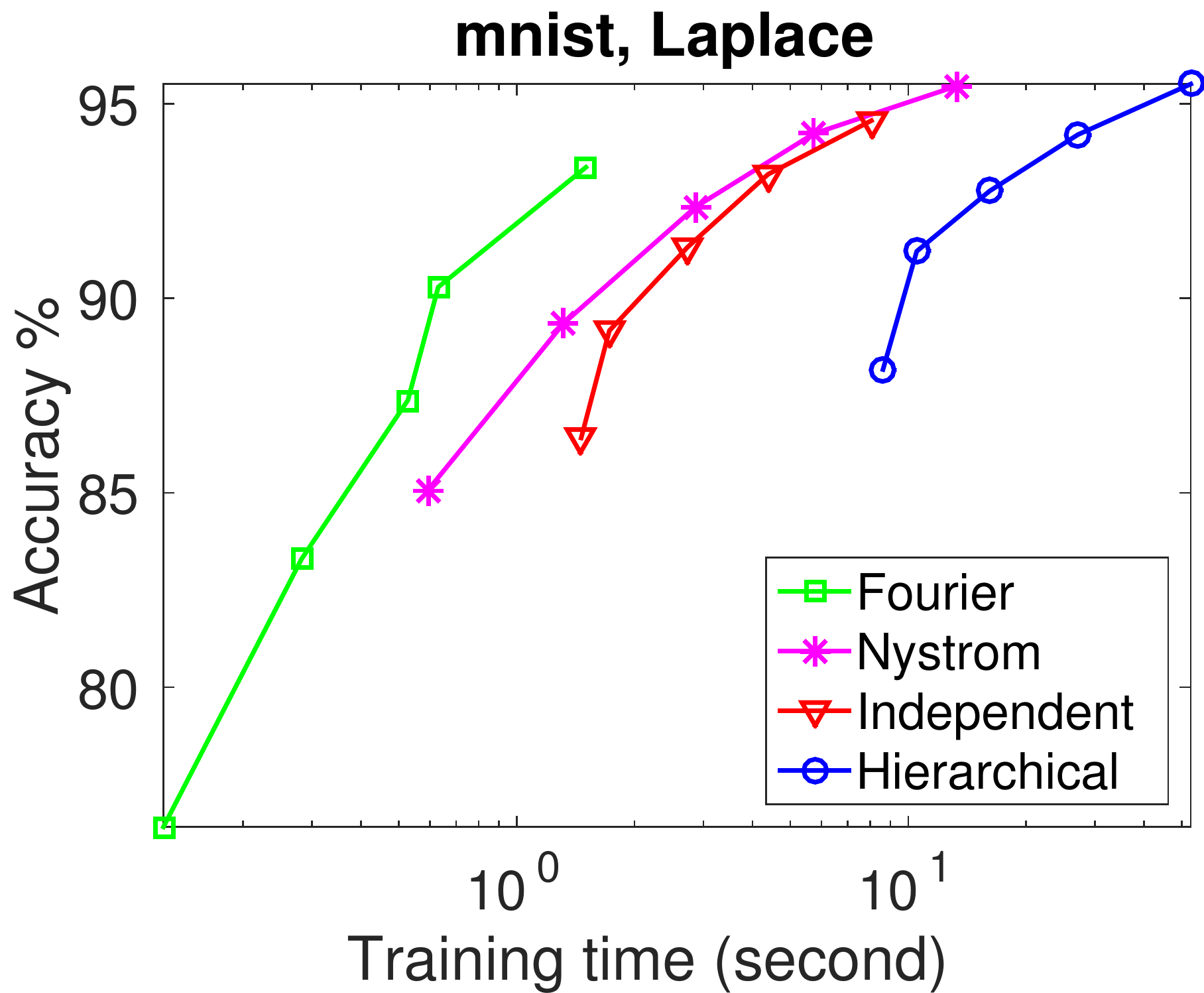}
\includegraphics[width=.33\linewidth]{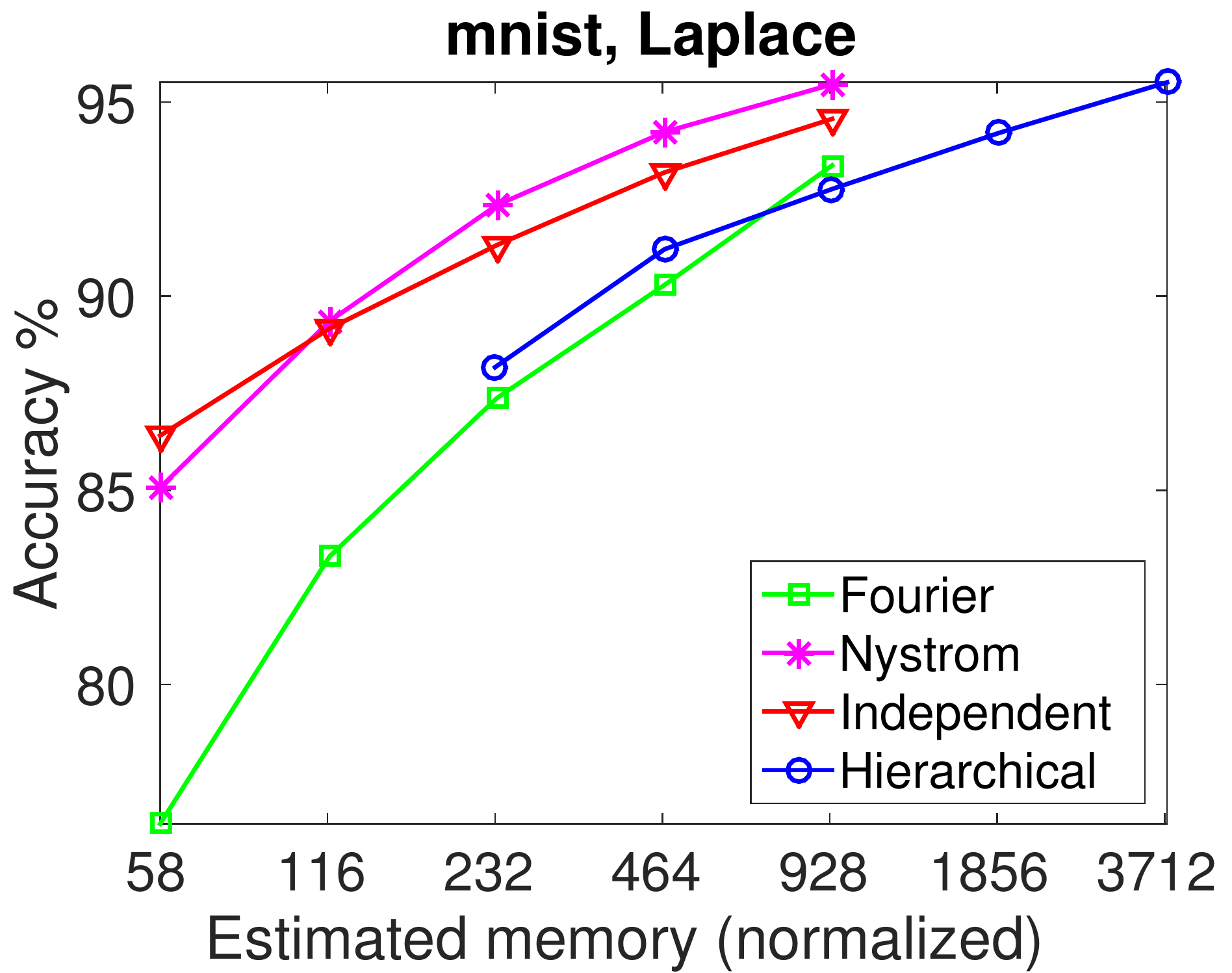}}
\subfigure[acoustic, multiclass classification]{
\includegraphics[width=.33\linewidth]{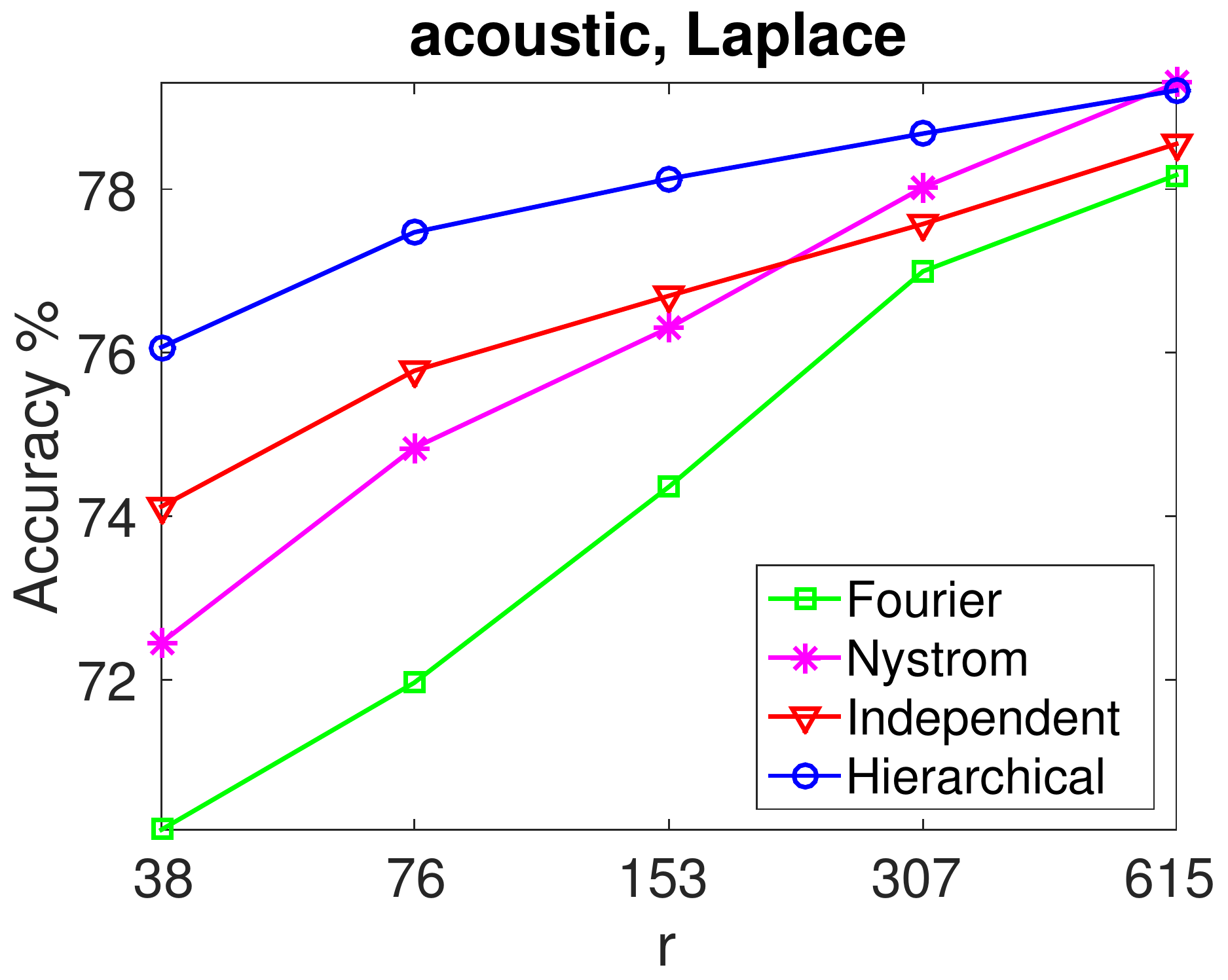}
\includegraphics[width=.32\linewidth]{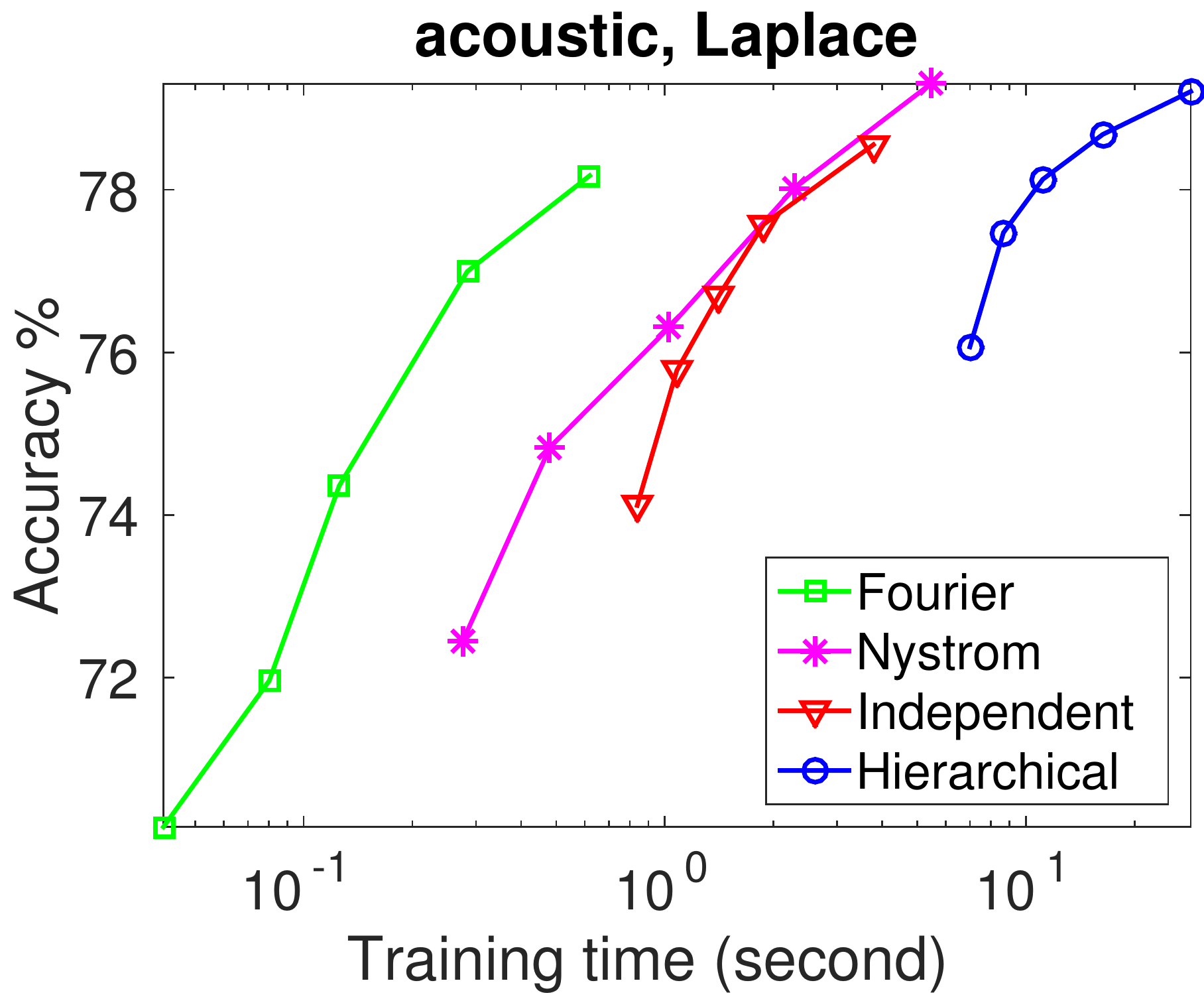}
\includegraphics[width=.33\linewidth]{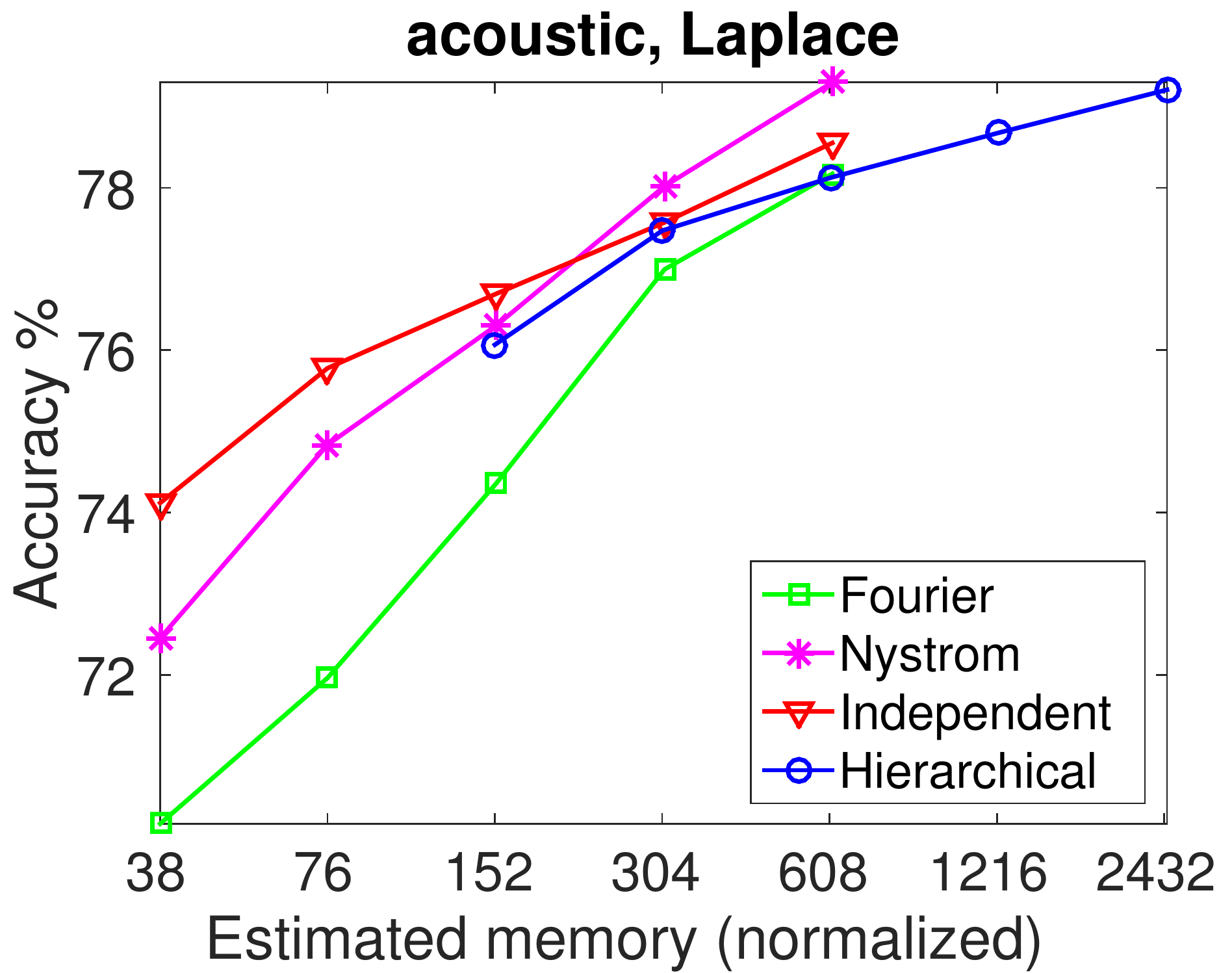}}
\subfigure[covtype, multiclass classification]{
\includegraphics[width=.33\linewidth]{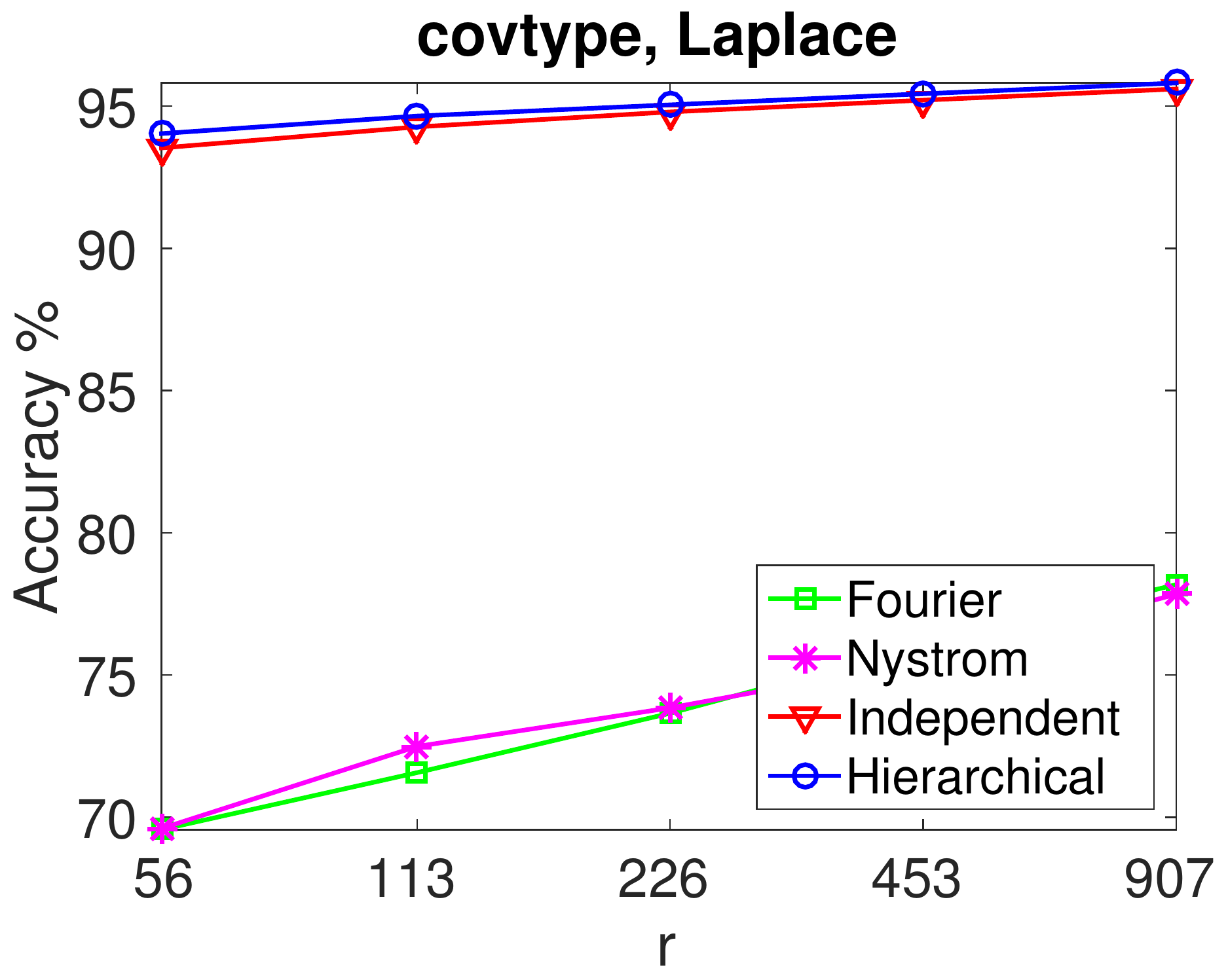}
\includegraphics[width=.32\linewidth]{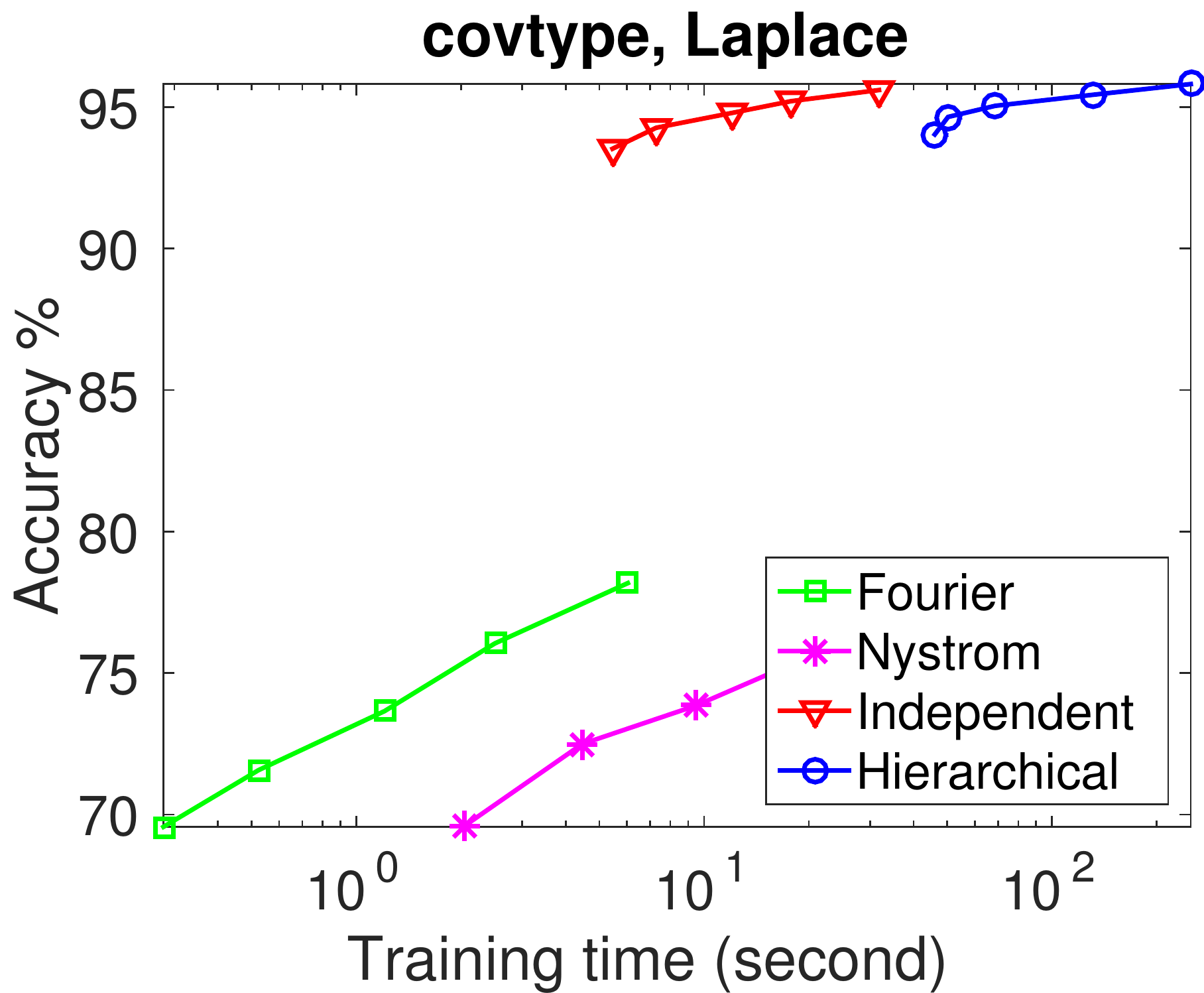}
\includegraphics[width=.33\linewidth]{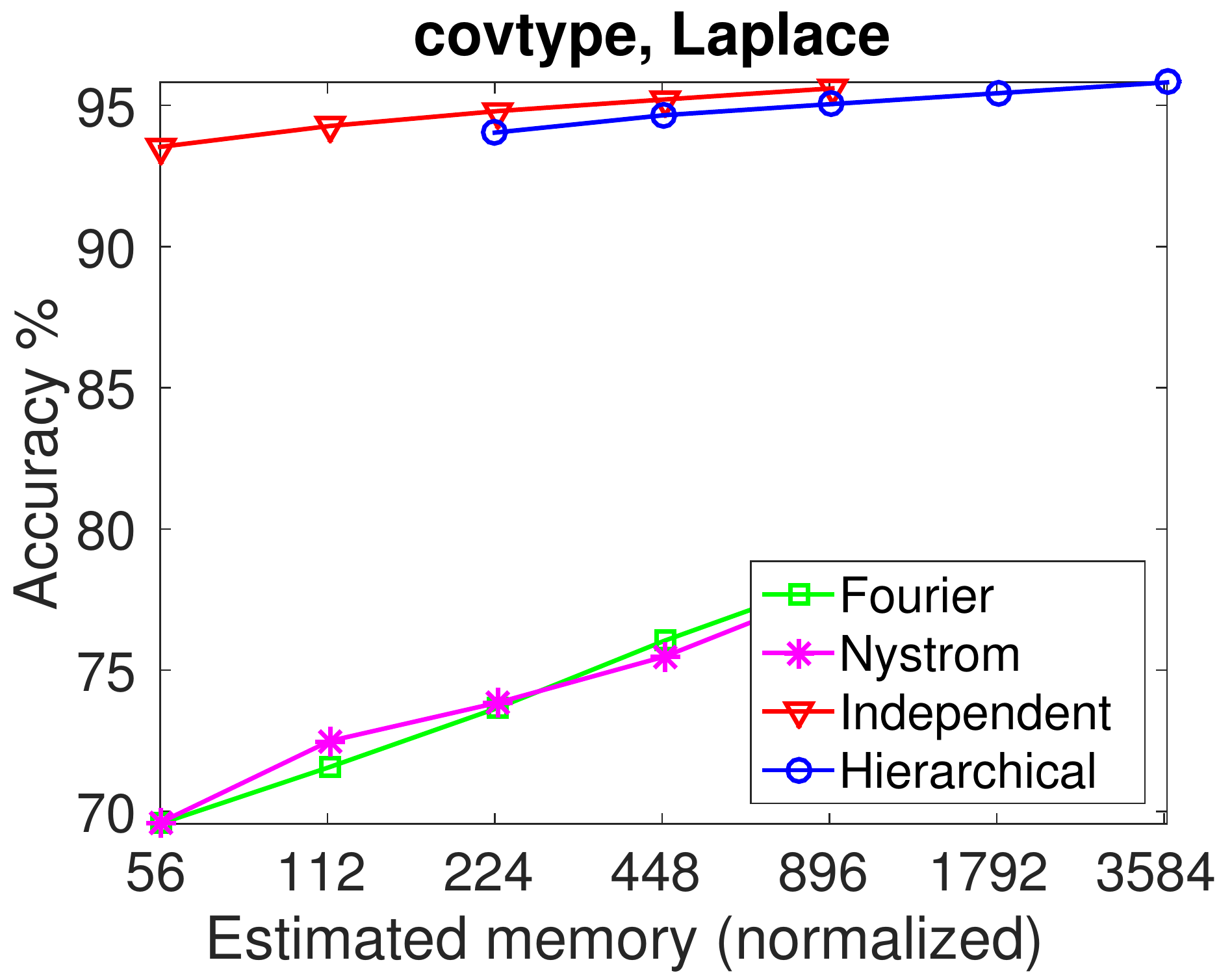}}
\caption{(Continued) Performance versus $r$, time, and memory. Laplace kernel.}
\label{fig:ZZ_plot_exp_4_laplace_2}
\end{figure}

\section{Performance Results with the Inverse Multiquadric Kernel}\label{sec:append.perf.invmultiquadric}
See Figures~\ref{fig:ZZ_plot_exp_7_invmultiquadric_1} and~\ref{fig:ZZ_plot_exp_7_invmultiquadric_2}. These plots are qualitatively similar to those of the Gaussian kernel (Figures~\ref{fig:ZZ_plot_exp_3_comprehensive_1} and~\ref{fig:ZZ_plot_exp_3_comprehensive_2}).

\begin{figure}[!ht]
\centering
\subfigure[cadata, regression]{
\includegraphics[width=.33\linewidth]{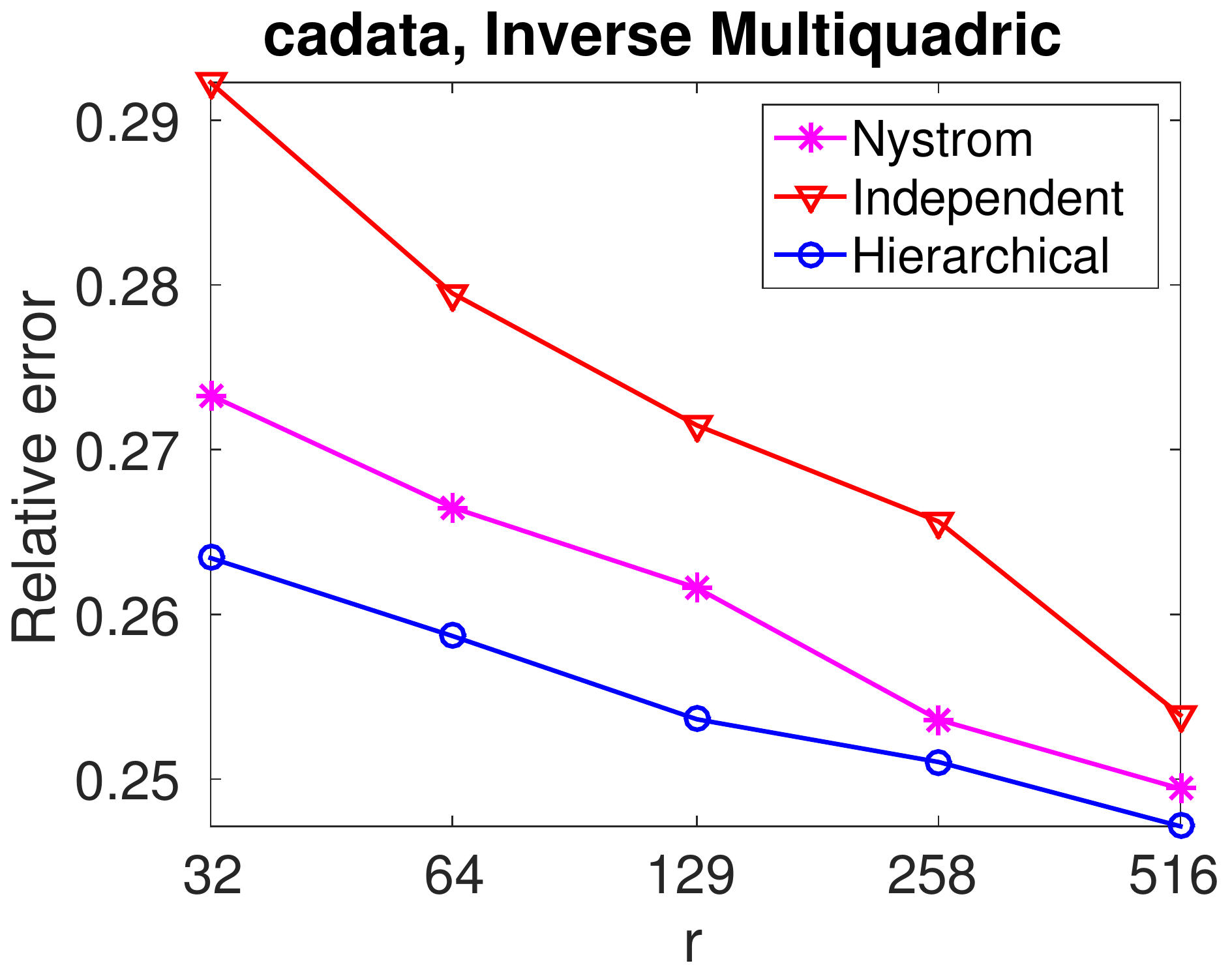}
\includegraphics[width=.32\linewidth]{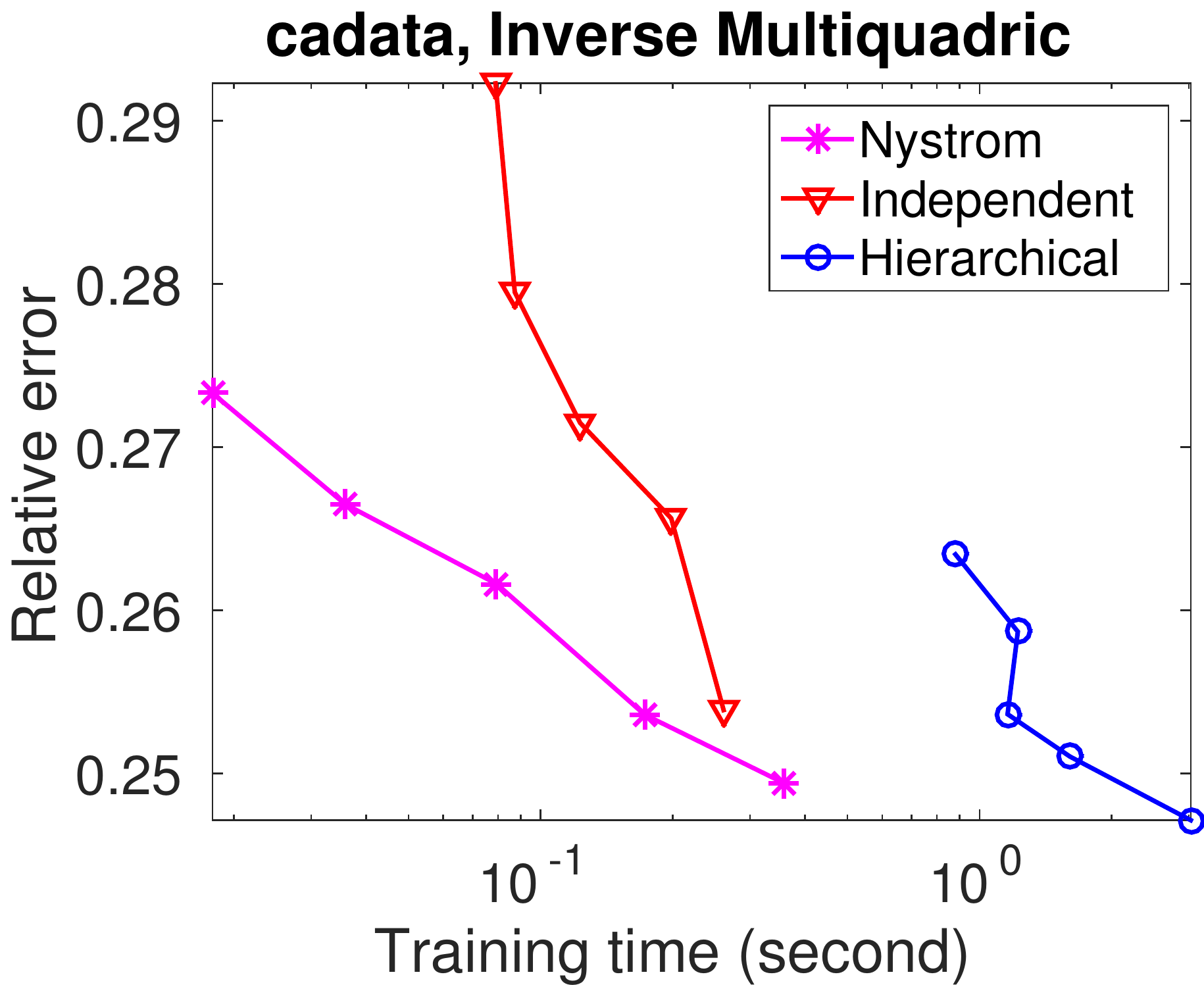}
\includegraphics[width=.33\linewidth]{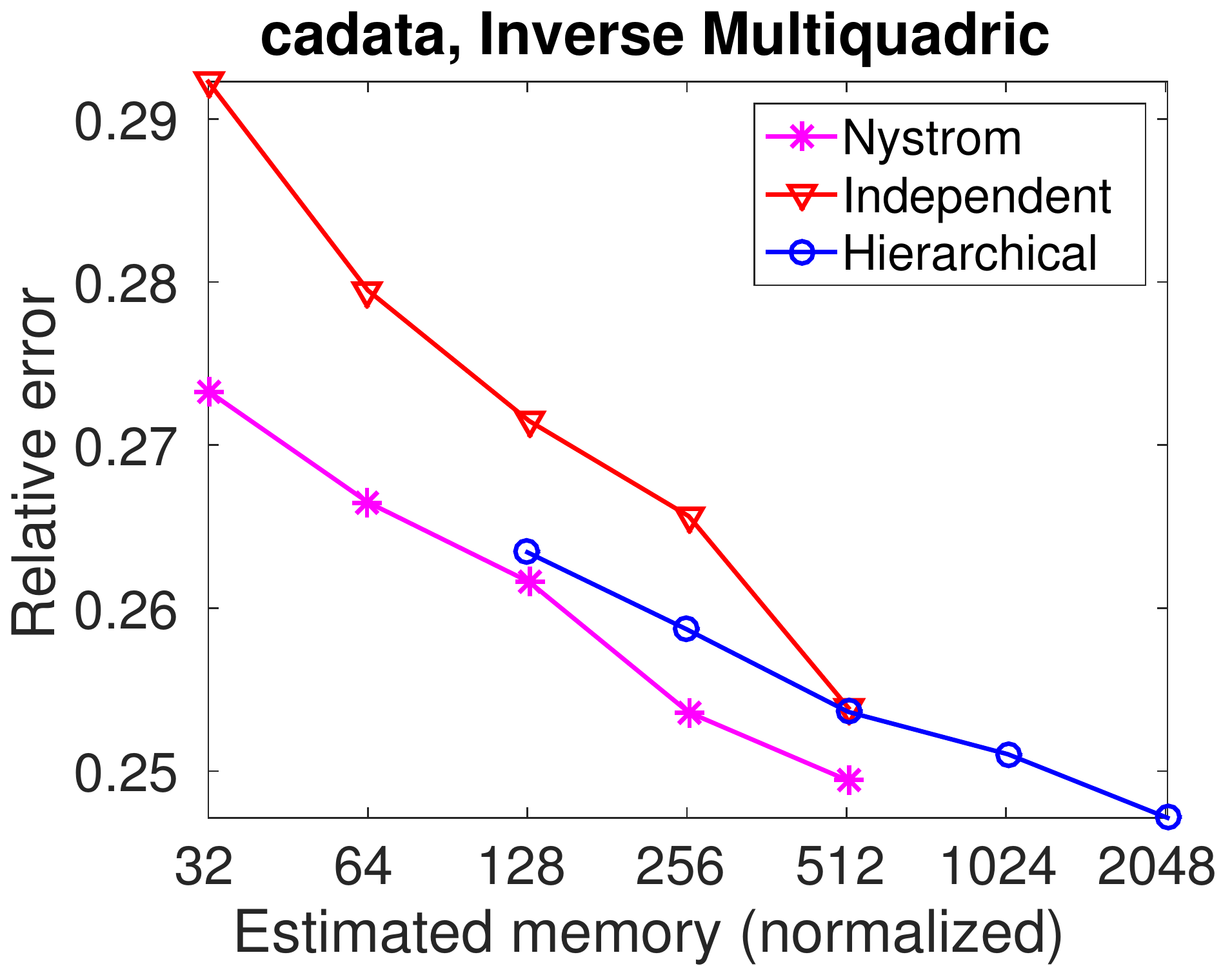}}
\subfigure[YearPredictionMSD, regression]{
\includegraphics[width=.33\linewidth]{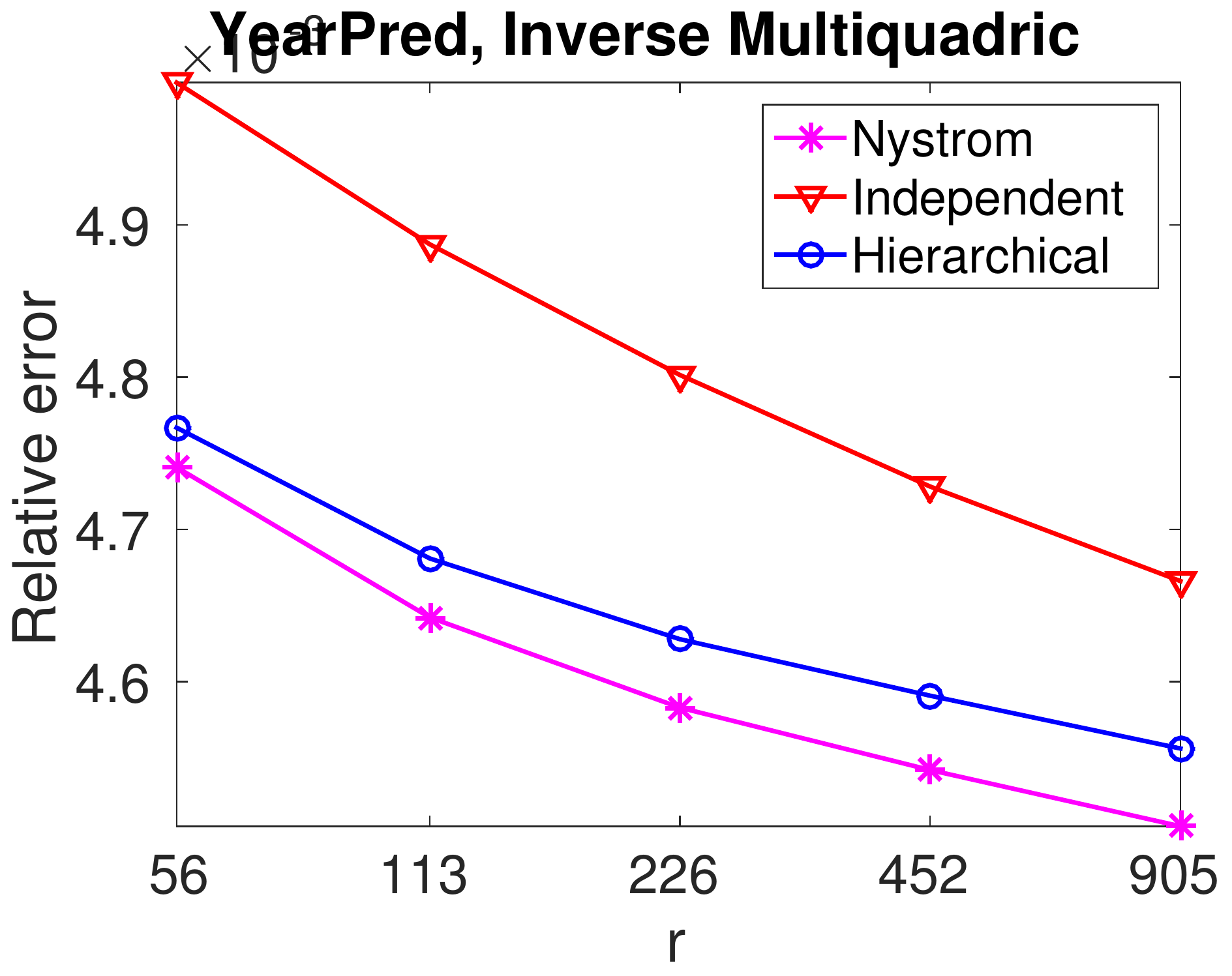}
\includegraphics[width=.32\linewidth]{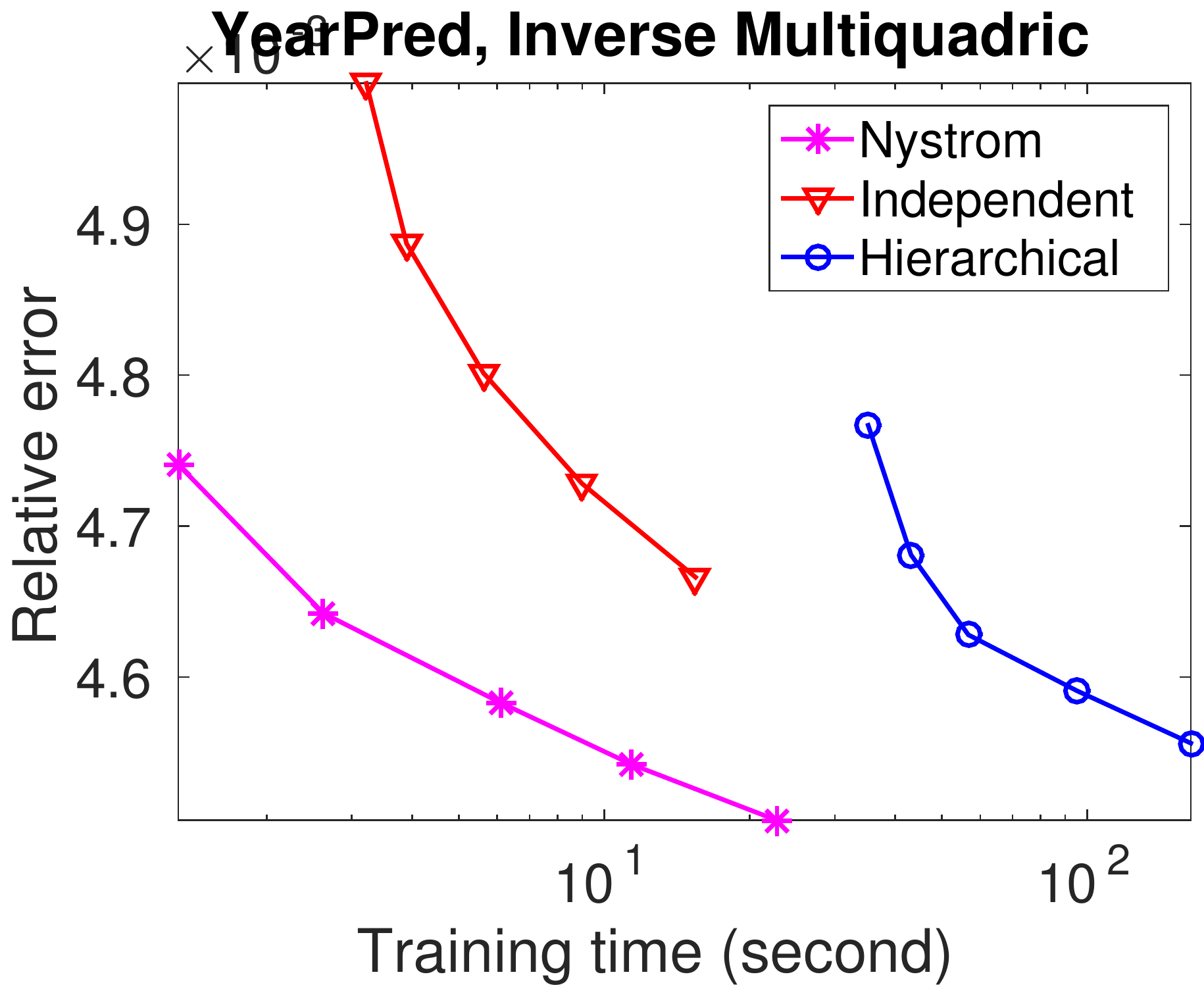}
\includegraphics[width=.33\linewidth]{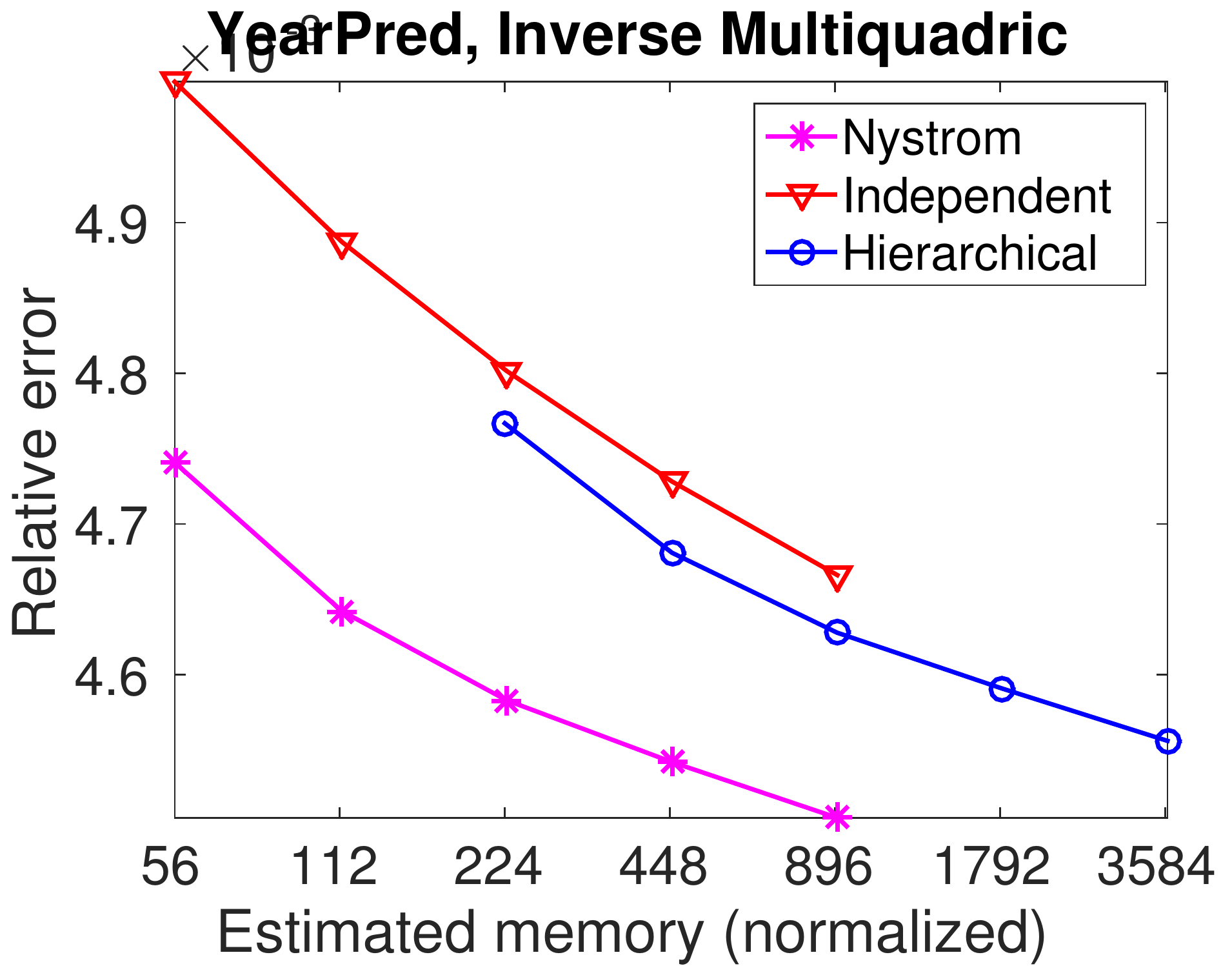}}
\subfigure[ijcnn1, binary classification]{
\includegraphics[width=.33\linewidth]{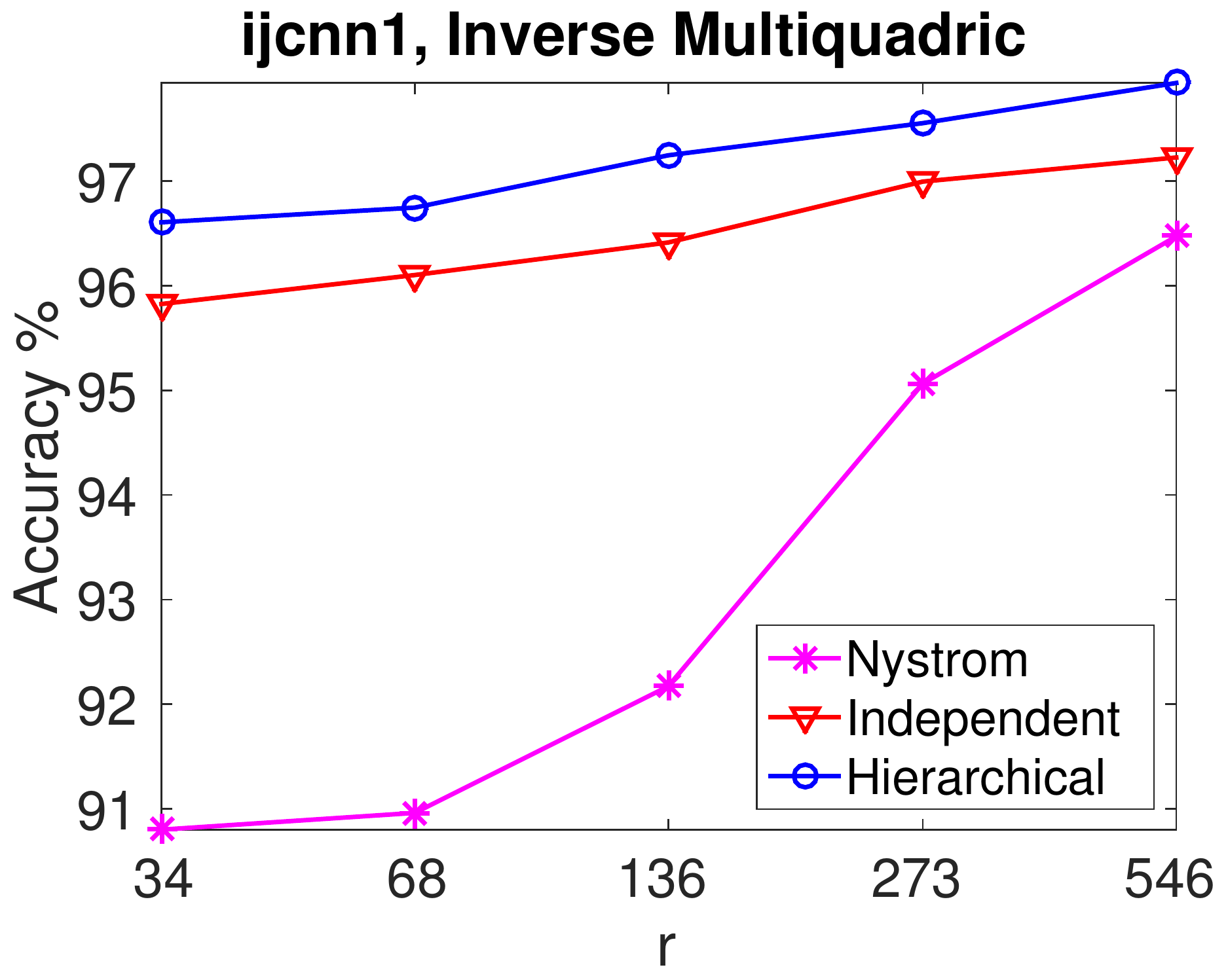}
\includegraphics[width=.32\linewidth]{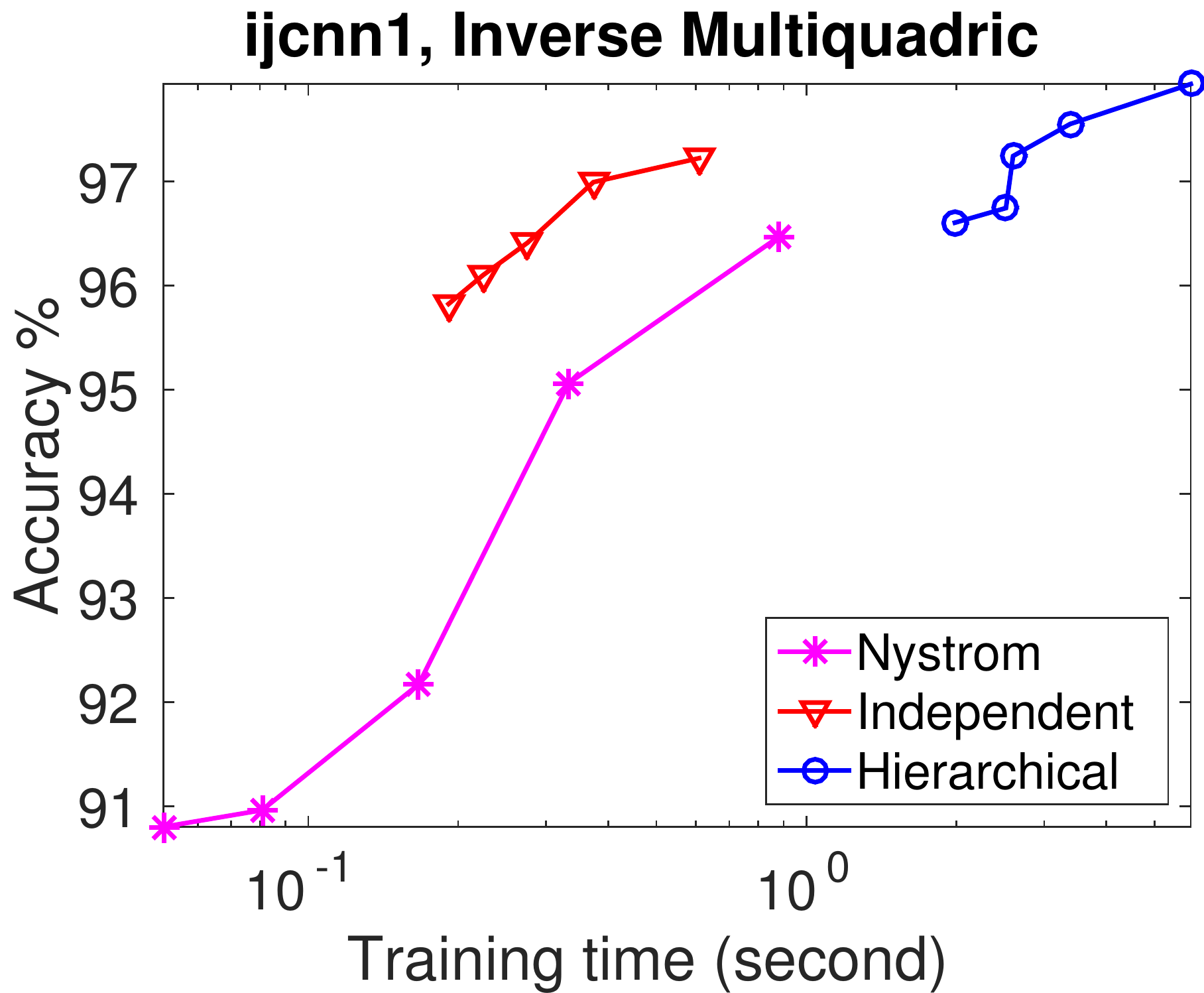}
\includegraphics[width=.33\linewidth]{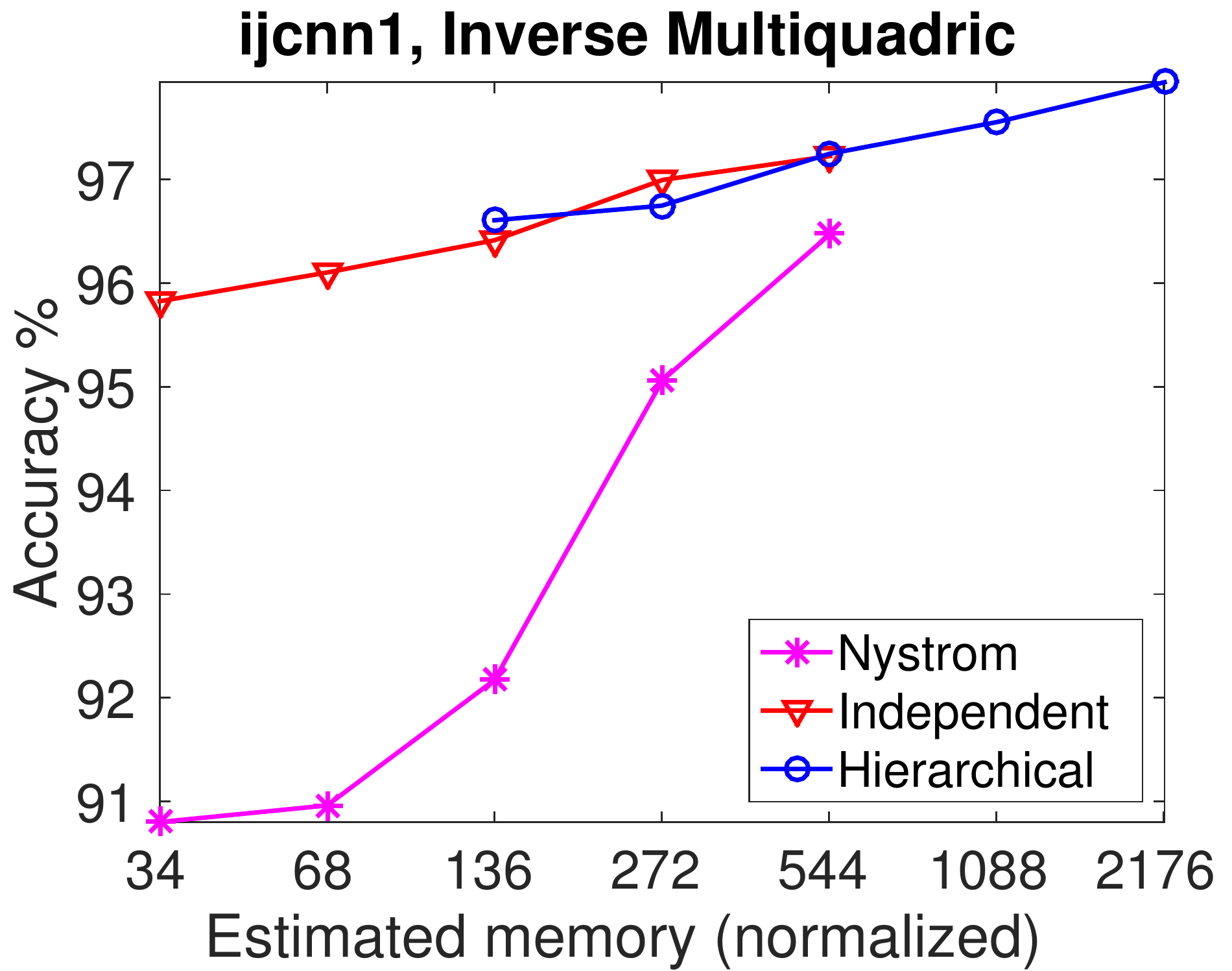}}
\subfigure[covtype.binary, binary classification]{
\includegraphics[width=.33\linewidth]{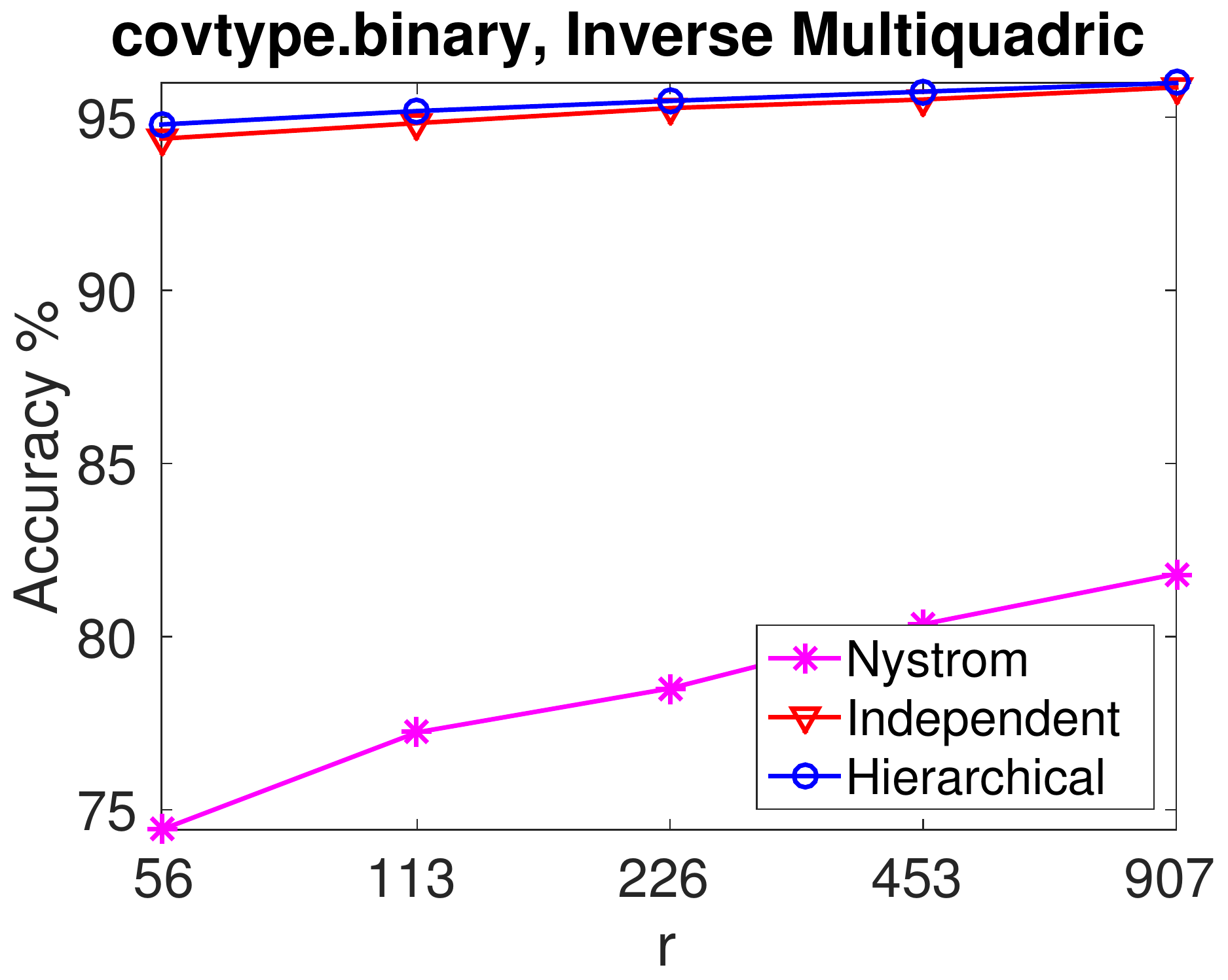}
\includegraphics[width=.32\linewidth]{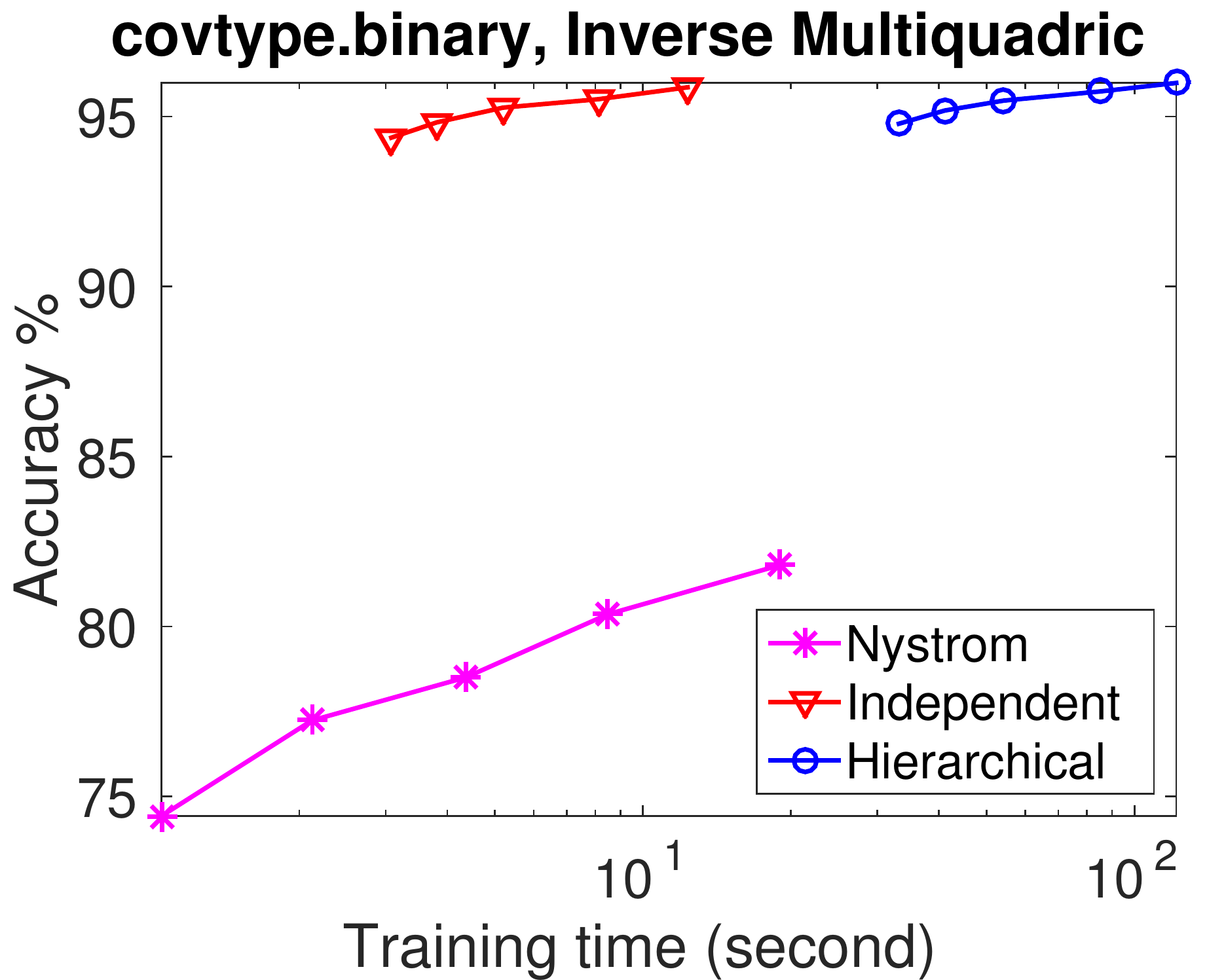}
\includegraphics[width=.33\linewidth]{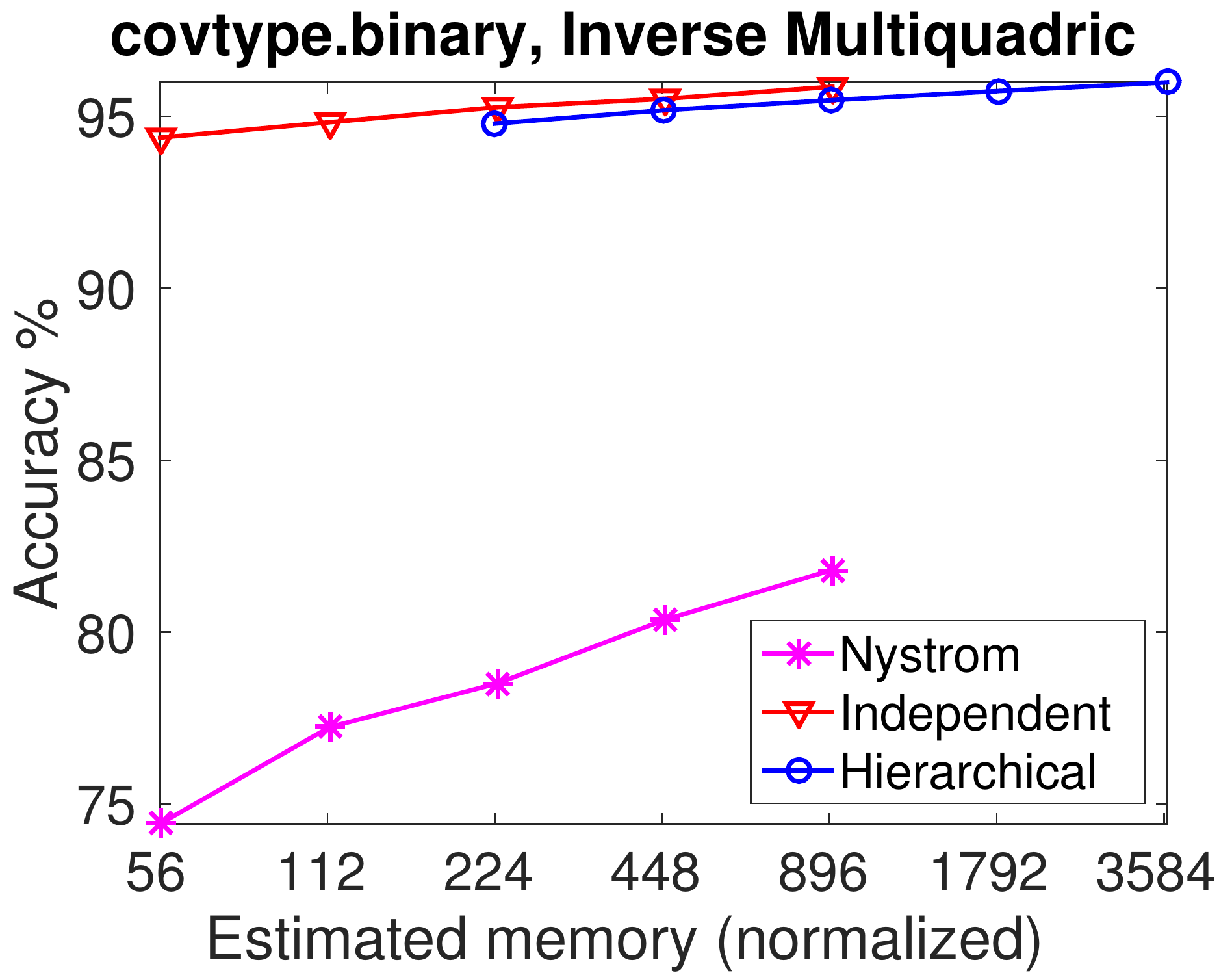}}
\caption{Performance versus $r$, time, and memory. Inverse multiquadric kernel.}
\label{fig:ZZ_plot_exp_7_invmultiquadric_1}
\end{figure}

\begin{figure}[!ht]
\centering
\subfigure[mnist, multiclass classification]{
\includegraphics[width=.33\linewidth]{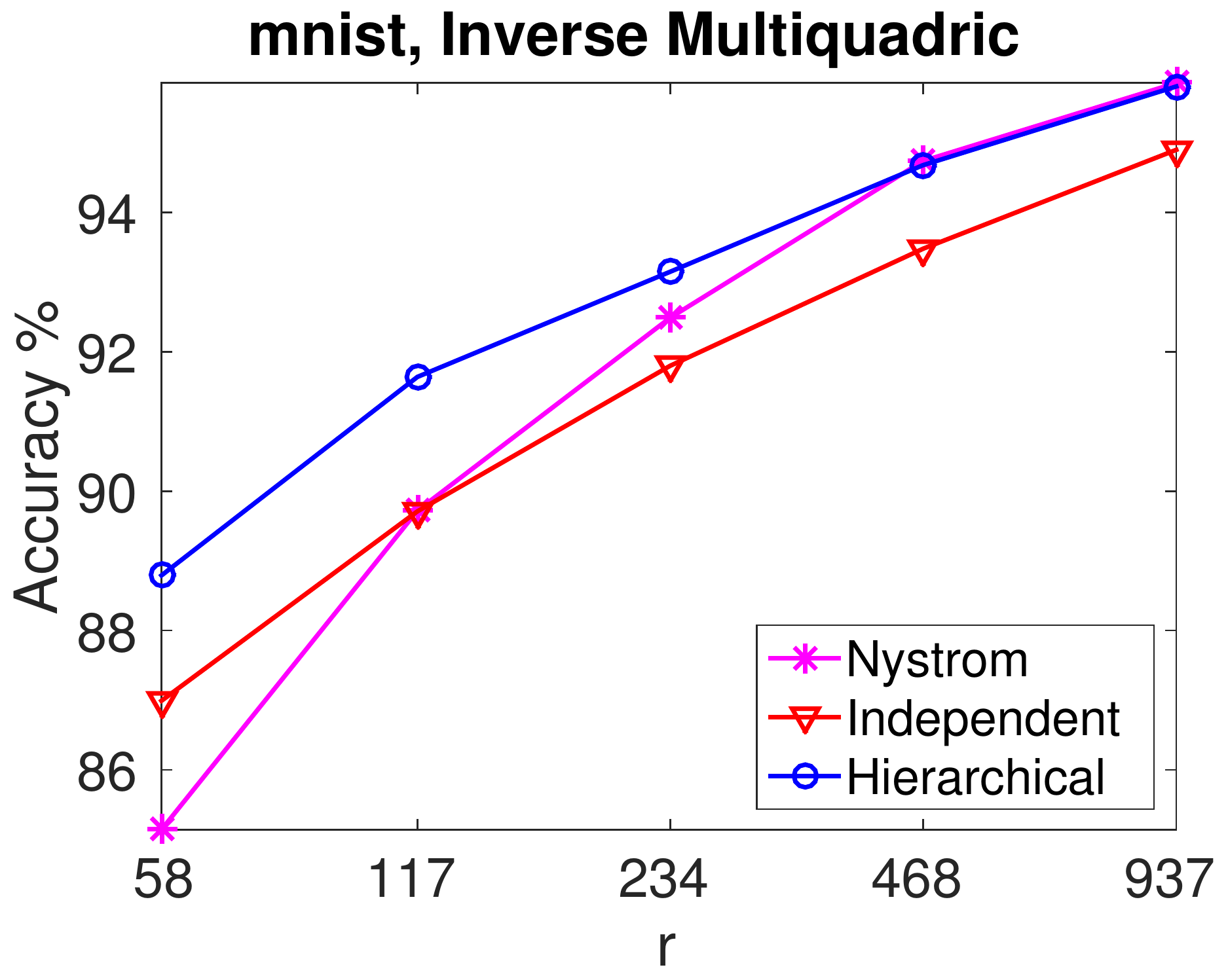}
\includegraphics[width=.32\linewidth]{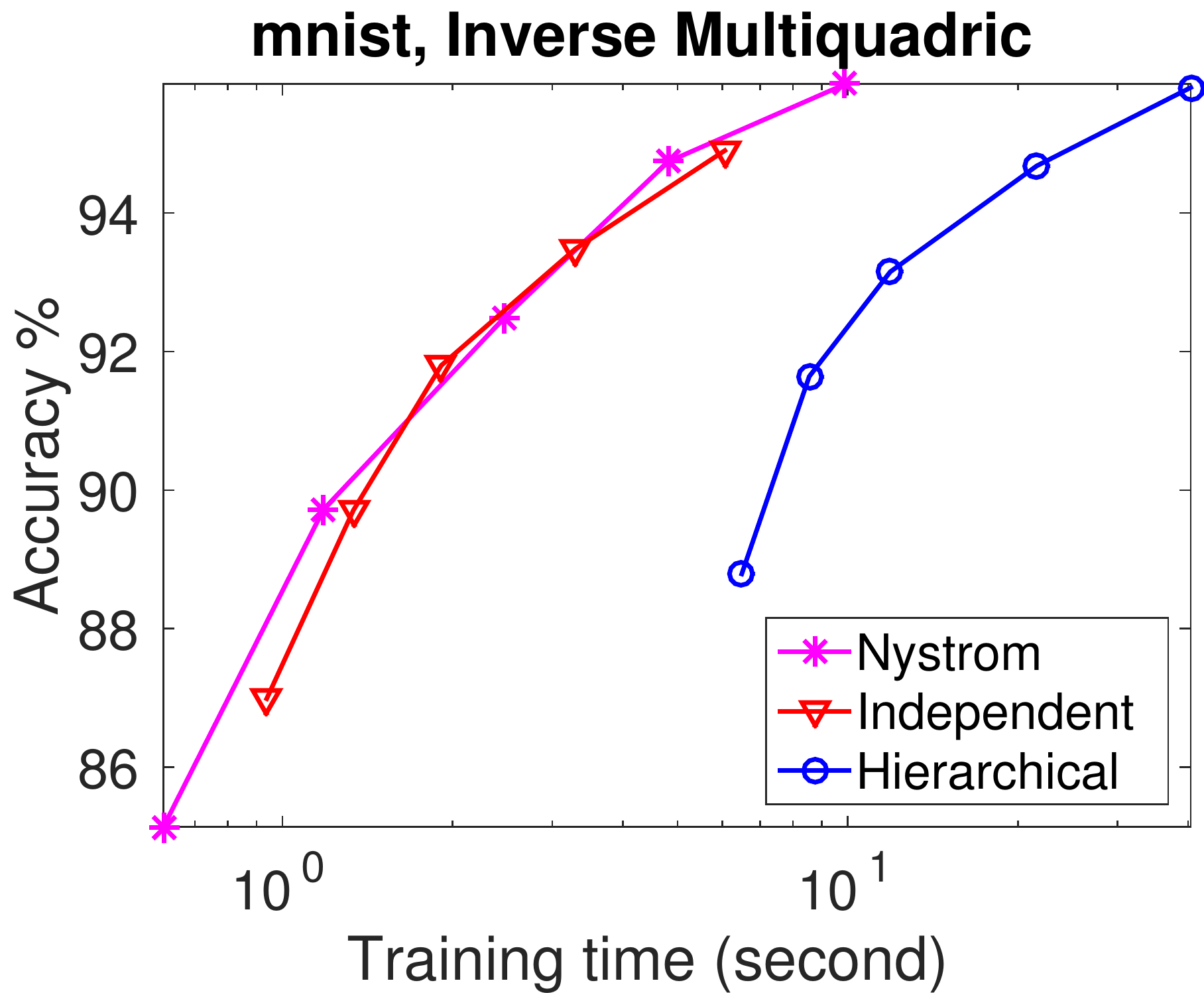}
\includegraphics[width=.33\linewidth]{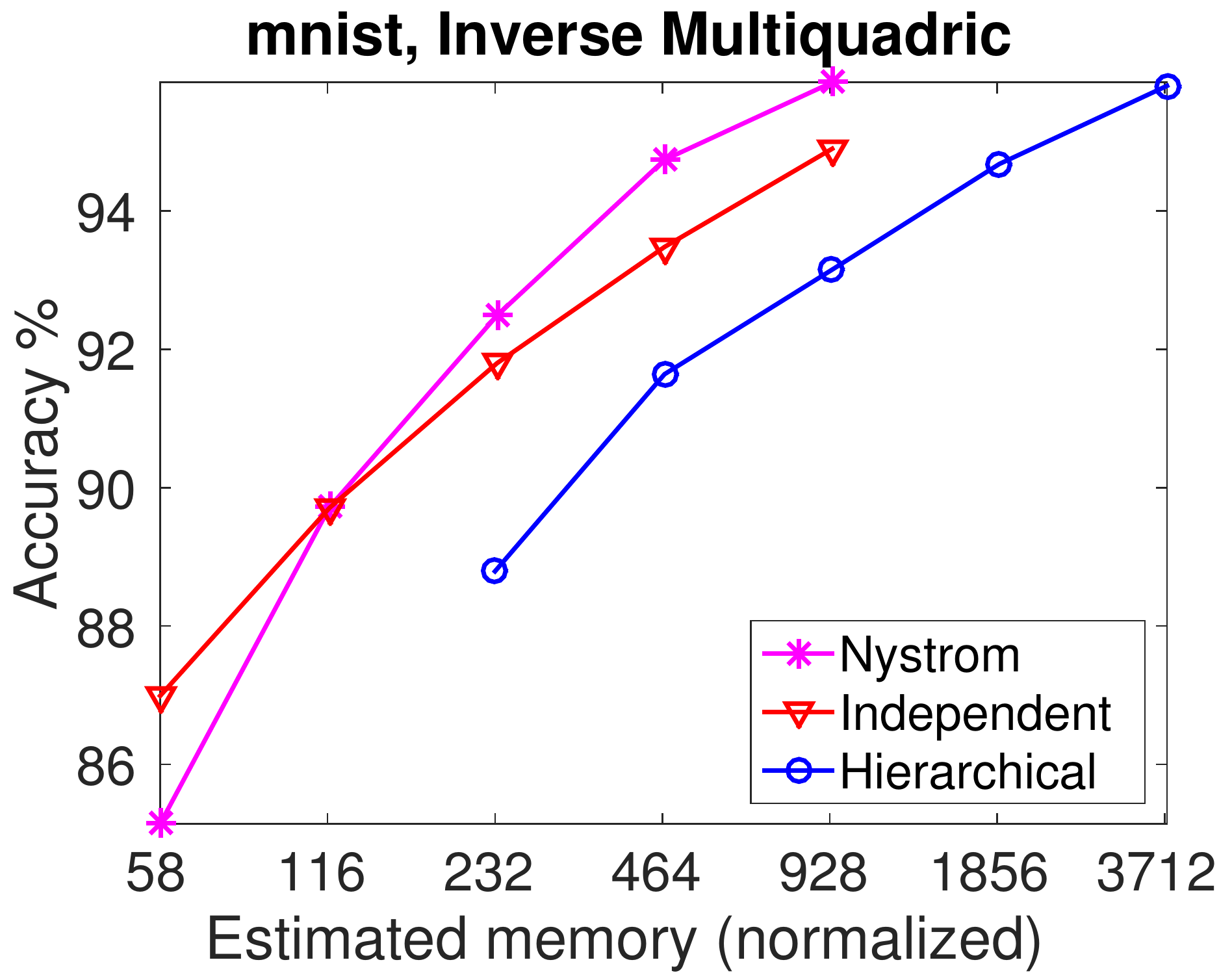}}
\subfigure[acoustic, multiclass classification]{
\includegraphics[width=.33\linewidth]{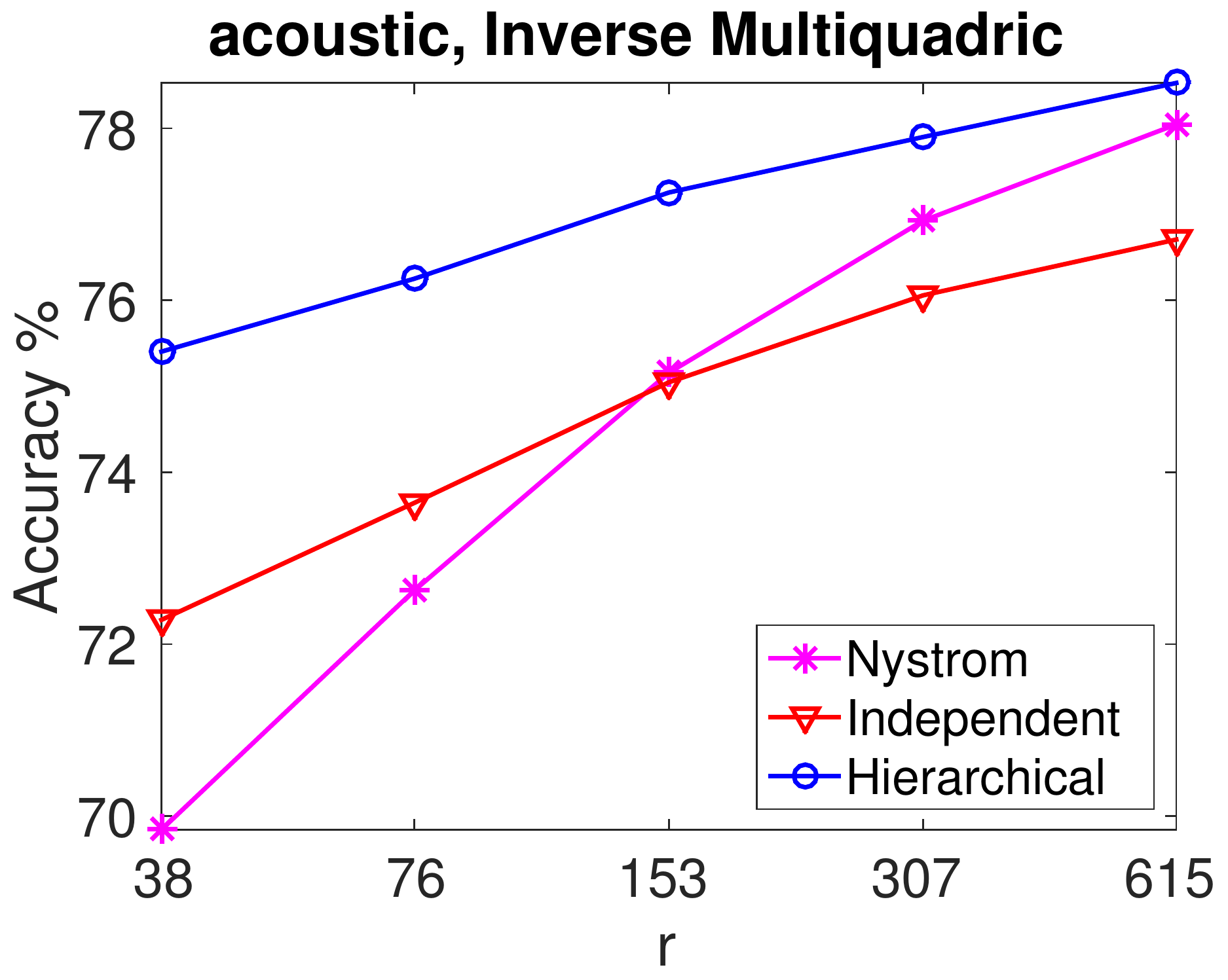}
\includegraphics[width=.32\linewidth]{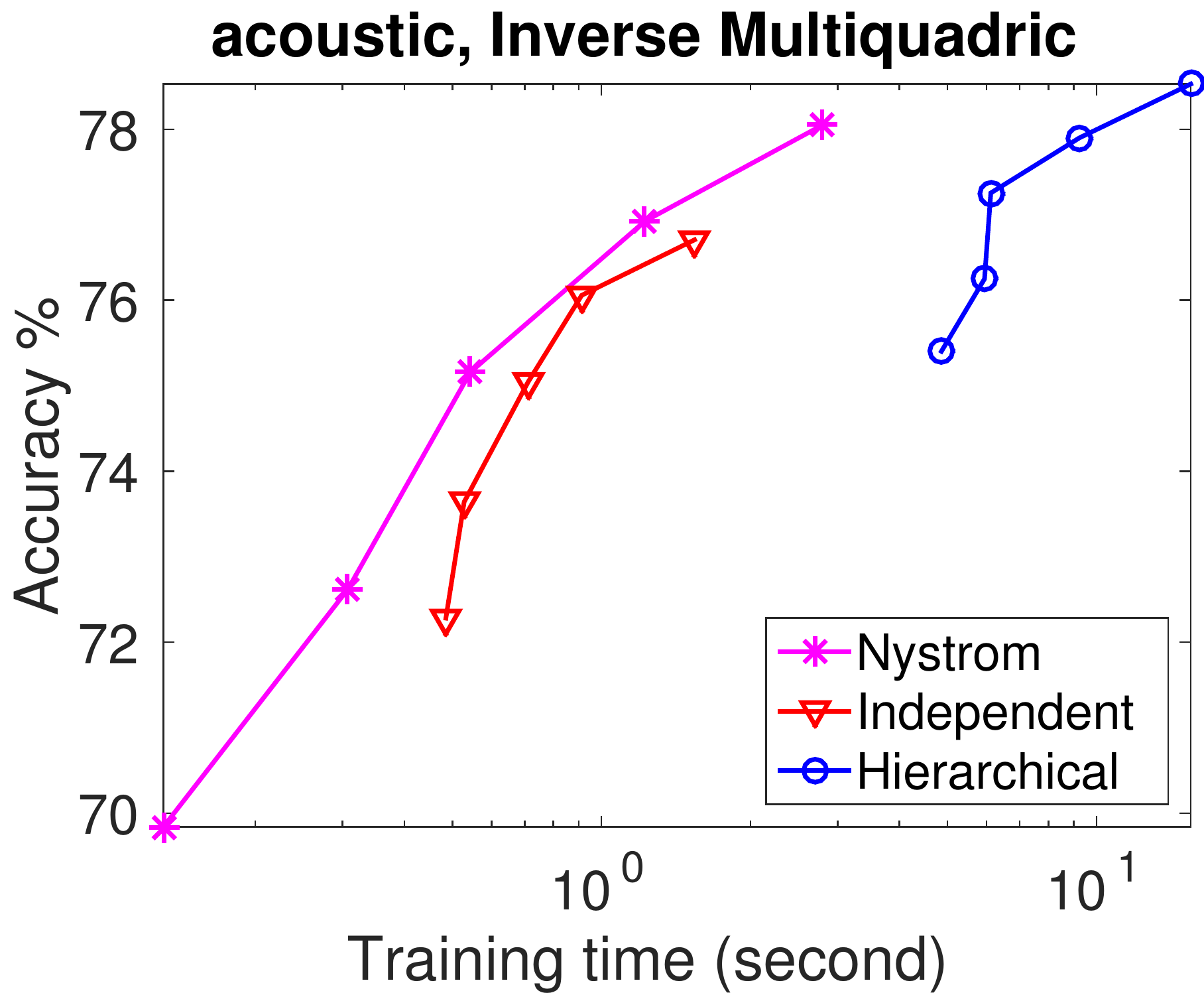}
\includegraphics[width=.33\linewidth]{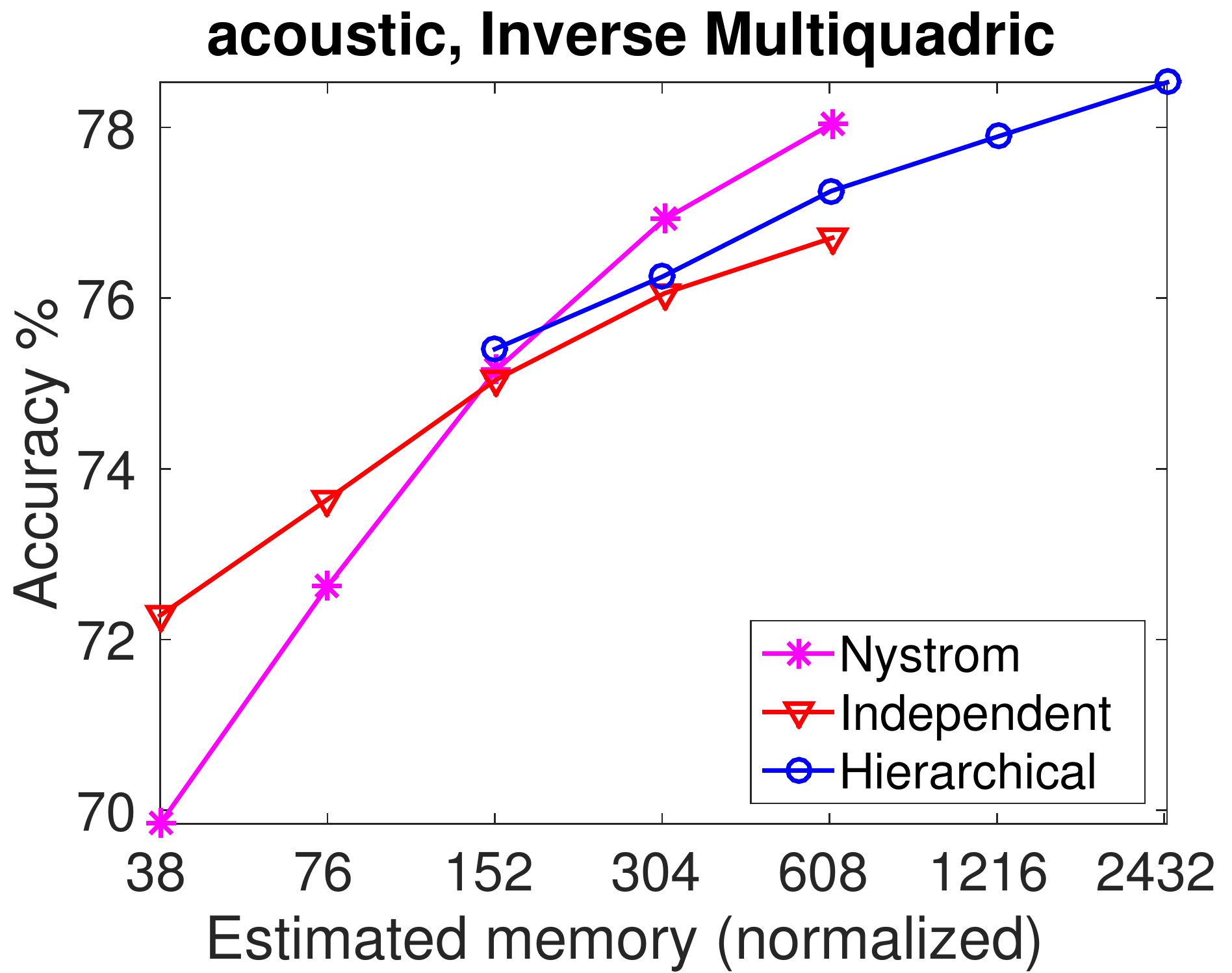}}
\subfigure[covtype, multiclass classification]{
\includegraphics[width=.33\linewidth]{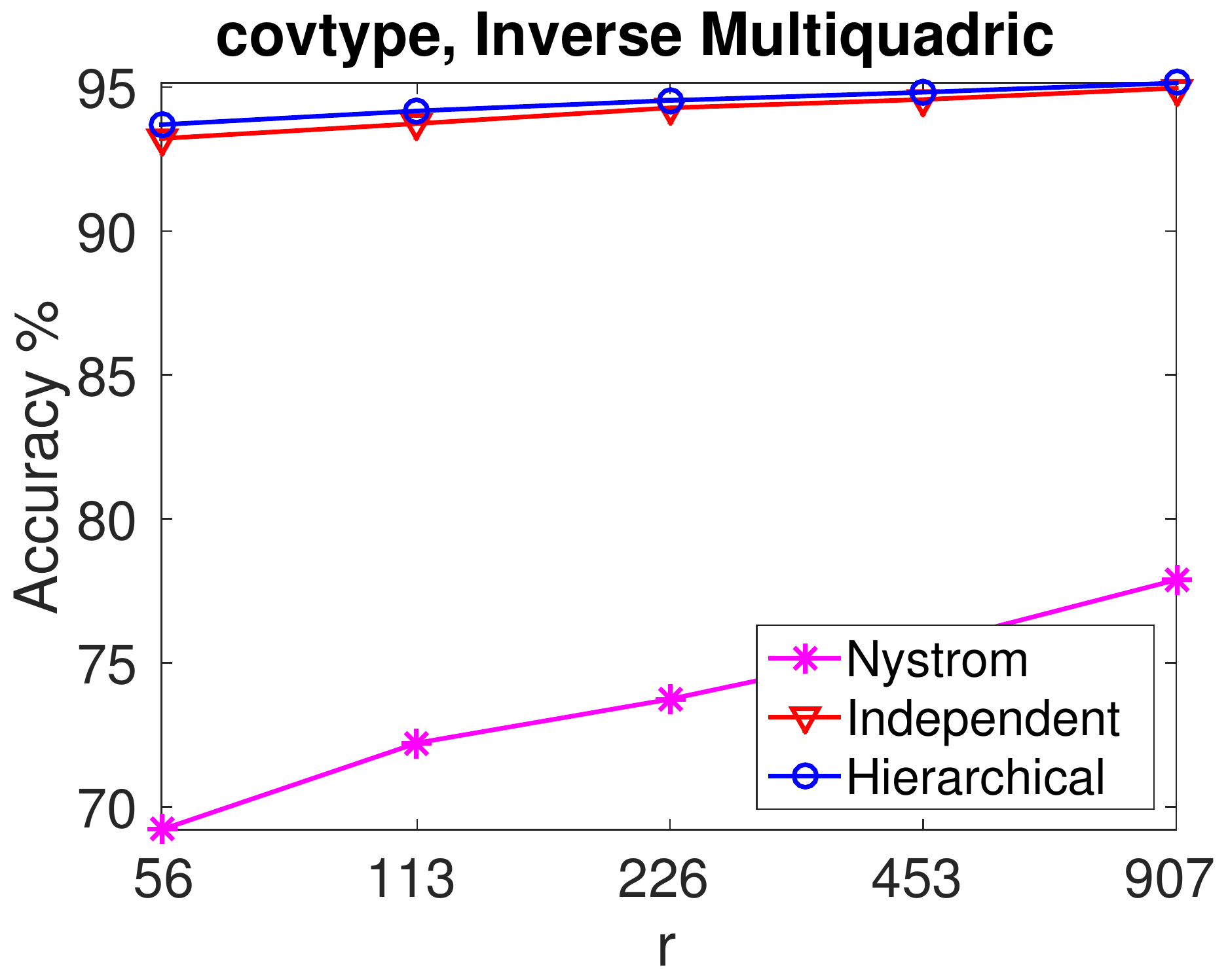}
\includegraphics[width=.32\linewidth]{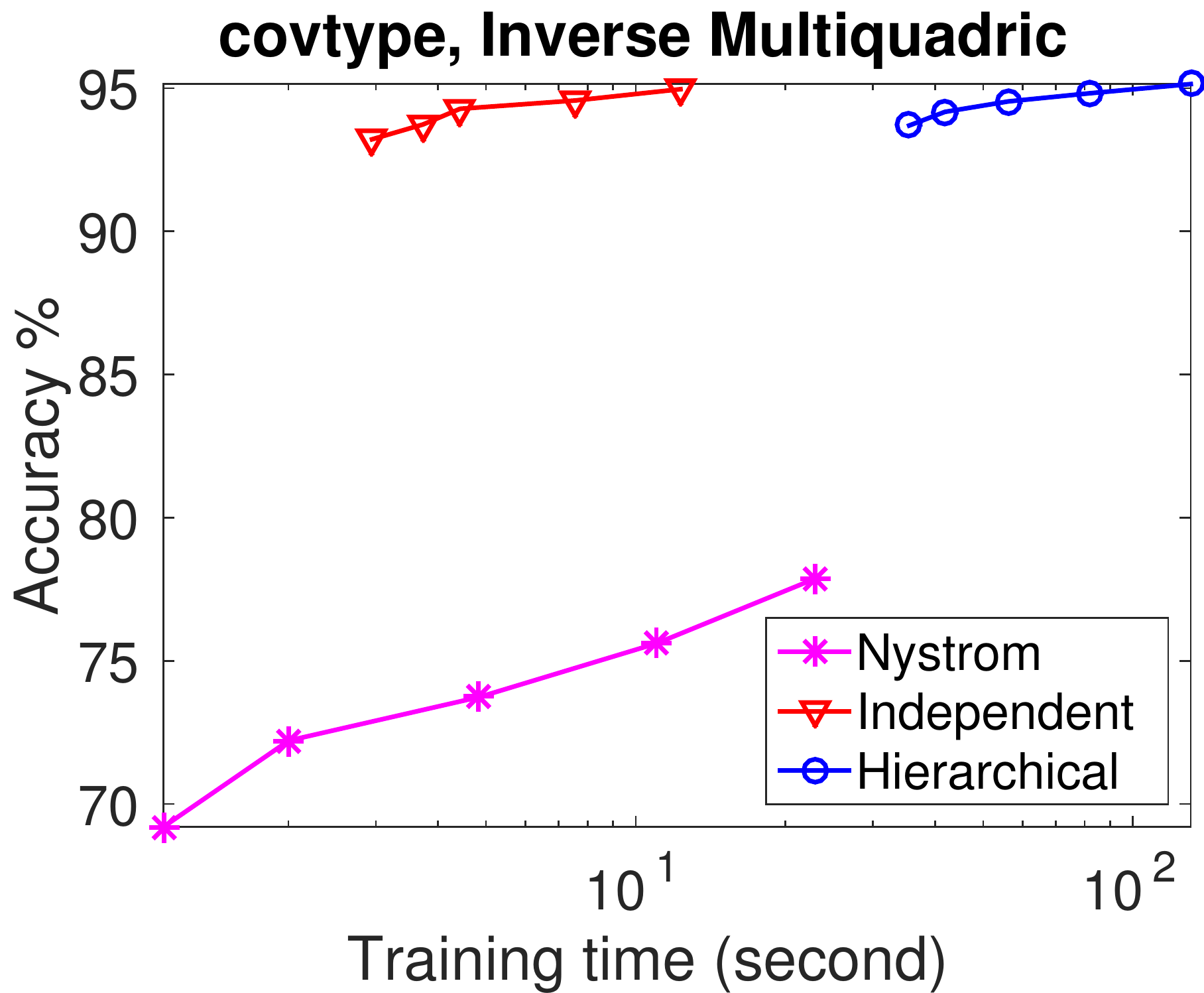}
\includegraphics[width=.33\linewidth]{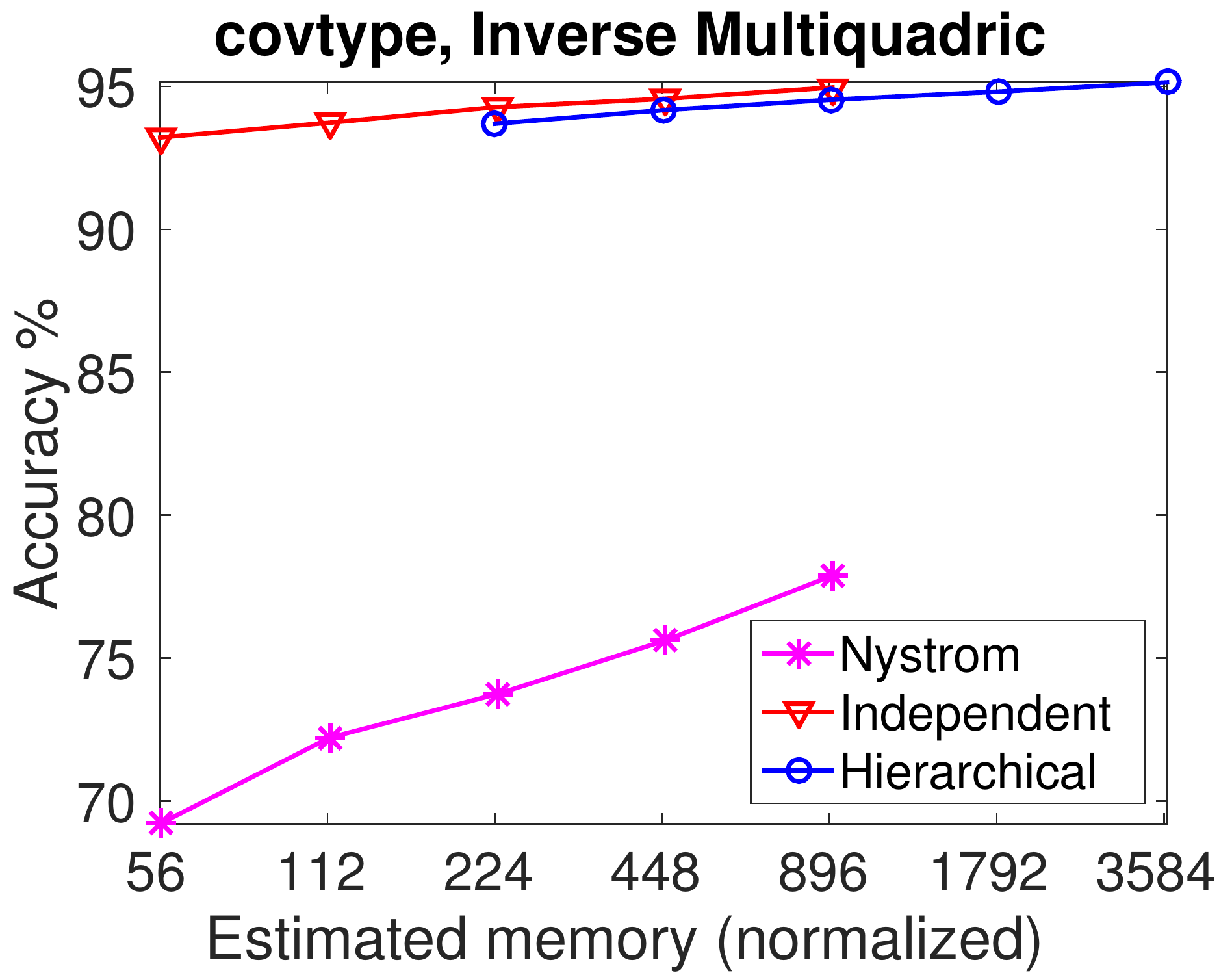}}
\caption{(Continued) Performance versus $r$, time, and memory. Inverse multiquadric kernel.}
\label{fig:ZZ_plot_exp_7_invmultiquadric_2}
\end{figure}

\end{document}